\documentclass[11pt]{article}

\usepackage{comment}
\usepackage{latexsym}
\usepackage{amsmath}
\usepackage{amssymb}
\usepackage{amsthm}
\usepackage{epsfig}
\usepackage{hyperref}
\hypersetup{
    colorlinks,
    linkcolor={blue!80!black},
    citecolor={green!50!black},
}
\usepackage{xcolor}
\usepackage{graphicx}
\usepackage{bm}
\usepackage{enumitem}

\usepackage[top=1in, bottom=1in, left=1in, right=1in]{geometry}

\renewcommand{\P}{\mathbb{P}}
\newcommand{\E}{\mathbb{E}}

\newcommand{\Z}{\mathbb{Z}}
\newcommand{\R}{\mathbb{R}}
\newcommand{\C}{\mathbb{C}}

\newcommand{\eps}{\varepsilon} 

\def\id{{\mathbf I}}

\newcommand{\<}{\langle}
\renewcommand{\>}{\rangle}

\newcommand{\sgn}{\mathrm{sgn}}
\newcommand{\diag}{\text{diag}}

\newcommand{\op}{{\rm op}}
\newcommand{\ones}{{\boldsymbol 1}}

\def\sT{{\mathsf T}}
\def\bzero{{\boldsymbol 0}}

\DeclareMathOperator*{\argmin}{arg\,min}

\usepackage{bbm}
\newtheorem{theorem}{Theorem}
\newtheorem*{theorem*}{Theorem}
\newtheorem{definition}[theorem]{Definition}
\newtheorem{proposition}[theorem]{Proposition}
\newtheorem{remark}[theorem]{Remark}
\newtheorem{lemma}[theorem]{Lemma}
\newtheorem{corollary}[theorem]{Corollary}
\newtheorem{example}[theorem]{Example}

\newtheorem{claim}[theorem]{Claim}

\DeclareSymbolFont{rsfs}{U}{rsfs}{m}{n}
\DeclareSymbolFontAlphabet{\mathscrsfs}{rsfs}

\newcommand{\poly}{\mathrm{poly}}

\def\bA{{\boldsymbol A}}
\def\bB{{\boldsymbol B}}

\def\bD{{\boldsymbol D}}

\def\bG{{\boldsymbol G}}

\def\bH{{\boldsymbol H}}

\def\bK{{\boldsymbol K}}

\def\bM{{\boldsymbol M}}
\def\bN{{\boldsymbol N}}

\def\bP{{\boldsymbol P}}
\def\bQ{{\boldsymbol Q}}
\def\bR{{\boldsymbol R}}

\def\bT{{\boldsymbol T}}

\def\bW{{\boldsymbol W}}

\def\ba{{\boldsymbol a}}
\def\bb{{\boldsymbol b}}

\def\bg{{\boldsymbol g}}

\def\bm{{\boldsymbol m}}

\def\bu{{\boldsymbol u}}
\def\bv{{\boldsymbol v}}
\def\bw{{\boldsymbol w}}
\def\bx{{\boldsymbol x}}
\def\by{{\boldsymbol y}}
\def\bz{{\boldsymbol z}}

\def\bbeta{{\boldsymbol \beta}}
\def\balpha{{\boldsymbol \alpha}}
\def\bgamma{{\boldsymbol \gamma}}
\def\bdelta{{\boldsymbol\delta}}

\def\bphi{{\boldsymbol \phi}}
\def\brho{{\boldsymbol \rho}}
\def\btheta{{\boldsymbol \theta}}

\def\bxi{{\boldsymbol \xi}}
\def\bnu{{\boldsymbol \nu}}

\def\bDelta{{\boldsymbol \Delta}}
\def\bLambda{{\boldsymbol \Lambda}}

\def\bTheta{{\boldsymbol \Theta}}

\def\hba{{\hat {\boldsymbol a}}}
\def\hf{{\hat f}}

\def\spn{{\rm span}}

\def\de{{\rm d}}

\def\de{{\rm d}}
\def\Unif{{\rm Unif}}

\def\spn{{\rm span}}

\def\cV{{\mathcal V}}

\def\cP{{\mathcal P}}

\def\cT{{\mathcal T}}
\def\cC{{\mathcal C}}

\def\cF{{\mathcal F}}

\def\cS{{\mathcal S}}
\def\cI{{\mathcal I}}
\def\cV{{\mathcal V}}

\def\cP{{\mathcal P}}

\def\cT{{\mathcal T}}
\def\cH{{\mathcal H}}

\def\Unif{{\sf Unif}}
\def\normal{{\sf N}}
\def\proj{{\mathsf P}}

\def\NN{{\sf NN}}

\def\naturals{{\mathbb N}}

\def\normal{{\sf N}}
\def\proj{{\mathsf P}}

\def\Unif{{\sf Unif}}
\def\normal{{\sf N}}

\def\proj{{\mathsf P}}

\def\NN{{\sf NN}}

\def\naturals{{\mathbb N}}
\def\proj{{\mathsf P}}

\def\hba{{\hat {\boldsymbol a}}}
\def\hf{{\hat f}}

\def\cX{{\mathcal X}}
\def\cF{{\mathcal F}}
\def\cS{{\mathcal S}}
\def\cI{{\mathcal I}}

\def\de{{\rm d}}
\def\Unif{{\rm Unif}}

\def\normal{{\sf N}}

\def\bDelta{{\boldsymbol \Delta}}

\def\cX{{\mathcal X}}

\def\bA{{\boldsymbol A}}
\def\btheta{{\boldsymbol \theta}}
\def\bTheta{{\boldsymbol \Theta}}
\def\bLambda{{\boldsymbol \Lambda}}

\def\cT{{\mathcal T}}
\def\cV{{\mathcal V}}
\def\bP{{\boldsymbol P}}
\def\diag{{\rm diag}}

\def\bD{{\boldsymbol D}}

\def\bb{{\boldsymbol b}}

\def\hf{\hat f}

\def\bR{{\boldsymbol R}}

\def\bc{{\boldsymbol c}}

\def\bzeta{{\boldsymbol \zeta}}

\def\ind{\mathbbm{1}}

\newcommand{\pd}[2]{\frac{\partial #1}{\partial #2}} 
\newcommand{\sm}{\setminus}
\newcommand{\fNN}{\hat{f}_{\NN}}
\newcommand{\EE}{\mathbb{E}}

\def\hatQ{\hat{Q}}
\def\hatbQ{\hat{\bQ}}
\def\hatbu{\hat{\bu}}

\def\hatbM{\hat{\bM}}

\def\barbw{\bar{\bw}}

\def\err{\mathsf{err}}

\def\br{{\boldsymbol r}}
\def\bn{{\boldsymbol n}}
\def\oa{\overline{a}}
\def\obu{\overline{\bu}}
\def\orho{\overline{\rho}}

\def\ubtheta{\underline{\btheta}}

\def\ubu{\underline{\bu}}

\def\hf{\hat f}

\def\os{\overline{s}}

\def\obtheta{\overline{\btheta}}

\def\bve{\big\vert}

\def\barbtheta{\bar{\btheta}}

\def\barbTheta{\bar{\bTheta}}
\def\barbu{\bar{\bu}}

\def\rA{{\rm A}}

\def\hh{\hat h}
\def\ou{\overline{u}}
\def\cR{\mathcal{R}}

\def\oR{\overline{R}}
\def\ocS{\overline{\cS}}
\def\hg{\hat{g}}
\def\hu{\hat{u}}
\def\hbu{\hat{\bu}}
\def\obH{\overline{\bH}}
\def\obLambda{\overline{\bLambda}}

\def\rD{{\rm D}}
\def\ubtheta{\underline{\obtheta}}

\def\om{\overline{m}}

\usepackage{cleveref}
\crefname{claim}{claim}{claims}
\crefname{fact}{fact}{facts}

\title{The merged-staircase property:   a necessary and nearly sufficient condition for SGD learning of sparse functions on two-layer neural networks}

\author{Emmanuel Abbe\thanks{Mathematics Institute, EPFL}, \;\; Enric Boix-Adser\`a\thanks{Department of Electrical Engineering and Computer Science, MIT}, \;\; Theodor Misiakiewicz\thanks{Department of
    Statistics, Stanford University} }

\begin{document}

\maketitle

\begin{abstract}
It is currently known how to characterize functions that neural networks can learn with SGD for two extremal parametrizations: neural networks in the linear regime, and neural networks with no structural constraints. However, for the main parametrization of interest ---non-linear but regular networks--- no tight characterization has yet been achieved, despite significant developments. 

We take a step in this direction by considering depth-2 neural networks trained by SGD in the mean-field regime. We consider functions on binary inputs that depend on a latent low-dimensional subspace (i.e., small number of coordinates). This regime is of interest since it is poorly understood how neural networks routinely tackle high-dimensional datasets and adapt to latent low-dimensional structure without suffering from the curse of dimensionality.
Accordingly, we study SGD-learnability with $O(d)$ sample complexity in a large ambient dimension $d$. 

Our main results characterize a hierarchical property ---the merged-staircase property--- that is both {\it necessary and nearly sufficient} for learning in this setting. 
 We further show that non-linear training is necessary: for this class of functions, linear methods on any feature map (e.g., the NTK) are not capable of learning efficiently. The key tools are a new ``dimension-free'' dynamics approximation result that applies to functions defined on a latent space of low-dimension, a proof of global convergence based on polynomial identity testing, and an improvement of lower bounds against linear methods for non-almost orthogonal functions.

\end{abstract}

\section{Introduction}

Major research activity has recently been devoted to understanding what function classes can be learned by SGD on neural networks. Two extremal cases are well understood. On one extreme, neural networks can be parametrized to collapse under SGD to linear models, for which a clear picture has been drawn \cite{jacot2018neural,li2018learning,du2018gradient,du2019gradient,allen2019convergence,allen2019learning,arora2019fine,zou2020gradient,oymak2020toward}. On the other extreme, neural networks with zero parametrization constraint (besides polynomial size) have been shown to be able to emulate essentially any efficient learning algorithm \cite{AS20,newAS} albeit with non-regular\footnote{Here we refer to `regular' for architectures used in tangent kernel results or more generally architectures used in neural network applications.} architectures. So both of these extremes admit a fairly complete characterization. However, none of these seem to capture the right behavior behind deep learning, or more specifically, behind {\it non-linear but regular networks}. Such networks are known to go beyond linear learning \cite{bach2017breaking,ghorbani2021neural,daniely20parities,refinetti2021classifying,allen19can,ghorbani2019limitations,yehudai19power,allen20backward,li20beyondntk} (even though the NTK can be competitive on several instances \cite{Geiger_2020}), and seem to exploit structural properties of the target functions in order to efficiently build their features.

Can we thus characterize learning in the non-linear regime for regular networks? Various important results have been developed in this direction, we focus here on the most relevant to us.   
\cite{chizat2018global,mei2018mean,rotskoff2018neural,sirignano2020mean} show that for a certain scaling at initialization, the SGD dynamics on large-width neural networks concentrates on a fully non-linear dynamics, the mean-field dynamics, described by a Wasserstein gradient flow, contrasting with the linear dynamics of the NTK regime \cite{jacot2018neural}.
In 
\cite{allen19can,allen2020backward}, the power of deep networks is demonstrated by showing how SGD and quadratic activations can efficiently learn a non-trivial teacher class hierarchically, with the notion of backward feature correction   \cite{allen2020backward}.

However, no tight necessary and sufficient characterization of what functions are learnable emerges from these works. The difficulty being that tight necessity results are difficult to obtain in such a setting since SQ-like arguments \cite{blum94weakly,query2,kearns1998efficient,parity_blum,feldman_gen,yang,vempala2,parity_conj,AS20,newAS,goel} are not expected to be tight (besides for the extreme case of unconstrained networks \cite{AS20,newAS}), and sufficiency results are significantly more difficult to obtain due to the more complex (non-linear) dynamics of SGD training.

\begin{center}
\emph{Is there hope to characterize tight necessary and sufficient conditions for function classes to be learnable by standard SGD on standard neural networks?}
\end{center}

As a first attempt in that direction, we focus in this paper on a natural setting: \textit{learning sparse functions} on the $d$-dimensional hypercube, i.e., functions that depend on a small latent (unknown) subset of coordinates of the input. We further restrict the optimization regime considered to \textit{two-layer neural networks} trained by one-pass batch-SGD in the \textit{mean-field regime}. This allows us to study a regime of optimization that goes beyond the linear regime while averaging out some of the complexity of studying non-linear SGD.

The motivation for the setting of learning sparse function is three-fold: (1) Linear (fixed features) methods do not adapt to latent sparsity, and suffer from the \textit{curse of dimensionality} \cite{bach2017breaking}. 
(2) On the contrary,  \cite{bach2017breaking,schmidt2020nonparametric} shows that neural networks can overcome this curse and learn sparse functions sample-efficiently. However, these works do not provide tractable algorithms and
the question of when SGD-trained neural networks can adapt to sparsity remains largely open. (3) Some sparse functions, such as monomials, are known to be much harder to learn than others from SQ-like lower bounds \cite{kearns1998efficient,blum94weakly,newAS}, and we expect SGD to inherit some of this complex behavior. Therefore, the problem of learning sparse functions presents a clear-cut separation between fixed-feature and feature learning methods, and can help understand the limits of SGD-training on neural networks. 

To gain insights on the interaction between SGD and the function structure that allows adaptivity to sparsity, we will ask the following question: \textit{Can one characterize necessary and sufficient conditions for a low-dimensional latent function to be learnable by standard SGD on standard neural networks in an arbitrarily large ambient dimension?} More precisely, we will consider a  $P$-dimensional latent function $h_* : \{+1,-1\}^P \to \R$ and consider learning sparse functions $f_* : \{+1,-1\}^d \to \R$ with $f_* (\bx) = h_* (\bz)$ for arbitrary ambient dimension $d$ and latent subset of coordinates $\bz := \bx_{\cI} = (x_{i_1} , \ldots , x_{i_P})$. As motivating examples, consider the  two functions:
\[
h_1 (\bz) = z_1 z_2 z_3 \, , \qquad h_2 (\bz) = z_1 + z_1z_2 + z_1z_2z_3 \, .
\]
Both of these functions depend on only 3 coordinates (i.e., they are 3-sparse), and because of the presence of the degree-3 monomial both require $\Omega(d^3)$ samples to be learned by a linear method. However, are these functions equivalent for SGD-trained neural networks? If not, can we obtain a fine-grained analysis that separates them?

In this paper, we introduce the following notion: we say that a latent function $h_*$ is \textbf{strongly SGD-learnable in $O(d)$-scaling}, if $O(d)$ samples are enough to  learn $f_* (\bx) = h_* (\bz) $ for arbitrary latent subspace $\bz$ and dimension $d$, using batch-SGD on a two-layer neural network in the mean-field regime. The main contribution of this paper is then to characterize with a {\bf necessary and nearly sufficient condition} the class of functions that are strongly SGD-learnable in $O(d)$-scaling. This is achieved with the {\bf merged-staircase property} (MSP), stating that the non-zero Fourier coefficients of $g$ can be ordered as subsets $\{S_1,\ldots,S_r\}$ such that for any $i \in [r]$, 
\begin{align*}
|S_i \sm \cup_{j=1}^{i-1} S_j  | \leq 1.
\end{align*}
For instance, $h_2(\bz)=z_1+z_1z_2+z_1z_2z_3$ has Fourier coefficients (monomials) that can be ordered as $\{\{1\},\{1,2\}, \{1,2,3\}\}$, and each new set is incremented only by one element at each time. So $h_2$ satisfies the MSP (or is an MSP function with a slight abuse of terminology) and so is the function $z_1+z_1z_2+z_2z_3 +z_1z_2z_3$. However, the function $h_1 (\bz) = z_1 z_2 z_3$ directly makes a leap to a degree-3 Fourier coefficient and does not therefore satisfy the MSP. Our main results thus imply that $h_2$ can be learned with $O(d)$ samples in this regime, but not $h_1$.
The \textit{near} sufficiency part in our result stands for the fact that the sufficiency result is proved for ``generic'' merged-staircase functions, i.e., excluding a measure zero subclass. This `genericity' is in fact needed, as we provide degenerate examples in Section \ref{sec:sufficient} for which the strong SGD-learnability in $O(d)$-scaling is indeed not achievable.

The terminology MSP comes from the fact that this condition generalizes the basic staircase property introduced in \cite{staircase1}, which only encompasses nested chains of coefficients with $|S_{i} \setminus S_{i-1}|=1$, such as the vanilla staircase function (e.g., $h_2$) and slight generalizations with multiple chains. 
In \cite{staircase1} it is shown that staircase functions are learnable by neural nets that are deep but sparse, and with an unconventional gradient-based training algorithm (see Section \ref{sec:related} for further discussion). Further \cite{staircase1} does not provide necessary conditions for learning, nor fine-grained complexity guarantees (beyond `polynomial').

Finally, while strong SGD-learnability is defined for a fixed latent function and fixed $P$, the number of samples required to fit MSP functions remains polynomial in $d$ for $P$ growing sufficiently slowly in $d$. This is of interest because in this regime, we can show that the considered functions are not learnable by any linear methods with any sample complexity (or feature space dimension) that is polynomial (using contribution (4) below). Thus the merged-staircase functions of such degree are efficiently learnable by SGD on networks of depth two but not by linear methods.

\subsection{Summary of main results}

Recall that any function $h_* :\{+1 , -1 \}^P \to \R$ can be decomposed in the Fourier-Walsh basis as
$
h_* ( \bz) = \sum_{S \subseteq [P]}\hh_* (S) \chi_S (\bz )\, , \text{where}\quad \hh_* (S) := \< h_* , \chi_S \> \, ,  \chi_S (\bz) := \prod_{i \in S} z_i
$, 
where we denoted the inner-product between two functions $\< f,g\> := \E_{\bz} [ f(\bz) g (\bz) ]$ with $\bz \sim \Unif ( \{ - 1, +1 \}^P )$. This corresponds to expressing the function $h_*(\bz)$ as a weighted sum of orthogonal monomials $\chi_S  (\bz)$, with weights $\hh_* (S)$ called the Fourier coefficients of $h_*$.

We now formally define the Merged-Staircase Property. Let us call any $\cS \subseteq 2^{[P]}$ a \textit{set structure}.
\begin{definition}
We say that $\cS = \{S_1 , \ldots , S_m\} \subseteq 2^{[P]}$ is a \emph{Merged-Staircase Property (MSP) set structure} if the sets can be ordered so that for each $i \in [m]$, $|S_i \setminus \cup_{i' < i} S_{i'}| \leq 1$.
\end{definition}

\begin{definition}[Merged-Staircase Property]
Let $\cS \subset 2^{[P]}$ be the non-zero Fourier coefficients of $h_*$, i.e., $\hh_* (S) \neq 0 $ iff $S \in \cS$. We say that $h_*$ satisfies the \emph{merged-staircase property} (MSP) if $\cS$ is a MSP set structure. \end{definition}

In words, $h_*$ satisfies the MSP if the monomials in its Fourier decomposition can be ordered sequentially such that the supports of the monomials grow by at most one at a time. Examples of MSP functions include vanilla staircases (i.e., $z_1 + z_1z_2 + \dots + \prod_{i=1}^P z_i$), $z_1 + z_1 z_2 + z_2 z_3 + z_3 z_4$, or $ z_1 + z_2 + z_3 + z_4 + z_1 z_2 z_3 z_4$, but not $ z_1  + z_1z_2 z_3 +z_1z_2z_3z_4$, $z_1 + z_1 z_2 + z_3 z_4$, or $z_1 z_2 z_3$. We briefly summarize our results here: 

\begin{description}
\item[(1) Dimension-free dynamics and equivalent characterization of strong SGD-learnability.] We introduce a \textit{dimension-free dynamics} (independent of $d$), which correspond to the gradient flow associated to learning $h_*$ with a certain two-layer neural network in the space of distributions on $\R^{P+2}$. We show $h_*$ is strongly $O(d)$-SGD learnable if and only if this dimension-free dynamics can reach $0$ risk when initialized with first-layer weights at $0$.

\item[(2) MSP necessity.] We show that for non-MSP $h_*$, the associated dimension-free dynamics stays bounded away from $0$. From the previous equivalence, we deduce that MSP is necessary for a function to be strongly $O(d)$-SGD-learnable.

\item[(3) MSP  near-sufficiency.] We first show that vanilla staircases are strongly $O(d)$-SGD-learnable for smooth activation functions as long as $\sigma^{(r)} (0) \neq 0$ for $r = 0 , \ldots, P$. 

For general MSP functions, however, some symmetric MSP functions have degenerate dynamics and are not strongly $O(d)$-SGD-learnable (see Section \ref{sec:sufficient}). We show instead that MSP $h_*$ are almost surely strongly $O(d)$-SGD-learnable. I.e., the degenerate examples are a measure-zero set. This is proved for generic degree-$L$ polynomial activations, and we explain how one can extend this result to generic smooth activations in the appendix.

\item[(4) Superpolynomial separation with linear methods.] 
One can take MSP functions (e.g., vanilla staircases) with $P$ slowly growing with $d$ so that the overall sample complexity of the above neural network results stay as $d^{O(1)}$, while we show that any linear method requires a sample complexity of $d^{\omega_d(1)}$.
\end{description}

These main results are further achieved with several side results of independent interest: (i) The approximation of the standard mean-field dynamics by the \textit{dimension-free dynamics}, valid for $P$-sparse target functions and $d \gg P$. We provide a new version of the non-asymptotic bounds from \cite{mei2018mean,mei2019mean}, which now compares SGD with this dimension-independent dynamics; (ii) A new proof technique to study layer-wise SGD dynamics which reduces the proof of global convergence to a polynomial identity testing problem, i.e., whether a certain polynomial is non-identically zero; (iii) An improvement of prior dimension lower-bounds for linear (kernel) methods \cite{hsu2021approximation,hsudimension,kamath_dim} that is tighter for function classes that are non-almost orthogonal (such as staircase functions, allowing for contribution (4) above). 

The rest of the paper is organized as follows. The next section overviews related work. Section \ref{sec:strong-SGD} provides a formal definition of strong SGD-learnability in $O(d)$-scaling. In Section \ref{sec:DF-PDE-Necessary}, we introduce the dimension-free dynamics and the equivalence with strong $O(d)$-SGD-learnability. The MSP necessary condition is then derived as a direct consequence of this equivalence. In Section \ref{sec:sufficient}, we provide our sufficient conditions for strong $O(d)$-SGD-learnability. In Section \ref{sec:linear-methods}, we discuss how this implies a separation with linear methods.

\subsection{Further related literature}
\label{sec:related}

\cite{staircase1} introduces a class of staircase functions, which our merged-staircase function class generalizes. They show that staircase functions are learnable by some neural nets with a gradient-based training algorithm. However, the approach remains non-standard: (i) the network's layers are sparse in order to guide the construction of the features; (ii) a coordinate descent variant of SGD is used that differs from the classical SGD algorithm. Further, the analysis is carried in the `polynomial scaling lens' rather than a finer sample complexity, and no necessity results are derived.  In contrast, we provide here both a necessary and nearly sufficient characterization for SGD-learning on a two-layer neural networks in the fine-grained $O(d)$-scaling.

Multiple works have used mean-field (also called distributional) dynamics to approximate the SGD trajectory. Relevant to us is \cite{chizat2020implicit} which showed that neural networks trained in the mean-field regime converge to a max-margin classifier that is independent of the dimension for latent low-dimensional target functions. However, these works do not provide quantitative results in terms of sample-complexity. A notable exception is \cite{mei2018mean} which studies classifying anisotropic gaussians: they show that the mean-field dynamics concentrates on a simplified low-dimensional dynamics as $d\to \infty$. However, this simplification is due to rotational invariance of the problem and not the sparsity of the target function.

In approximation theory, it has been understood for a long time that sparse functions are naturally well approximated by neural networks \cite{barron1993universal}. Recent work \cite{bach2017breaking,schmidt2020nonparametric,ghorbani2021neural,celentano2021minimum} have shown that neural networks can learn sparse functions more sample-efficiently than linear methods. However, these works do not provide tractable algorithms.

Finally, a string of works \cite{yehudai19power,allen19can,allen20backward,li20beyondntk,daniely20parities,refinetti2021classifying,ghorbani2021linearized, ghorbani2021neural,quantifying,karp2021local,suzuki2020benefit} have shown separation results between gradient-trained neural networks and fixed-features models. We refer to Appendix B of \cite{quantifying} for a detailed survey. In particular,  \cite{daniely20parities}  considers the learning of parity functions, with a modified input distribution that gives correlation to the response and allows for domain extraction; it also uses the population dynamics (infinite samples). In \cite{malach2020implications},  the learning of Boolean circuits of logarithmic depth is considered via neural networks with layer-wise gradient descent, but with an architecture that is required to match the Boolean circuit being learned, i.e., not with a `regular' or `blackbox' architecture. 
Lastly,  \cite{basri2019convergence,ijcai2021-304} show that during training, SGD on  2-layer $\mathsf{ReLU}$ networks learns faster the lower frequency components of a target function, in similar spirit to low degree monomials, but the approach relies on the linear regime rather than the non-linear regime of interest here, and  suffers from an exponential dependency on the degree.

\section{Strong SGD-learnability in $O(d)$-scaling}\label{sec:strong-SGD}

Consider $n$ iid data points $(\bx_i , y_i)_{i \in [n]}$ with covariates $\bx_i \sim \Unif ( \{+1 , -1 \}^d)$ and responses $y_i = f_* (\bx) + \eps_i$ with bounded independent noise $\E [ \eps_i ] = 0$. We assume that $f_* : \{ +1 , -1\}^d \to \R$ is a sparse function with latent $P$-dimensional function $h_*$, i.e., there exists an (unknown) subset of coordinates $\bz = \bx_{\cI} = (x_{i_1} , \ldots , x_{i_P})$ (the signal part of the input) such that $f(\bx) = h_* ( \bz)$. We consider fitting this data using a two-layer fully-connected neural network with $N$ hidden units and weights $\bTheta := (\btheta_j)_{j\in [N]} = (a_j,\bw_j)_{j \in [N]} \in \R^{N(d+1)}$:
\begin{equation}\label{eq:NN}
    \hf_\NN  (\bx ; \bTheta) = \frac{1}{N} \sum_{j \in [N]} a_j \sigma ( \< \bw_j , \bx \>)\, . \tag{2-NN}
\end{equation}
We train the parameters $\bTheta$ using batch-SGD with square loss and batch size $b$. We allow for time-varying step sizes $\{(\eta^a_k , \eta_k^w)\}_{k \geq 0}$, and layer-wise $\ell_2$-regularization with parameters $\lambda^a , \lambda^w \geq 0$. Given samples $\{ ( \bx_{ki} , y_{ki} )_{i \in [b]} \}_{k \geq 0}$ and initialization $(\btheta_j^0 )_{j \in [N]} \sim_{iid} \rho_0$, the weights are updated at each step:
\begin{equation}\label{eq:bSGD}\tag{bSGD}
\btheta_j^{k+1} = \btheta_j^k + \frac{1}{b} \sum_{i \in [b]} \{ y_{ki} - \hat f_{\NN} ( \bx_{ki} ; \bTheta^k ) \} \cdot \bH_k  \nabla_{\btheta} [a_j^k \sigma ( \< \bw_j^k , \bx_{ki} \> ) ] - \bH_k \bLambda \btheta_j^k\, ,
\end{equation}
where we introduced $\bH_k = \diag (\eta^a_k , \eta^w_k \cdot \id_d )$ and $\bLambda = \diag (\lambda^a , \lambda^w \cdot \id_d )$.
We will be interested in the prediction error (test error) $R ( f_* , \hf ) = \E_{\bx } \big[ \big\{ f_* ( \bx ) - \hf ( \bx ) \big\}^2 \big]$.

We first consider a general definition for a class of sparse functions to be learnable. We take $\{ P(d) \}_{d \geq 1}$ a sequence of integers (here, we allow the sparsity parameter $P$ to grow with $d$) and consider a general class of functions defined as $\cH = \{ \cH_{P(d)} \}_{d \geq 1}$ with $\cH_{P(d)} \subseteq L^2 ( \{+1,-1\}^{P(d)})$. 

\begin{definition}[SGD-learnability in $O(d^\alpha)$-scaling]
We say that a function class $\cH$ is \emph{SGD-learnable in $O(d^\alpha)$-scaling} if the following hold for some $C ( \cdot , \cH): \R_{>0} \to \R_{>0}$. For any $h \in \cH_{P(d)}$, $\eps >0$ there exist hyperparameters $(N,b, \sigma, \lambda^a, \lambda^w, \{ \eta_k^a , \eta_k^w   \}_{k \in [0,k_0]})$ and initialization $\rho_0$, such that: (1) for a sample size\footnote{Note that $n \leq bk_0$ with equality if we assume fresh samples at each iteration, as in the next definition.} $n  \leq C ( \eps, \cH)   d^{\alpha}$; and (2) for any $\cI \subseteq [d], | \cI | = P(d)$, and target function $f_*(\bx) = h (\bx_{\cI})$, $k_0$ steps of batch stochastic gradient descent \eqref{eq:bSGD} achieves prediction error $\eps$ with prob.\ at least $9/10$.
\end{definition}

This definition covers many scenarios that occur in practice where the practitioner is allowed to tune the hyperparameters of the dynamics. While this choice leaves the question of tractability open, we note that the requirement that learnability must hold uniformly over all possible latent subspaces excludes many irregular scenarios. Furthermore, the next definition will require strong regularity on the hyperparameters, and our sufficiency results will hold for simple choices of hyperparameters.

In order to introduce strong SGD-learnability, we will restrict the previous definition in three major ways: (1) we consider a fixed dimension $P$ and a fixed function $\cH = \{h_*\}$, which is still nontrivial to learn since we do not know the set $\cI \subseteq [d]$ such that $f_*(\bx) = h_*(\bx_{\cI})$; (2) we consider the scaling\footnote{Extending our results to $\alpha>1$, and establishing how this relates to the `leap' in the staircase definition (i.e., how can one jump monomial degrees) is a natural future direction to this work.} of $\alpha = 1$; (3) we restrain the hyperparameters to be in either of two regimes (i) small batch size $b = o(d)$ and step size $\eta = o(1)$ trained for $\Theta(1/\eta)$ steps (``continuous''); and (ii) large batch size $b = \Theta(d)$ and step size $\eta = \Theta(1)$ trained for a total number of $\Theta(1)$ steps (``discrete''). For the sake of presentation, we will only present the continuous regime in the main text and defer the presentation of the discrete regime to Appendix \ref{app:strong-discrete}. We will assume that the hyperparameters obey the following for some constant $K$ (independent of $d$):
\begin{itemize}

    \item[$\rA 0.$]  \textit{(Activation)} $\sigma:\R\to\R$ is three times differentiable with $\| \sigma^{(k)} \|_{\infty}\leq K$ for $k = 0 , \ldots , 3$. 
    
    \item[${\rm A1.}$] \textit{(One-pass)} We have fresh samples at each steps, meaning $\{ ( \bx_{k_i} , y_{k_i} )\}_{k \geq 0, i \in [b]}$ are iid. Furthermore, the response variable is bounded $|y| \leq K$.
    
    \item[${\rm A2.}$] \textit{(Initialization)} The initialization verifies $(a_i^0, \sqrt{d} \cdot \bw_i^0) \sim \mu_a \otimes  \mu_w^{\otimes d}$ where the distributions $\mu_a , \mu_w \in \cP (\R)$ are independent of $d$ with $|a|\leq K$ on the support of $\mu_a$ and $\mu_w$ is symmetric and $K^2$-sub-Gaussian. We will denote $m_2^w := \E_{W \sim \mu_w} [W^2]^{1/2}$.
    
    \item[${\rm A3.}$] \textit{(Boundedness and lipschitzness of hyperparameters)} There exists a constant $\eta>0$ such that $\eta_k^a, \eta_k^w \leq \eta K$, $|\eta_{k+1}^a - \eta_k^a | \leq \eta^2 K$ and $|\eta_{k+1}^w - \eta_k^w | \leq \eta^2 K$. Furthermore, $\lambda^a , \lambda^w \leq K$.
\end{itemize}

\begin{definition}[Strong SGD-learnability in $O(d)$-scaling]
We say that a function $h_* : \{ - 1, +1 \}^P \to \R$ is \emph{strongly $O(d)$-SGD-learnable} if the following hold for some $C (\cdot , h_*), T(\cdot,h_*):\R_{>0} \to \R_{>0}$. For any $\eps >0$, $d \geq C ( \eps, h_*)   $, $n \geq C ( \eps, h_*)   d$ and $e^d \geq N \geq C (\eps , h_*)$, there exists hyperparameters $(\sigma, b, \lambda^a, \lambda^w, \{ \eta_k^a , \eta_k^w   \}_{k \in [0,k_0]})$ and initialization $\rho_0$ satisfying $\rA0$-$\rA3$, $b \leq d$ and $k_0 = n/b  \leq  T ( \eps , h_*)/\eta$ s.t.\ for any $\cI \subseteq [d], | \cI | = P$ and target function $f_*(\bx) = h_* (\bx_{\cI})$, $k_0$ steps of batch stochastic gradient descent \eqref{eq:bSGD} achieves test error $\eps$ with prob.\ at least $9/10$.
\end{definition}

Conditions $\rA0$-$\rA3$ guarantee that as long as $d,n,N$ are taken sufficiently large, there exists a continuous mean-field dynamics that well-approximates batch-SGD up to (continuous) time $T$ depending on $\eps,h_*$. An analogous statement is true for strong-SGD-learnability in the ``discrete regime'', except convergence is to a family of limiting discrete-time dynamics (deferred to Appendix~\ref{app:strong-discrete}). This allows us to get a necessary condition for strong-learnability by studying the limiting dynamics (see next section). 

Finally, we note that for any degree-$k$ sparse function $h_*$, any linear method (e.g., arbitrary kernel or random feature methods) will require $\Omega(d^k)$ samples to fit functions $f_* (\bx) = h_* (\bz)$ uniformly well over all latent subspaces $\bz = \bx_{\cI}$ (see Section \ref{sec:linear-methods} for a formal statement). As emphasized in the introduction, this bound is not adaptive to the sparsity parameter $P$. In particular, any non-linear $h_*$ that is strongly $O(d)$-SGD-learnable provides a separation result between SGD-trained neural networks and linear methods. 

\section{Continuous dimension-free dynamics and necessary condition}
\label{sec:DF-PDE-Necessary}

For simplicity, the results in this section are stated in the `continuous regime' of strong SGD-learnability. Discrete versions can be found, with little modification, in Appendix \ref{app:strong-discrete}. 

\paragraph*{Mean-field approximation:} A recent line of work \cite{chizat2018global,mei2018mean,rotskoff2018neural,sirignano2020mean,mei2019mean} showed that one-pass batch-SGD \eqref{eq:bSGD} can be well approximated in some regime by a continuous dynamics in the space of probability distributions on $\R^{d+1}$, which we will refer to as the \textit{mean-field dynamics}.

Before describing this limiting dynamics, we first introduce a few definitions. To any distribution $\rho \in \cP (\R^{d+1} )$, we associate the infinite-width neural network
\begin{equation}
\hf_{\NN} (\bx ; \rho) = \int a \sigma ( \< \bw , \bx \>) \rho(\de a \, \de \bw ) \, .
\end{equation}
In particular, \eqref{eq:NN} corresponds to taking the empirical distribution $\hat \rho^{(N)} = N^{-1} \sum_{j \in [N]} \delta_{\btheta_j}$. We assume further that there exist functions $\xi^a , \xi^w : \R_{\geq 0} \to \R_{\geq 0}$ and a parameter $\eta>0$ (the time discretization) such that $\eta_k^a = \eta \xi^a (k\eta )$ and $\eta_k^w = \eta \xi^w (k\eta )$.
We replace Assumption $\rA 3$ by :
\begin{itemize}
    \item[$\rA 3 '$.] $\xi^a , \xi^w $ are bounded Lipschitz  $\| \xi^a \|_{\infty}, \| \xi^w \|_{\infty},\| \xi^a \|_{\text{Lip}}, \| \xi^w \|_{\text{Lip}} \leq K$ and $\lambda^a,\lambda^w \leq K$.
\end{itemize}
Note that for any $(\eta_k^a,\eta_k^w)_{k \geq 0}$ obeying $\rA3$, there exists functions $\xi^a,\xi^w$ such that $\rA3'$ holds with same constant $K$. Conversely, any $\eta$ discretization of $\xi^a,\xi^w$ obeys $\rA3$ with constants $\eta,K$.

Consider the empirical distribution of the weights $\hat \rho_k^{(N)} $ after $k$ batch-SGD steps, i.e., $\hat \rho_k^{(N)} = N^{-1} \sum_{j \in [N]} \delta_{\btheta_j^k}$. For large $N$ and small step size $\eta$, setting $k = t / \eta$, $\hat \rho_k^{(N)} $ is well approximated by a distribution $\rho_t \in \cP(\R^{d+1})$ that evolves according to the following PDE:
\begin{equation}\label{eq:MF-PDE}\tag{MF-PDE}
\begin{aligned}
    \partial_t \rho_t = &~ \nabla_\btheta \cdot (\rho_t  \bH (t) \nabla_\btheta \psi ( \btheta ; \rho_t ) ) \, ,\\
    \psi ( \btheta ; \rho_t ) = &~  a\E_{\bx} \Big[ \big\{\hat f_{\NN} ( \bx ; \rho_t) -  f_* (\bx) \big\} \sigma ( \< \bw , \bx \>) \Big] + \frac{1}{2} \btheta^\sT \bLambda \btheta \, ,
\end{aligned}
\end{equation}
with initial distribution $\rho_0$, and where we introduced $\bH(t) = \diag (\xi^a (t) , \xi^w (t) \id_d )$. This PDE corresponds to a Wasserstein gradient flow on the square-loss test error $R(\rho):= \E\big[ \big\{ f_* (\bx) - \hf_{\NN} (\bx ;  \rho) \big\}^2 \big]$ with regularization $\int \btheta^\sT \bLambda \btheta \rho (\de \btheta)$ and learning schedule $\bH(t)$.

\paragraph*{Dimension-free dynamics:} For a sparse function $f_* (\bx) := h_* (\bz)$, the \eqref{eq:MF-PDE} concentrates to a dimension-free dynamics when $d \to \infty$. Decompose the input $\bx = ( \bz , \br)$ and the weights $\bw_i = (\bu_i , \bv_i)$ with $\bu_i \in \R^{P}$ aligned with $\bz$ and $\bv_i \in \R^{d-P}$ with $\br$. By Assumption ${\rA 2}$, $\bw_i^0$ has iid symmetric coordinates, which implies that $\hf_{\NN} ( \bx ; \rho_0)$ does not depend on $\br$. In fact, by symmetry of \eqref{eq:MF-PDE}, the mean-field solution $\hf_{\NN} ( \bx ; \rho_t)$ stays independent of $\br$ throughout the dynamics:
\begin{equation}\label{eq:NN_MF_sym}
\forall t \geq 0\, , \qquad \hf_\NN (\bx ; \rho_t ) = \int a^t \sigma( \< \bx , \bw^t \>) \rho_t (\de \btheta^t) = \int a^t \E_\br \big[ \sigma ( \< \bu^t , \bz \> + \< \bv^t ,  \br \>  )\big]  \rho_t (  \de \btheta^t) \, ,
\end{equation}
and we denote with a slight abuse of notation, $\hf_{\NN} ( \bz ; \rho_t) := \hf_{\NN} ( \bx ; \rho_t)$.

 With $\br \sim \Unif ( \{-1 , +1 \}^{d-P} )$, one can show that $\< \bv^t , \br \>$ can be well approximated by $\| \bv^t \|_2 G$ with $G \sim \normal (0,1)$ when $d\gg P$ and $t = O_d(1)$.  We introduce effective parameters $\obtheta^t = ( \oa^t , \obu^t , \os^t )$ with distribution $\orho_t \in \cP(\R^{P+2})$, and we replace the neural network \eqref{eq:NN_MF_sym} by an effective neural network (with a slight abuse of notation)
\begin{equation}
    \hf_{\NN} (\bz ; \orho_t ) := \int \oa^t \E_G \big[ \sigma ( \< \obu^t , \bz \> + \os^t G) \big] \orho_t(\de \obtheta^t )\, .
\end{equation}
We see that $\hf_{\NN} (\cdot ; \orho_t )$ can be seen as a two layer neural network in dimension $P$, with adaptive Gaussian smoothing. Taking $d \to \infty$ with $P$ fixed, $(a^0, \bu^0 , \| \bv^0 \|_2 )$ (with distribution $\rho_0$ satisfying $\rA2$) converges in distribution to $(\oa^0 , \obu^0 , \os^0) \sim \orho_0$ with $\oa^0 \sim \mu_a$, $\obu^0 = \bzero$ and $\os^0 = m_2^w$, and the dynamics \eqref{eq:MF-PDE} simplifies into the following dimension-free dynamics 
\begin{equation}\tag{DF-PDE}\label{eq:DF-PDE}
\begin{aligned}
        \partial_t \orho_t = &~ \nabla_{\obtheta} \cdot \big(\orho_t  \obH (t) \nabla_{\obtheta} \psi ( \obtheta ; \orho_t ) \big) \, ,\\
    \psi ( \obtheta ; \orho_t ) = &~  \frac{1}{2}\E_{\bz,G} \Big[ \big\{\hat f_{\NN} ( \bz ; \orho_t) -  f_* (\bz) \big\} \oa \sigma ( \< \obu , \bz \> +\os G) \Big] + \frac{1}{2} \obtheta^\sT \obLambda \obtheta \, ,
    \end{aligned}
\end{equation}
where $\obH (t) = \diag (\xi^a (t) , \xi^w (t) \id_{P+1})$ and $\obLambda = \diag (\lambda^a , \lambda^w \id_{P+1} ) $. Equivalently, \eqref{eq:DF-PDE} can be seen as a Wasserstein gradient flow over the test error $R(\orho) = \E_{\bz} \big[ \big\{ h_* ( \bz) - \hf_{\NN} ( \bz ; \orho ) \big\}^2 \big]$ in the space $\bz \sim \Unif (\{+1,-1\}^P)$ with initialization $\orho_0$ and regularization $\int \obtheta^\sT \obLambda \obtheta \orho (\de \obtheta)$. We put further intuition for this result in Appendix \ref{app:intuition-DF-PDE}.

The following theorem provides a non-asymptotic bound between the \eqref{eq:bSGD} solution $\hf_{\NN} (\cdot ; \hat\rho^{(N)}_t)$ and the \eqref{eq:DF-PDE} solution $\hf_{\NN} (\cdot ; \orho_t )$:

\begin{theorem}\label{thm:bSGD-to-DF}
Assume conditions $\rA0$-$\rA2$,$\rA3'$ hold, and let $T \geq 1$. There exist constants $K_0$ and $K_1$ depending only on the constants in $\rA 0$-$\rA 2$,$\rA3'$ (in particular, independent of $d,P,T$), such that for any $b\leq d$, $N\leq e^d$, $\eta \leq e^{-K_0 T^3} b/(d+\log(N)) $,  we have
\begin{equation*}
\begin{aligned}
 \big\Vert \hat f_{\NN} (\cdot; \bTheta^k ) - \hat f_{\NN} (\cdot; \orho_{k\eta} ) \big\Vert_{L^2}
\leq&~  e^{K_1T^7} \left\{ \sqrt{\frac{P + \log(d)}{d}} + \sqrt{\frac{\log N}{N}} +  \sqrt{\frac{d+\log N}{b}}  \sqrt{ \eta} \right\} \, ,
\end{aligned}
\end{equation*}
for all $k \in [T/\eta] \cap \naturals$, with probability at least $1 - 1/N$.
\end{theorem}

The proof of Thm.\ \ref{thm:bSGD-to-DF} can be found in App.\ \ref{app:proof-bSGD-to-Df}. We first extend the results in \cite{mei2019mean} to bound the difference between \eqref{eq:bSGD} and \eqref{eq:MF-PDE} dynamics, and then we use a propagation-of-chaos argument to bound the distance between the \eqref{eq:MF-PDE} and \eqref{eq:DF-PDE} solutions.

\begin{figure}[t]
\begin{center}
    \includegraphics[width=14cm]{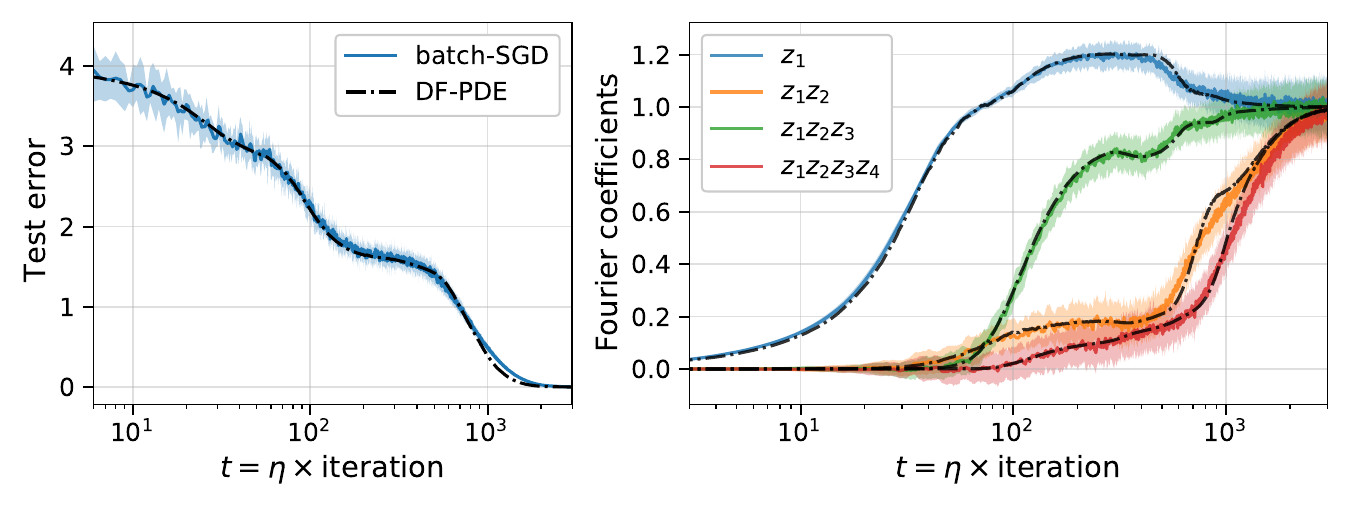}
\end{center}

\vspace{-25pt}
\caption{Comparison between \eqref{eq:bSGD} and \eqref{eq:DF-PDE} dynamics for $h_* (\bz) = z_1 + z_1 z_2 + z_1 z_2 z_3 + z_1 z_2 z_3 z_4$. Left: Test error. Right: Fourier coefficients of $\hf_{\NN} (\bx ; \bTheta^{t/\eta})$ and $\hf_{\NN} (\bz ; \orho_t)$. The dashed-dotted black lines correspond to \eqref{eq:DF-PDE} and the continuous colored line to \eqref{eq:bSGD}. The test errors and Fourier coefficients are evaluated with $m=300$ test samples and for \eqref{eq:bSGD}, we report the average and $95\%$ confidence interval over 10 experiments.  \label{fig:Staircase_SGDvsDF}}
\end{figure}

\paragraph*{Equivalence with SGD-learnability:} From Theorem \ref{thm:bSGD-to-DF}, \eqref{eq:DF-PDE} is a good approximation of \eqref{eq:bSGD} as long as $d,N,1/\eta$ are taken sufficiently large while keeping $T =\eta n /b$ bounded. This leads to the equivalence described in the introduction (the proof can be found in Appendix \ref{app:proof-bSGD-to-Df}):

\begin{theorem}\label{thm:equivalence} A function $h_* : \{+1,-1\}^P \to \R$ is strongly $O(d)$-SGD-learnable if and only if for any $\eps >0$, there exists $\lambda^a, \lambda^w \geq 0$ and Lipschitz $\xi^a, \xi^w : \R_{>0} \to \R_{>0}$, such that $\inf_{t \geq 0} R(\orho_t) < \eps$.
\end{theorem}

For generic activation, we have $\inf_{\orho} R (\orho) = 0$. Hence, Theorem \ref{thm:equivalence} states that $h_*$ is strongly $O(d)$-SGD-learnable if and only if the global minimizer is dynamically reachable by a gradient flow initialized at $\orho_0 = \mu_a \otimes \delta_{\obu^0 = \bzero} \otimes \delta_{\os^0 = c}$. See Appendix~\ref{app:numerics} for additional discussions and numerical illustrations. 
In Figure \ref{fig:Staircase_SGDvsDF}, we plotted a comparison between \eqref{eq:bSGD} and \eqref{eq:DF-PDE} for $h_* ( \bz) = z_1 + z_1 z_2 + z_1 z_2 z_3 + z_1 z_2 z_3 z_4$ and shifted sigmoid activation $\sigma ( x) = (1 + e^{-x+0.5} )^{-1}$. We fix $d = N = 100$, $b = 150$, $\lambda^a = \lambda^w = 0$, $\eta^a_k  = \eta^w_k = 1/2$, $\mu_a = \Unif([+1,-1])$ and $\mu_w = \normal (0,1)$. Let us emphasize a few prominent features of this plot: 1) The \eqref{eq:DF-PDE} approximation tracks well \eqref{eq:bSGD} until convergence even for moderate $d,N,b/\eta$, despite a convergence with nontrivial structure. 2) The monomials are picked up sequentially to a nonnegligible amount with increasing degree, which agrees with the intuition that lower-degree monomials guide SGD to learn higher degree monomials. 3) \eqref{eq:DF-PDE} reaches a global minimum, which by Theorem \ref{thm:equivalence} implies that $h_*$ is strongly SGD-learnable in $O(d)$-scaling.

\paragraph*{MSP is necessary:} We can show that the \eqref{eq:DF-PDE} dynamics with $h_*$ without MSP cannot reach arbitrarily small test error when initialized with $\orho_0$. By Theorem \ref{thm:equivalence}, this implies that MSP is necessary for strong SGD-learnability in $O(d)$-scaling.

\begin{theorem}\label{thm:MSP_necessary}
Let $h_* : \{+1,-1\}^P \to \R$ be a function without MSP. Then there exists $c >0$ such that for any $\xi^a, \xi^w : \R \to \R$ and regularizations $\lambda^a, \lambda^w \geq 0$, we have $\inf_{t \geq 0} R(\orho_t) \geq c$.
\end{theorem}

This result is based on the following simple observation: for $h_*$ without MSP and with $\obu^0 = \bzero$ initialization, some coordinates stay equal to $0$ throughout the dynamic, i.e., $\ou_i^t = 0$. In that case, any Fourier coefficient that contains $i \in S$ is not learned: $\E_{\bz} [ \hf_{\NN} ( \bz ; \orho_t) \chi_S (\bz) ] = 0 $. We report the proof to Appendix \ref{app:proof-MSP-necessary} and simply detail one example $h_* ( \bz) = z_1 + z_1 z_2 z_3$. Consider the first-layer weight evolution $\dot{\ou}_i^t = \oa^t \E [(z_1 +z_1z_2z_3)\sigma' (\< \obu^t , \bz \>) z_i ]$ (for the sake of intuition, we take $\hf_{\NN} ( \bz ;\orho_t) = 0$ and $\os^t = 0$). Notice that the evolution equations are symmetric under exchange $\ou_2^t \leftrightarrow \ou_3^t$ with $\ou_2^0 = \ou_3^0 = 0$ and therefore $\ou_2^t = \ou_3^t =: u_{23}^t$. Denoting $z_{23} = z_2 + z_3$ and integrating out $z_{23}$, 
\[
\dot{u}_{23}^t = \oa^t \E_{\bz} \big[ z_1 z_{23}\sigma' (z_1 \ou_1^t + z_{23} u_{23}^t) \big] =\oa^t  u_{23}^t  \E_{z_1} \big[ z_1 \sigma'' (z_1 \ou_1^t + r( u_{23}^t) ) \big]  \,,
\]
for some $r(u_{23}^t) \in [-2u_{23}^t,2u_{23}^t]$ using the mean value theorem. Recalling that $u_{23}^0 = 0$, we deduce that $u^t_{23} = \ou_2^t = \ou_3^t = 0$.

\section{Sufficient conditions for strong SGD-learnability}\label{sec:sufficient}

In the previous section, we saw that having MSP is necessary for strong $O(d)$-SGD-learnability. Is the converse true? Is any MSP function strongly SGD-learnable in the $O(d)$-scaling?

\paragraph{Degenerate cases:} It turns out that one first has to exclude some special cases. Some MSP functions present degenerate dynamics due to their symmetries and are not strongly $O(d)$-SGD-learnable. For example, take $h_* ( \bz) = z_1 + z_2 + z_1 z_3 + z_2 z_4$, which is invariant by permutation $(1,2,3,4) \leftrightarrow (2,1,4,3)$ of its input coordinates. During the \eqref{eq:DF-PDE} dynamics, $\ou_1^t = \ou_2^t$ and $\ou_3^t = \ou_4^t$, which implies that a solution with $\E_{\bz} [ \hf_{\NN} (\bz ;\orho_t) z_1 z_3 ]  = \E_{\bz} [ \hf_{\NN} (\bz ;\orho_t) z_2 z_3 ]$ is found and therefore the risk of the \eqref{eq:DF-PDE} dynamics is always bounded away from zero. See Section \ref{app:numerics} for numerical simulations and further discussion on degenerate MSPs.

\paragraph{Generic MSP functions are learnable:} To bypass this difficulty, we prove a learnability result that holds for ``generic'' MSP functions -- i.e., that holds almost surely over a random choice of non-zero Fourier coefficients. Formally, for any set structure $\cS = \{S_1,\ldots,S_m\} \subseteq 2^{[P]}$, let us define a measure over functions that have those Fourier coefficients.
\begin{definition}\label{def:leb-msp}
For any set structure $\cS \subseteq 2^{[P]}$ define the measure $\mu_{\cS}$ over functions $h_* : \{+1,-1\}^P \to \R$ induced by taking $h_*(\bz) = \sum_{S \subseteq [P]} \alpha_S \chi_S(\bx)$, where the Fourier coefficients satisfy $\alpha_S = 0$ if $S \not\in\cS$, and $(\alpha_S)_{S \in \cS}$ have Lebesgue measure on $\R^{|\cS|}$.
\end{definition}

Our main sufficiency result shows that the degenerate cases are a measure-zero set. In this sense, there are very few bad examples, and so MSP structure is ``nearly'' sufficient for strong $O(d)$-SGD-learnability.
\begin{theorem}\label{thm:discrete-msp}
For any MSP set structure $\cS \subseteq 2^{[P]}$, $h_*$ is strongly $O(d)$-SGD-learnable almost surely with respect to $\mu_{\cS}$, using activation function $\sigma(x) = (1+x)^L$ where $L = 2^{8P}$.\footnote{Technically speaking, for the strong SGD-learnability definition we cannot take $\sigma(x) = (1+x)^L$ as it is not bounded. However, we take an activation function that equals $(1+x)^L$ on the interval $(-1,1)$ and is bounded elsewhere.}
\end{theorem}

The converse to this result is implied by the necessity result of the previous section, which states that for any $h_*$ with non-zero Fourier coefficients (set structure) $\cS$ that is \textit{not} MSP, $h_*$ is \textit{not} strongly $O(d)$-SGD-learnable. While we prove Theorem \ref{thm:discrete-msp} for a particular activation, we note that the proof implies that the same is true for any degree-$L$ polynomial activation almost surely over its $(L+1)$-coefficients (see Theorem \ref{thm:discrete-msp-more-activations} in Appendix \ref{app:discrete-msp-proof}). In Appendix~\ref{ssec:genericmspproof}
we show how this result extends to generic smooth (non-polynomial) activations as long as a certain polynomial is not identically $0$ for a given set structure (which we show with a small technical caveat).

\paragraph{Vanilla staircase, learnable without genericity:}
In the special case of functions with ``vanilla staircase''  structure we do not need a genericity assumption, and we require weaker assumptions on the activation function.
\begin{theorem}\label{thm:vanillastaircasesuffrestated}
Let $h_*$ be of the form $h_* ( \bz) = \alpha_1 z_1 + \alpha_2 z_1z_2 + \ldots + \alpha_P z_1 z_2 \cdots z_P$ where $\alpha_{i} \neq 0$ for $i \in [P]$. Then $h_*$ is strongly $O(d)$-SGD-learnable using any activation function $\sigma \in \cC^{2^{P-1}+1} (\R)$ with nonzero derivatives $\sigma^{(r)}(0) \neq 0$ for $r = 0 , \ldots, P$. 
\end{theorem}

\paragraph{Proof ideas}
The proofs for Theorems~\ref{thm:discrete-msp} and \ref{thm:vanillastaircasesuffrestated} follow a similar approach. From the equivalence stated in Theorem \ref{thm:equivalence}, it is sufficient to display, for each $\eps >0$, hyperparameters such that the \eqref{eq:DF-PDE} dynamics reaches $\eps$-risk. We choose $\lambda^a = \lambda^w = 0$ (no regularization) and initialization $\mu_a = \Unif ( [-1,+1])$ and $\mu_w = \delta_0$ (this choice simplifies the analysis as $\os^t = \os^0 = 0$). We split the learning in two phases: in \textit{Phase 1}, we train the first layer weights $\obu^t$ for time $t \in [0,T_1]$ while keeping $\oa^t=\oa^0$ fixed, and in \textit{Phase 2}, we train the second layer weights $\oa^t$ for time $t \in [T_1,T_2]$ while keeping $\obu^t= \obu^{T_1}$ fixed. 

At the end of Phase 1, denote $(\oa^0, \obu^{T_1} (\oa^0))$ the weights obtained from the evolution \eqref{eq:DF-PDE} from initialization $(\oa^0, \obu^0 = \bzero)$ (note that $\obu^{T_1} (\oa^0)$ is a deterministic function of $\oa^0$). Phase 2 corresponds to a linear training phase with kernel $K^{T_1}(\bz , \bz ') = \E_{\bar a^0 \sim \mu_a} [\sigma(\<\bar\bu^{T_1}(\bar a^0),\bz\>)\sigma(\<\bar\bu^{T_1}(\bar a^0),\bz'\>)]$. In particular, the risk decreases as $\exp ( - \lambda_{\min} (\bK^{T_1}) t)$ during this phase, where we denote by $\bK^{T_1} = (K^{T_1} (\bz , \bz') )_{\bz,\bz' \in \{+1,-1\}^P}$ the kernel matrix. Showing global convergence reduces to showing that $\lambda_{\min} (\bK^{T_1})>0$ for some $T_1$ and taking $T_2 = T_1 + \log(1/\eps)/\lambda_{\min} (\bK^{T_1})$.

The goal of the analysis in Phase 1 is therefore to prove this lower bound on the eigenvalues of the kernel matrix. Phase 1 corresponds to  a nonlinear dynamics, and is a priori unclear how to analyze. In the case of vanilla staircases, we show that it is enough to track the leading order in $t$ for each coordinates $(\ou_i^t)_{i \in [P]}$ and take $T_1$ small enough. For example, when learning $h_*(\bz) = z_1 + z_1z_2 + z_1z_2z_3$,  for small time $t \leq T_1$ we can roughly show that $\bar{u}_{1}^t(a) \propto a t$, that $\bar{u}_{2}^t(a) \propto a t^2$, and that $\bar{u}_{3}^t(a) \propto a t^4$. In other words, the weight corresponding to $z_1$ increases in magnitude the fastest, followed by the weight corresponding to $z_2$, and then weight corresponding to $z_3$. We can then use this explicit calculation to lower bound the eigenvalues of $\bK^{T_1}$, crucially using that the second-layer weights $\bar a^0 \sim \mu_a$ are chosen at random, which ensures that the neurons are diverse enough. See Appendix \ref{ssec:vanillaproof} for the detailed proof.

For general MSP set structure, it is not enough to only track the weights $\barbu^t(\bar a^0)$ to leading order in $t$. We show instead that it suffices to lower bound a kernel matrix $\hat{\bK}^{T_1}$ obtained from a simplified dynamics $\hat{\bu}^t (\oa^0)$. The weights $\hat{u}_i^t (\oa^0)$ can be written in terms of polynomials in the second-layer weights $\oa^0$, the Fourier coefficients $(\alpha_S)_{S \in \cS}$ and the derivatives of the activation $(\sigma^{(r)}(0))_{r = 0 , \ldots , L}$, with coefficients defined explicitly by a recurrence relation and only depending on the set structure $\cS$. Using algebraic facts about the linear independence of large powers of polynomials we show that $\det(\hat{\bK}^{T_1})$ is a nonzero polynomial in the second-layer weights and Fourier coefficients. Therefore, plugging in random second-layer weights $\oa^0 \sim \Unif([-1,1])$, and random Fourier coefficients we show that $\det ( \hat{\bK}^{T_1} ) \neq 0$ almost surely, by anti-concentration of polynomials. This implies in particular that $\sum_{S \in \cS} \alpha_S$ is almost surely strongly $O(d)$-SGD-learnable. See Appendix \ref{app:discrete-msp-proof}.

\section{Separation with linear methods}
\label{sec:linear-methods}

It is known that linear methods with $\poly(d)$ many features or samples cannot learn the class of degree-$P$ monomials if $P$ grows with the input dimension $d$ \cite{hsu2021approximation,hsudimension,kamath_dim}. One way of proving this is by using SQ lower bounds \cite{blum94weakly}, which imply lower bounds on linear methods \cite{kamath_dim}. However, this proof strategy fails for staircase functions of growing degree, since the hierarchical structure makes these efficiently SQ learnable by sequentially querying the monomials of increasing degree\footnote{Making at most $d$ queries per degree for vanilla staircases, e.g., at most $d \log d$ queries when $P=\log d$.}. We thus need a lower-bound on linear methods that goes beyond SQ lower-bounds, which we obtain by using subspace projections.

Consider a general linear method which is defined by a Hilbert space $(\cH , \< \cdot, \cdot \>_{\cH} )$, a feature map $\psi : \{+1 , -1 \}^d \to \cH$, an empirical loss function $L:\R^{2n} \to  \R \cup {\infty}$ and a regularization parameter $\lambda >0$. Given data points $(y_i , \bx_i )_{i \in [n]}$, the linear method construct a prediction model $\hf( \bx) := \< \hba , \psi (\bx ) \>_{\cH}$ where $\hba \in \cH$ is obtained by minimizing the regularized empirical risk functional 
\begin{equation}\label{eq:linear_method}
\hba = \argmin_{\ba \in \cH} \Big\{ L \big((y_i, \< \ba , \psi (\bx_i) \>)_{i\in [n]}\big) + \lambda \| \ba \|_{\cH}^2 \Big\}\, .
\end{equation}
We will further denote $q = \dim  (\cH)$. %
Popular examples include random feature models ($q$ is equal to the number of random features) and kernel methods ($q =\infty$ typically). While the optimization problem \eqref{eq:linear_method} is over a (potentially) infinite dimensional space $\cH$, it is an easy exercise to verify that $\hba \in \spn \{ \psi ( \bx_i ) : i \in [n] \}$ which has dimension bounded by $\min(n,q)$.

We consider learning a class of functions $\cF_M = \{ f_1 , \ldots , f_M \} \subseteq L^2 ( \cX)$ from $n$ evaluations at points $(\bx_i)_{i\in[n]}$. For any linear method which, for $j \in [M]$, outputs the model $\hf_j$ obtained by \eqref{eq:linear_method} on $(f_j(\bx_i) + \eps_{ij} , \bx_i)_{i \in [n]}$, we define the average prediction error on $\cF$ as $\oR_n ( \cF_M ) = \frac{1}{M} \sum_{j \in [M]} \E_{\cX} [ (f_j (\bx) - \hf_j ( \bx) )^2]$. Several lower bounds on the sample complexity have appeared in the literature such as \cite{hsu2021approximation,hsudimension,kamath_dim}, here we present an improvement on \cite{hsu2021approximation,hsudimension} that is tighter for target functions that are not almost orthogonal, and an improvement of \cite{kamath_dim} that is tighter for functions like vanilla staircases of growing degree (which have polynomial SQ-dimension).

\begin{proposition}\label{prop:dimens-lower-bound} 
Let $\Omega \subseteq L^2(\cX)$ a linear subspace. Let $\cF_M = \{f_1 , \ldots , f_M \} \subset L^2 (\cX)$ such that $\| \proj_{\Omega} f_j \|_{L^2}^2 = 1 - \kappa$ and $\| \proj_{\Omega}^\perp f_j \|_{L^2}^2 = \kappa$ for all $f_j \in \cF_M$. For any linear method, if $\oR_n ( \cF_M ) \leq 1 - \eta$, then we must have
\begin{equation}\label{eq:dim_r}
    \min ( n ,q)  \geq \frac{\eta - \kappa}{ \max_{i \in [M]} \frac{1}{M} \sum_{j \in [M]}  | \< f_i , \proj_{\Omega} f_j \>_{L^2} | }\, .
\end{equation}
\end{proposition}

Define $\Omega_k$ to be the subspace spanned by all degree-$k$ monomials $L^2 (\{+1 , -1\}^P)$. Consider a function $h_{{\text{Poly-}k}} \in L^2 (\{+1 , -1\}^P)$ such that $\| \proj_{\Omega_k} h_{\text{Poly-}k} \|_{L^2}^2 = \| \proj_{\Omega_k}^\perp h_{\text{Poly-}k} \|_{L^2}^2 = 1/2$ and $\proj_{\Omega_k} h_{\text{Poly-}k}$ is supported over $m$ monomials. Second, consider $h_{\text{Str-}P} (\bz) = P^{-1/2} ( z_1 + z_1 z_2 + \ldots + z_1 \cdots z_P )$ the degree-$P$ staircase. We consider two sets of functions obtained by $h_{{\text{Poly-}k}}$ and $h_{\text{Str-}P}$ with all the permutations of their input signal:
$
\cF_{\ell} = \Big\{ f = h_\ell \circ \tau  : \tau \in \Pi (d) \Big\} \, , \ell \in \{\text{Poly-}k ,\text{Str-}P\} \,$,
where $\Pi (d)$ corresponds to the group of all permutation on $[d]$. Applying Proposition \ref{prop:dimens-lower-bound}, we get the following sample-complexity lower bounds:
\begin{proposition}\label{prop:lower_bound_stairs}
For any linear method, if $\oR_n ( \cF_{{\text{Poly-}k}} ) \leq 1/2\cdot (1 - \eta)$ then we must have $\min ( n, q) \geq \frac{\eta }{m} {{d}\choose{k}}$.
Similarly, if $\oR_n ( \cF_{{\text{Str-}P}} ) \leq 1 - \eta$ then we must have $\min(n,q) \geq \frac{\eta}{2}{{d}\choose{\lfloor \frac{\eta P}{2} \rfloor}}$.
\end{proposition}

Note that kernel and random features methods achieve the lower bound for $\cF_{{\text{Poly-}k}}$ \cite{ghorbani2021linearized,mei2021generalization}. Comparing Proposition \ref{prop:dimens-lower-bound} with the result of Section \ref{sec:sufficient}, we get the following separation results between SGD-trained neural networks and linear methods:
\begin{itemize}
    \item[(1)] SGD on two-layer neural networks outperforms linear methods almost surely on non-linear MSP functions ($n = O(d)$ versus $n = \Omega_d (d^k)$ for degree-$k$ MSP).
    
    \item[(2)] We obtain from Proposition \ref{prop:lower_bound_stairs} that for any $P=\omega_d(1)$, linear methods must have $\min(n,q)=d^{\omega_d(1)}$ to learn the vanilla staircase of degree $P$, while Theorem \ref{thm:vanillastaircasesuffrestated} can still guarantee a sample complexity of $d^{O(1)}$ for $P$ growing slowly enough with $d$. 
\end{itemize}

\section{Conclusion and future directions}\label{sec:open-problems}

In this paper, we considered learning sparse functions in arbitrarily large ambient dimension, using two-layer neural networks trained by batch-SGD in the mean-field regime. We proved that the merged-staircase property is a necessary and nearly-sufficient condition for such functions to be learnable on such models in $O(d)$ sample-complexity. The near sufficiency part, which excludes a measure-zero subset, is unavoidable as there exist symmetric MSP functions with degenerate dynamics that are not strongly SGD-learnable in $O(d)$-scaling. This provides a regime where one can achieve a tight characterization of functions that are learnable by regular SGD on regular neural networks, while going beyond the linear regime.

One venue for future work is to characterize more precisely the set of degenerate MSP functions: current examples correspond to MSP functions with some group invariance (see Appendix \ref{app:numerics}) which arise naturally in applications. More importantly, the current MSP condition hinges on the particular setting considered in this paper: $P$ fixed (or sufficiently slowly growing), $O(d)$ sample-complexity and two-layer neural networks. In particular, this definition ignores composition-order, i.e., how many monomials are composed to create a new one, and how many fresh variables are involved in such a composition. We anticipate more complex categories to appear as we move away from this setting. For example, we conjecture that $l$-leap MSP (i.e., $|S_i \sm \cup_{j=1}^{i-1} S_j | \leq l$) are SGD-learnable in $O(\varphi_l(d))$-scaling with $\varphi_2(d) = d \log (d)$ and $\varphi_l(d) = \Tilde{O} (d^{l/2})$ for $l >2$ (this paper considers $l=1$ and showed $\varphi_1 (d) = d$ is tight in the mean-field regime).
Furthermore, the compositionality-order, i.e., the number $k$ of monomials that can be composed in order to produce new monomials, should also be factored in for a finer complexity analysis once $P$ (and $k$) are no longer constant. 
The depth of the architecture is also expected to play a role when $P$ is diverging: for instance, it is shown in \cite{staircase1} that vanilla staircases (i.e., nested chains with $k=l=1$) are learnable with $P$-layer neural networks (but unconventional gradient-based training) in $\poly (d,P,1/\eps)$ samples, while our proof techniques (Theorem \ref{thm:discrete-msp} and Theorem \ref{thm:vanillastaircasesuffrestated}, see Appendix \ref{sec:sample-complexity-explicit} for the statement of the explicit bounds) yield $O( e^{e^{P}} d/ \eps^C)$ sample-complexity\footnote{One can show $e^{O(P)}$-dependency is sufficient if all degree-$1$ monomials are included in the MSP set structure -- see Appendix~\ref{ssec:msp-one-included}. We conjecture that such an exponential scaling is needed, i.e., deeper neural networks are required to achieve a $\poly (P)$ dependency.}  with two-layer neural networks. More generally, the sparsity parameter $P$ will not be necessarily the right complexity measure for deeper networks: for example, some functions with small leap $l$ and large $P$ will be easier to learn than some functions with smaller $P$ but larger $l$.

Finally, it is natural to seek counterparts of the results in this work and counterparts of the staircase notions for other Hilbert spaces, such as $L^2$ functions with respect to the Gaussian measure.

\section*{Acknowledgements} 
We thank Guy Bresler, Dheeraj Nagaraj, and Nati Srebro for stimulating discussions. We thank the Simons Foundations and the NSF for supporting us through the Collaboration on the Theoretical Foundations of Deep Learning (deepfoundations.ai). This work was done (in part) while E.B. and T.M. were visiting the Simons Institute for the Theory of Computing and while E.B.\ was visiting the Bernoulli Center at EPFL.

\bibliography{bibliography.bib}
\bibliographystyle{alpha}

\clearpage

\tableofcontents

\clearpage 

\appendix

\section{Additional numerical simulations}\label{app:numerics}

In this Appendix, we provide further background and numerical illustrations on the strong $O(d)$-SGD learning setting, merged-staircase functions and the dimension-free dynamics.

\paragraph*{Global convergence of the dimension-free dynamics.} As stated in Theorem \ref{thm:equivalence}, a function $h_* : \{ +1 , -1\}^P \to \R$ is strongly $O(d)$-SGD-learnable if and only if a Wasserstein gradient flow on $R(\orho) = \E_{\bz} [ (h_* (\bz) - \hf_{\NN} ( \bz ; \orho ) )^2 ]$ (the \eqref{eq:DF-PDE} dynamics) can converge to the global optimizer when initialized with $\oa^0 \sim \mu_a$, $\obu^0 = \bzero$ and $\os^0 = m_2^w$ (in this paragraph, we consider the continuous regime, as this is the regime that has attracted the most attention). Showing global convergence results for such dynamics is generically challenging: $R(\orho)$ presents many bad stationary points (e.g., measures $\orho$ that are distributed on an insufficient number of atoms to represent $h_*$). While some progress has been made to show such results in the context of mean-field neural networks \cite{chizat2018global,nguyen2020rigorous,wojtowytsch2020convergence}, existing global convergence results assume typically that (1) the dynamics converges to a limiting distribution as $t\to \infty$; and (2) we are given a good ``spread-out'' initialization. Condition (2) usually holds for initialization with bounded density on an open set around $0$. Condition (1) is more challenging and presents counter-examples that are difficult to rule out. In Section \ref{sec:sufficient}, we avoid these difficulties by considering layer-wise training: global convergence reduces to showing that the final linear-training phase converges to $0$-risk, which is implied by a certain kernel matrix being full rank.

While global convergence proofs are challenging, the \eqref{eq:DF-PDE} dynamics is a low dimensional problem and can be efficiently solved numerically. In the rest of this section, we provide a few numerical simulations to illustrate phenomena alluded to in the main text. We will fix the activation to be a shifted sigmoid $\sigma ( x) = (1 + e^{-x+1} )^{-1}$, and choose learning schedules $\xi^a(t) = \xi^w (t) = 1$, zero regularization parameters $\lambda^a = \lambda^w = 0$, and initialization $\mu_a = \Unif([+1,-1])$ and $\os^0 = 1$. 
In Figure \ref{fig:MSP_app}, we consider four MSP functions and plot the evolution of their Fourier coefficients during the \eqref{eq:DF-PDE} dynamics. In particular, the two top row examples, $h_1 (\bz) = z_1 + z_1 z_2 +z_3 +z_1z_2z_3z_4$ and $h_2 (\bz) = z_1 + z_1 z_2 +z_2z_3 +z_3z_4+ z_1z_2z_3z_4$, converge to the global minimum and are therefore strongly $O(d)$-SGD-learnable. The bottom row examples, $h_3 (\bz) = z_1 + z_1 z_2 +z_3 +z_3z_4$ and $h_4 (\bz) = z_1 + z_2 +z_3 + z_1z_2z_3$, do not converge and have risks bounded away from $0$. Functions $h_3$ and $h_4$ are two examples of $G$-invariant MSP functions.

\begin{figure}[t]
\begin{center}
    \includegraphics[width=14cm]{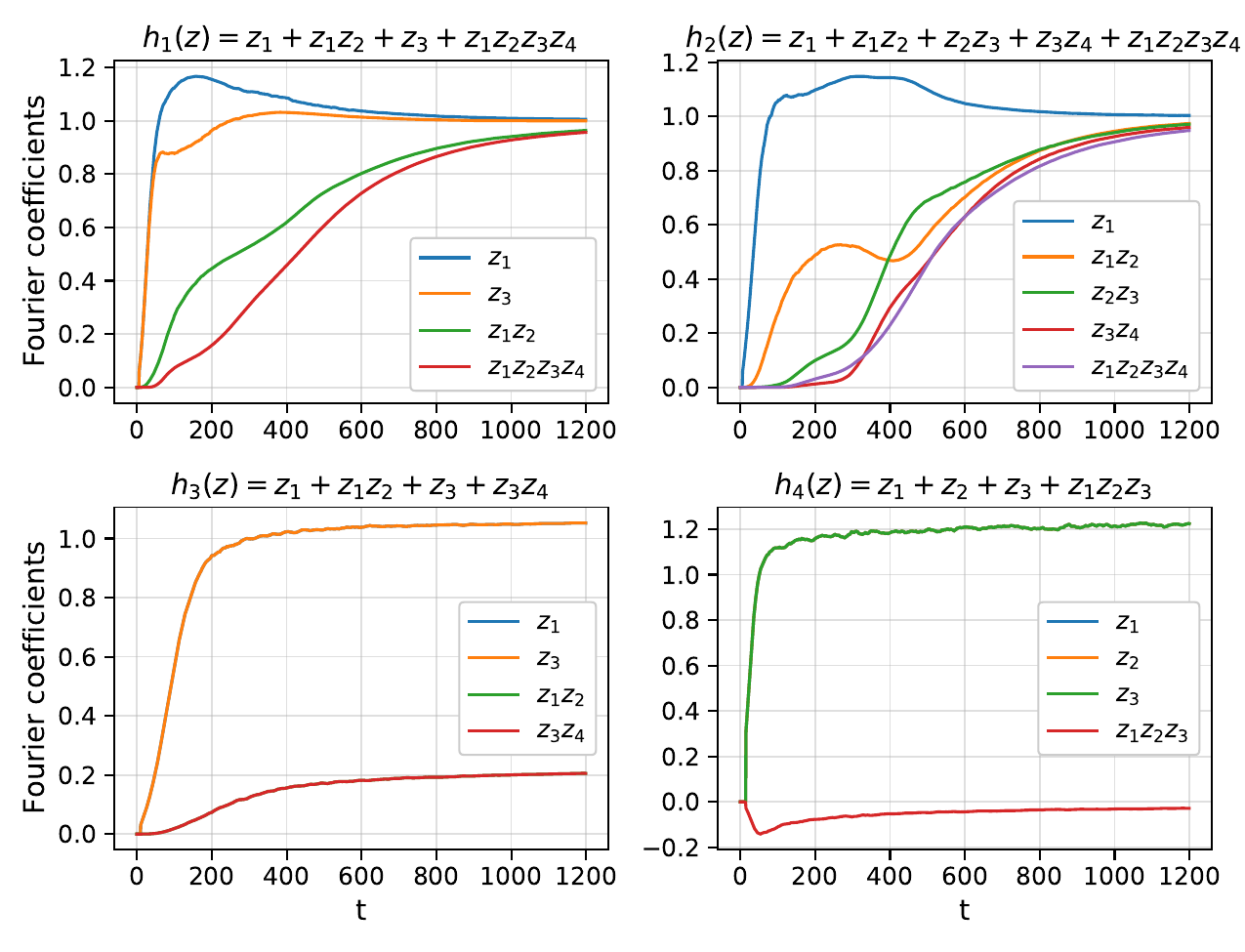}
\end{center}

\vspace{-25pt}
\caption{Evolution of the Fourier coefficients during the \eqref{eq:DF-PDE} dynamics for $4$ MSP functions. \label{fig:MSP_app}}
\end{figure}

\paragraph*{$G$-invariant MSP functions.} We call $h_*$ a $G$-invariant MSP function if $h_*$ is invariant under a group of transformations, i.e., there exists $\tau : \{+1,-1\}^P \to \{+1 , -1 \}^P$ (invertible and $\tau \neq \text{id}$) such that $h_* (\tau(\bz)) = h_* ( \bz)$. For these functions, the $\obu^t$ weight distribution remains invariant by this same group of transformations during the \eqref{eq:DF-PDE} dynamics, regardless of the choice of parameters. For example, $h_3$ is invariant by permutation $(1,2,3,4) \leftrightarrow (3,4,1,2)$ of its input, and $\ou_1^t  = \ou_3^t$, $\ou_2^t = \ou_4^t$ for all $t \geq 0$. $h_4$ is invariant by any permutation, and $\ou_1^t = \ou_2^t = \ou_3^t$ for all $t \geq 0$. For $G$-invariant MSP functions, the weights $\obu^t$ remain constrained in a linear subspace of dimension $<P$. We can then prove that a function is not strongly $O(d)$-SGD-learnable if no global minimizers lie on this subspace. For example, this is the case of $h_3$: as argued in Section \ref{sec:sufficient} of the main text, $\E_{\bz} [ \hf_{\NN} (\bz ; \orho_t) z_1 z_2] =  \E_{\bz} [ \hf_{\NN} (\bz ; \orho_t) z_1 z_3]$, and $\hf_{\NN} (\cdot ; \orho_t)$ can never coincide with $h_3$. For $h_4$, however, there exists $\orho$ with $\obu = u (1,1,1)$ that achieves $0$ risk (indeed, define $X = z_1 +z_2 +z_3$, then $h_* (X) = (X^2 -1)/2 + (-1)^{(X+3)/2}$ and can be fitted with a cubic polynomial). It is difficult to check whether for such functions, the \eqref{eq:DF-PDE} dynamics will converge to $0$-risk for some initialization and activation function. (Let us just mention that some $G$-invariant MSP functions are indeed strongly $O(d)$-SGD-learnable, such as $h_* (\bz) = z_1 + z_2$.)

Such $G$-invariant functions appear naturally in applications and we believe that understanding their dynamics is an important future direction. However, in this paper we consider instead to perturb the Fourier coefficients, which breaks the symmetries, and we show that any MSP function is strongly $O(d)$-SGD-learnable almost surely over this perturbation. In Figure \ref{fig:smoothMSP}, we plot the evolution of the Fourier coefficients of the original $h_4 (\bz) = z_1 + z_2 +z_3 + z_1z_2z_3$ and its perturbation $\Tilde{h}_4 (\bz) = z_1 + 0.99 z_2 + 1.01 z_3 + z_1z_2z_3$. We see that $\Tilde{h}_4$ is no longer $G$-invariant and $\Tilde{h}_4$ is strongly $O(d)$-SGD-learnable.

\begin{remark}
In this paper, we only prove that the set of MSP functions that are not strongly $O(d)$-SGD-learnable is of Lebesgue measure $0$. We do not characterize this set beyond this and do not prove that $G$-invariant MSP functions coincides with this set (in particular, we do not show that $G$-invariant MSP functions are the only functions that might not be strongly $O(d)$-SGD-learnable).
\end{remark}

\begin{figure}
\begin{center}
    \includegraphics[width=15cm]{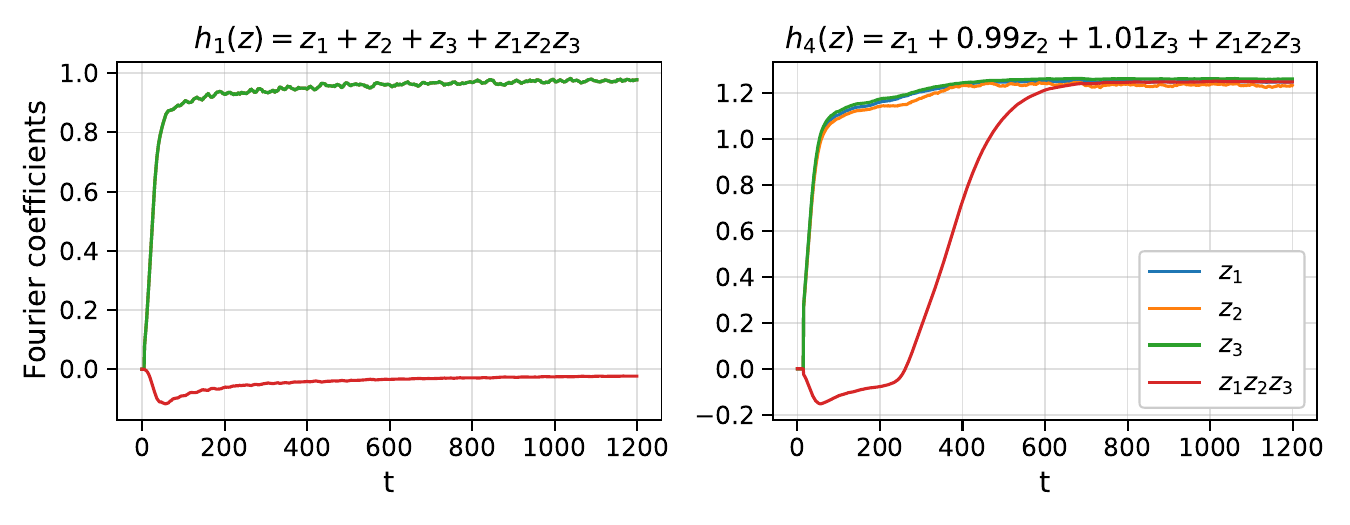}
\end{center}

\vspace{-20pt}
\caption{Evolution of Fourier coefficients during the \eqref{eq:DF-PDE} dynamics for the degenerate MSP $h_4 (\bz) = z_1 + z_2 + z_3 + z_1z_2z_3$ (left) and perturbed MSP function $\Tilde{h}_4 (\bz) = z_1 + 0.99 z_2 + 1.01 z_3 + z_1z_2z_3$ (right). \label{fig:smoothMSP}}
\end{figure}

We conclude this section with a final comment about the necessity condition of MSP, which holds only when considering \textit{arbitrarily large $d$}.

\paragraph*{Escaping the saddle-space.} The proof that non-MSP functions are not strongly $O(d)$-SGD-learnable relies on the fact that, when $d$ goes to infinity, the initialization $u_i^0 \to 0$ for $i \in [P]$. However, for $d$ fixed, $u_i^0 \approx 1/ \sqrt{d}$ and waiting sufficiently long, one-pass \eqref{eq:bSGD} escapes the neighborhood of the subspace $u_i^0 = 0$. In this case, the time to escape the subspace has to grow with $d$, and we are not in the $O(d)$-scaling anymore (indeed $n = Tb/\eta \approx Td$ for one pass \eqref{eq:bSGD}). In Figure \ref{fig:NonMSP_SGDvsDF}, we consider the same experimental setting as Figure \ref{fig:Staircase_SGDvsDF} but with $h_*$ missing one (left) or two (right) stairs. We see that \eqref{eq:DF-PDE} remains trapped in the saddle-space, while one-pass \eqref{eq:bSGD} escapes around $n \approx d^{2}$ and $n\approx d^{2.7}$ respectively. This agrees with the intuition that staircases with larger leaps are harder to learn with SGD.

\begin{figure}
\begin{center}
    \includegraphics[width=15cm]{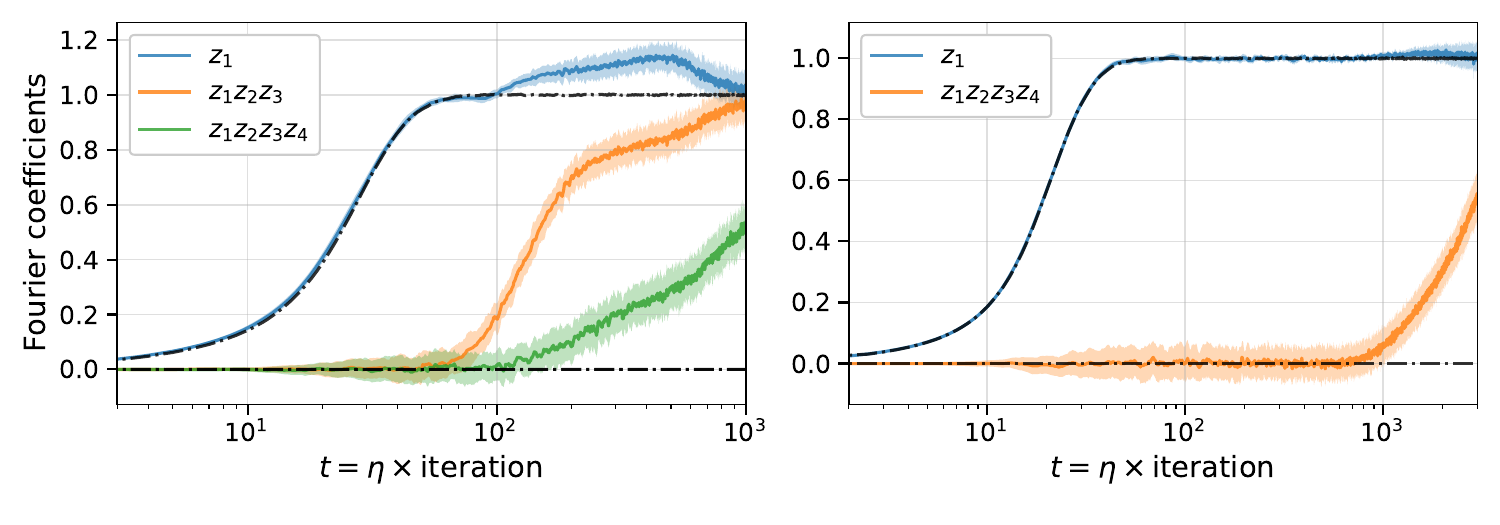}
\end{center}

\vspace{-20pt}
\caption{Fourier coefficients of one-pass \eqref{eq:bSGD} and \eqref{eq:DF-PDE} solutions throughout the dynamics for $h_* (\bz) = z_1 + z_1 z_2 z_3 + z_1 z_2 z_3 z_4$ (left) and $h_* (\bz) = z_1 + z_1 z_2 z_3 z_4$ (right). \label{fig:NonMSP_SGDvsDF}}
\end{figure}

\clearpage

\section{Proofs for continuous mean-field and dimension-free dynamics}
\label{app:MF-DF}

In this appendix, we provide proofs and discussions for the results presented in Section \ref{sec:DF-PDE-Necessary}, which corresponds to the `continuous-time regime' of strong $O(d)$-SGD-learnability. A discrete version of these results and proofs are presented in Appendix \ref{app:strong-discrete} and require little modifications.

Throughout this section, we will denote by $K$ a constant that depends only on the constants in Assumptions $\rA0$-$\rA2$,$\rA3'$ (in particular, $K$ is independent of $d,P,T$). The value of this constant is allowed to change from line to line.

\subsection{Justification for the dimension-free dynamics equations}
\label{app:intuition-DF-PDE}

Here, we provide more details and intuition on how to derive the equations of the dimension-free dynamics \eqref{eq:DF-PDE}. We report to Section \ref{app:proof-bSGD-to-Df} a rigorous proof of Theorem \ref{thm:bSGD-to-DF}, which shows a non-asymptotic bound between \eqref{eq:bSGD} and \eqref{eq:DF-PDE} dynamics.

First, by Assumption $\rA2$, the coordinates of $\bw^0$ are iid and symmetric and therefore 
\[
\begin{aligned}
\hf_{\NN} (\bx ; \rho_0) =&~ \int a^0 \sigma ( \< \bu^0 , \bz \> + \< \bv^0 , \br \>) \rho_0(\de \btheta^0) \\
=&~ \int a^0 \E_\br \big[ \sigma ( \< \bu^0 , \bz \> + \< \bv^0 ,  \br \>  )\big]  \rho_t (  \de \btheta^0) =: \hf_\NN (\bz ; \rho_0 )\, .
\end{aligned}
\]
By symmetry of \eqref{eq:MF-PDE}, the following lemma shows that the neural network stays independent of the uninformative part $\br$ of the input during the whole trajectory.

\begin{lemma}\label{lem:PDE_solution_symmetry}
The solution $(\rho_t)_{t \geq 0}$ of \eqref{eq:MF-PDE} with initialization $\rho_0$ satisfying $\rA 2$, obeys:
\begin{equation}\label{eq:NN_MF_sym_app}
\forall t \geq 0\, , \qquad \hf_\NN (\bx ; \rho_t ) = \int a^t \E_\br \big[ \sigma ( \< \bu^t , \bz \> + \< \bv^t ,  \br \>  )\big]  \rho_t (  \de \btheta^t) =: \hf_\NN (\bz ; \rho_t )\, .
\end{equation}
\end{lemma}

\begin{proof}[Proof of Lemma \ref{lem:PDE_solution_symmetry}] It is sufficient to show that for any $\br \in \{ -1 , +1 \}^{d-P}$, the weights $(a^t , \bu^t , \bv^t  \odot \br)$ have the same distribution as $(a^t,\bu^t , \bv^t ) \sim \rho_t$ where $ \bv^t \odot \br = (v_1^t r_1 , \ldots , v_{d-P}^t r_{d-P})$. Consider $\rho^{\# \br}_t = \rho_t \circ \varphi_{\br}$ where $\varphi_{\br} (\btheta) = (a, \bu , \bv \odot \br)$. First, notice that by assumption on $\rho_0$, we have $\rho^{\# \br}_0 = \rho_0$. Second, it is easy to check that for any bounded function $h : \R^{d+1} \to \R $, we have
\[
\begin{aligned}
\frac{\de }{\de t} \int h ( \btheta ) \rho_t^{\# \br} (\de \btheta ) =&~ \frac{\de }{\de t} \int h ( \varphi_{\br} (\btheta) ) \rho_t (\de \btheta ) \\
 = &~ - \int \< \nabla_\btheta h ( \varphi_{\br} (\btheta) ) , \bH(t) \nabla_{\btheta} \psi ( \btheta ; \rho_t ) \> \rho_t (\de \btheta) \\
 = &~ - \int \< (\ones , \br )\odot \nabla_\btheta h ( \btheta ) , \bH(t)  (\ones , \br )\odot \nabla_{\btheta} \psi ( \btheta ; \rho_t^{\# \br} ) \> \rho_t^{\# \br} (\de \btheta) \\
 = &~ - \int \<  \nabla_\btheta h ( \btheta ) , \bH(t)   \nabla_{\btheta} \psi ( \btheta ; \rho_t^{\# \br} ) \> \rho_t^{\# \br} (\de \btheta) \, ,
\end{aligned}
\]
where we used in the third line that $\psi ( \varphi_{\br}(\btheta) ; \rho_t ) =  \psi ( \btheta ; \rho_t^{\# \br} ) $. Hence $\rho_t^{\# \br}$ is the solution of the \eqref{eq:MF-PDE} dynamics with initialization $\rho^{\# \br}_0 = \rho_0$. Hence by uniqueness of the solution, we deduce that $\rho^{\# \br}_t = \rho_t$ for any $t\geq 0$.
\end{proof}

As mentioned in the main text, one can show that conditional on $\bv^t$, the noise part of the signal $\< \bv^t , \br \>$ for $\br \sim \Unif ( \{ -1 , +1 \}^{d-P})$ is well approximated by $\| \bv^t \|_2 G$ with $G \sim \normal (0,1)$, as long as $\max_i | v_i^t|/ \| \bv^t \|_2$ remains small. This is the case at $t = 0$ by Assumption $\rA2$ that the $v_i^0$ are iid and sub-Gaussian, and we show in Section \ref{app:proof-MF-DF-comp} that it remains true up to times $T = O_d(1)$. 
This motivates the introduction of effective parameters $\obtheta^t = (\oa^t, \obu^t, \os^t)$ with effective distribution $\orho_t \in \cP ( \R^{P+2})$. The new parameter $\os^t$ plays the role of $\| \bv^t \|_2 $ and we replace the neural network \eqref{eq:NN_MF_sym_app} by an effective neural network 
\begin{equation}\label{eq:effNN_app}
\hf_\NN (\bz ; \orho_t ) = \int \oa^t \E_G \big[ \sigma ( \< \obu^t , \bz \> + \os^t G  )\big]  \orho_t (  \de \obtheta^t)\, .
\end{equation}

The evolution equations of $(a^t, \bu^t , \| \bv^t \|_2)$ associated to the \eqref{eq:MF-PDE} dynamics are given by 
\begin{equation}\label{eq:evolution_MF}
\begin{aligned}
\frac{\de}{\de t} a^t =&~ \xi^a (t) \E_{\bz , \br } \Big[ \big\{ f_* (\bz) - \hf_{\NN} (\bz ; \rho_t) \big\} \sigma ( \< \bu^t , \bz \> + \< \bv^t , \br \> ) \Big] - \xi^a (t)   \lambda^a a^t  \, , \\
\frac{\de}{\de t } \bu^t =&~ \xi^w(t) a^t \E_{\bz , \br} \Big[ \big\{ f_* (\bz) - \hf_{\NN} (\bz ; \rho_t) \big\} \sigma' ( \< \bu^t , \bz \> + \< \bv^t , \br \> ) \bz \Big] - \xi^w (t) \lambda^w  \bu^t \, , \\
\frac{\de}{\de t } \|\bv^t \|_2 =&~ \xi^w (t)  a^t \E_{\bz , \br} \Big[ \big\{ f_* (\bz) - \hf_{\NN} (\bz ; \rho_t) \big\} \sigma' ( \< \bu^t , \bz \> + \< \bv^t , \br \> ) \< \bv^t/ \| \bv^t \|_2 , \br \> \Big]  \\
&~ -\xi^w (t)  \lambda^w \| \bv^t \|_2  \, ,
\end{aligned}
\end{equation}
where we used that $\frac{\de}{\de t } \|\bv^t \|_2  = \frac{1}{\| \bv ^t \|_2} \< \bv^t , \frac{\de}{\de t } \bv^t \> $ to write the last equation. 

For $P$ fixed and $d\to \infty$, we see that the distribution of $( a^0 , \bu^0 , \| \bv^0 \|_2 )$ converges in distribution to $(a^0 , \bzero , m_2^w ) $ which we denote $\orho_0$ and we recall that $m_2^w = \E_{\mu_w} [ W^2]^{1/2}$. As argued above, the mean-field neural network $\hf_{\NN} ( \bz; \rho_t)$ converges to $\hf_{\NN} (\bz ; \orho_t)$ for any $t \geq 0$. The evolution equations of $\orho_t$ can be obtained by taking $d \to \infty$ in Eq.~\eqref{eq:evolution_MF}, and replacing $\< \bv^t , \br \>$ by $\os^t G$:
\begin{equation}\label{eq:evolution_effective_MF}
\begin{aligned}
\frac{\de}{\de t} \oa^t =&~ \xi^a (t) \E_{\bz , G } \Big[ \big\{ f_* (\bz) - \hf_{\NN} (\bz ; \orho_t) \big\} \sigma ( \< \obu^t , \bz \> + \os^t G ) \Big] - \xi^a (t) \lambda^a \oa^t \, , \\
\frac{\de}{\de t } \obu^t =&~ \xi^w(t) \oa^t\E_{\bz , G} \Big[ \big\{ f_* (\bz) - \hf_{\NN} (\bz ; \orho_t) \big\} \sigma' ( \< \obu^t , \bz \> + \os^t G ) \bz \Big] - \xi^w(t) \lambda^w \obu^t \, , \\
\frac{\de}{\de t } \os^t =&~ \xi^w(t) \oa^t\E_{\bz , \br} \Big[ \big\{ f_* (\bz) - \hf_{\NN} (\bz ; \orho_t) \big\} \sigma' ( \< \obu^t , \bz \> + \os^t G ) G \Big] - \xi^w(t) \lambda^w \os^t \, .
\end{aligned}
\end{equation}

Denoting $\obH (t ) = \diag (\xi^a (t) , \xi^w (t) \id_{P+1})$ and regularization parameters $\obLambda = \diag (\lambda^a , \lambda^w \id_{P+1} ) $, the equations \eqref{eq:evolution_effective_MF} are the parameter evolution equations associated to the following PDE in the space of probability distributions on $\R^{P+2}$:
\begin{equation}\label{eq:DF-PDE_app}
\begin{aligned}
        \partial_t \orho_t = &~ \nabla_{\obtheta} \cdot \big(\orho_t  \obH (t) \nabla_{\obtheta} \psi ( \obtheta ; \orho_t ) \big) \, ,\\
    \psi ( \obtheta ; \orho_t ) = &~  \frac{1}{2}\E_{\bz,G} \Big[ \big\{\hat f_{\NN} ( \bz ; \orho_t) -  f_* (\bz) \big\} \oa \sigma ( \< \obu , \bz \> +\os G) \Big] + \frac{1}{2} \obtheta^\sT \obLambda \obtheta \, ,
    \end{aligned}
\end{equation}

Equivalently, this PDE corresponds to the gradient flow in the Wasserstein space (with $\obH(t)$ rescaling) over the regularized risk functional:
\[
E (\orho) = \frac{1}{2}\E_{\bz} \big[ \big\{ h_* (\bz)  - \hf_{\NN }  ( \bz ; \orho ) \big\}^2 \big] + \frac{1}{2}\int \obtheta^\sT \obLambda \obtheta \orho (\de \obtheta) \, .
\]

\subsection{Proof of the results in Section \ref{sec:DF-PDE-Necessary}} 
\label{app:proof-results-DF-MSP}

In this section, we gather the proofs for the results on the dimension free dynamics and the necessity condition. The longer and more technical arguments are deferred to Sections \ref{app:proof-MF-approx} and \ref{app:proof-MF-DF-comp}.

\subsubsection{Proof of Theorem \ref{thm:bSGD-to-DF}}
\label{app:proof-bSGD-to-Df}

We use the mean-field dynamics \eqref{eq:MF-PDE} as an intermediary dynamics for the bound. Theorem \ref{thm:bSGD-to-DF} is a direct consequence of the following two bounds:

\begin{proposition}\label{prop:general-PDE-SGD-bound}
Assume conditions $\rA0$-$\rA2$,$\rA3'$, and let $T \geq 1$. There exists constants $K_0$ and $K_1$ depending only on the constants in $\rA 0$-$\rA 2$,$\rA3'$ (in particular, independent of $d,P,T$), such that for any $\eta \leq e^{-K_0 T^3} \big[ \{b/(d+\log(N))\} \wedge 1 \big]$, we have 
\[
\sup_{k \in [T/\eta]\cap \naturals } \big\Vert \hat f_{\NN} (\cdot; \bTheta^k ) - \hat f_{\NN} (\cdot; \rho_{k\eta} ) \big\Vert_{L^2} \leq K_1 e^{K_1T^3} \left\{ \sqrt{\frac{\log N}{N}} + \left[ \sqrt{\frac{d+\log N}{b}} \vee 1 \right] \sqrt{ \eta} \right\} \, ,
\]
with probability at least $1 - 1/N$.
\end{proposition}

This proposition follows from a straightforward extension of \cite{mei2019mean} to batch-SGD and anisotropic step sizes, and can be found in Section \ref{app:proof-MF-approx}. In particular, Proposition \ref{prop:general-PDE-SGD-bound} implies that, if we consider $T,K = O_d(1)$, then $N = \Omega_{d} (1) $ and $1 / \eta = \Omega_{d} ( d/b)$ are sufficient for the mean-field PDE to be an accurate approximation of batch-SGD up to time $T$ (recall that $T = \eta n/b$ by one-pass assumption and therefore $n = O_d (d)$).

\begin{theorem}\label{thm:MF-dimension-free}
Assume conditions $\rA0$-$\rA2$,$\rA3'$, and let $T \geq 1$. There exists a constant $K_1$ depending only on the constants in $\rA0$-$\rA2$,$\rA3'$ (in particular, independent of $d,P,T$), such that 
\[
\sup_{t \in [0,T]} \big\| \hf_{\NN} ( \cdot ; \rho_t ) - \hf_{\NN} ( \cdot ; \orho_t ) \big\|_{L^2} \leq K_1e^{K_1 T^7}  \sqrt{\frac{P + \log (d)}{d}}\, .
\]
\end{theorem}

The proof of Theorem \ref{thm:MF-dimension-free} can be found in Section \ref{app:proof-MF-DF-comp}.

\subsubsection{Proof of Theorem \ref{thm:equivalence}}
\label{app:proof-equivalence}

Fix $\eps >0$. Consider Lipschitz $\xi^a, \xi^w : \R_{>0} \to \R_{>0}$ and $\lambda^a , \lambda^w \geq 0$ such that $\lim_{t \to \infty} R (\orho_t )< \eps/6$. Take $T$ such that $R( \orho_T ) = \eps/6$. Let $K$ be sufficiently large such that $\rA_0-\rA_2$ are satisfied, and $\rA_3'$ is satisfied on $[0,T]$. By Theorem \ref{thm:bSGD-to-DF}, there exists constants $K_0,K_1$ that only depend on $K$ such that the bound holds with probability at least $1 - 1/N$ for $\eta \leq e^{-K_0T^3} [\{ b / (d + \log(N)) \} \wedge 1 ]$. Consider $c_1 $ such that $c_1 /\log(c_1) = 81 K_1^2 e^{2K_1T^3} / \eps$ and take $C ( h_* , \eps ) = 2(T\vee 1) ( c_1 \vee e^{K_0 T^3})$ and $T (\eps, h_*) = T $. Then for any $d\geq C(\eps,h_*)$, $n \geq C(\eps , h_*)d$ and $e^d \geq N \geq C(\eps, h_*)$, taking $b = d + \log (N)$, we have $\sqrt{\eta} \leq \eps e^{-K_1 T^6}/(9K_1)$, and with probability $1 - 1/N$, taking $k_0 = \lfloor T/\eta \rfloor$,
\[
\| \hf_{\NN} (\cdot ; \bTheta^{k_0} ) - \hf_{\NN} (\cdot ; \orho_{\eta k_0} ) \|_{L^2} \leq  \frac{\sqrt{\eps}}{3}\, .
\]
Hence
\[
\begin{aligned}
R ( h_* ,  \hf_{\NN} (\cdot ; \bTheta^{k_0} )) \leq&~ 3R ( h_* ,  \hf_{\NN} (\cdot ; \orho_{T} )) + 3\| \hf_{\NN} (\cdot ; \bTheta^{k_0} ) - \hf_{\NN} (\cdot ; \orho_{\eta k_0} ) \|_{L^2}^2 \\
&~ +3\| \hf_{\NN} (\cdot ; \orho_T ) - \hf_{\NN} (\cdot ; \orho_{\eta k_0} ) \|_{L^2}^2\\
\leq&~ \frac{\eps}{2} + \frac{\eps}{3} + \frac{\eps}{9}\leq \eps \, .
\end{aligned}
\]

Conversely, assume that $h_*$ is strongly SGD-learnable in $O(d)$-scaling. Let $(b, \lambda^a, \lambda^w, \{ \eta_k^a , \eta_k^w   \}_{k \in [0,k_0]})$ be the hyperparameters that satisfy strong learnability for $\eps/4$: in particular, $k_0 = n/b$ and $R( \bTheta^{k_0} ) \leq \eps/4$ with probability at least $9/10$. Take $T= n\eta / b \leq T (h_* , \eps/ 4)$ and let $\eta^a,\eta^w$ be piecewise linear functions such that $\eta^a (\eta k) = \eta_k^a /\eta $ and $\eta^w (\eta k) = \eta_k^w /\eta $. Consider $(\orho_t)_{t \geq 0}$ the solution of \eqref{eq:DF-PDE} with $\eta^a,\eta^w,\lambda^a,\lambda^w$. From Theorem \ref{thm:bSGD-to-DF}, there exists constants $K_0$ and $K_1$ that only depend on $\eta^a,\eta^w,\lambda^a,\lambda^w$ through the constants in assumption $\rA3$, such that 
\[
\begin{aligned}
&~\| \hf_{\NN} (\cdot; \bTheta^{k_0} ) - \hf_{\NN} (\cdot ; \orho_{\eta k_0} ) \|_{L^2} \\
\leq&~ K_1 e^{K_1 c(h_*, \eps/4)^6}
\left\{ \sqrt{\frac{P + \log(d)}{d}} + \sqrt{\frac{\log N}{N}} + \left[ \sqrt{\frac{d+\log N}{b}} \vee 1 \right] \sqrt{ \eta} \right\} \, ,
\end{aligned}
\]
with probability at least $1 - 1/N$. We can therefore take $d,N,n$ sufficiently big such that the right-hand side is less than $\sqrt{\eps}/2$. On the intersection of this event and the event $R( \bTheta^{k_0} ) \leq \eps/4$ (which happens with positive probability), we have
\[
\begin{aligned}
R ( \orho_{\eta k_0}) \leq&~ 2R (\bTheta^{k_0}) + 2\| \hf_{\NN} (\cdot ; \bTheta^{k_0} ) - \hf_{\NN} (\cdot ; \orho_{\eta k_0} ) \|_{L^2}^2 
\leq \frac{\eps}{2} + \frac{\eps}{2} = \eps \, ,
\end{aligned}
\]
which finishes the proof.

\subsubsection{Proof of Theorem \ref{thm:MSP_necessary}}
\label{app:proof-MSP-necessary}

Consider $h_* : \{ +1 , -1 \}^P \to \R$ that is not MSP, and denote by $\{ \hh (S) \}_{S \in \cS^*}$ its non-zero Fourier coefficients $h_* ( \bz) = \sum_{S \in \cS_*} \hh (S)\chi_S (\bz)$. Denote by $\ocS_* \subset \cS_*$ the biggest subset of $\cS_*$ such that $\ocS_* = \{S_1,\ldots,S_r\}$ can be ordered with $|S_i \sm \cup_{i=1}^{r-1} S_i| \leq 1$ for any $i \in [r]$. By assumption $\cS_* \setminus \ocS_*$ is not empty, and for any $S\in \cS_* \setminus \ocS_*$, there exists at least two coordinates $i_1,i_2 \in S$ such that $i_1, i_2 \not\in \bigcup_{S \in \ocS_*}S$. Denote $\Omega = [ P] \setminus \big( \bigcup_{S \in \ocS_*}S \big)$. We show that $\ou^t_i = 0 $ during the whole dynamics for every $i \in \Omega$. In particular, this implies that for any $S\in \cS_* \setminus \ocS_*$, $\E_{\bz} [ \chi_S (\bz) \hf_{\NN} (\bz ; \orho_t ) ] = 0$, and
\[
R(h_* ; \hf ( \cdot ; \orho_t) ) \geq \sum_{S\in \cS_* \setminus \ocS_*}  \hh (S)^2 > 0 \, .
\]
This lower bound does not depend on the details of the dynamics (parameters $\xi^a,\xi^w,\lambda^a,\lambda^w$, activation and initialization $\mu_a$,$m_2^w$). Let $\bz_{i,+}$ and $\bz_{i,-}$ denote the vector $\bz \in \{+1,-1\}^P$, with $z_i = +1$ and $z_i = -1$ respectively, and note that by Lemma \ref{lem:bound_evoluation_a}, $|\oa^t| \leq K (1 +t )$. Using Assumption $\rA0$, we have by integrating out $z_i$:
\[
\begin{aligned}
\big\vert \E_{z_i} \big[ \hf_{\NN} ( \bz ; \orho_t) \sigma ' ( \< \bz , \obu^t \>) z_i \big] \big\vert \leq&~ \frac{1}{2}\big\vert \hf_{\NN} ( \bz_{i,+} ; \orho_t) \big\{ \sigma ' ( \< \bz_{i,+} , \obu^t \>) - \sigma ' ( \< \bz_{i,-} , \obu^t \>)  \big\} \big\vert  \\
&~ + ~ \frac{1}{2}\big\vert \big\{ \hf_{\NN} ( \bz_{i,+} ; \orho_t)  - \hf_{\NN} ( \bz_{i,-} ; \orho_t) \big\}  \sigma ' ( \< \bz_{i,-} , \obu^t \>)  \big\vert  \\
\leq &~ K (1 +t) \| \sigma \|_{\infty} \| \sigma ''\|_{\infty} |  \ou_i^t | +K (1 +t) \| \sigma' \|_{\infty} \| \sigma '\|_{\infty} |  \ou_i^t |\\
\leq&~ K (1 +t) |  \ou_i^t | \, .
\end{aligned}
\]
Similarly, for $i \in \Omega$ and $S \in \ocS_*$ (in particular, since $i \not\in S$)
\[
\big\vert \E_{\bz} [ \chi_S ( \bz) \sigma ' (\< \bz , \obu^t \> ) z_i ] \big\vert  \leq \| \sigma '' \|_\infty | \ou_i^t | \, ,
\]
while if $S \in \cS_* \setminus \ocS_*$, then there exists $j \in \Omega \cap S$ with $j\neq i$, hence
\[
\big\vert \E_{\bz} [ \chi_S ( \bz) \sigma ' (\< \bz , \obu^t \> ) z_i ] \big\vert  \leq \| \sigma '' \|_\infty | \ou_j^t | \, .
\]
Denoting $m_{\Omega}^t = \max_{i \in \Omega} | \ou_i^t |$ (recall $m_{\Omega}^0 = 0$), we conclude that for any $i \in \Omega $:
\[
\Big\vert \frac{\de }{\de t } \ou_i^t \Big\vert = | \oa^t | \big\vert \E_{\bz} \big[ (h_* (\bz) - \hf_{\NN} (\bz ; \orho_t) ) \sigma ' ( \< \obu^t , \bz \> ) z_i \big]\big\vert \leq K(1+t)^2 m_{\Omega}^t \, ,
\]
and therefore $m_{\Omega}^t = 0$ during the whole dynamics.

\subsection{Proof of Proposition \ref{prop:general-PDE-SGD-bound}}
\label{app:proof-MF-approx}

The proof is an application of an extension of Theorem 1.(B) in \cite{mei2019mean} to batch-SGD and anisotropic step sizes. This extension is straightforward and we simply list below the two main differences with the proof in Appendix C of \cite{mei2019mean}:

\begin{itemize}
    \item Recall that we defined the regularized risk $E(\rho) := \frac{1}{2} R(\rho) + \frac{1}{2} \int \btheta^\sT \bLambda \btheta \rho (\de \btheta)$. We have
 \[
 \frac{\de}{\de t } E(\rho_t) = - \int \|\nabla_{\btheta} \psi (\btheta^t ; \rho_t ) \|^2_{\bH(t)} \rho_t (\de \btheta^t ) \leq 0\, ,
 \]
where $\psi (\btheta^t ; \rho_t )$ is defined in Eq.~\eqref{eq:MF-PDE} and we denoted $\|\bv \|_{\bA} = \| \bA^{1/2} \bv \|_2 $. We conclude that $E(\rho_t)$ is nonincreasing. The rest of the proof only uses that $\bH (t)$ verifies $\| \bH \|_{\infty}, \| \bH \|_{\text{Lip}} \leq K$.

\item The concentration between the batch-SGD and gradient descent  (Appendix C.5 in \cite{mei2019mean}) uses that there is an extra $1/b$ factor in the sub-Gaussian constant.

\end{itemize}

The proof of Proposition \ref{prop:general-PDE-SGD-bound} simply amounts to checking that our setting (with Assumptions $\rA0$-$\rA2,\rA3'$) falls under the general framework of Theorem 1.(B) in \cite{mei2019mean}.

\begin{proof}[Proof of Proposition \ref{prop:general-PDE-SGD-bound}] 
First, from conditions $\rA0$ and $\rA1$, we have $\| \sigma \|_{\infty} \leq K$ and $|y_k | \leq  K$. Furthermore, note that $\bx$ is a sub-Gaussian vector and $\sigma ' $ is bounded ($\| \sigma ' \|_{\infty} \leq K$ by condition $\rA0$). Then, for any $\bw \in \R^d$, the gradient $\nabla_\bw \sigma ( \< \bx , \bw \> ) = \bx \sigma ' ( \< \bx , \bw \>)$ is $K$-sub-Gaussian. Hence, assumption ${\rm A_2}$ of \cite{mei2019mean} is verified.

Denote $v(\bw) = \E_{\bx} [ f_* ( \bx) \sigma ( \< \bx , \bw \>)]$ and $u(\bw_1 , \bw_2 ) = \E_{\bx} [ \sigma ( \< \bx , \bw_1\>) \sigma ( \< \bx , \bw_2 \>)]$. Consider $\bn \in \R^d$ with $\| \bn \|_2 = 1$. Then, we have
\[
\begin{aligned}
    \< \nabla v ( \bw) , \bn \> = &~ \E_{\bx} [ f_* ( \bx) \sigma ' ( \< \bx , \bw \>) \< \bn , \bx \> ] \leq K \E [ \< \bn , \bx \>^2 ]^{1/2} = K \, , \\
    \< \nabla_{\bw_1} u (\bw_1 , \bw_2 ) , \bn \> = &~  \E_{\bx} [  \sigma ' ( \< \bx , \bw_1 \>) \< \bn , \bx \> \sigma  ( \< \bx , \bw_2 \>) ] \leq K^2 \E [ \< \bn , \bx \>^2 ]^{1/2} = K^2 \, , \\
    |\< \nabla^2 v ( \bw) , \bn^{\otimes 2} \>| = &~ \E_{\bx} [ |f_* ( \bx) \sigma '' ( \< \bx , \bw \>)| \< \bn , \bx \>^2 ] \leq K \E [ \< \bn , \bx \>^2 ] =PK \, .
\end{aligned}
\]
Finally, consider $\bn_1 , \bn_2 \in \R^{d}$ with $\| \bn_1 \|_2^2 + \| \bn_2 \|_2^2$. Then,
\[
\begin{aligned}
|\< \nabla_{(\bw_1,\bw_2)}^2 u (\bw_1 , \bw_2 ) , (\bn_1 , \bn_2)^{\otimes 2} \>| \leq &~  \E_{\bx} [ | \sigma '' ( \< \bx , \bw_1 \>) \sigma ( \< \bx , \bw_2 \>) | \< \bn_1 , \bx \>^2  ] \\
&~ + 2 \E_{\bx} [  | \sigma ' ( \< \bx , \bw_1 \>) \sigma '  ( \< \bx , \bw_2 \>) \< \bn_1 , \bx \> \< \bn_2 , \bx \> | ] \\
&~ + \E_{\bx} [  | \sigma  ( \< \bx , \bw_1 \>) \sigma '' ( \< \bx , \bw_2 \>) | \< \bn_2 , \bx \>^2  ] \\
\leq &~ 4 K^2 \, .
\end{aligned}
\]
We conclude that $\| \nabla v (\bw) \|_2 ,\| \nabla u (\bw_1 , \bw_2) \|_2,\|\nabla^2 v (\bw) \|_\op ,\| \nabla^2 u (\bw_1 , \bw_2) \|_{\op} \leq K$, and assumption ${\rm A3}$ in \cite{mei2019mean} is verified.
\end{proof}

\subsection{Proof of Theorem \ref{thm:MF-dimension-free}: bound between \eqref{eq:MF-PDE} and \eqref{eq:DF-PDE} dynamics}
\label{app:proof-MF-DF-comp}

We will assume throughout this section that the assumptions and the setting of Theorem \ref{thm:MF-dimension-free} hold. In particular, we will use Assumptions $\rA0$-$\rA2,\rA3'$ without mention when clear from context. For clarity, we will write the proof in the case $\xi^a (t) = \xi^w (t) = 1$ and $\lambda^a = \lambda^w = 0$. The general case follows easily, using $\| \xi^a  \|_{\infty} , \|  \xi^w \|_{\infty} , \lambda^a,\lambda^w \leq K$ by Assumption $\rA3'$.

We bound the distance between the mean-field and the dimension-free solutions by coupling the two dynamics through their initialization. Denote $\btheta^t = ( a^t , \bu^t , \bv^t)$ and $\obtheta^t = (\oa^t , \obu^t , \os^t)$ the parameters obtained by the evolution equations \eqref{eq:evolution_MF} and \eqref{eq:evolution_effective_MF} from initial parameters $\btheta^0 = ( a^0 , \bu^0 , \bv^0)$ and $\obtheta^0 = (\oa^0, \obu^0 , \os^0)$ respectively. Recall that we initialize independently $a^0 \sim \mu_a$ and $\sqrt{d} \cdot (\bu^0, \bv^0 ) \sim \mu_w^{\otimes d}$, and $\oa^0 \sim \mu_a$, $\obu^0 = \bzero$ and $\os^0 = m^w_2 := \E_{W \sim \mu_w} [ W^2 ]^{1/2}$. We couple the two dynamics by taking $a^0 = \oa^0$ (because of $(\bu^t, \bv^t)$ this coupling is not deterministic), and denote $\gamma_t$ the obtained joint distribution on $(\btheta^t , \obtheta^t)$.

The goal is to bound
\begin{equation}\label{eq:distanceNNs_integral}
\begin{aligned}
&~\Big\| \hf_{\NN} ( \cdot ; \rho_t ) - \hf_{\NN} (\cdot ; \orho_t ) \Big\|_{L^2}^2 \\
=&~ \E_{\bz} \Big[ \Big( \int \Big\{ a^t \E_{\br} [\sigma ( \< \bu^t , \bz\> + \< \bv^t , \br \>) ] - \oa^t \E_{G} [ \sigma ( \< \obu^t , \bz \> + \os^t G) ] \Big\} \gamma_t (\de \btheta^t \de \obtheta^t ) \Big)^2 \Big]  \, ,
\end{aligned}
\end{equation}
where we used Lemma \ref{lem:PDE_solution_symmetry} to remove the dependency in $\br$. It will be useful to introduce the residuals of the dynamics: $\hg ( \bz ; \rho_t) = h_* (\bz) - \hf_{\NN} (\bz ; \rho_t)$  and $\hg(\bz ; \orho_t ) = h_* (\bz) - \hf_{\NN} (\bz ; \orho_t)$. Recall that we denote by $R (\rho_t) = \E [ (f_* (\bz) - \hf_{\NN} (\bz ; \rho_t))^2] = \E_{\bz} [\hg ( \bz ; \rho_t)^2]$ and $R (\orho_t)= \E_{\bz} [\hg ( \bz ; \orho_t)^2]$ the prediction risks at time $t$.

The value of the integrand in Eq.~\eqref{eq:distanceNNs_integral} only depends on $\{(a^t, \bu^t, \< \bv^t , \br \> )\}_{t \geq 0}$ and $\{(\oa^t, \obu^t, \os^t G) \}_{t \geq 0}$ with $\br \sim \Unif ( \{ -1 , +1 \}^{d-P})$ and $G \sim \normal(0,1)$ independent of $\{(\btheta^t, \obtheta^t)\}_{t \geq 0}$.  Conditioning on $(\btheta^t, \obtheta^t)$, we consider the $1$-Wasserstein distance
\begin{equation}
W_1 \big( \< \bv^t , \br \> , \os^t G \big) \leq  \sqrt{2 \pi} \big\vert \| \bv^t \|_2 - \os^t \big\vert + W_1 \big( \< \bv^t , \br \> , \| \bv^t \|_2 G \big) \, ,
\end{equation}
where we recall that $W_1$ is defined by
\[
\begin{aligned}
W_1 (X,Y) =&~ \inf_{\gamma \in \Gamma (X,Y) } \E_{(X,Y) \sim \gamma} \big[ | X - Y | \big] \\
=&~  \sup_{f: \R \to \R, \| f \|_{\text{Lip}} \leq 1} \big\vert \E [ f(X)] - \E [ f(Y) ] \big\vert \, .
\end{aligned}
\]
Lemma \ref{lem:Berry-Esseen} in Section \ref{app:auxiliary_DF_PDE} shows that
\begin{equation}\label{eq:W1_bound_max_v_i}
\begin{aligned}
W_1 \big( \< \bv^t , \br \> , \| \bv^t \|_2 G \big) \leq &~ 3 \frac{\| \bv^t \|_3^3}{\| \bv^t \|_2^2 } \leq 3 \max_{i \in [d-P]} | v^t_i | \, .
\end{aligned}
\end{equation}
The following lemma bounds the right hand-side through the value of $\bv^t$ at initialization:
\begin{lemma}\label{lem:bound_thirdMoment_evolution}
Consider the same setting and assumptions as Theorem \ref{thm:MF-dimension-free}. There exists a constant $K$ independent of $d,P$ and depending only on the Assumptions $\rA0$-$\rA2,\rA3'$  such that for any $T \geq 0$,
\[
\sup_{t \in [0,T]} \max_{i \in [d-P]} | v_i^t| \leq   K e^{KT^2} \max_{i \in [d-P]} | v_i^0| \, .
\]
\end{lemma}

\begin{proof}[Proof of Lemma \ref{lem:bound_thirdMoment_evolution}]
We have 
\[
\begin{aligned}
\Big\vert \frac{\de }{\de t } v_i^t \Big\vert = &~ \Big\vert a^t \E_{\bz,\br} \big[ \hg (\bz ; \rho_t) \sigma' ( \< \bu^t , \bz \> + \< \bv^t , \br \> ) r_i \big]  \Big\vert \\
=&~ \Big\vert a^t  v_i^t  \E_{\bz,\br_{-i}} \Big[ \hg (\bz ; \rho_t) \sigma'' \big( \< \bu^t , \bz \> + \< \bv^t_{-i} , \br_{-i} \> + \xi_{v_i^t} \big)  \Big] \Big\vert \\
\leq &~ |a^t| \cdot | v_i^t|  \cdot \E_{\bz} \big[ \hg (\bz ; \rho_t)^2 \big]^{1/2} \|  \sigma'' \|_{\infty} 
\leq K (1 +T) | v_i^t| \, ,
\end{aligned}
\]
where we expanded the expectation on $r_i$ in the second line and used the mean value theorem, and used Eq.~\eqref{eq:bound_evol_a} in Lemma \ref{lem:bound_evoluation_a} in the last line. We deduce that
\[
| v_i^t| \leq e^{K(1+T)t} |v_i^0| \, ,
\]
which concludes the proof.
\end{proof}

Using Lemma \ref{lem:bound_thirdMoment_evolution} in the bound \eqref{eq:W1_bound_max_v_i} yields (conditional on $\btheta^0$):
\begin{equation}
\sup_{t \in [0,T]} W_1 \big( \< \bv^t , \br \> , \| \bv^t \|_2 G \big) \leq K e^{KT^2}  \max_{i \in [d- P]} | v_i^0 |\, .
\end{equation}
By Lemma \ref{lem:exp_max_subG} in Section \ref{app:auxiliary_DF_PDE}, the following holds for any fixed $q \in \naturals$, $q \leq K$,
\begin{equation} \label{eq:bound_W1}
\begin{aligned}
\int \Big\{ \sup_{t \in [0,T]}  W_1 \big( \< \bv^t , \br \> , \| \bv^t \|_2 G \big)  \Big\}^{q}  \rho_t ( \de \btheta ) \leq&~ Ke^{KT^2} \E_{\sqrt{d} \cdot \bv^0 \sim \mu_w^{\otimes (d-P)}} \Big[ \max_{i \in [d- P]} | v_i^0 |^{2q} \Big] \\
\leq&~  K e^{KT^2}  \left(\frac{\log d}{d}\right)^{q/2}\, .
\end{aligned}
\end{equation}

Using Eq.~\eqref{eq:bound_W1} and the coupling described above, we will bound \eqref{eq:distanceNNs_integral}. Introduce the random quantity
\begin{equation}
\delta (t) =  \big\vert a^t - \oa^t  \big\vert \vee \big\| \bu^t - \obu^t \big\|_2 \vee  \big\vert \os^t -\| \bv^t \|_2  \big\vert \,, 
\end{equation}
and the square root of its second moment
\begin{equation}
\Delta (t) = \Big( \int \delta (t)^2 \gamma_t ( \de \btheta^t \,\de \obtheta^t) \Big)^{1/2}\, .
\end{equation}

We will show the following technical bounds:
\begin{lemma}\label{lem:prop_chaos_bounds}
Consider the same setting and assumptions as Theorem \ref{thm:MF-dimension-free}. There exists a constant $K$ independent of $d,P$ and depending only on the Assumptions $\rA0$-$\rA2,\rA3'$ such that for any $T \geq 0$,
\begin{equation}\label{eq:bound_dist_NN_DF_MF}
\begin{aligned}
\big\| \hf_{\NN} ( \cdot ; \rho_t) - \hf_{\NN} (\cdot ; \orho_t) \big\|_{L^2} =&~ \E_{\bz} \Big[ \Big\{ \hg (\bz ; \rho_t ) - \hg ( \bz ; \orho_t ) \Big\}^2 \Big]^{1/2} \\
 \leq&~ K(1+T) \Delta (t) + K e^{KT^2} \sqrt{\frac{\log d}{d}} \, , \\
\end{aligned}
\end{equation}
where 
\begin{equation}\label{eq:bound_on_Delta}
\begin{aligned}
\frac{\de }{\de t } \Delta (t) \leq K ( 1+T)^6 \Delta (t) +  K e^{KT^2} \sqrt{\frac{\log d}{d}}\, .
\end{aligned}
\end{equation}
\end{lemma}

From this lemma, we can now complete the proof of Theorem \ref{thm:MF-dimension-free}:

\begin{proof}[Proof of Theorem \ref{thm:MF-dimension-free}]
From Gronwall's lemma applied to Eq.~\eqref{eq:bound_on_Delta} in Lemma \ref{lem:prop_chaos_bounds}, we have
\[
\Delta (t) \leq \Big[ \Delta (0) + K e^{KT^2} \sqrt{\frac{\log d}{d}} \Big] e^{K(1+T)^6 t} \, ,
\]
where 
\[
\begin{aligned}
\Delta (0)^2 =&~  \int \Big\{ \big\vert a^0 - \oa^0 \big\vert  \vee \big\Vert \bu^0 - \obu^0  \big\Vert_2 \vee \big\vert \os^t - \| \bv^0 \|_2 \big\vert  \Big\}^2 \gamma_0 (\de \btheta^0 \, \de \obtheta^0 )  \\
\leq&~ \E_{\sqrt{d} \cdot \bu^0 \sim \mu_w^{\otimes P}} \big[ \| \bu^0 \|_2^2 \big] +  \E_{\sqrt{d} \cdot \bv^0 \sim \mu_w^{\otimes P}} \big[ (\| \bv^0 \|_2 - m^w_2)^2 \big] \\
\leq&~ K\frac{P}{d} + \frac{K}{d} + \frac{KP}{d} \leq K \frac{P}{d}\, .
\end{aligned}
\]
Injecting this bound in Eq.~\eqref{eq:bound_dist_NN_DF_MF} concludes the proof. 
\end{proof}

\subsubsection{Proof of Lemma \ref{lem:prop_chaos_bounds}}
\label{app:DFproof_main_lemma}

Throughout the proof, we will use the following decomposition for any differentiable $\varphi : \R \to \R$: 
\begin{equation}\label{eq:decompo_interpolation}
\begin{aligned}
&~ \Big\vert \E_{\br,G} [ \varphi ( \< \bu^t , \bz \> + \< \bv^t , \br \>) - \varphi ( \< \obu^t , \bz \>  + \os^t G ) ]  \Big\vert\\
 \leq&~ \Big\vert \E_{\br} [ \varphi ( \< \bu^t , \bz \> + \< \bv^t , \br \>) - \varphi ( \< \obu^t , \bz \>  + \< \bv^t , \br \>) ]  \Big\vert \\
&~ + \Big\vert \E_{\br,G} [ \varphi ( \< \obu^t , \bz \>  + \< \bv^t , \br \>) - \varphi ( \< \obu^t , \bz \>  + \|\bv^t \|_2 G ) ]  \Big\vert\\
&~ + \Big\vert \E_{\br,G} [ \varphi ( \< \obu^t , \bz \>  + \|\bv^t \|_2 G ) - \varphi ( \< \obu^t , \bz \>  + \os^t G ) ]  \Big\vert \\
\leq &~ \| \varphi ' \|_{\infty} \big\vert \< \bu^t - \obu^t , \bz \> \big\vert + \| \varphi ' \|_{\infty} W_1 ( \< \bv^t , \br \> , \| \bv^t \|_2 G ) +  \| \varphi ' \|_{\infty} \big\vert  \| \bv^t \|_2 - \os^t \big\vert  \E [ |G|]\, .
\end{aligned}
\end{equation}
The proof consists in carefully bounding the evolution of the distance between the parameters in the two dynamics.

\noindent
{\bf Step 1. Bound on $\| \hf_{\NN} ( \cdot ; \rho_t ) - \hf_{\NN} (\cdot ; \orho_t ) \|_{L^2}$.}

We can bound the difference between the two functions with
\[
\begin{aligned}
&~ | \hf_{\NN} ( \bz ; \rho_t )  - \hf_{\NN} ( \bz ; \orho_t ) | \\
= &~ \Big\vert \int a^t \E_{\br} \big[ \sigma ( \< \bu^t , \bz \> + \< \bv^t , \br \>) \big] \rho_t ( \de \btheta^t ) - \int \oa^t \E_{G} \big[ \sigma ( \< \obu^t , \bz \> + \os^t G ) \big] \orho_t ( \de \obtheta^t ) \Big\vert \\
\leq &~ ({\rm I})+({\rm II})\, ,
\end{aligned}
\]
where
\[
\begin{aligned}
({\rm I}) = &~ \Big\vert \int (a^t - \oa^t) \E_{\br} \big[ \sigma ( \< \bu^t , \bz \> + \< \bv^t , \br \>) \big] \de \gamma_t \Big\vert \, , \\
({\rm II}) =  &~  \Big\vert \int \oa^t  \E_{\br,G} \big[ \big\{ \sigma ( \< \bu^t , \bz \> + \< \bv^t , \br \>) - \sigma ( \< \obu^t , \bz \> + \os^t G )\big\} \big] \de \gamma_t  \Big\vert \, .
\end{aligned}
\]
The first term can simply be bounded by
\begin{equation}\label{eq:stp1_b1}
\begin{aligned}
({\rm I})  \leq &~ \| \sigma \|_{\infty} \int | a^t - \oa^t | \de \gamma_t  \leq K \Delta (t) \, ,
\end{aligned}
\end{equation}
while we use Eq.~\eqref{eq:decompo_interpolation} for the second term
\begin{equation}\label{eq:stp1_b2}
\begin{aligned}
({\rm II})  \leq&~  \| \oa^t \|_{\infty}  \| \sigma ' \|_{\infty} \int \Big\{ \big\vert \< \bu^t - \obu^t , \bz \> \big\vert + \big\vert \| \bv^t  \|_2 - \os^t \bve + W_1 \big( \< \bv^t , \br \> , \| \bv^t \|_2 G \big) \Big\} \de \gamma_t \\
\leq&~ K(1+T) \int \big\vert \< \bu^t - \obu^t , \bz \> \big\vert \de \gamma_t +  K (1+T) \Delta(t) +  K e^{KT^2}  \sqrt{\frac{\log d}{d}}\, ,
\end{aligned}
\end{equation}
where we used Eq.~\eqref{eq:bound_evol_a} in Lemma \ref{lem:bound_evoluation_a} and Eq.~\eqref{eq:bound_W1} with $q=1$.

Combining bounds \eqref{eq:stp1_b1} and \eqref{eq:stp1_b2} and by Jensen's inequality,
\[
\begin{aligned}
&~\big\| \hf_{\NN} ( \cdot ; \rho_t ) - \hf_{\NN} (\cdot ; \orho_t ) \big\|_{L^2} \\
\leq&~ K(1+T) \Big( \int \E_{\bz} [ | \< \bu^t - \obu^t , \bz \>|^2 ] \de \gamma_t \Big)^{1/2} + K (1+T) \Delta (t) + K e^{KT^2}  \sqrt{\frac{\log d}{d}} \\
\leq&~ K (1+T) \Delta (t) + K e^{KT^2}  \sqrt{\frac{\log d}{d}} \, ,
\end{aligned}
\]
which proves Eq.~\eqref{eq:bound_dist_NN_DF_MF}.

\noindent
{\bf Step 2. Bound on $(a^t - \oa^t)^2$.}

Let us bound the derivative
\[
\begin{aligned}
\Big\vert \frac{\de}{\de t } (a^t - \oa^t ) \Big\vert =&~ \Big\vert (a^t - \oa^t ) \E_{\bz,\br,G} \Big[ \hg (\bz ; \rho_t) \sigma ( \< \bu^t, \bz\> + \< \bv^t , \br \> ) - \hg ( \bz ; \orho_t ) \sigma ( \< \obu^t, \bz\> + \os^t G ) \Big] \Big\vert \\
\leq &~ ({\rm I})  + ({\rm II}) \, ,
\end{aligned}
\]
where 
\[
\begin{aligned}
({\rm I})  =&~ \Big\vert \E_{\bz,\br,G} \Big[ \Big\{ \hg (\bz ; \rho_t)  - \hg ( \bz ; \orho_t ) \Big\} \sigma ( \< \bu^t, \bz\> + \< \bv^t , \br \> ) \Big] \Big\vert  \, ,\\
({\rm II})  =&~ \Big\vert  \E_{\bz,\br,G} \Big[ \hg ( \bz ; \orho_t )  \Big\{ \sigma ( \< \bu^t, \bz\> + \< \bv^t , \br \> ) - \sigma ( \< \obu^t, \bz\> + \os^t G ) \Big\} \Big] \Big\vert  \, .
\end{aligned}
\]
Noting that $\hg (\bz ; \rho_t)  - \hg ( \bz ; \orho_t ) = \hf_{\NN} (\bz ; \orho_t) - \hf_{\NN} (\bz ; \rho_t)$, the first term can be bounded as in step 1 by 
\begin{equation}\label{eq:stp2_b1}
\begin{aligned}
({\rm I})  \leq &~  \| \sigma \|_{\infty} \E_{\bz} \Big[ \Big\{ \hg (\bz ; \rho_t)  - \hg ( \bz ; \orho_t ) \Big\}^2 \Big]^{1/2} 
\leq K (1+T) \Delta (t) + K e^{KT^2}  \sqrt{\frac{\log d}{d}}\, .
\end{aligned}
\end{equation}
For the second term, we use Eq.~\eqref{eq:bound_evol_hR} in Lemma \ref{lem:bound_evoluation_a} and the decomposition \eqref{eq:decompo_interpolation}:
\begin{equation}\label{eq:stp2_b2}
\begin{aligned}
({\rm II})  \leq &~   \|\hg ( \cdot ; \orho_t )  \|_{\infty} \E_{\bz,\br,G} \Big[ \big\vert \sigma ( \< \bu^t, \bz\> + \< \bv^t , \br \> ) - \sigma ( \< \obu^t, \bz\> + \os^t G ) \big\vert \Big]\\ 
\leq &~ K (1+T) \Big\{ \E_{\bz} \big[ \bve \< \bu^t - \obu^t ; \bz \> \bve \big] + \bve \| \bv^t \|_2 - \os^t \bve + W_1 \big( \< \bv^t , \br \> , \| \bv^t \|_2 G \big) \Big\} \\
\leq&~  K (1+T) \delta(t) +  K (1+T)W_1 \big( \< \bv^t , \br \> , \| \bv^t \|_2 G \big) \, .
\end{aligned}
\end{equation}
Combining Eqs.~\eqref{eq:stp2_b1} and \eqref{eq:stp2_b2} and applying Cauchy-Schwarz inequality yield
\begin{equation}
\begin{aligned}
&~ \frac{1}{2 }\Big\vert\frac{\de}{\de t } \int ( a^t - \oa^t )^2 \de \gamma_t  \Big\vert\\
 \leq&~  \Delta(t)^2 + \int \Big\vert \frac{\de}{\de t } (a^t - \oa^t ) \Big\vert^2 \de \gamma_t \\
\leq &~ K (1+T)^2 \Delta(t)^2 + Ke^{KT^2} \frac{\log d}{d} + K(1+T)^2 \int W_1 \big( \< \bv^t , \br \> , \| \bv^t \|_2 G \big)^2 \gamma_t ( \de \btheta^t \, \de \obtheta^t ) \\
\leq &~K (1+T)^2 \Delta(t)^2 + Ke^{KT^2} \frac{\log d}{d} \, ,
\end{aligned}
\end{equation}
where we used the bound \eqref{eq:W1_bound_max_v_i} on $W_1$ and Eq.~\eqref{eq:bound_W1} in Lemma \ref{lem:exp_max_subG} in the last line. We deduce that for $t \in [0,T]$,
\begin{equation}\label{eq:Gronwall_a}
\int ( a^t - \oa^t )^2 \de \gamma_t  \leq Ke^{KT^2} \frac{\log d}{d} +K(1+T)^2 \int_0^T \Delta(t)^2 \de t\, , 
\end{equation}
where we used that $a^0 = \oa^0$ at initialization.

\noindent
{\bf Step 3. Bound on $\| \bu^t - \obu^t\|_2^2$.}

Again we first bound the derivative:
\[
\begin{aligned}
\Big\vert \frac{1}{2}\frac{\de}{\de t} \| \bu^t - \obu^t \|_2^2 \Big\vert  =&~ \Big\vert \< \bu^t - \obu^t , \frac{\de}{\de t} ( \bu^t - \obu^t ) \> \Big\vert
\leq ({\rm I}) + ({\rm II}) + ({\rm III})  \, ,
\end{aligned}
\]
where
\[
\begin{aligned}
({\rm I}) =&~ \Big\vert (a^t - \oa^t ) \E_{\bz, \br} \big[ \hg (\bz ; \rho_t) \sigma ' ( \< \bu^t , \bz \> + \< \bv^t , \br \>) \< \bz , \bu^t - \obu^t \> \big] \Big\vert \, , \\
({\rm II}) =&~ \Big\vert  \oa^t \E_{\bz, \br} \big[ \big\{ \hg (\bz ; \rho_t) - \hg (\bz ; \orho_t) \big\} \sigma ' ( \< \bu^t , \bz \> + \< \bv^t , \br \>) \< \bz , \bu^t - \obu^t \> \big] \Big\vert \, , \\
({\rm III}) =&~ \Big\vert  \oa^t \E_{\bz, \br,G} \big[  \hg (\bz ; \orho_t) \big\{ \sigma ' ( \< \bu^t , \bz \> + \< \bv^t , \br \>) - \sigma ' ( \< \obu^t , \bz \> + \os^t G) \big\} \< \bz , \bu^t - \obu^t \> \big] \Big\vert \,  .
\end{aligned}
\]
These terms are bounded respectively by
\begin{equation}\label{eq:stp3_b1}
\begin{aligned}
({\rm I})  \leq &~ | a^t - \oa^t | \| \hg (\cdot ; \rho_t ) \|_{\infty} \| \sigma ' \|_{\infty} \E [ \< \bz , \bu^t - \obu^t \>^2 ]^{1/2}
\\
\leq&~  K(1+T) \delta(t)^2 \, , \\
({\rm II})  \leq &~ \| \oa^t \|_{\infty} \| \sigma ' \|_{\infty} \E_{\bz} \Big[  \big\{ \hg (\bz ; \rho_t) - \hg (\bz ; \orho_t) \big\}^2 \Big]^{1/2} \E_\bz [ \< \bz , \bu^t - \obu^t \>^2 ]^{1/2} \\
\leq&~ K (1+T)^2  \delta(t ) \Delta (t) + Ke^{KT^2} \delta(t) \sqrt{\frac{\log d}{d }} \\
\leq &~ K (1+T)^2  \big[ \Delta (t)^2 + \delta (t)^2 \big] + Ke^{KT^2} \frac{\log d}{d } \, , 
\end{aligned}
\end{equation}
and 
\begin{equation}\label{eq:stp3_b2}
\begin{aligned}
({\rm III}) 
\leq &~ \| \oa^t \|_{\infty} \| \hg ( \cdot ; \orho_t) \|_{\infty}\|\bu^t - \obu^t \|_2 \\
&~ \phantom{AAAAAA} \times \E_{\bz} \Big[ \big( \E_{\br , G} \big[  \sigma ' ( \< \bu^t , \bz \> + \< \bv^t , \br \>) - \sigma ' ( \< \obu^t , \bz \> + \os^t G) \big] \big)^2 \Big]^{1/2} \\
\leq &~ K(1+T)^2 \delta (t)^2 + K(1+T)^2 W_1 \big( \< \bv^t , \br \> , \| \bv^t \|_2 G \big)^2 \, .
\end{aligned}
\end{equation}

Combining inequalities \eqref{eq:stp3_b1} and \eqref{eq:stp3_b2} yields 
\begin{equation}
\begin{aligned}
 \frac{1}{2 }\Big\vert\frac{\de}{\de t } \int \|\bu^t - \obu^t \|_2^2\,  \de \gamma_t  \Big\vert
\leq &~K (1+T)^2 \Delta(t)^2 + Ke^{KT^2} \frac{\log d}{d} \, ,
\end{aligned}
\end{equation}
where we again used the bound \eqref{eq:bound_W1}. We deduce that for $t \in [0,T]$,
\begin{equation}\label{eq:Gronwall_u}
\int \| \bu^t - \obu^t \|_2^2 \de \gamma_t  \leq Ke^{KT^2} \frac{\log d}{d} + K \frac{P}{d}+ K(1+T)^2 \int_0^T \Delta(t)^2 \de t \, , 
\end{equation}
where we used that $\obu^0 = \bzero$ and $\sqrt{d} \cdot \bu^0 \sim \mu_w^{\otimes P}$ at initialization, and $\E_{\mu_w^{\otimes P}} [ \| \bu \|_2^2 ] \leq K P/d$.

\noindent
{\bf Step 4. Bound on $\big\vert \os^t - \| \bv^t \|_2 \big\vert^2 $.}

First, notice that we have the following simple upper bounds on the evolution of $\| \bv^t \|_2$ and $\os^t$:
\[
\begin{aligned}
\Big\vert \frac{\de}{\de t} \| \bv^t \|_2 \Big\vert \leq&~ \| a^t \|_{\infty} \| \sigma '\|_{\infty} \E_{\bz}  \big[ \hg (\bz ; \rho_t )^2 \big]^{1/2} \E_{\br} \big[ \< \bv^t , \br \>^2 / \| \bv^t \|_2^2 \big]^{1/2} 
\leq  K(1 +T) \, , \\
\Big\vert \frac{\de}{\de t} \os^t \Big\vert \leq&~ \| \oa^t \|_{\infty}  \| \sigma '\|_{\infty} \E_{\bz}  \big[ \hg (\bz ; \orho_t )^2 \big]^{1/2} \E_{G} \big[ G^2 \big]^{1/2} 
\leq K(1 +T) \, , \\
\end{aligned}
\]
which yields
\begin{equation}\label{eq:bound_vt_ost}
\sup_{t \in [0,T]}  \| \bv^t \|_2  \leq \| \bv^0 \|_2 + K (1+T)^2 \, , \qquad\qquad \sup_{t \in [0,T]}  \os^t  \leq K (1+T)^2\, .
\end{equation}
Furthermore, we have by Gaussian integration by part
\begin{equation}\label{eq:GPP}
\frac{\de}{\de t} \os^t 
=   \os^t \oa^t \E_{\bz, G} \big[ \hg ( \bz ; \orho_t ) \sigma '' ( \< \obu^t , \bz \> + \os^t G )   \big]  \, .
\end{equation}
Similarly,
we have by expanding the expectation over the $r_i$'s and using the mean-value theorem:
\begin{equation}\label{eq:GPP_vt}
\begin{aligned}
\frac{\de }{\de t } \| \bv^t \|_2= &~  a^t \E_{\bz, \br } \big[ \hg (\bz ; \rho_t ) \sigma ' ( \< \bu^t , \bz \> + \< \bv^t , \br \> ) \< \bv^t / \| \bv^t \|_2 , \br \> \big] \\
= &~ \frac{a^t }{\| \bv^t \|_2} \sum_{i \in [d-P]}  \E_{\bz,\br} \big[ \hg (\bz ; \rho_t ) \sigma ' ( \< \bu^t , \bz \> + \< \bv^t , \br \> ) v^t_i r_i \big]   \\
=&~  \frac{a^t }{\| \bv^t \|_2}  \sum_{i \in [d-P]}  (v_i^t)^2 \E_{\bz,\br} \big[ \hg (\bz ; \rho_t ) \sigma '' ( \< \bu^t , \bz \> + \< \bv^t_{-i} , \br_{-i} \> + \xi_i )\big]    \\
=&~ a^t \| \bv^t \|_2 \E_{\bz,\br} \big[ \hg (\bz ; \rho_t ) \sigma '' ( \< \bu^t , \bz \> + \< \bv^t, \br\>)\big]  + M_t\, ,
\end{aligned}
\end{equation}
where
\begin{equation}\label{eq:stp4_Mt_def}
\begin{aligned}
M_t
=&~ \frac{a^t }{\| \bv^t \|_2}  \sum_{i \in [d-P]}  (v_i^t)^2 \\
&~ \phantom{AAAA} \times \E_{\bz,\br} \Big[ \hg (\bz ; \rho_t ) \big\{ \sigma '' ( \< \bu^t , \bz \> + \< \bv^t_{-i} , \br_{-i} \> + \xi_i )  - \sigma '' ( \< \bu^t , \bz \> + \< \bv^t , \br\> )  \big\} \Big] 
\end{aligned}
\end{equation}

We can now bound the evolution in time of $(\| \bv^t \|_2 - \os^t)$. Using the expressions in Eqs.~\eqref{eq:GPP} and \eqref{eq:GPP_vt}, we decompose
\[
\begin{aligned}
&~ \Big\vert \frac{\de}{\de t} \big\{ \| \bv^t \|_2 - \os^t \big\} \Big\vert \\
=&~ \Big\vert a^t \| \bv^t \|_2  \E_{\bz,\br} \big[ \hg ( \bz ; \rho_t ) \sigma'' ( \< \bu^t , \bz\> + \< \bv^t , \br \>) \big] + M_t - \oa^t  \os^t \E_{\bz,G} \big[ \hg ( \bz ; \orho_t ) \sigma'' ( \< \obu^t , \bz\> + \os^t G) \big]\Big\vert \\
\leq &~  ({\rm I}) + ({\rm II}) + ({\rm III})  + ({\rm IV}) + |M_t| \, ,
\end{aligned}
\]
where
\[
\begin{aligned}
({\rm I}) =&~ \Big\vert (\| \bv^t \|_2 - \os^t ) a^t \E_{\bz, \br} \Big[ \hg (\bz ; \rho_t) \sigma '' ( \< \bu^t , \bz \> + \< \bv^t , \br \>)  \Big] \Big\vert \, , \\
({\rm II}) =&~ \Big\vert (a^t - \oa^t ) \os^t \E_{\bz, \br} \Big[ \hg (\bz ; \rho_t) \sigma '' ( \< \bu^t , \bz \> + \< \bv^t , \br \>)  \Big] \Big\vert \, , \\
({\rm III}) =&~ \Big\vert \os^t \oa^t \E_{\bz, \br} \Big[ \big\{ \hg (\bz ; \rho_t) - \hg (\bz ; \orho_t) \big\} \sigma '' ( \< \bu^t , \bz \> + \< \bv^t , \br \>)\Big] \Big\vert \, , \\
({\rm IV}) =&~ \Big\vert  \oa^t \os^t \E_{\bz, \br,G} \Big[  \hg (\bz ; \orho_t) \big\{ \sigma '' ( \< \bu^t , \bz \> + \< \bv^t , \br \>)  - \sigma '' ( \< \obu^t , \bz \> + \os^t G)  \big\} \Big] \Big\vert \,  .
\end{aligned}
\]
These four quantities can be bounded as previously:
\begin{equation}\label{eq:stp4_b1}
\begin{aligned}
({\rm I}) \leq &~ \big\vert \| \bv^t \|_2 - \os^t \big\vert \| a^t \|_{\infty} \| \sigma '' \|_{\infty} \|\hg (\cdot ; \rho_t) \|_{L^2} 
\leq  K(1+T) \delta (t) \, , \\
({\rm II}) \leq &~ \big\vert a^t - \oa^t \big\vert \cdot| \os^t | \cdot \| \sigma '' \|_{\infty} \|\hg (\cdot ; \rho_t) \|_{L^2} 
\leq  K(1+T) \delta (t) \, , \\
({\rm III}) \leq &~ \| \oa^t \|_{\infty} \cdot | \os^t | \cdot \| \sigma '' \|_{\infty} \|\hg (\cdot ; \rho_t)  -\hg (\cdot ; \orho_t)  \|_{L^2} \\
    \leq &~ K(1+T)^3 \Delta (t) + K e^{KT^2} \sqrt{\frac{\log d}{d}}\, , \\
({\rm IV}) \leq &~ \| \oa^t \|_{\infty} \cdot | \os^t | \cdot \| \sigma '' \|_{\infty} \|\hg (\cdot ; \orho_t)  \|_{L^2} \\
&~ \phantom{AAAAAAA} \times\E_{\bz,\br,G} \Big[ \big\{ \sigma '' ( \< \bu^t , \bz \> + \< \bv^t , \br \>)  - \sigma '' ( \< \obu^t , \bz \> + \os^t G)  \big\}^2 \Big]^{1/2}\\
\leq &~ K(1+T)^2 \delta (t) + K (1+T)^2 W_1 ( \< \bv^t , \br \> , \| \bv^t \|_2 G) \, .
\end{aligned}
\end{equation}
For the last term, we use Eq.~\eqref{eq:stp4_Mt_def} and that $| \xi_i | \leq |v_i^t|$:
\begin{equation}\label{eq:stp4_b2}
\begin{aligned}
&~| M_t | \\
\leq&~ \frac{\|a^t\|_{\infty} }{\| \bv^t \|_2}  \sum_{i \in [d-P]}  (v_i^t)^2 \\
&~\phantom{AAAA} \times \Big\vert \E_{\bz,\br} \Big[ \hg (\bz ; \rho_t ) \big\{ \sigma '' ( \< \bu^t , \bz \> + \< \bv^t_{-i} , \br_{-i} \> + \xi_i )  - \sigma '' ( \< \bu^t , \bz \> + \< \bv^t , \br\> )  \big\} \Big]  \Big\vert  \\
\leq&~ 2 \frac{\|a^t\|_{\infty} }{\| \bv^t \|_2}  \| \hg (\cdot ; \rho_t ) \|_{L^2} \|\sigma '''\|_{\infty} \sum_{i \in [d-P]}  |v_i^t|^3 \\
\leq &~ K \|a^t\|_{\infty}  \|\bv^t\|_{2} \max_{i \in [d-P]} | v_i^t | \leq K(1+T) [ \| \bv^0 \|_2  + K (1+T) ] \max_{i \in [d-P]} | v_i^t |\, .
\end{aligned}
\end{equation}
Combining Eqs.~\eqref{eq:stp4_b1} and \eqref{eq:stp4_b2} and applying Cauchy-Schwarz inequality yield
\begin{equation}
\begin{aligned}
 \frac{1}{2 }\Big\vert\frac{\de}{\de t } \int ( \| \bv^t \|_2 - \os^t )^2 \de \gamma_t  \Big\vert
 \leq&~  \Delta(t)^2 + \int \Big\vert \frac{\de}{\de t } \{ \| \bv^t \|_2 - \os^t \} \Big\vert^2 \de \gamma_t \\
\leq &~K (1+T)^6 \Delta(t)^2 + Ke^{KT^2} \frac{\log d}{d} \, ,
\end{aligned}
\end{equation}
where we used that
\[
\int \|\bv^0 \|_2^2 \max_{i \in [d-P]} | v_i^t |^2 \de \rho_t \leq \Big( \int \|\bv^0 \|_2^4 \de \rho_t \Big)^{1/2} \Big( \int  \max_{i \in [d-P]} | v_i^t |^4 \de \rho_t \Big)^{1/2}  \leq K e^{KT^2} \frac{\log d}{d} \, .
\]
We deduce that for $t \in [0,T]$,
\begin{equation}\label{eq:Gronwall_s}
\int \big\vert \| \bv^t \|_2 - \os^t \big\vert^2 \de \gamma_t  \leq Ke^{KT^2} \frac{\log d}{d} + K \frac{P}{d}+ K(1+T)^6 \int_0^T \Delta(t)^2 \de t \, , 
\end{equation}
where we used that $\sqrt{d} \cdot \bv^0 \sim \mu_w^{\otimes (d-P)}$ and $\os^0 =  \E_{\mu_w} [ W^2 ]^{1/2}$ at initialization, and
\[
\begin{aligned}
\int  \big\vert \| \bv^0 \|_2 - \os^0 \big\vert^2 \de \gamma_0 \leq&~  \E_{\bv^0} \Big[ \big\vert \| \bv^0 \|_2^2 - (\os^0)^2 \big\vert \Big]\\
 \leq&~ \E_{\bv^0} \Big[ \big\{ \| \bv^0 \|_2^2 - \E_{\bv^0} [ \| \bv^0 \|_2^2 ] \big\}^2 \Big]^{1/2}  + \Big\vert \os^0 - \E_{\bv^0} [ \| \bv^0 \|_2^2 ]\Big\vert 
 \leq K \frac{P}{d} \, .
\end{aligned}
\]

\noindent
{\bf Step 5. Concluding the proof.}

We can now combine inequalities \eqref{eq:Gronwall_a}, \eqref{eq:Gronwall_u} and \eqref{eq:Gronwall_u} to get
\[
\begin{aligned}
\Delta (t)^2 =&~ \int \big\{ \big\vert a^t - \oa^t \big\vert \vee \| \bu^t - \obu^t \|_2 \vee \big\vert \| \bv^t \|_2 - \os^t \big\vert \big\}^2 \de \gamma_t \\
\leq &~ Ke^{KT^2} \frac{\log d}{d} + K \frac{P}{d}+ K(1+T)^6 \int_0^T \Delta(t)^2 \de t \, ,
\end{aligned}
\]
which concludes the proof. 

\subsection{Auxiliary lemmas} 
\label{app:auxiliary_DF_PDE}

\begin{lemma}\label{lem:bound_evoluation_a} Denote the residuals of the dynamics $\hg ( \bz ; \rho_t ) = h_* ( \bz) - \hf_{\NN} ( \bz ; \rho_t)$ and $\hg ( \bz ; \orho_t ) = h_* ( \bz) - \hf_{\NN} ( \bz ; \orho_t)$. By the properties of gradient flows, the risks 
\begin{equation}\label{eq:GF_bound_risk}
\| \hg ( \cdot ; \rho_t ) \|_{L^2} \leq \| \hg ( \cdot ; \rho_0 ) \|_{L^2} \leq K\, , \qquad \| \hg ( \cdot ; \orho_t ) \|_{L^2} \leq \| \hg ( \cdot ; \orho_0 ) \|_{L^2} \leq K \, .
\end{equation}
In particular, this implies
\begin{align}
\sup_{t \in [0,T]} \| a^t \|_{\infty} \vee \| \oa^t \|_{\infty}  \leq&~ K (1 +T)\, , \label{eq:bound_evol_a}\\
 \sup_{t \in [0,T]} \| \hg ( \cdot ; \rho_t ) \|_{\infty} \vee \| \hg ( \cdot ; \orho_t ) \|_{\infty} \leq&~ K(1+T) \label{eq:bound_evol_hR}\, .
\end{align}
\end{lemma}

\begin{proof}[Proof of Lemma \ref{lem:bound_evoluation_a}]
By definition $(\rho_t)_{t \geq 0}$ and $(\orho_t)_{t \geq 0}$ are the solutions of a gradient flow:
\[
\frac{\de}{\de t} R ( \rho_t) = - \int \| \psi ( \btheta ; \rho_t ) \|_2^2 \rho_t (\de \btheta ) \leq 0 \, , 
\]
and therefore
\[
R ( \rho_t) = \| \hg ( \cdot ; \rho_t ) \|_{L^2}^2 \leq \| \hg ( \cdot ; \rho_0 ) \|_{L^2}^2 \leq 2\| f_* \|_{\infty}^2 + 2\| a^0 \|_{\infty}^2 \| \sigma \|_{\infty}^2 \leq K \, ,
\]
and similarly for $\| \hg ( \cdot ; \orho_t ) \|_{L^2}$. 

Furthermore, by Jensen inequality,
\[
\Big\vert \frac{\de}{\de t} a^t \Big\vert = \Big\vert \E_{\bz,\br} \big[ \hg(\bz ; \rho_t) \sigma  ( \< \bu^t , \bz \> + \< \bv^t , \br \>) \big]\Big\vert \leq \| \sigma \|_{\infty} \| \hg ( \cdot ; \rho_t ) \|_{L^2} \leq K \, .
\]
We deduce that 
\[
| a^t | \leq |a^0| + \int_0^t \Big\vert \frac{\de}{\de s} a^s \Big\vert \de s \leq K(1+T) \, .
\]
A similar result holds for $\oa^t$. Finally,
\[
\| \hg ( \cdot ; \rho_t ) \|_{\infty} \leq \| a^t \|_{\infty} \| \sigma \|_{\infty} + \| f_* \|_{\infty} \leq K(1+T) \, ,
\]
which concludes the proof.
\end{proof}

\begin{lemma}[Berry-Esseen bound in Wasserstein metric]\label{lem:Berry-Esseen}
Let $(X_i)_{i \geq 1}$ be independent random variables with mean zero. Denote $v_n = \sum_{i \in [n]} \E[ X_i^2]$ and $S = v_n^{-1/2} \sum_{i \in [n]} X_i$. Then
\begin{equation}
W_1 ( S , G ) \leq \frac{3}{v_n^{3/2}}  \sum_{i \in [n]} \E [ | X_i|^3]\, ,
\end{equation}
where we denoted $G \sim \normal (0,1)$.
\end{lemma}

\begin{proof}[Proof of Lemma \ref{lem:Berry-Esseen}]
This is a simple application of Stein's method. Consider $f $ twice differentiable such that $\| f \|_{\infty} \leq 1 $, $\| f ' \|_{\infty} \leq \sqrt{2/\pi}$ and $\| f '' \|_{\infty} \leq 2$. Introduce $S_i = S - v_n^{-1/2} X_i  = v_n^{-1/2} \sum_{j \neq i} X_j$. By expanding, we get
\begin{equation}\label{eq:BE_W1}
\begin{aligned}
\E [ S f(S)] = &~ v_n^{-1/2} \sum_{i \in [n]} \E [ X_i f(S)] \\
=&~ v_n^{-1/2} \sum_{i \in [n]}  \E [ X_i (S - S_i) f' (S_i)] + M  \\
=&~ v_n^{-1} \sum_{i \in [n]}  \E [ X_i^2] \E[ f' (S_i)] + M \, ,
\end{aligned}
\end{equation}
where, by Taylor's theorem,
\begin{equation}\label{eq:BE_W2}
\begin{aligned}
|M| = &~ \Big\vert v_n^{-1/2} \sum_{i \in [n]} \E \big[ X_i \{ f(S) - f(S_i) - (S-S_i) f'(S_i) \}\big] \Big\vert \\
\leq&~ \frac{1}{2} v_n^{-1/2} \sum_{i \in [n]}  \E \big[ | X_i (S-S_i)^2| \| f''\|_{\infty} \big]\\
\leq &~ v_n^{-3/2} \sum_{i \in [n]} \E [ |X_i|^3] \, .
\end{aligned}
\end{equation}
Finally, note that 
\begin{equation}\label{eq:BE_W3}
\begin{aligned}
\Big\vert v_n^{-1} \sum_{i \in [n]} \E[ X_i^2] \E [ f' (S_i)] - \E [ f'(S)] \Big\vert \leq&~ v_n^{-1} \| f'' \|_{\infty}  \sum_{i \in [n]} \E [ X_i^2] \E [| S - S_i |] \\
\leq&~ 2 v_n^{-3/2} \sum_{i\in [n]} \E [ X_i^2] \E [ |X_i| ]\\
\leq &~ 2 v_n^{-3/2} \sum_{i \in [n]} \E [ | X_i|^3] \, ,
\end{aligned}
\end{equation}
where we used Jensen's inequality in the last line. Combining bounds Eqs.~\eqref{eq:BE_W2} and \eqref{eq:BE_W3} in the identity \eqref{eq:BE_W1} yields
\[
\Big\vert \E [ S f(S) ] - \E [ f' (S) ] \Big\vert \leq 3 v_n^{-3/2} \sum_{i \in [n]} \E [ | X_i|^3] \, .
\]
The result follows by Stein's lemma.
\end{proof}

\begin{lemma}\label{lem:exp_max_subG}
Let $d \geq 2$ be an integer. Consider $\{X_i\}_{i \in [d]}$ iid $(\tau^2/d)$-sub-Gaussian random variables with $0$ mean. Then for any $q\in \naturals$, there exists a universal constant $C_q>0$ such that 
\[
\E \Big[ \max_{i \in [d]} |X_i |^q \Big] \leq C_q \left( \tau^2 \frac{\log d}{d} \right)^{q/2}\, .
\]
\end{lemma}

\begin{proof}[Proof of Lemma \ref{lem:exp_max_subG}]
By sub-Gaussianity, there exists a universal constant $c>0$ such that
\[
\begin{aligned}
\P \Big( \max_{i \in [d]} |X_i |^q  >t \Big) = 1 - \{1 - \P ( |X_i | >t^{1/q} )\}^d \leq&~ 1 - \Big\{ 1 - 2 e^{ - c d t^{2/q} / \tau^2 } \Big\}^d\\
\leq &~ 1 \wedge \Big( 2d e^{ - cdt^{2/q} / \tau^2 } \Big) \, .
\end{aligned}
\]
Consider $t_c = \Big( \tau^2 \frac{\log(2d)}{\kappa_q cd} \Big)^{q/2}$ with $\kappa_q = 2^{2(q-2)/q^2 \vee 0 }$, such that $2d e^{ - cd\kappa_q t_c^{2/q} / \tau^2 } =1$. Then, we have the following upper bound: 
\[
\begin{aligned}
\E \Big[ \max_{i \in [d]} |X_i |^q \Big]  = &~ \int_0^\infty \P \Big( \max_{i \in [d]} |X_i |^q  >t \Big) \de t \\
\leq &~ t_c + \int_{t_c}^\infty  2d e^{ - cdt^{2/q} / \tau^2} \de t \\
\leq&~ t_c +2d e^{ - cd \kappa_q t_c^{2/q} / \tau^2} \int_{0}^\infty   e^{ - cd \kappa_q t^{2/q} / \tau^2} \de t  
= \left( \tau^2 \frac{\log(2d)}{\kappa_q cd}\right)^{q/2} + c_q \left( \frac{\tau^2}{\kappa_q cd} \right)^{q/2}\, ,
\end{aligned}
\]
which concludes the proof.
\end{proof}

\clearpage

\section{Strong SGD-learnability in the discrete-time regime}\label{app:strong-discrete}

In this appendix, we define strong SGD-learnability in the discrete-time regime, i.e., for large batch size $b$ and large $\eta$. We keep the same assumptions $\rA0$-$\rA2$, and replace Assumption $\rA3$ by 
\begin{itemize}
\item[$\rD 3.$] \textit{(Boundedness of hyperparameters)} We have $\eta_k^a, \eta_k^w \leq  K$ and $\lambda^a , \lambda^w \leq K$.
\end{itemize}

While the continuous-time regime requires step size $\eta$ to be small enough compared to $n/b$, the discrete-time regime requires the batch size $b$ to be big enough compared to $cn$ for $c \ll 1$ (recall $b\leq n$ by one-pass assumption) in the discrete regime. 

\begin{definition}[Strong SGD-learnability in $O(d)$-scaling (discrete time)]
We say that a function $h_* : \{ - 1, +1 \}^P \to \R$ is \emph{strongly $O(d)$-SGD-learnable} if the following hold for some $C (\cdot , h_*), T(\cdot,h_*):\R_{>0} \to \R_{>0}$. For any $\eps >0$, $d \geq C ( \eps, h_*)   $, $n \geq C ( \eps, h_*)   d$ and $e^d \geq N \geq C (\eps , h_*)$, there exists hyperparameters $(\sigma, b, \lambda^a, \lambda^w, \{ \eta_k^a , \eta_k^w   \}_{k \in [0,k_0]})$ and initialization $\rho_0$ satisfying $\rA0$-$\rA2, \rD3$ and $k_0 = n/b    \leq T ( \eps , h_*)$ such that for any $\cI \subseteq [d], | \cI | = P$ and target function $f_*(\bx) = h_* (\bx_{\cI})$, $k_0$ steps of batch stochastic gradient descent \eqref{eq:bSGD} achieves test error $\eps$ with probability at least $9/10$.
\end{definition}

Again, conditions $\rA0$-$\rA2,\rD3$ guarantee that as long as $d,n,N$ are taken sufficiently large, there exist a discrete mean-field dynamics that well-approximates batch-SGD up to a constant number of steps that depends on $\eps,h_*$.

\subsection{Discrete time mean-field and dimension-free dynamics}

We first give the discrete time mean-field dynamics to which batch-SGD converges. Recall that when $N \to \infty$ and $\eta \to 0$, the dynamics converge to the continuous \eqref{eq:MF-PDE}. Here instead, we fix the step sizes and consider $N,b \to \infty $, and get the following discrete mean-field dynamics $(\rho_k)_{k \geq 0}$  (with $\rho_k \in \cP (\R^{d+1})$) described by the initialization $(a^0,\bw^0) \sim \rho_0$ and the recurrence relation: $(a^{k+1},\bw^{k+1}) \sim \rho_{k+1}$ the distribution of the updated weights 
\begin{equation}\label{eq:d-MF-PDE}\tag{d-MF-PDE}
    \begin{aligned}
    a^{k+1} =&~ (1 - \eta_k^a \lambda^a) a^k + \eta_k^a \E_{\bx} \big[ \big\{ f_* (\bx) - \hf_{\NN} (\bx ; \rho_k) \big\} \sigma ( \< \bx , \bw^k \> ) \big] \, ,\\
    \bw^{k+1} =&~ (1 - \eta_k^w \lambda^w) \bw^k + \eta_k^w a^k \E_{\bx} \big[ \big\{ f_* (\bx) - \hf_{\NN} (\bx ; \rho_k) \big\} \sigma' ( \< \bx , \bw^k \> ) \bx \big] \, ,
    \end{aligned}
\end{equation}
where $(a^k, \bw^k) \sim \rho_k$.

Similarly to the continuous regime, the discrete dynamics simplify when $d \to \infty$ with $P$ fixed, to the following discrete dimension-free dynamics $(\orho_k)_{k \geq0}$ (with $\orho_k \in \cP( \R^{P+2})$) defined by the initialization $(\oa^0 , \obu^0 , \os^0) \sim \orho_0$ (with $\oa^0 \sim \mu_a$, $\obu^0 = \bzero$ and $\os^0 = m_2^w$) and the recurrence relation
\begin{equation}\label{eq:d-DF-PDE}\tag{d-DF-PDE}
    \begin{aligned}
    \oa^{k+1} =&~ (1 - \eta_k^a \lambda^a) \oa^k + \eta_k^a \E_{\bz,G} \big[ \big\{ h_* (\bz) - \hf_{\NN} (\bz ; \orho_k) \big\} \sigma ( \< \bz , \obu^k \> + \os^k G) \big] \, ,\\
    \obu^{k+1} =&~ (1 - \eta_k^w \lambda^w) \obu^k + \eta_k^w \oa^k \E_{\bz,G} \big[ \big\{ h_* (\bz) - \hf_{\NN} (\bz ; \orho_k) \big\} \sigma' ( \< \bz , \obu^k \>  + \os^k G) \obu^k \big] \, ,\\
    \os^{k+1} =&~ (1 - \eta_k^w \lambda^w) \os^k + \eta_k^w \oa^k \E_{\bz,G} \big[ \big\{ h_* (\bz) - \hf_{\NN} (\bz ; \orho_k) \big\} \sigma' ( \< \bz , \obu^k \>  + \os^k G) G \big]\, .
    \end{aligned}
\end{equation}

We have the new non-asymptotic bound between the \eqref{eq:bSGD} and \eqref{eq:d-DF-PDE} dynamics, analogous to Theorem \ref{thm:bSGD-to-DF}, but with a worse dependency on the number of iterations.

\begin{theorem}\label{thm:discrete-bSGD-to-DF}
Assume conditions $\rA0$-$\rA2$,$\rD3$ hold, and let $k_0 \geq 0$. There exists a constant $K$ depending only on the constants in $\rA 0$-$\rA 2$,$\rD3$ (in particular, independent of $d,P,k_0$), such that 
\begin{equation}\label{eq:d-bSGD_DF_comp}
\begin{aligned}
&~\sup_{k = 0 , \ldots , k_0 } \big\Vert \hat f_{\NN} (\cdot; \bTheta^k ) - \hat f_{\NN} (\cdot; \orho_{k} ) \big\Vert_{L^2}\\
\leq&~ K e^{e^{K k_0^2}} \left\{ \sqrt{\frac{P + \log(d)}{d}} + \sqrt{\frac{\log N}{N}} +  \sqrt{\frac{d+\log N}{b}}  \right\} \, ,
\end{aligned}
\end{equation}
with probability at least $1 - 1/N$.
\end{theorem}

From there, it is straightforward, following the same arguments as for Theorems \ref{thm:equivalence} and \ref{thm:MSP_necessary},
to get the equivalence of strong $O(d)$-SGD-learnability in the discrete-time regime and global convergence of the  discrete \eqref{eq:d-DF-PDE} dynamics, and the MSP necessary condition:

\begin{theorem}\label{thm:discrete-equivalence} A function $h_* : \{+1,-1\}^P \to \R$ is strongly $O(d)$-SGD-learnable in the discrete-time regime if and only if for any $\eps >0$, there exists $\lambda^a, \lambda^w \geq 0$ and bounded step-sizes $\{\eta_k^a, \eta_k^w \}_{k\geq 0}$, such that $\inf_{k \in \naturals} R(\orho_k) < \eps$, where $\orho_k$ is the solution of the discrete \eqref{eq:d-DF-PDE} dynamics.
\end{theorem}

\begin{theorem}\label{thm:MSP_necessary-discrete}
Let $h_* : \{+1,-1\}^P \to \R$ be a function without MSP. Then there exists $c >0$ such that for any regularizations $\lambda^a, \lambda^w \geq 0$ and step-sizes $\{\eta_k^a, \eta_k^w \}_{k\geq 0}$, we have $\inf_{k \in \naturals} R(\orho_k) \geq c$.
\end{theorem}

\subsection{Proof of Theorem \ref{thm:discrete-bSGD-to-DF} }

The proof relies on first comparing the \eqref{eq:bSGD} dynamics to the discrete mean-field dynamics \eqref{eq:d-MF-PDE}, using an extension of the results in \cite{mei2019mean} to the discrete  \eqref{eq:d-DF-PDE} dynamics (see Appendix \ref{app:discrete-prop-of-chaos}).

\begin{proposition}\label{prop:discrete-PDE-SGD-bound}
Assume conditions $\rA0$-$\rA2$,$\rD3$, and let $k_0 \in \naturals$. There exists a constant $K$ depending only on the constants in $\rA 0$-$\rA 2$,$\rD3$ (in particular, independent of $d,P,T$), such that 
\[
\sup_{k = 0, \ldots , k_0 } \big\Vert \hat f_{\NN} (\cdot; \bTheta^k ) - \hat f_{\NN} (\cdot; \rho_{k} ) \big\Vert_{L^2} \leq K e^{e^{K k_0}} \left\{ \sqrt{\frac{\log N}{N}} +  \sqrt{\frac{d+\log N}{b}}  \right\} \, ,
\]
with probability at least $1 - 1/N$.
\end{proposition}

The proof of this proposition follows from applying Proposition \ref{prop:discrete_prop}, with the assumptions already verified in Appendix \ref{app:proof-MF-approx}.

The proof of Theorem \ref{thm:discrete-bSGD-to-DF} then follows by combining the above result with the following bound between the discrete mean-field dynamics \eqref{eq:d-MF-PDE} and the discrete dimension-free dynamics \eqref{eq:d-DF-PDE}:

\begin{theorem}\label{thm:discrete-MF-dimension-free}
Assume conditions $\rA0$-$\rA2$,$\rD3$, and let $k_0 \geq 0$. There exists a constant $K$ depending only on the constants in $\rA0$-$\rA2$,$\rD3$ (in particular, independent of $d,P,T$), such that 
\[
\sup_{k = 0, \ldots , k_0 } \big\| \hf_{\NN} ( \cdot ; \rho_k ) - \hf_{\NN} ( \cdot ; \orho_k ) \big\|_{L^2} \leq K e^{e^{K k_0^2}}  \sqrt{\frac{P + \log (d)}{d}}\, .
\]
\end{theorem}

\begin{proof}[Proof of Theorem \ref{thm:discrete-MF-dimension-free}]
The proof follows similarly to the proof in the continuous case (see Section \ref{app:proof-MF-DF-comp}) and we will simply highlight the differences. First, by the same argument as in the proof of Proposition \ref{prop:discrete_prop}, we replace the bounds from Lemma \ref{lem:bound_evoluation_a} by 
\begin{align}
\sup_{k = 0, \ldots , k_0} \| a^k \|_{\infty} \vee \| \oa^k \|_{\infty}  \leq&~ K e^{Kk_0}\, , \label{eq:d-bound_evol_a}\\
 \sup_{k = 0, \ldots , k_0} \| \hg ( \cdot ; \rho_k ) \|_{\infty} \vee \| \hg ( \cdot ; \orho_k ) \|_{\infty} \leq&~ K e^{Kk_0} \label{eq:d-bound_evol_hR}\, .
\end{align}
and from Lemma \ref{lem:bound_thirdMoment_evolution} by
\[
 \sup_{k = 0, \ldots , k_0} \max_{i \in [d-P]} | v_i^k| \leq   K e^{e^{Kk_0}} \max_{i \in [d-P]} | v_i^0| \, .
\]
We define 
\begin{equation}
\delta (k) =  \big\vert a^k - \oa^k  \big\vert \vee \big\| \bu^k - \obu^k \big\|_2 \vee  \big\vert \os^k -\| \bv^k \|_2  \big\vert \,, 
\end{equation}
and the square root of its second moment
\begin{equation}
\Delta (k) = \Big( \int \delta (k)^2 \gamma_t ( \de \btheta^k \,\de \obtheta^k) \Big)^{1/2}\, .
\end{equation}
The proof follows by using discrete Gr\"onwall lemma in Lemma \ref{lem:discrete-prop_chaos_bounds} stated in the next section, which is the analogous of Lemma \ref{lem:prop_chaos_bounds} in discrete time.
\end{proof}

\subsection{Auxiliary lemma}

\begin{lemma}\label{lem:discrete-prop_chaos_bounds}
Consider the same setting and assumptions as Theorem \ref{thm:discrete-MF-dimension-free}. There exists a constant $K$ independent of $d,P$ and depending only on the Assumptions $\rA0$-$\rA2,\rD3$ such that for any $k_0 \in \naturals$,
\begin{equation}\label{eq:d-bound_dist_NN_DF_MF}
\begin{aligned}
\big\| \hf_{\NN} ( \cdot ; \rho_k) - \hf_{\NN} (\cdot ; \orho_k) \big\|_{L^2} =&~ \E_{\bz} \Big[ \Big\{ \hg (\bz ; \rho_k ) - \hg ( \bz ; \orho_k ) \Big\}^2 \Big]^{1/2} \\
 \leq&~ Ke^{K k_0} \Delta (k) + K e^{e^{K k_0}} \sqrt{\frac{\log d}{d}} \, , \\
\end{aligned}
\end{equation}
where 
\begin{equation}\label{eq:bound_on_Delta2}
\begin{aligned}
\Delta (k) \leq   K e^{e^{Kk_0}} \sqrt{\frac{P + \log d}{d}}  + K e^{Kk_0^2} \sum_{j = 0}^{k-1} \Delta (j)\, .
\end{aligned}
\end{equation}
\end{lemma}

\begin{proof}[Proof of Lemma \ref{lem:discrete-prop_chaos_bounds}]
The proof proceeds similarly to the proof of Lemma \ref{lem:prop_chaos_bounds} in Section \ref{app:DFproof_main_lemma}, where we use discrete Gr\"onwall instead. Step 1 to Step 3 are very similar, using that
\[
\begin{aligned}
&~ \| \bu^{k+1} - \obu^{k+1} \|_2^2 \\
\leq&~\| \bu^{k} - \obu^{k} \|_2^2 + | \< \bu^{k+1} - \obu^{k+1} , (\bu^{k+1} - \bu^k) - (\obu^{k+1} - \obu^k) \> \, .
\end{aligned}
\]
For Step 4, notice that, denoting $\bP^k = a^k \E_{\bz,\br} [ \hg (\bz ; \rho_k) \sigma ' ( \< \bu^k , \bz \> + \< \bv^k, \br \>) \br ]$,
\[
\begin{aligned}
\| \bv^{k+1} \|_2^2 = \| \bv^k \|_2^2 + 2 \eta_k \< \bv_k , \bP^k \> + \eta_k^2 \| \bP^k \|_2^2 \, .
\end{aligned}
\]
Note that integrating out $r_i$ and with mean-value theorem, we get $P_i^k = v_i^k \E_{\bz,\br} \big[ \hg (\bz ; \rho_k ) \sigma '' ( \< \bu^k , \bz \> + \< \bv^k_{-i} , \br_{-i} \> + \xi_i )\big]$. Denote $m_k = a^k \E_{\bz,\br} [ \hg (\bz ; \rho_k) \sigma '' ( \< \bu^k , \bz \> + \< \bv^k, \br \>)  ]$ and $\bT^k = (T^k_i) = \bP^k - m_k \bv^k$ with
\[
\begin{aligned}
T^k_i = a^k v_i^k \E_{\bz,\br} [ \hg (\bz ; \rho_k) (\sigma '( \< \bu^k , \bz \> + \< \bv^k, \br \> + \xi_i) - \sigma '' ( \< \bu^k , \bz \> + \< \bv^k, \br \>) ) ]  \,.
\end{aligned}
\]
Then, we can decompose
\[
\| \bv^{k+1} \|_2^2 = \| \bv^k \|_2^2 (1 + \eta_k m_k)^2 + 2 \eta_k (1 + \eta_k m_k) \< \bv^k , \bT^k \> + \eta_k^2 \| \bT^k \|_2^2 \, .
\]
Denote $\om_k = \oa^k \E_{\bz,G} [ \hg (\bz ; \orho_k) \sigma '' ( \< \obu^k , \bz \> + \os^k G)  ]$ and note that $\os^{k+1} = \os^k (1 + \eta_k \om_k)$ (using Gaussian integration by parts). 

We decompose:
\[
\begin{aligned}
 \big\vert \| \bv^{k+1} \|_2 - \os^{k+1} \big\vert
\leq &~ ({\rm I}) + ({\rm II}) \, ,
\end{aligned}
\]
where 
\[
\begin{aligned}
({\rm I}) =&~  \| \bv^{k} \|_2 \Big\vert  \sqrt{(1 + \eta_k m_k)^2 +  2 \eta_k (1 + \eta_k m_k) \< \bv^k / \| \bv^k \|_2^2 , \bT^k \> + \eta_k^2 \| \bT^k \|_2^2 / \| \bv_k \|_2^2} - |1 + \eta_k m_k| \Big\vert\, , \\
({\rm II}) =&~  \Big\vert \| \bv^{k} \|_2 |1 + \eta_k m_k| - \os^k |1 + \om_k | \Big\vert\, .
\end{aligned}
\]
The first term is bounded by
\[
({\rm I}) \leq ~ 2 \eta_k | \< \bv_k / \| \bv^k \|_2 , \bT^k \> |  + \eta_k \| \bT^k \|_2 \, .
\]
Note that $\< \bv_k / \| \bv^k \|_2 , \bT^k \> =M_k$ as defined in Eq.~\eqref{eq:stp4_Mt_def} and we can use the bound in Eq.~\eqref{eq:stp4_b2}:
\[
\begin{aligned}
| \< \bv_k / \| \bv^k \|_2 , \bT^k \> |  \leq Ke^{Kk_0} [ \| \bv^0 \|_2  + 1 ] \max_{i \in [d-P]} | v_i^k | \, .
\end{aligned}
\]
Similarly, we have
\[
\begin{aligned}
\| \bT^k \|_2^2 =&~ (a^k)^2  \sum_{i \in [d-P]}  (v_i^t)^2 \\
&~ \phantom{AAAA} \times \E_{\bz,\br} \Big[ \hg (\bz ; \rho_k ) \big\{ \sigma '' ( \< \bu^k , \bz \> + \< \bv^k_{-i} , \br_{-i} \> + \xi_i )  - \sigma '' ( \< \bu^k , \bz \> + \< \bv^k , \br\> )  \big\} \Big]^2  \\
\leq&~ 2 \|a^k\|_{\infty}^2   \| \hg (\cdot ; \rho_k ) \|_{L^2}^2 \|\sigma '''\|_{\infty}^2 \sum_{i \in [d-P]}  |v_i^k|^4 \\
\leq&~ K \|a^k\|_{\infty}^2  \|\bv^k\|_{2}^2 \max_{i \in [d-P]} | v_i^k |^2 \leq Ke^{Kk_0} [ \| \bv^0 \|_2^2  + 1 ] \max_{i \in [d-P]} | v_i^k |^2 \, .
\end{aligned}
\]
We deduce
\begin{equation}\label{eq:d-s-b1}
    ({\rm I}) \leq K e^{Kk_0} [ \| \bv^0 \|_2  + 1 ] \max_{i \in [d-P]} | v_i^k | \, .
\end{equation}
We can further bound $({\rm II}) $ using the same decomposition as in Eq.~\eqref{eq:stp4_b1}:
\begin{equation}\label{eq:d-s-b2}
 ({\rm II})  \leq K e^{Kk_0} \Big[ \delta (k ) + \Delta(k) +  W_1 ( \< \bv^t , \br \> , \| \bv^t \|_2 G) \Big] + K e^{e^{Kk_0}} \sqrt{\frac{\log d}{d}} \, .
\end{equation}

Combining Eqs.~\eqref{eq:d-s-b1} and \eqref{eq:d-s-b2}, we obtain
\[
\Delta ( k+1)^2 \leq K e^{Kk_0} \Delta (k)^2 +  K e^{e^{Kk_0}} \frac{\log d}{d} + K e^{e^{Kk_0}} \frac{P}{d}\, .
\]
This concludes the proof.
\end{proof}

\clearpage

\section{Vanilla staircase functions are strongly $O(d)$-SGD-learnable: Proof of Theorem~\ref{thm:vanillastaircasesuffrestated}}\label{ssec:vanillaproof}

We start by providing the proof that vanilla staircases are strongly $O(d)$-SGD-learnabile, as described in Theorem~\ref{thm:vanillastaircasesuffrestated}. This proof will outline the main ideas behind our global convergence results, without the technical complexity of dealing with general MSP set structure.

\paragraph{Assumption on activation function} We will assume the following hold for the activation $\sigma$:

\begin{itemize}
    \item[$\rA$0'.] Let $\sigma: \R \to \R$ be an activation function that satisfies Assumption $\rA 0$. Furthermore, assume that for some $L \in \naturals$ and $\eta >0$ such that $\sigma$ is $L+1$ differentiable on $(-\eta,\eta)$ with
    \[
    \sup_{x \in (-\eta, \eta)} | \sigma^{(L+1)} (x)  | \leq K \, .
    \]
    We will denote $m_r = \sigma^{(r)} (0)$ and $\bm = (m_0 , \ldots, m_{L} ) \in \R^{L+1}$.
\end{itemize}

In particular, this assumption implies that we have the following polynomial approximations of $\sigma$ and $\sigma '$ around $0$: for any $x \in (-\eta, \eta)$,
\begin{equation}\label{eq:poly_approx_sigma}
    \begin{aligned}
     \Big\vert \sigma (x) - \sum_{r = 0 }^{L} \frac{m_r}{r!} x^r \Big\vert \leq&~ K |x|^{L+1}\, , \\
     \Big\vert \sigma ' (x) - \sum_{r = 0 }^{L-1} \frac{m_{r+1}}{r!} x^r \Big\vert \leq&~ K |x|^{L}\, .
    \end{aligned}
\end{equation}

The Assumption $\rA 0'$ is simply to connect with the definition of strong $O(d)$-SGD-learnability. However, in the proof below, we will choose hyperparameters such that $|\< \ubu^t, \bz \>| < \eta$ (the input of the activation) during the whole dynamics, so that $\rA 0 $ can be lifted. In particular, any activation that is $\cC^{L+1} (\R)$ will satisfy $\rA 0'$.

Now recall that vanilla staircases $h_* : \{+1,-1\}^P \to \R$ are those functions of the form
\begin{equation}\label{eq:vanillastaircasestructure}
h_* (\bz) = \alpha_{\{1\}} z_1 + \alpha_{\{1,2\}} z_1z_2 + \alpha_{\{1,2,3\}} z_1z_2z_3 + \ldots + \alpha_{\{1,\ldots,P\}} z_1 z_2 \cdots z_P\, .
\end{equation}
for some Fourier coefficients $\alpha_{\{1\}},\ldots,\alpha_{\{1,\ldots,P\}} \in \R \sm \{0\}$. We will prove that any such function $h_*$ is strongly SGD-learnable in the $O(d)$-scaling.

\paragraph{Choice of hyperparameters}  Recall from the equivalence with \eqref{eq:DF-PDE} (Theorem \ref{thm:equivalence}) that it is sufficient to show for any $\eps >0$, there exist hyperparameters satisfying $\rA0$-$\rA2$, $\rA3'$ such that \eqref{eq:DF-PDE} dynamics reaches $\eps$-risk. We consider the following hyperparameters:
\begin{itemize}
    \item We do not regularize, i.e., $\lambda^a = \lambda^w= 0$.
    \item We initialize the first layer to deterministically $0$ weights, and the second layer to uniform random weights. I.e., we take $\mu_a = \Unif ( [ +1, -1])$ and $\mu_W = \delta_0$. Although initializing the first layer to 0 may at first glance seem restrictive, there turns out to be enough randomness in the initialization of the second layer to ensure that the neural network learns. For the dimension-free dynamics, this corresponds to taking  $(\oa^0 , \obu^0 , \os^0 ) \sim \orho_0$ with $\oa^0 \sim \Unif ( [ +1, -1])$, $\obu^0 = \bzero$ and $\os^0 = 0$. In particular, $\os^t = 0$ during the whole dynamics, which allows for a simpler analysis.
    \item Our learning rate schedule has two phases:
\begin{description}
\item[Phase 1:] We train the first layer weights $\obu^t$ while keeping the second layer weights fixed $\oa^t = \oa^0$. We set $\xi^a (t) = 0$ and $\xi^w (t) = 1$ for $t \in [0,T_1]$. 

\item[Phase 2:] We train the second layer weights $\oa^t$ while keeping the first layer weights fixed at $\obu^t = \obu^{T_1}$. We set $\xi^a (t) = 1$ and $\xi^w (t) = 0$ for $t \in [T_1,T_2]$.
\end{description}
\end{itemize}

\begin{remark}
As written above, the learning rate schedules $\xi^a,\xi^w$ are not Lipschitz at $T_1$. Note that we can always do the following change of time variable on $[0,T_1]$: $t ' =2 t T_1 - t^2$ such that $\xi^w (t') = 2(\sqrt{T_1} - t' )_+$ is Lipschitz on $\R_{\geq 0}$ (and we have now $T_1' = T_1^2$). Similarly, we can do a change of time variable on $[T_1',T_2']$ such that $\xi_a (t) = 2 \min ( (t - \sqrt{T_1})_+ , 1)$. We will proceed with the simpler learning schedule $\xi^a (t) = \ind_{t \geq T_1}$ and $\xi^w (t) = \ind_{t \leq T_1}$ with the understanding that we can do the above change of variables to obtain Lipschitz learning schedules and therefore fall under the assumptions of strong SGD learnability.
\end{remark}

We restate the sufficient condition in the case of the vanilla staircase.
\begin{theorem}[Theorem~\ref{thm:vanillastaircasesuffrestated} restated]\label{thm:vanillastaircasesuffrestated-restated}
Suppose that $h_* : \{+1,-1\}^P \to \R$ has the vanilla staircase structure \eqref{eq:vanillastaircasestructure}. Suppose also that the activation function $\sigma$ has nonzero derivatives $m_r \neq 0$ for $r =0 , \ldots, P$ and satisfies $\rA$0' for some $L > 2^{P-1}$. Then, for any $\eps > 0$, there are $T_1,T_2 > 0$ such that training the dimension-free PDE with the above hyperparameters will learn $h_*$ to accuracy $\eps$. Therefore, $h_*$ is strongly $O(d)$-SGD-learnable.
\end{theorem}

\subsection{Outline of the proof}
\label{app:outline_staircase}

Consider $(\orho_t)_{t \geq 0}$ the solution of \eqref{eq:DF-PDE} with the hyperparameters described above. Denote $\obu^t (\oa^0)$ the solution of the evolution equations \eqref{eq:evolution_effective_MF} obtained from initialization $ (\oa^0, \bzero, 0)$. For clarity, we will suppress some notations in the proof: we will denote $\bu$ instead of $\obu$, and $a$ instead of $\oa^0$. We will further forget about $\os^t = 0$ and simply consider $\orho_t \in \cP(\R^{P+1})$ the distribution of $(a,\bu)$. This last simplification can be done since we initialize the first-layer weights to 0, so in particular $\os^0 = 0$, and by the evolution equation of \eqref{eq:evolution_effective_MF} we have $\os^t = 0$ throughout training. Furthermore, we will denote $K$ a generic constant that only depends on $P$ and the constants in the assumptions. The value of $K$ can change from line to line.

The proof analyzes Phase 1 and Phase 2 of training separately.

\paragraph*{Phase 1 (nonlinear dynamics):} In this phase, we train the first layer, which has nonlinear dynamics, and so it is a priori unclear how to analyze. Nevertheless, since $h_*$ is specially structured, the structure in the weights during training is particularly simple and it is enough to track the smallest order terms in the weights.

Specifically, in Proposition \ref{prop:staircase_poly_approx} (see next section), we prove that there exist constants $c,C>0$ such that for all $t \leq c$ and $k \in [P]$, we have $|u_k^t (a) - \hat{u}_k^t (a) | \leq  C t^{2^{k-1}+1}$, where
\begin{equation}
    \hat{u}_k^t (a) = 2^{1 - 2^{k-1}} (at)^{2^{k-1}} \cdot \prod_{i \in [k]} (m_i\alpha_{\{1,\ldots,i\}})^{2^{\max((k-1-i),0)}} \, .
\end{equation}
Denote $\nu_k (t) = 2^{1 - 2^{k-1}} t^{2^{k-1}}\prod_{i \in [k]} (m_i\alpha_{\{1,\ldots,i\}})^{2^{\max((k-1-i),0)}}$ such that $u_k^t (a) = \nu_k (t) ( a^{2^{k-1}} + O(t) )$.

\paragraph*{Phase 2 (linear dynamics):}

In this phase, we train the second layer, and the training has linear dynamics. Denote $g_t ( \bz ) : = h_* (\bz) - \hf_{\NN} (\bz ; \orho_t)$ the residual function at time $t$. During this phase, we have the following evolution on the risk:
\begin{equation}
    \frac{\de}{\de t} R ( \orho_t) = - \E_{\bz,\bz'} \big[ g_t ( \bz ) K^{T_1} (\bz , \bz ' ) g_t (\bz ' )\big] \, ,
\end{equation}
where the kernel is given by
\begin{equation}
K^{T_1} ( \bz , \bz ' ) = \E_{a \sim \mu_a} \big[ \sigma ( \< \bu^{T_1} (a) , \bz \> ) \sigma ( \< \bu^{T_1} (a) , \bz ' \> ) \big] \, .
\end{equation}
(This is indeed the kernel, since at the end of Phase 1, the distribution $\orho_{T_1}$ of the parameters is given by $(a, \bu^{T_1} (a))$ with $a \sim \mu_a$, and the first-layer weights are kept constant during Phase 2.)

Let us decompose these quantities in the Fourier basis: denote $g_t(S) = \E_{\bz} [ g_t (\bz) \chi_S (\bz) ] $ and $K^{T_1} (S,S') = \E_{\bz,\bz' } [ K^{T_1} (\bz , \bz' ) \chi_S(\bz) \chi_{S'} (\bz') ]$, and the vector $\bg_t = (g_t (S))_{S\subseteq [P]}$ and matrix $\bK^{T_1} = ( K^{T_1} (S,S') )_{S,S' \subseteq[P]}$. Noting that $R ( \orho_t) = \| \bg_t \|_{2}^2$, we have
\begin{equation}
    \frac{\de}{\de t}  \| \bg_t \|_{2}^2 = - \bg_t^\sT \bK^{T_1} \bg_t \leq - \lambda_{\min} ( \bK^{T_1}  )  \| \bg_t \|_{2}^2\, .
\end{equation}
This implies that for $t \geq T_1$, we have $\| \bg_t \|_{2}^2 \leq e^{-\lambda_{\min} ( \bK^{T_1}  ) (t-T_1)} \| \bg_{T_1} \|_{2}^2$. By assumption we have $\| \bg_{T_1} \|_{2}^2 \leq K$ for a constant $K$.\footnote{This is since $\|h_*\|_{\infty} \leq K$ and $\|\fNN(\cdot;\orho_t)\|_{\infty} \leq \E[|a|\|\sigma\|_{\infty}] \leq K$, because we choose initialization with $|a| \leq 1$. So $\|\bg_{T_1}\|_{2} \leq \sqrt{2^P}\|\bg_{T_1}\|_{\infty} \leq \sqrt{2^P}(2K)$.} We deduce that, if we prove that $\lambda_{\min} (\bK^{T_1}) > c$ for some constant $c > 0$, then it is sufficient to consider $T_2 = T_1+ \log(K/\eps)/c$, to guarantee that $R (\orho_{T_2}) \leq \eps$. This would conclude the proof of strong $O(d)$-SGD-learnability.

\paragraph*{Lower bound on $\lambda_{\min} (\bK^{T_1})$:} It only remains to lower-bound $\lambda_{\min}(\bK^{T_1})$. For this we use the structure on $\bu^{T_1}$ that we prove holds in Phase 1. For all $S \subseteq [P]$, denote $$\nu_S (T_1) = \prod_{k \in S} \nu_k (T_1) \mbox{ and } \beta (S) = \sum_{k \in S} 2^{k-1}.$$ From Lemma \ref{lem:fourier_coeff_approx} (see next section), there exists a constant $C>0$ depending only on $P$ (and independent of $T_1$) such that for any $S \subseteq [P]$,
\[
\Big\vert \E_{\bz} \big[ \chi_S (\bz) \sigma ( \< \bu^{T_1}(a) , \bz \>) \big] - m_{|S|} \nu_S (T_1)  a^{\beta(S)} \Big\vert \leq C m_{|S|} \nu_S (T_1) T_1 \, .
\]
Denote $D_S = m_{|S|} \nu_S (T_1)$ and $\bD = \diag( (D_S)_{S\subseteq [P]} )$. We have $$|K^{T_1} (S,S') - D_S D_{S'} \E_a [ a^{\beta(S) + \beta(S')} ] | \leq C D_{S} D_{S'} T_1\, .$$
Introduce $\bM = ( \E_{a \sim \mu_a} [ a^{\beta(S) + \beta(S')} ] )_{S,S' \subseteq [P]}$, then we have
\[
\bK^{T_1} = \bD ( \bM + \bDelta ) \bD \, ,
\]
where $\| \bDelta \|_{\op} \leq C T_1 P$.

Note that $\beta(S)$ takes value $\{0, \ldots , 2^{P}-1\}$, and $\bM$ is the Gram matrix of the monomials $[1,X , \ldots , X^{2^P-1}]$ in $L^2 ( [+1,-1], \Unif)$, which are linearly independent. We deduce that $\lambda_{\min} (\bM)$ is bounded away from $0$ (independent of $T_1$). We can therefore take $T_1 \leq  \lambda_{\min} (\bM)/(2P) $, so that $ \lambda_{\min} ( \bM + \bDelta ) \geq  \lambda_{\min} ( \bM )/2$, and $\lambda_{\min} (\bK^{T_1}) \geq \{ \min_S D_S^2\} \lambda_{\min} ( \bM  )/2  > 0$.

\subsection{Approximating the $u_i^t$ with polynomials}

First, we have the following simple bound on $\| \bu^t \|_1$:

\begin{claim}\label{claim:bound_ut_l1}
There exists a constant $C$ depending on $K,P$ such that $\| \bu^t \|_1 \leq Ct$. 
\end{claim}

\begin{proof}[Proof of Claim \ref{claim:bound_ut_l1}]
By Assumptions $\rA0$ and $\rA1$, we have $\| \sigma'\|_{\infty}, \| h_* \|_{\infty} \leq K$ and also $\| \hf_{\NN} ( \cdot ; \orho_t) \|_{\infty} \leq \int |a| \| \sigma \|_{\infty } \de \mu_a \leq K$. Combining these bounds, we get for $t \leq T_1$:
\[
\Big\vert \frac{\de}{\de t} u_k^t \Big\vert = \big\vert a \E_{\bz} \big[ g_t(\bz) \sigma ' ( \< \bu^t , \bz \> ) z_k \big] \big\vert  \leq K \,
\]
and therefore $\frac{\de}{\de t} \| \bu^t \|_1 \leq P K$. Recalling, $\bu^0 = \bzero$, we conclude $\| \bu^t \|_1 \leq KP t $.
\end{proof}

The following lemma give the leading order in $t$ approximation of the Fourier coefficients of $\sigma ( \< \bu^t , \bz\>)$:

\begin{lemma}\label{lem:fourier_coeff_approx}
There exists a constant $c>0$ that depend on $\eta,K,P$ such that for any $t \leq c$, $S \subseteq [P]$ and $i \in \{0,1\}$,
\begin{equation}
 \E_{\bz} [ \chi_S (\bz) \sigma^{(i)} ( \< \bu^t , \bz \> ) ]  = m_{|S|+i} \Big( \prod_{k \in S } u_k^t \Big)\cdot (1 + O(t) ) + O (t^L)\, .
\end{equation}
\end{lemma}

\begin{proof}[Proof of Lemma \ref{lem:fourier_coeff_approx}]
From Claim \ref{claim:bound_ut_l1}, we can choose $c$ sufficiently small such that $\| \bu^t \|_1 < \eta$, and $| \< \bz , \bu^t \> | < \eta$. We can therefore use the polynomial approximation Eq.~\eqref{eq:poly_approx_sigma} of $\sigma^{(i)}$:
\begin{equation}\label{eq:four_deco}
    \begin{aligned}
     \E_{\bz} [ \chi_S (\bz) \sigma ( \< \bu^t , \bz \> ) ] = &~ \sum_{r=0}^{L-1} \frac{m_r}{r!} \E_{\bz} [ \chi_S (\bz)   \< \bu^t , \bz \>^r ] + O(t^L) \, .
    \end{aligned}
\end{equation}
Note that $\E_{\bz} [ \chi_S (\bz)   \< \bu^t , \bz \>^r ] =0$ for $r < |S|$, $\E_{\bz} [ \chi_S (\bz)   \< \bu^t , \bz \>^r ] =r! \prod_{k \in S} u^t_k$ for $r = |S|$, and for $|S| = l < r$ (such that $r - l = 2s$)
\[
\begin{aligned}
\big\vert \E_{\bz} [ \chi_S (\bz)   \< \bu^t , \bz \>^r ] \big\vert  \leq&~ \frac{r!}{(2s)!} \cdot \Big(\prod_{k \in S} u^t_k\Big) \frac{(2s)!}{s!2^s} \sum_{i_1, \ldots , i_s \in [P]} (u_{i_1}^t)^2 \cdots (u_{i_s}^t )^2 \\
=&~ \frac{r!}{s! 2^s}  \cdot \Big(\prod_{k \in S} u^t_k\Big) \cdot \| \bu^t \|_2^{2s} = \Big(\prod_{k \in S} u^t_k\Big) \cdot O(t) \, .
\end{aligned}
\]
Injecting these bounds in Eq.~\eqref{eq:four_deco} yields the result.
\end{proof}

We can now prove the main structural result on the $\bu^t$, on which the rest of the proof relies.

\begin{proposition}\label{prop:staircase_poly_approx}
Assume $L > 2^{P-1}$ and denote
\begin{equation}
    \hu_k^t (a) = 2^{1 - 2^{k-1}} (at)^{2^{k-1}} \cdot \Big\{ \prod_{i \in [k]} (\alpha_{\{1,\ldots,i\}} m_i )^{2^{(k-1-i) \vee 0}} \Big\} \, .
\end{equation}
There exists constants $c,C>0$ depending on $\eta,K,P$, such that for all $t \leq c$ and $k \in [P]$, $| u_k^t (a) - \hu_k^t (a)  | \leq C t^{2^{k-1} +1}$.
\end{proposition}

\begin{proof}[Proof of Proposition \ref{prop:staircase_poly_approx}]
Denote $\hbu^t = (\hu_k^t)_{k \in [P]}$. Notice that 
\[
\frac{\de}{\de t} \hu_k^t (a) =  a \alpha_{\{1, \ldots ,k\}} m_k \prod_{j<k} \hu_j^t (a) \, .
\]
Denote $\Delta_k^t = \sup_{s \in [0,t]} \sup_{a \in [-1,+1]} |u_k^t(a) - \hu_k^t (a)|$. By Gr\"onwall's lemma, it is sufficient to show that $ \frac{\de}{\de t} \Delta_k^t \leq K (t^{2^{k-1}} + \Delta_k^t)$ for some constant $K>0$. We will consider $c$ sufficiently small to apply Lemma \ref{lem:fourier_coeff_approx}.

We recall the evolution equations: 
\[
\begin{aligned}
\frac{\de}{\de t} u_k^t = a\E_{\bz} [h_* (\bz) z_k \sigma ' ( \< \bu^t, \bz \>)] - a\E_{\bz} [\hf_{\NN} ( \bz ; \orho_t) z_k \sigma ' ( \< \bu^t, \bz \>)] \, .
\end{aligned}
\]
Let us first show that $|\E_{\bz} [\hf_{\NN} ( \bz) z_k \sigma ' ( \< \bu^t, \bz \>)]  | \leq \Delta_k^t + O(t^L)$. Denote the Fourier coefficients $\hf_{\NN} (S ; \orho_t) = \E_{\bz} [ \chi_S (\bz) \hf_{\NN} (\bz ; \orho_t)]$. If $k \in S$, then
\begin{equation}\label{eq:S_fNN_kin}
\begin{aligned}
 \big\vert \hf_{\NN} (S ; \orho_t) \big\vert  \leq &~ \int \big\vert a \E_{\bz} [ \chi_S (\bz) \sigma ( \< \bu^t (a) , \bz \> ) ] \big\vert \mu_a (\de a ) \\
 \leq &~ \int \Big\vert  m_{|S|} \prod_{ i \in S} u_i^t (a) \cdot (1 + O(t)) \Big\vert \mu_a (\de a )  + O(t^L)
 \leq  K \Delta_k^t + O(t^L) \, ,
 \end{aligned}
\end{equation}
where we used Lemma \ref{lem:fourier_coeff_approx}. Furthermore, note that $|\hf_{\NN} (S ; \orho_t) | \leq K$ for any $S \subseteq [P]$. By expanding $\hf_{\NN}$ in the Fourier basis, we get
\[
\begin{aligned}
|\E_{\bz} [\hf_{\NN} ( \bz) z_k \sigma ' ( \< \bu^t, \bz \>)]  | \leq&~ \sum_{S \subseteq [P]} \big\vert  \hf_{\NN} (S ; \orho_t)  \E_{\bz} [ \chi_{S \oplus k} (\bz) \sigma ' ( \< \bu^t , \bz \> )] \big\vert \\
\leq &~ K \Delta_k^t + O(t^L) + \sum_{S \subseteq [P], k\not\in S} K \big\vert \E_{\bz} [ \chi_{S \cup k} (\bz) \sigma ' ( \< \bu^t , \bz \> )] \big\vert \\
\leq &~ K \Delta_k^t + O(t^L) \, ,
\end{aligned}
\]
where we used Eq.~\eqref{eq:S_fNN_kin} in the second line and Lemma \ref{lem:fourier_coeff_approx} in the third line. We see therefore that 
\[
\Big\vert \frac{\de}{\de t } ( u_k^t - \hu_k^t ) \Big\vert \leq \Big\vert a \E_{\bz} [h_* (\bz) z_k \sigma ' ( \< \bu^t, \bz \>)] - a \alpha_{\{1, \ldots ,k\}} m_k \prod_{j<k} \hu_j^t (a) \Big\vert + K \Delta_k^t + O(t^L)\, .
\]
We can separate the first term into three contributions:
\[
\Big\vert \E_{\bz} [h_* (\bz) z_k \sigma ' ( \< \bu^t, \bz \>)] -  \alpha_{\{1, \ldots ,k\}} m_k \prod_{j<k} \hu_j^t (a) \Big\vert  \leq ({\rm I}) + ({\rm II}) + ({ \rm III}) \, ,
\]
where 
\[
\begin{aligned}
({\rm I}) =&~ \sum_{i < k } \Big\vert \alpha_{\{1 , \ldots , i\}} \E_{\bz} \big[ \chi_{\{1 , \ldots , i\} \cup \{ k\}} (\bz) \sigma ' ( \< \bu^t, \bz \>) \big] \Big\vert \leq K \Delta_k^t + O(t^L) \, , \\
({\rm II}) = &~  \Big\vert  \alpha_{\{1 , \ldots , k\}} \E_{\bz} \big[ \chi_{\{1 , \ldots , k-1\} } (\bz) \sigma ' ( \< \bu^t, \bz \>) \big]  -  \alpha_{\{1 , \ldots , k\}} m_k \prod_{j<k} \hu_j^t (a) \Big\vert \\
\leq &~ K \Big\vert m_k \prod_{i \in [k-1]} u_i^t  \cdot (1+ O(t)) - \prod_{i \in [k-1]} \hu_i^t \Big\vert + O(t^L) \\
\leq &~ K \sum_{i \in [k-1]}  \Delta_i^t \prod_{j \in [k-1], j\neq i} (\Delta_j^t + | \hu_j^t | ) +O(t^L) \, , \\
({\rm III}) = &~ \sum_{i >k} \Big\vert \alpha_{\{1 , \ldots , i\}} \E_{\bz} \big[ \chi_{\{1 , \ldots , i\} \setminus \{ k\}} (\bz) \sigma ' ( \< \bu^t, \bz \>) \big] \Big\vert \\
\leq &~ K \prod_{j \in [k-1]} (\Delta_j^t + | \hu_j^t | )  \sum_{i>k+1} | u_i^t |   +O(t^L) \leq K t \prod_{i \in [k-1]} (\Delta_i^t + | \hu_i^t | ) + O(t^L) \, ,
\end{aligned}
\]
where we used in the last line that $\| \bu^t \|_1 \leq Ct$ from Claim \ref{claim:bound_ut_l1}. In particular, notice that for any $i < k$, $\frac{\de}{\de t} \Delta_k^t \leq \Delta_i^t$. We can therefore prove recursively that $| \Delta_k^t | \leq O (t^{2^{k-1}+1})$ by noting that 1) $\Delta_1^t \leq K t^2 $; 2) $|\hu^t_k | = \Theta ( t^{2^{k-1}})$ and $\prod_{j<k} |\hu^t_j | \leq K t^{2^{k-1} - 1}$; and 3) $t^L = O( t^{2^{k-1}})$ for any $k \in [P]$, and do not contribute to the leading terms.
\end{proof}

\clearpage

\section{Generic MSP functions are strongly $O(d)$-SGD-learnable: Proof of Theorem~\ref{thm:discrete-msp} (discrete-time regime)}\label{app:discrete-msp-proof}

In this appendix, we prove Theorem~\ref{thm:discrete-msp}, which states that generic functions with MSP structure are strongly SGD-learnable in the $O(d)$-scaling. While the proof for vanilla staircases in Appendix \ref{ssec:vanillaproof} is done in the continuous-time regime, we use here the discrete-time regime as defined in Appendix \ref{app:strong-discrete}, with $O(1)$-steps of size $\eta = \Theta (1)$. Furthermore, we will consider the activation function to be a degree-$L$ polynomial, with $L$ sufficiently large. In Appendix \ref{ssec:genericmspproof}, we provide a more general proof of this result for smooth (non-polynomial) activations (see Theorem~\ref{thm:genericmspsuffrestated}) and using the continuous-time regime, with one technical caveat: the activation function needs to be perturbed at some point during training (the result holds almost surely over this perturbation, see Appendix \ref{ssec:discussionperturb} for a discussion on this technical caveat). 

Recall the definition of an MSP set structure.

\begin{definition}
We say that $\cS = \{S_1,\ldots,S_m\}$ is a \emph{Merged-Staircase Property (MSP) set structure} on the variables $z_1,\ldots,z_P$ if the sets are (without loss of generality) ordered so that for each $i \in [m]$, $|S_i \setminus \cup_{i' < i} S_{i'}| \leq 1$.
\end{definition}
Given an MSP set structure $\cS \subset 2^{[P]}$ and a function $h_* : \{+1,-1\}^P \to \R$, we say that $h_*$ has MSP structure $\cS$ if $h_*$ can be written as 
\[
h_* (\bz) = \sum_{S \in \cS} \alpha_S \chi_S (\bz) \, ,
\]
where $\alpha_S \in \R \sm \{0\}$ for all $S \in \cS$. In other words, $h_*$ has MSP structure $\cS$ if its nonzero Fourier coefficients are $\cS$.

Ideally, we would like prove that for any MSP set structure $\cS$, then any function $h_*$ with nonzero Fourier coefficients $\cS$ is strongly $O(d)$-SGD-learnable. However, there are degenerate examples of functions such as $h_*(\bz) = z_1 + z_2 + z_1z_3 + z_2 z_4$ which satisfy MSP structure but are not strongly $O(d)$-SGD-learnable (see Section~\ref{app:numerics}). Therefore, it is not possible to prove a result that holds for every MSP function. The existence of degenerate functions satisfying MSP also adds difficulty to the problem of showing that specific functions satisfying MSP are learnable.

Nevertheless, in this section we are able to show that for any MSP set structure $\cS$ there are very few degenerate functions $h_*$. In fact, almost all functions with MSP structure $\cS$ are non-degenerate and are strongly $O(d)$-SGD-learnable. 

More precisely, for any set structure $\cS \subseteq 2^{[P]}$, define the following measure over functions:

\begin{definition}[Definition~\ref{def:leb-msp} restated]
For any set structure $\cS \subseteq 2^{[P]}$ define the measure $\mu_{\cS}$ over functions $h_* : \{+1,-1\}^P \to \R$ induced by taking $h_*(\bz) = \sum_{S \subseteq [P]} \alpha_S \chi_S(\bx)$, where the Fourier coefficients satisfy $\alpha_S = 0$ if $S \not\in\cS$, and $(\alpha_S)_{S \in \cS}$ have Lebesgue measure on $\R^{|\cS|}$.
\end{definition}
For any MSP structure $\cS$, we prove that $h_*$ is almost surely strongly $O(d)$-SGD-learnable with respect to $\mu_{\cS}$:

\begin{theorem}[Theorem~\ref{thm:discrete-msp} restated]\label{thm:discrete-msp-restated}
For any MSP set structure $\cS \subseteq 2^{[P]}$, $h_*$ is strongly $O(d)$-SGD-learnable almost surely with respect to $\mu_{\cS}$, using activation function $\sigma(x) = (1+x)^L$ where $L = 2^{8P}$.
\end{theorem}

\begin{remark}
We note that although $\sigma(x) = (1+x)^L$ does not satisfy Assumption $\rA 0$, we can instead use an activation function such that $\sigma(x) = (1+x)^L$ in the interval $(-1,1)$, and $\sigma (x)$ is smoothly thresholded outside this interval. In the proof, we control the growth of the first-layer weights and the input of the activation remains $|x| \leq 1$, so such a thresholding does not impact training.
\end{remark}

We also prove the following variation on the theorem, which shows that we can take activation function that is a polynomial of degree $L \geq 2^{8P}$ with random coefficients. This proves that almost surely any polynomial activation will work, so it does not hold just for activation $(1+x)^L$:

\begin{theorem}\label{thm:discrete-msp-more-activations}
For any MSP structure $\cS \subseteq 2^{[P]}$, and any $L \geq 2^{8P}$, if we draw $\bm \sim \mathrm{Unif}[-1,1]^{\otimes L+1}$, then $h_*$ is strongly $O(d)$-SGD-learnable almost surely with respect to $\mu_{\cS}$, using activation function $\sigma(x) = \sum_{i=0}^{L} m_i x^i$.
\end{theorem}

\subsection{Outline of the proof}

\paragraph{Choice of hyperparameters}

We train in the discrete-time regime with $\Theta(1)$ steps of size $\Theta(1)$ and $\Theta(d)$ batch size. Recall from \eqref{eq:d-DF-PDE} (Theorem \ref{thm:discrete-equivalence}) that it is sufficient to show for any $\eps >0$, there exist hyperparameters satisfying $\rA0$-$\rA2$, $\rD3$ such that \eqref{eq:d-DF-PDE} reaches $\eps$-risk. We consider the following hyperparameters. 

\begin{itemize}
    \item We do not regularize. I.e., $\lambda^w= 0$, and $\lambda^a = \lambda^w = 0$, same as Section~\ref{ssec:vanillaproof}.
    \item We initialize the first layer to deterministically $0$ weights, and the second layer to uniform random weights. I.e., we take $\mu_a = \Unif ( [ +1, -1])$ and $\mu_W = \delta_0$. This is the same as in the vanilla staircase proof of Section~\ref{ssec:vanillaproof}. For the dimension-free dynamics, this corresponds to taking  $(\oa^0 , \obu^0 , \os^0 ) \sim \orho_0$ with $\oa^0 \sim \Unif ( [ +1, -1])$, $\obu^0 = \bzero$ and $\os^0 = 0$. In particular, $\os^k = 0$ during the whole dynamics, which lets us ignore it and allows for a simpler analysis.
    \item Our learning rate schedule has two phases, with learning rate given by parameter $\eta > 0$:
\begin{description}
\item[Phase 1:] For $k_1$ steps we train the first layer weights $\obu^k$ while keeping the second layer weights fixed $\oa^k = \oa^0$. We set $\eta^a_k = 0$ and $\eta^w_k = \eta$ for $k \in \{0,\ldots,k_1-1\}$. 

\item[Phase 2:] For $k_2$ steps we train the second layer weights $\oa^k$ while keeping the first layer weights fixed at $\obu^k = \obu^{k_1}$. We set $\eta^a_k = \eta$ and $\eta^w_k = 0$ for $k \in \{k_1,\ldots,k_2 - 1\}$.
\end{description}
\end{itemize}

We also take $\eta > 0$ to be a small enough constant, and $b = \Omega(d)$ for a large enough constant depending on $P,\eps,\eta$. For the first phase, we will train for $k_1 = P$ time steps, since this turns out to be sufficient to prove learnability. For the second phase, we train for $k_2 = \Theta(1)$ time steps, where $k_2$ is a constant depending on $\eta, \eps$, and $P$, to be determined later. We prove that \eqref{eq:d-DF-PDE} with such hyperparameters will reach $\eps$-risk, which, by the equivalence stated Theorem~\ref{thm:discrete-equivalence}, implies the strong SGD-learnability in $O(d)$-scaling.

\paragraph{Assumption on the activation}
We will assume that on the interval $(-1,1)$ our activation is given by a polynomial of degree at most $L$. I.e., for all $x \in (-1,1)$, we have $\sigma(x) = \sum_{i=0}^L \frac{m_i}{i!} x^i$ for $\bm = [m_0,\ldots,m_L] = [\sigma(0),\sigma^{(1)}(0),\ldots,\sigma^{(L)}(0)]$.

\subsubsection{Phase 2 (linear training)}\label{sssec:phase2-discrete-msp}

Let us first present the analysis of Phase 2. We train the second layer and keep the first layer weights fixed. This is kernel gradient descent with kernel $K^{k_1} : \{+1,-1\}^P \times \{+1,-1\}^P \to \R$ given by
\begin{align*}
K^{k_1}(\bz,\bz') = \E_{a \sim \mu_a}[\sigma(\<\bu^{k_1}(a),\bz\>)\sigma(\<\bu^{k_1}(a),\bz'\>)].
\end{align*}
So the residual $g_k(\bz) = h_*(\bz) - \fNN(\bz;\bar\rho_k)$, evolves, for any $k \in \{k_1,\ldots,k_2-1\}$, as:
\begin{align*}
g_{k+1}(\bz) = g_k(\bz) - \eta \E_{\bz'}[K^{k_1}(\bz,\bz')g_k(\bz')].
\end{align*}
The evolution of the risk is given by:
\begin{align*}
R(\bar\rho_{k+1}) &= \frac{1}{2} \E[g_{k+1}(\bz)^2] \\
&= R(\bar\rho_k) - \eta \E_{\bz,\bz'}[g_{k}(\bz)K^{k_1}(\bz,\bz')g_{k}(\bz')] \\
&\qquad\qquad+ \frac{\eta^2}{2} \E_{\bz,\bz',\bz''}[K^{k_1}(\bz,\bz')K^{k_1}(\bz,\bz'')g_k(\bz')g_k(\bz'')] \\
&\leq R(\bar\rho_k) - \eta \Big(1 - \frac{\eta \lambda_{\mathrm{max}}(\bK^{k_1})}{2^{P+1}} \Big) \E_{\bz,\bz'}[g_{k}(\bz)K^{k_1}(\bz,\bz')g_{k}(\bz')],
\end{align*}
where $\bK^{k_1} = (K^{k_1}(\bz,\bz'))_{\bz,\bz'}$ is the $2^P \times 2^P$ kernel matrix. Note that $\lambda_{\mathrm{max}}(\bK^{k_1}) \leq \|\bK^{k_1}\|_F \leq 2^P \|\sigma\|_{\infty}^2 \leq 2^P K^2$. So if we take any learning rate $\eta \leq 1 / K^2$, we have
\begin{align*}
R(\bar\rho_{k+1}) &\leq R(\bar\rho_k) - \frac{\eta}{2} \E_{\bz,\bz'}[g_{k}(\bz)K^{k_1}(\bz,\bz')g_{k}(\bz')] \\
&\leq \Big(1-\frac{\eta \lambda_{\mathrm{min}}(\bK^{k_1})}{2^P} \Big)R(\bar\rho_{k}).
\end{align*}

Finally, note that $R(\bar\rho_{k_1}) \leq \frac{1}{2}(\|\sigma\|_{\infty} + \|h_*\|_{\infty})^2 \leq 2K^2$, so if we take any $k_2 \geq k_1 + \log(\eps / 2K^2) / \log(1 - \eta \lambda_{\mathrm{min}}(\bK^{k_1}) / 2^P)$, we ensure that $R(\bar\rho_{k_2}) \leq \eps$. It remains only to show that $\lambda_{\mathrm{min}}(\bK^{k_1}) \geq c$, for a constant $c > 0$ depending only on $\eta,h_*,P,k_1$, and $\bm$.

\subsubsection{Phase 1 (nonlinear training)}

Now let us show how to analyze Phase 1, and in particular how to prove that $\lambda_{\mathrm{min}}(\bK^{k_1})$ is bounded away from $0$.

\paragraph{Writing the weight evolution with a polynomial} First, we show that if we train for a constant number $k_1$ of steps, then we can write the weights obtained by the dimension-free dynamics as a constant-degree polynomial in the second-layer weights. This is because the activation is a polynomial in the interval $(-1,1)$, and the weights of the first layer do not grow enough to leave this interval.
\begin{lemma}[Training dynamics are given by a polynomial]\label{lem:pki-baru-discrete}
Let $\bxi = (\xi_{S,k})_{S \subseteq [P], 0 \leq k \leq k_1 -1} \in \R^{2^P k_1}$, $\zeta \in \R$, and $\brho \in \R^{L+1}$ be variables.

For each $i \in [P]$ define $p_{0,i}(\zeta,\bxi,\brho) \equiv 0$. For each $0 \leq k \leq k_1-1$, define $p_{k+1,i}(\zeta,\bxi,\brho)$ with the recurrence relation:
\begin{equation*}
\begin{aligned}
p_{k+1,i}(\zeta,\bxi,\brho) =&~ p_{k,i}(\zeta,\bxi,\brho) + \zeta \rho_1 \xi_{\{i\},k} \\
&~ + \zeta \sum_{r=1}^{L-1} \frac{\rho_{r+1}}{r!} \sum_{(i_1,\ldots,i_r) \in [P]^r} \xi_{\{i\} \oplus \{i_1\} \oplus \dots \{i_r\},k} \prod_{l=1}^r p_{k,i_l}(\zeta,\bxi,\brho)
\end{aligned}
\end{equation*}
There is a constant $c > 0$ depending only on $k_1,P,K$, such that for any $0 < \eta < c$,
\begin{align*}
\bar u_{i}^k (a)  = p_{k,i}(\eta a, \bbeta, \bm)
\end{align*}
where $\bbeta = (\beta_{S,k})_{S \subseteq [P], 0 \leq k \leq k_1 - 1}$ has values given by, for all $S \subseteq [P]$,
\begin{align*}
\beta_{S,k} &= \E[(-\fNN(\bz;\bar\rho_k) + h_*(\bz))\chi_S(\bz)].
\end{align*}
\end{lemma}

Because of the term $\fNN(\cdot;\bar\rho_k)$, which evolves nonlinearly, this is nontrivial to directly analyze. However, if the step size $\eta$ is taken small enough, then the interaction term $\fNN(\cdot; \bar\rho_k)$ is small, of order $O(\eta k)$, and we show that it can be ignored. Formally, we define the simplified dynamics $\hatbu^k(a)$ for each $0 \leq k \leq k_1$ by letting $\hatbu^0(a) = \bzero$ and inductively setting for each $k \in \{0,\ldots,k_1-1\}$,
\begin{align*}
\hatbu^{k+1}(a) = \hatbu^k(a) - \eta a_j^0 \EE_{\bx}[-h_*(\bz) \sigma'(\<\hatbu^k(a),\bz\>)\bz].
\end{align*}
This differs from the definition of the dynamics for $\bar\bu^k$ in that we have dropped the $\fNN(\bz;\rho_k)$ term in the update equation. By a similar argument, we may show:
\begin{lemma}[Simplified training dynamics are given by a polynomial]\label{lem:pki-hatu-discrete}
There is a constant $c > 0$ depending only on $k_1,P,K$, such that for any $0 < \eta < c$, any $i \in [P]$ and any $0 \leq k \leq k_1$, we have
$$\hat u_{i}^k(a) = p_{k,i}(\eta a, \balpha,\bm),$$
where we abuse notation (since $\balpha = (\alpha_S)_{S \subseteq [P]}$ otherwise) and let $\balpha = (\alpha_{S,k})_{S \subseteq [P], 0 \leq k \leq k_1 - 1}$ be given by
\begin{align*}
\alpha_{S,k} &= \alpha_S = \E[h_*(\bz) \chi_S(\bz)]
\end{align*}
\end{lemma}

We now show that the simplified dynamics $\hat\bu^k$ is a good enough approximation to $\bar\bu^k$, and it suffices to analyze $\hat\bu^k$.

\paragraph{Reducing to analyzing simplified dynamics} We lower-bound $\lambda_{\mathrm{min}}(\bK^{k_1})$ in terms of the determinant of a certain random matrix. Let $\bzeta = [\zeta_1,\ldots,\zeta_{2^P}]$ be a vector of $2^P$ variables. Define $\bM = \bM(\bzeta,\bxi,\brho) \in \R^{2^P \times 2^P}$ to be the matrix indexed by $\bz \in \{+1,-1\}^P$ and $j \in [2^P]$ with entries
\begin{align}\label{eq:bM-def-MSP}
M_{\bz,j}(\bzeta,\bxi,\brho) = \sum_{r=0}^{L} \frac{\rho_r}{r!} \left(\sum_{i=1}^P z_i p_{k,i}(\zeta_j,\bxi,\brho)\right)^r.
\end{align}

This matrix is motivated by the following fact:
\begin{lemma}\label{lem:plug-into-M}
There is a constant $c > 0$ depending only on $k_1,P,K$, such that for any $0 < \eta < c$, and any $\ba = [a_1,\ldots,a_{2^P}] \in [-1,1]^{2^P}$, we have
\begin{align*}
M_{\bz,j}(\eta \ba, \bbeta, \bm) &= \sigma(\<\barbu^{k_1}(a_j), \bz\>) \\
M_{\bz,j}(\eta \ba, \balpha, \bm) &= \sigma(\<\hatbu^{k_1}(a_j), \bz\>)
\end{align*}
\end{lemma}

Using this we can show:
\begin{lemma}\label{lem:k-min-from-det-m}
There is a constant $c > 0$ depending on $K,P$ such that for any $0 < \eta < c$,
$$\lambda_{\mathrm{min}}(\bK^{k_1}) \geq c \E_{\ba \sim \mu_a^{\otimes 2^P}}[\det(\bM(\eta \ba, \bbeta, \bm))^2].$$
\end{lemma}

On the other hand, we can prove a lower-bound on $\E[\det(\bM(\eta \ba, \bbeta, \bm))^2]$ simply by lower-bounding the sum of magnitudes of coefficients of $\det(\bM(\bzeta, \balpha, \bm))$ when viewed as a polynomial in $\bzeta$. This is because of (a) the fact that $\det(\bM(\bzeta, \balpha, \bm))$ and $\det(\bM(\bzeta, \bbeta, \bm))$ have coefficients in $\bzeta$ that are $O(\eta)$-close for $\eta$ small, and (b) the fact that polynomials anti-concentrate over random inputs:
\begin{lemma}\label{lem:poly-anticoncentration-and-bar-hat-comparison}
There is $D > 0$ depending only on $P,k_1,L$, and there are $C,c > 0$ depending only on $P,k_1,K,L$ such that if we write $$\det(\bM(\bzeta, \balpha, \bm)) = \sum_{\bgamma \in \{0,\ldots,D\}^{2^P}} h_{\bgamma} \bzeta^{\bgamma},$$ then
$$\E_{\ba \sim \mu_a^{\otimes 2^P}}[\det(\bM(\eta \ba, \bbeta, \bm))^2] \geq c \sum_{\bgamma \in \{0,\ldots,D\}^{2^P}} \eta^{2\|\bgamma\|_1}\max(0,|h_{\bgamma}| - C\eta)^2.$$
\end{lemma}

Combining the above lemmas, it holds that if $\det(\bM(\bzeta,\balpha,\bm))$ is a nonzero polynomial in $\bzeta$, then $h_*$ is strongly-$O(d)$ learnable:

\begin{lemma}\label{lem:intermediate-reduction}
Suppose that $\det(\bM(\bzeta,\balpha,\bm)) \not\equiv 0$ as a polynomial in $\bzeta$. Then the function $h_*(\bz) = \sum_{S \subseteq [P]} \alpha_S \chi_S(\bz)$ is strongly $O(d)$-SGD-learnable with any activation function $\sigma$ that is equal to $\sigma(x) = \sum_{i=0}^L \frac{m_i}{i!}x^i$ on the interval $x \in (-1,1)$.
\end{lemma}
\begin{proof}
Let $k_1$ be a constant depending on $P$, and let $C,c > 0$ be constants depending on $k_1,P,K,L$ such that Lemmas~\ref{lem:k-min-from-det-m} and \ref{lem:poly-anticoncentration-and-bar-hat-comparison} hold. Then taking any learning rate $$0 < \eta < \min \Big(c,\max_{\bgamma \in \{0,\ldots,D\}^{2^P}} |h_{\bgamma}| / (2C) \Big),$$ we have
\begin{align}\label{eq:Kmin-bound-smoothed-msp}
\lambda_{\mathrm{min}}(\bK^{k_1}) \geq c^2 \eta^{2^P D} |h_{\bgamma}|^2 / 4 > 0,
\end{align}
which is a nonnegative constant that does not depend on $d$. So by the analysis of Phase 2 in Section~\ref{sssec:phase2-discrete-msp}, we can set $k_2$ to be a large enough constant that $R(\bar\rho_{k_2}) \leq \eps$. By Theorem~\ref{thm:discrete-equivalence} (which gives the equivalence between \eqref{eq:d-DF-PDE} and strong $O(d)$-SGD-learnability in the discrete-time setting), this implies strong $O(d)$-SGD-learnability.
\end{proof}

\paragraph{Analyzing simplified dynamics}
By the above arguments, the problem has been reduced to proving that $\det(\bM(\bzeta,\balpha,\bm)) \not\equiv 0$ as a polynomial in $\bzeta$. In other words, by Lemma~\ref{lem:plug-into-M}, this means that it suffices to analyze the simplified dynamics $\hat\bu^k$.

We wish to prove that $\det(\bM(\bzeta,\balpha,\bm)) \not\equiv 0$ almost surely over the choice of $\balpha$. Since we take $h_*(\bz) = \sum_{S \in \cS} \alpha_S \chi_S(\bz)$ to be a generic function satisfying MSP, we could hope that it would be sufficient to prove that $\det(\bM(\bzeta,\bxi,\bm)) \not\equiv 0$ as a polynomial over $\bzeta$ and $\bxi$. However, there is an important technical subtlety. Although $\balpha = (\balpha_S)_{S \in \cS}$ can be chosen to be generic, the vector $(\balpha_{S,k})_{S \subseteq [P], k \in \{0,\ldots,k_1-1\}}$ has the constraints that $\alpha_{S,k} = \alpha_S$ for all $S,k$, and that $\alpha_{S} = 0$ for all $S \not\in \cS$. To take this into account, let $\bphi = (\phi_S)_{S \in \cS}$ be a vector of variables and define the following matrix $\bN(\bzeta,\bphi,\brho) \in \R^{2^P \times 2^P}$, indexed by $\bz \in \{+1,-1\}^P$ and $j \in [2^P]$:
\begin{align}\label{eq:N-construction-discrete-msp}
N_{\bz,j}(\bzeta,\bphi,\brho) = M_{\bz,j}(\bzeta,\bxi,\brho) \mid_{\xi_{S,k} = 0 \mbox{ for all } S \not\in \cS, \mbox{ and } \xi_{S,k} = \phi_S \mbox{ for all } S \in \cS}.
\end{align}
The matrix $\bN$ differs from $\bM$ only in that we have changed the variables from $(\xi_{S,k})_{S,k}$ to variables $(\phi_S)_{S \in \cS}$, effectively incorporating the constraints on $\balpha$. This is helpful, because suppose that we can prove that
\begin{align}\label{eq:suffices-poly-nonzero}
\det(\bN(\bzeta,\bphi,\bm)) \not\equiv 0 \mbox{  as a polynomial in $\bzeta$ and $\bphi$}.
\end{align}
Then almost surely over the Lebesgue measure on $(\alpha_S)_{S \in \cS}$, we have that
$\det(\bN(\bzeta,\balpha,\bm)) \not\equiv 0$ as a polynomial over $\bzeta$. And indeed, $\det(\bN(\bzeta,\balpha,\bm)) \equiv \det(\bM(\bzeta,\balpha,\bm)) \not\equiv 0$, which is what we wanted to show. So it suffices to prove \eqref{eq:suffices-poly-nonzero}.

We prove \eqref{eq:suffices-poly-nonzero} by analyzing the recurrence relations for $p_{k,i}$ to show that to first-order the polynomials $p_{k_1,i}$ are distinct for all $i \in [P]$, and then leveraging the algebraic result of \cite{newman1979waring} that large powers of distinct polynomials are linearly independent. We show:

\begin{lemma}\label{lem:nonzero-poly-discrete-msp}
Suppose that $L \geq 2^{8P}$ and let $m_i = i! \binom{L}{i}$ for all $0 \leq i \leq L$, corresponding to activation function $\sigma(x) = (1+x)^L$. Also let $k_1 = P$. Then $\det(\bN(\bzeta,\bphi,\bm)) \not\equiv 0$ (i.e., \eqref{eq:suffices-poly-nonzero} holds).
\end{lemma}

This also yields the immediate corollary:

\begin{corollary}\label{cor:nonzero-poly-discrete-msp}
Suppose that $L \geq 2^{8P}$, and let $\bm \sim \mathrm{Unif}[-1,1]^{\otimes L+1}$, corresponding to a random polynomial activation function. Then $\det(\bN(\bzeta,\bphi,\bm)) \not \equiv 0$ (i.e., \eqref{eq:suffices-poly-nonzero} holds) almost surely over $\bm$.
\end{corollary}
\begin{proof}
Lemma~\ref{lem:nonzero-poly-discrete-msp} implies that $\det(\bN(\bzeta,\bphi,\brho))$ is a nonzero polynomial in $\bzeta,\bphi,\brho$. Since we choose $\bm \sim \mathrm{Unif}[-1,1]^{\otimes L+1}$, this means that $\det(\bN(\bzeta,\bphi,\bm)) \not\equiv 0$ almost surely over the choice of $\bm$.
\end{proof}

This allows us to prove Theorems~\ref{thm:discrete-msp} and \ref{thm:discrete-msp-more-activations}.

\begin{proof}[Proof of Theorem~\ref{thm:discrete-msp}]
Taking $m_i = i! \binom{L}{i}$ corresponds to activation function $\sigma(x) = (1+x)^L$. By Lemma~\ref{lem:nonzero-poly-discrete-msp}, we have $\det(\bN(\bzeta,\balpha,\bm)) \not\equiv 0$ almost surely over $\balpha$ with respect to the Lebesgue measure. So by Lemma~\ref{lem:intermediate-reduction}, $h_*(\bz) = \sum_{S \in \cS} \alpha_S \chi_S(\bz)$ is strongly $O(d)$-SGD-learnable with activation $\sigma(x) = (1+x)^L$, almost surely over $h_*$ with respect to $\mu_{\cS}$.
\end{proof}

\begin{proof}[Proof of Theorem~\ref{thm:discrete-msp-more-activations}]
The argument is the same, except using Corollary~\ref{cor:nonzero-poly-discrete-msp}.
\end{proof}

\subsection{Proof of Lemmas~\ref{lem:pki-baru-discrete}, \ref{lem:pki-hatu-discrete}, and \ref{lem:plug-into-M}}

We show that if the learning rate $\eta$ is small then for $0 \leq k \leq k_1$ the weights of $\barbu^k$ and $\hatbu^k$ remain small enough that the activation $\sigma$ only ever has inputs in the range $(-1,1)$, meaning that we can treat the activation $\sigma$ as exactly given by the polynomial $\sum_{i=0}^L \frac{m_i}{i!} x^i$.
\begin{claim}\label{claim:ubarwtbound}
For any time step $0 \leq k \leq k_1$ any $j \in [N]$, and any learning rate $\eta < 1/ (4 K^2  P k )$, and any $a \in [-1,1]$ we have $$\|\barbu^k(a)\|_1, \|\hatbu^k(a)\|_1 \leq 2\eta K^2 P k \leq 1/2.$$
\end{claim}
\begin{proof}
The proof is by induction on $k$. The base case is clear since $\barbu^0 = \hatbu^0 = \bzero$. For the inductive step, $\fNN(\bz;\bar\rho_k) \leq \EE_{a \sim \mu_a}[|a| |\sigma(\<\barbu^k,\bz\>)|] \leq \|\sigma\|_{\infty} \leq K$, since $a \sim \mathrm{Unif}[-1,1]$.
Therefore
\begin{align*}
\|\barbu^{k+1}(a)\|_1 &\leq \|\barbu^k(a)\|_1 + \eta \|\EE_{\bx}[(\fNN(\bz;\bar\rho_k) + h_*(\bx))a \sigma'(\<\barbu^k, \bz\>) \bz]\|_1 \\
&\leq \|\barbu^k(a)\|_1 + 2\eta K^2 P \leq 2\eta K^2 P k.
\end{align*}
The bound for $\|\hat\bu^k(a)\|_1$ is similar. 
\end{proof}

This allows us to prove Lemmas~\ref{lem:pki-baru-discrete} and \ref{lem:pki-hatu-discrete}.

\begin{proof}[Proof of Lemmas~\ref{lem:pki-baru-discrete} and \ref{lem:pki-hatu-discrete}]
Let $\zeta \in \R$, $\bxi = (\xi_{S,k})_{S \subseteq [P], k \in \{0,\ldots,k_1-1\}}$, and $\brho \in \R^{L+1}$ be variables. Define $s_0,\ldots,s_{k_1-1} : \{+1,-1\}^P \to \R$ to be $s_k(\bz) = \sum_{S \subseteq [P]} \xi_{S,k} \chi_S(\bz)$. Consider the recurrence relation $\bnu^k \in \R^P$, where we initialize $\bnu^0 = \bzero$ and, for $0 \leq k \leq k_1-1$,
\begin{align}\label{eq:genericrecrelation}
\bnu^{k+1} = \bnu^k + \zeta \E_{\bz}\Big[s_k(\bz) \sum_{r=0}^{L-1} \frac{\rho_{r+1}}{r!}\<\bnu^k,\bz\>^r \bz \Big].
\end{align}
Substituting in $\zeta = \eta a$ and $\brho = \bm$, this recurrence relation is satisfied by $\bar\bu^k(a)$ with $s_k(\bz) = -\fNN(\bz;\bar\rho_k) + h_*(\bz) = \sum_{S} \beta_{S,k} \chi_S(\bz)$ and by $\hat\bu^k(a)$ with $s_k(\bz) = h_*(\bz) = \sum_{S} \alpha_S \chi_S(\bz)$. This is because $|\<\barbu^k,\bz\>|, |\<\hatbu^k,\bz\>| \leq 1/2 < 1$ by Claim~\ref{claim:ubarwtbound} and in the interval $(-1,1)$ $\sigma(x) = \sum_{r=0}^L \frac{m_r}{r!} x^r$. 

It remains to show that
$$\nu_i^k = p_{k,i}(\zeta,\bxi,\brho).$$
The proof is by induction on $k$. For $k = 0$, it is true that $p_{0,i}(\zeta, \bxi) = 0 = \nu_i^0$. For the inductive step, notice that for any $r \geq 1$ and $i \in [d]$, we can write
\begin{align*}
\E_{\bz}[s_k(\bz)\<\bnu^k,\bz\>^r z_i] &= \E_{\bz} \Big[s_k(\bz) z_i \sum_{(i_1,\ldots,i_r) \in [P]^r} \prod_{l=1}^r \nu_{i_{l}}^k z_{i_{l}} \Big] \\
&= \sum_{(i_1,\ldots,i_r) \in [P]^r}
\E_{\bz} \Big[s_k(\bz)\chi_i(\bz)\prod_{l=1}^r \chi_{i_l}(\bz)\Big] \prod_{l=1}^r p_{k,i_l}(\zeta,\bxi,\brho) \\
&= \sum_{(i_1,\ldots,i_r) \in [P]^r} \xi_{\{i\} \oplus \{i_1\} \oplus \dots \oplus \{i_r\},k} \prod_{l=1}^r p_{k,i_l}(\zeta,\bxi,\brho),
\end{align*}
and $\E_{\bz}[s_k(\bz) \<\bnu^k,\bz\>^0 z_i] = \E_{\bz}[s_k(\bz) z_i] = \xi_{\{i\},k}$.
The inductive step follows by linearity of expectation.
\end{proof}

Finally, we prove Lemma~\ref{lem:plug-into-M}:
\begin{proof}[Proof of Lemma~\ref{lem:plug-into-M}]
This is immediate from Lemmas~\ref{lem:pki-baru-discrete} and \ref{lem:pki-hatu-discrete}, using the fact from Claim~\ref{claim:ubarwtbound} that $\|\barbu^{k_1}(a)\|_1, \|\hatbu^{k_1}(a)\|_1 \leq 1/2$, so $\<\barbu^{k_1}(a),\bz\>, \<\hatbu^{k_1}(a),\bz\> \in (-1/2,1/2)$, and in this interval $\sigma(x) = \sum_{i=0}^L \frac{m_i}{i!}x^i$.
\end{proof}

\subsection{Proof of Lemma~\ref{lem:k-min-from-det-m}}

\begin{proof}[Proof of Lemma~\ref{lem:k-min-from-det-m}]
For short-hand write $\bB(\ba) = \bM(\eta \ba, \bbeta, \bm)$. By Lemma~\ref{lem:plug-into-M}, $B_{\bz,j}(\ba) = \sigma(\<\bar\bu^{k_1}(a_j),\bz\>)$, so
\begin{align*}
\bK^{k_1}_{\bz,\bz'} &= \E_{a \sim \mu_a}[\sigma(\<\bar\bu^{k_1}(a),\bz\>)\sigma(\<\bar\bu^{k_1}(a),\bz'\>)] \\
&= \E_{\ba \sim \mu_a^{\otimes 2^P}}\Big[\frac{1}{2^P}\sum_{j=1}^{2^P}\sigma(\<\bar\bu^{k_1}(a_j),\bz\>)\sigma(\<\bar\bu^{k_1}(a_j),\bz'\>)\Big] \\
&= \frac{1}{2^P}\E_{\ba \sim \mu_a^{\otimes 2^P}}[\bB(\ba) \bB(\ba)^{\top}].
\end{align*}
So $\lambda_{\mathrm{min}}(\bK^{k_1}) \geq \frac{1}{2^P}\E_{\ba \sim \mu_a^{\otimes 2^P}}[\lambda_{\mathrm{min}}(\bB(\ba) \bB(\ba)^{\top})]$, and $$\lambda_{\mathrm{min}}(\bB(\ba) \bB(\ba)^{\top}) \geq \det(\bB(\ba))^2 / (\lambda_{\mathrm{max}}(\bB(\ba)))^{2^{P+1}-2} \geq c \det(\bB(\ba))^2,$$ for $c = 1 / (2^P K)^{2^{P+1}-2} > 0$.
\end{proof}

\subsection{Proof of Lemma~\ref{lem:poly-anticoncentration-and-bar-hat-comparison}}

Let us first show that $\beta_{S,k} = \E[(-\fNN(\bz;\bar\rho_k) + h_*(\bz))\chi_S(\bz)]$ is close to $\alpha_{S,k} = \alpha_S = \E[h_*(\bz)\chi_S(\bz)]$.
\begin{claim}\label{claim:alpha-beta-close}
There are constants $C,c > 0$ depending on $k_1,P,K$ such that for any $0 < \eta < c$, any $S \subseteq [P]$, and any $k \in \{0,\ldots,k_1-1\}$,
\begin{align*}
|\beta_{S,k} - \alpha_{S,k}| \leq C\eta
\end{align*}
\end{claim}
\begin{proof}
It suffices to show that $\|\fNN(\cdot;\bar\rho_k)\|_{\infty} \leq C \eta$. This is true since Claim~\ref{claim:ubarwtbound} implies $\|\hat\bu^k(a)\|_1 \leq C\eta$, so $|\fNN(\bz;\bar\rho_k)| \leq \E_{a \sim \mu_a} [a \sigma(0) + |a|\|\sigma'\|_{\infty} |\<\bar\bu^k(a),\bz\>|] \leq K\|\bar\bu^k(a)\|_1 \leq KC\eta \leq C\eta$.
\end{proof}

We now show the lemma.

\begin{proof}[Proof of Lemma~\ref{lem:poly-anticoncentration-and-bar-hat-comparison}]

Write $\det(\bM(\bzeta,\balpha,\bm)) = \sum_{\bgamma \in \{0,\ldots,D\}^{2^P}} \hat{h}_{\bgamma} \bzeta^{\bgamma}$. Let us prove that there is a constant $C$ depending on $k_1,P,K,L$ such that $|h_{\bgamma} - \bar h_{\bgamma}| \leq C\eta$ for all $\bgamma$. To see this, notice that $\det(\bM(\bzeta,\bxi,\brho))$ is a polynomial in $\bzeta,\bxi,\brho$, whose degree and coefficients depend only on $k_1,P,L$ (this is because each entry of $\bM(\bzeta,\bxi,\brho)$ is a polynomial in $\bzeta,\bxi,\brho$ with coefficients depending on $k_1,P,L$, and it is a $2^P \times 2^P$ matrix). Since $\|\bm\|_{\infty} \leq K$ and $\|\balpha\|_{\infty}, \|\bbeta\|_{\infty} \leq 2K$, and $\|\balpha - \bbeta\|_{\infty} \leq C\eta$ by Claim~\ref{claim:alpha-beta-close}, we conclude that there is a constant $C$ depending on $k_1,P,K,L$ such that $|h_{\bgamma} - \bar h_{\bgamma}| \leq C\eta$ for all $\bgamma$.

By anti-concentration of polynomials (i.e., Lemma~\ref{lem:not-shifted-polynomial-anticoncentration}), we have that there exists a constant $c > 0$ depending on $k_1,L,P$ such that
\begin{align*}
\E_{\ba \sim \mu_a^{\otimes 2^P}}[\det(\bM(\eta \ba,\balpha,\bm))^2] &\geq c \sum_{\bgamma \in \{0,\ldots,D\}^{2^P}} \eta^{2\|\bgamma\|_1} |\bar h_{\bgamma}|^2\\
&\geq c \sum_{\bgamma \in \{0,\ldots,D\}^{2^P}} \eta^{2\|\bgamma\|_1} \max(0,|h_{\bgamma}|-C\eta)^2,
\end{align*}
concluding the lemma.
\end{proof}

\subsection{Proof of Lemma~\ref{lem:nonzero-poly-discrete-msp}}

For this section, fix $\bm \in \R^{L+1}$ to be $m_i = i! \binom{L}{i}$ for all $i \in \{0,\ldots,L\}$. This corresponds to the activation function $\sigma(x) = (1+x)^L$.

\subsubsection{Reducing to minimal MSP set structures}
To show that $\det(\bN(\bzeta,\bphi,\bm)) \not\equiv 0$, we first show that it suffices to consider ``minimal'' MSP set structures.

\begin{claim}
Let $\cS' \subseteq \cS$ be such that $\cS'$ is an MSP set structure. Then if $$\det(\bN(\bzeta,\bphi,\bm))\mid_{\phi_S = 0 \mbox{ for all } S \in \cS \sm \cS'} \not\equiv 0,$$ we have
$$\det(\bN(\bzeta,\bphi,\bm)) \not\equiv 0.$$
\end{claim}
\begin{proof}
Substituting 0 for $\phi_S$ for all $S \in \cS \sm \cS'$.
\end{proof}

Therefore it suffices to prove the lemma for minimal MSP structures. Without loss of generality (up to permutation of the variables), we assume that we can write $$\cS' = \{S_1,\ldots,S_P\},$$ where, for all $i \in [P]$, $$i \in S_i \mbox{ and } S_i \subseteq [i].$$ Otherwise, we could remove a set from $\cS$ and still have a MSP set structure.

\subsubsection{Computing the weights to leading order}
Let us define the polynomials $q_{k,i}$ in variables $\bzeta,\bphi,\brho$. For all $k \in \{0,\ldots,k_1-1\}$ and $i \in [P]$,
\begin{align*}
q_{k,i}(\zeta,\bphi,\brho) = p_{k,i}(\zeta,\bxi,\brho) \mid_{\xi_{S,k} = 0 \mbox{ for all } S \not\in \cS \mbox{ and } \xi_{S,k} = \phi_S \mbox{ for all } S \in \cS}.
\end{align*}
Therefore $\bN(\bzeta,\bphi,\brho)$ has entries $N_{\bz,j}(\bzeta,\bphi,\brho) = \sum_{r=0}^L \frac{\rho_r}{r!} \left(\sum_{i=1}^Pq_{k,i}(\zeta_j,\bphi,\brho)\right)^r$.
Let us explicitly compute the nonzero term of $q_{k,i}$ that is of lowest-degree in $\zeta$. First, we show that many terms are zero.

\begin{claim}\label{claim:oi-def}
Recursively define $o_i = 1 + \sum_{i' \in S_i \sm \{i\}} o_{i'}$ for all $i \in [P]$.\footnote{The sum over an empty set is $0$ by convention.} Then $q_{k,i}(\zeta,\bphi,\bm)$ has no nonzero terms of degree less than $o_i$ in $\zeta$.
\end{claim}
\begin{proof}
The proof is by induction on $k$. In the base case of $k = 0$ it is true since $q_{0,i} \equiv 0$. In the inductive step, we assume it is true for all $k' \in \{0,\ldots,k\}$ and we prove the claim for $k+1$. By the recurrence dynamics,
\begin{align*}
&q_{k+1,i}(\zeta,\bphi,\bm) \\
&\quad= q_{k,i}(\zeta,\bphi,\bm) +  \zeta m_1 \phi_{\{i\}} \ones(\{i\} \in \cS) \\
&\quad\quad\quad + \zeta \sum_{r=1}^{L-1} \frac{m_{r+1}}{r!} \sum_{(i_1,\ldots,i_r) \in [P]^r} \phi_{\{i\} \oplus \{i_1\} \oplus \dots \{i_r\}} \ones(\{i\} \oplus \{i_1\} \oplus \dots \{i_r\} \in \cS) \prod_{l=1}^r q_{k,i_l}(\zeta,\bphi).
\end{align*}
The first term, $q_{k,i}(\zeta,\bphi,\bm)$, is handled by the inductive hypothesis. The second term is nonzero only in the case that $\{i\} \in \cS$, in which case $S_i = \{i\} \not\in \cS'$ and $o_i = 1$, so we do not have a contradiction.
The last terms can be handled by the inductive hypothesis: for any $(i_1,\ldots,i_r)$, each $q_{k,i_l}$ has no terms of degree less than $o_{i_l}$ in $\zeta$. So $\zeta \prod_{l} q_{k,i_l}(\zeta,\balpha)$ has no terms of degree less than $1 + \sum_{l=1}^r o_{i_l}$ in $\zeta$. 
We break into cases. \textit{Case a}. If $\{i\}\oplus \{i_1\} \oplus \dots \{i_r\} = S_i$, then $S_i \sm \{i\} \subset \{i_1,\ldots,i_r\}$, so $1 + \sum_{i=1}^l o_{i_l} \geq o_i$, and so no new terms of degree less than $o_i$ are added. \textit{Case b}. If $\{i\}\oplus \{i_1\} \oplus \dots \{i_r\} = S_{i'}$ for some $i' \neq i$, then either $i \in \{i_1,\ldots,i_r\}$, in which case $1 + \sum_{l=1}^r o_{i_{l}} > o_i$. Otherwise, we must have $i' > i$. But in this case $o_{i'} > o_i$ since $i \in S_{i'}$, so we also have $\sum_{l=1}^r o_{i_{l}} > o_i$ and again no new terms of degree less than $o_i$ are added. In fact, only terms of degree strictly more than $o_i$ are added.
\end{proof}

Finally, we give a recurrence for the degree-$o_i$ term in $\zeta$ of $q_{k,i}(\zeta,\bphi,\bm)$. Because of the previous claim, when this term is nonzero, it is the smallest-degree nonzero term. Denote this term by $\tilde{q}_{k,i}(\bphi,\bm) = [\zeta^{o_i}]q_{k,i}(\zeta,\bphi,\bm)$.
\begin{claim}\label{claim:oi-monomial-distinct}
If $k_1 \geq P$, then $\tilde{q}_{k_1,i}(\bphi,\bm)$ is a nonzero monomial in the variables $\bphi$. Furthermore, for any $i \neq i' \in [P]$, the monomials $\tilde{q}_{k_1,i}(\bphi,\bm)$ and $\tilde{q}_{k_1,i'}(\bphi,\bm)$ are not constant multiples of each other.
\end{claim}
\begin{proof}Following the analysis of the previous claim used to prove that $[\zeta^{l}] q_{k,i}(\zeta,\bphi,\bm) = 0$ for all $l < o_i$, only certain terms contribute in the recurrence. So we can simplify it to:
\begin{align*}
\tilde{q}_{k+1,i}(\bphi,\bm) &= \tilde{q}_{k,i}(\bphi,\bm) +  m_{|S_i|} \phi_{S_i} \prod_{i' \in S_i \sm \{i\}} \tilde{q}_{k,i'}(\bphi,\bm).
\end{align*}
Define $s_i = 1$ for all $i$ such that $|S_i| = 1$. And recursively define $s_i = 1 + \max \{s_{i'} : i' \in S_i \sm \{i\}\}$ for all other $i \in [P]$. Inductively on $k$, for all $k < s_i$ we have $\tilde{q}_{k,i} \equiv 0$. This is clear from the base case $\tilde{q}_{0,i} \equiv 0$ and the recurrence.

Next, for all $k \geq s_i$ we prove that 
\begin{align}\label{eq:monomial-recurrence}
\tilde{q}_{k,i}(\bphi,\bm) = \gamma_{k,i} \phi_{S_i} \prod_{i' \in S_i \sm \{i\}} \tilde{q}_{s_{i'},i'}(\bphi,\bm) \not\equiv 0 \end{align}
for some nonzero constant $\gamma_{k,i} > 0$ that depends on $\bm$. This is proved inductively on $k$. For $k = s_i > 1$, we have $$\tilde{q}_{k,i}(\bphi,\bm) = m_{|S_i|} \phi_{S_i} \prod_{i' \in S_i \sm \{i\}} \tilde{q}_{k-1,i'}(\bphi,\bm) = m_{|S_i|} \phi_{S_i}\prod_{i' \in S_i \sm \{i\}} \gamma_{k-1,i'} \tilde{q}_{s_{i'},i'}(\bphi,\bm),$$ so it is true since $m_{|S_i|} > 0$. For the inductive step, if $k > s_i$, $$\tilde{q}_{k,i}(\bphi,\bm) = \Big(\gamma_{k-1,i} + m_{|S_i|}\prod_{i' \in S_i \sm \{i\}} \gamma_{k-1,i'}\Big)\phi_{S_i} \prod_{i' \in S_i \sm \{i\}} \tilde{q}_{s_{i'},i'}.$$ So $\gamma_{k,i} =\gamma_{k-1,i} + m_{|S_i|}\prod_{i' \in S_i \sm \{i\}} \gamma_{k-1,i'} > 0$ since $m_{|S_i|} > 0$ by nonnegativity. This concludes the induction for \eqref{eq:monomial-recurrence}.

Using this recurrence relation \eqref{eq:monomial-recurrence} for $\tilde{q}_{k,i}$, by induction on $k$ we conclude that for any $k \geq P > s_i$ we have that $\tilde{q}_{k,i}(\bphi,\bm)$ is a nonzero monomial. Also, $\tilde{q}_{k,i}(\bphi,\bm)$ and $\tilde{q}_{k,i'}(\bphi,\bm)$ are distinct for all $i \neq i'$, since if $s_i \geq s_{i'}$ then $\phi_{S_i}$ divides $\tilde{q}_{k,i}$, but it does not divide $\tilde{q}_{k,i'}$.
\end{proof}

Recall that the interpretation of $q_{k,i}$ with respect to the simplified dynamics: for any second-layer weight $a \in [-1,1]$, the first-layer weights after training the simplified dynamics are $\hat{\bu}_i^{k_1}(a) = q_{k,i}(\eta a,\balpha,\bm)$. What we have shown in the previous two claims is that for any $i \neq i'$ to leading order $\hat{\bu}_i^{k_1}(a)$ and $\hat{\bu}_{i'}^{k_1}$ have different dependence on the Fourier coefficients $\balpha$ of the target function $h_*$.  Now we use this to essentially show that $\<\hat\bu^{k_1}(a),\bz\>$ and $\<\hat\bu^{k_1}(a),\bz'\>$ are distinct for all $\bz \neq \bz'$. 
\begin{claim}\label{claim:rzdistinct-discrete}
Define $$r_{\bz}(\zeta,\bphi,\bm) = \sum_i z_i q_{k_1,i}(\zeta,\bphi,\bm).$$
Then, for each distinct pair $\bz,\bz' \in \{+1,-1\}^P$, we have
$r_{\bz}(\zeta,\bphi,\bm) - r_{\bz'}(\zeta,\bphi,\bm) \not\equiv 0$ as a polynomial in $\zeta$ and $\bphi$.
\end{claim}
\begin{proof}
Recall the definition of $o_i$ from Claim~\ref{claim:oi-def}. Let $i \in [P]$ be such that $z_i \neq z'_i$ and $o_i$ is minimized. By Claim~\ref{claim:oi-def},
\begin{align*}
[\zeta^{o_i}](r_{\bz}(\zeta,\bphi,\bm) - r_{\bz'}(\zeta,\bphi,\bm)) &= \sum_{i' \mbox{ s.t. } o_{i'} = o_i, z_{i'} \neq z'_{i'}} (z_{i'} - z'_{i'}) [\zeta^{o_i}]q_{k_1,i}(\zeta,\bphi,\bm) \\
&= \sum_{i' \mbox{ s.t. } o_{i'} = o_i, z_{i'} \neq z'_{i'}} (z_{i'} - z'_{i'}) \tilde{q}_{k_1,i'}(\bphi,\bm),
\end{align*}
but $\tilde{q}_{k_1,i'}$ are distinct nonzero monomials in $\bphi$ by Claim~\ref{claim:oi-monomial-distinct}. So $r_{\bz}(\zeta,\bphi,\bm) - r_{\bz}(\zeta,\bphi,\bm) \not\equiv 0$.
\end{proof}

\subsubsection{Applying linear independence of powers of polynomials}

We conclude the proof of the lemma by using the following result of \cite{newman1979waring} showing that large powers of distinct polynomials are linearly independent. \begin{proposition}[Remark 5.2 in \cite{newman1979waring}]\label{prop:newmanslater}
Let $R_1,\ldots,R_m \in \C[\zeta]$ be non-constant polynomials such that for all $i \neq i' \in [m]$ we have $R_i(\zeta)$ is not a constant multiple of $R_{i'}(\zeta)$. Then for $L \geq 8m^2$ we have that $(R_1)^L,\ldots,(R_m)^L \in \C[\zeta]$ are $\C$-linearly independent.
\end{proposition}

We are ready to prove that $\det(\bN(\bzeta,\bphi,\bm)) \not\equiv 0$.

\begin{proof}[Proof of Lemma~\ref{lem:nonzero-poly-discrete-msp}]

Let us fix $\balpha = (\alpha_S)_{S \in \cS}$ such that for all $\bz \neq \bz'$ we have $r_{\bz}(\zeta,\balpha,\bm) - r_{\bz'}(\zeta,\balpha,\bm) \not\equiv 0$ as polynomials in $\zeta$. This can be ensured by drawing $\alpha_S \sim \mathrm{Unif}[-1,+1]$ for all $S \in \cS$, since for all $\bz \neq \bz'$ we have $r_{\bz}(\zeta,\bphi,\bm) - r_{\bz'}(\zeta,\bphi,\bm) \not\equiv 0$ as polynomials in $\zeta,\bphi$ by Claim~\ref{claim:rzdistinct-discrete}. Let us write $\tilde{r}_{\bz}(\zeta) = r_{\bz}(\zeta,\balpha,\bm)$ to emphasize that we have fixed the variables $\bphi = \balpha$ and $\brho = \bm$, and that we are looking at a polynomial over $\zeta$.

Since we have chosen $m_i = i! \binom{L}{i}$ for all $i \in \{0,\ldots,L\}$, we have $$N_{\bz,j}(\bzeta,\balpha,\bm) = (1 + \tilde{r}_{\bz}(\zeta))^L.$$ From the recurrence relations $\zeta$ divides $q_{k_1,i}(\zeta,\bphi,\bm)$ for each $i \in [P]$, so $\zeta$ divides $\tilde{r}_{\bz}(\zeta) = \sum_{i=1}^P z_i q_{k_1,i}(\zeta,\balpha,\bm)$. Therefore, no two polynomials $(1 + \tilde{r}_{\bz}(\zeta)) ,(1 + \tilde{r}_{\bz'}(\zeta))$ are constant multiples of each other for each distinct $\bz,\bz'$. Otherwise, if $(1 + \tilde{r}_{\bz}(\zeta)) \equiv c (1 + \tilde{r}_{\bz'}(\zeta))$, then we would have $1 = (1 + \tilde{r}_{\bz}(0)) = c (1 + \tilde{r}_{\bz'}(0)) = c$, which would imply $c = 1$, but $(1 + \tilde{r}_{\bz}(\zeta)) \not\equiv (1 + \tilde{r}_{\bz'}(\zeta))$ since $\tilde{r}_{\bz}(\zeta)$ and $\tilde{r}_{\bz'}(\zeta)$ are distinct.

Construct the Wronskian matrix over the $L$th power polynomials $\{(1 + \tilde{r}_{\bz}(\zeta))^L\}_{\bz \in \{+1,-1\}^P}$. This is a $2^P \times 2^P$ matrix $\bH(\zeta)$ whose entries are indexed by $\bz$ and $j \in [2^P]$ and defined by:
\begin{align*}
H_{\bz,j}(\zeta) = \pd{^{j-1}}{\zeta^{j-1}} (1 + r_{\bz}(\zeta))^L.
\end{align*}
By Proposition~\ref{prop:newmanslater}, the polynomials $\{(1 +\tilde{r}_{\bz}(\zeta))^L\}_{\bz \in \{+1,-1\}^P}$ are linearly-independent, so the Wronskian determinant is nonzero as a polynomial in $\zeta$:
\begin{align*}
\det(\bH(\zeta)) \not\equiv 0.
\end{align*}
Finally notice that we can write $\det(\bH(\zeta)) = \pd{}{\zeta_2} \pd{}{\zeta_3^2} \dots \pd{^{2^P-1}}{\zeta_{2^P}} \det(\bN(\bzeta,\balpha,\bm)) \mid_{\zeta = \zeta_1 = \dots = \zeta_{2^P}}$. 

Therefore $\det(\bN(\bzeta,\balpha,\bm)) \not\equiv 0$ as a polynomial in $\bzeta$. So $\det(\bN(\bzeta,\bphi,\bm)) \not\equiv 0$ as a polynomial in $\bzeta$ and $\bphi$.
\end{proof}

\clearpage

\section{Generic MSP functions are strongly $O(d)$-SGD-learnable with continuous-time dynamics and activation perturbation}\label{ssec:genericmspproof}

In this appendix, we provide a more general approach to proving strong $O(d)$-SGD-learnability for generic MSP functions that goes beyond polynomial activation functions. The reason to include this second approach is two-fold:
\begin{enumerate}
    \item[1.] We consider the continuous-time regime (as opposed to the discrete-time regime as in Appendix \ref{app:discrete-msp-proof}), which is closer to practice, with small batch and step sizes. (Note that the extension to non-polynomial activations would also hold in discrete time.)
    
    \item[2.] For continuous time and non-polynomial activations, the first layer weights $\obu^t$ are not polynomials in $\oa^0$ anymore. However, we show that they can still be approximated by polynomials and that global convergence reduces to showing that certain (universal) polynomials are not identically $0$.
\end{enumerate}

Using this approach, we show in Theorem \ref{thm:genericmspsuffrestated} that generic MSP functions are strongly $O(d)$-SGD-learnable for smooth activation functions (as long as $\sigma^{(r)} (0) \neq 0$ for $r= 0 , \ldots , P$), with one technical caveat: we need to introduce a random perturbation to the activation function at one point during the training dynamics. While unnatural, this modification allows us to prove that the polynomials are non-zero for general MSP structure, using a ``Vandermonde trick''. See Section~\ref{ssec:discussionperturb} for a discussion on this technicality.

\subsection{Statement of the result}

Recall the definition of the measure over functions with MSP set structure $\cS$:
\begin{definition}[Definition~\ref{def:leb-msp} restated]
For any set structure $\cS \subseteq 2^{[P]}$ define the measure $\mu_{\cS}$ over functions $h_* : \{+1,-1\}^P \to \R$ induced by taking $h_*(\bz) = \sum_{S \subseteq [P]} \alpha_S \chi_S(\bx)$, where the Fourier coefficients satisfy $\alpha_S = 0$ if $S \not\in\cS$, and $(\alpha_S)_{S \in \cS}$ have Lebesgue measure on $\R^{|\cS|}$.
\end{definition}

\paragraph{Choice of hyperparameters:}  Recall from the equivalence with \eqref{eq:DF-PDE} (Theorem \ref{thm:equivalence}) that it is sufficient to show for any $\eps >0$, there exists hyperparameters satisfying $\rA3'$ such that \eqref{eq:DF-PDE} reaches $\eps$-risk. We consider the following hyperparameters, which are the same as in the proof for the vanilla staircase in Section~\ref{ssec:vanillaproof}:
\begin{itemize}
    \item We do not regularize, i.e., $\lambda^a = \lambda^w= 0$, same as Section~\ref{ssec:vanillaproof}.
    \item We initialize the first layer to deterministically $\obu^0 =\bzero$, and the second layer to uniform random weights on $[-1,-1]$, i.e., $\mu_a = \Unif ( [ +1, -1])$ and $\mu_W = \delta_0$.
    \item Our learning rate schedule is the same as in Section~\ref{ssec:vanillaproof},
\begin{description}
\item[Phase 1:] We train the first layer weights $\obu^t$ while keeping the second layer weights fixed $\oa^t = \oa^0$. We set $\xi^a (t) = 0$ and $\xi^w (t) = 1$ for $t \in [0,T_1]$.

\item[Phase 2:] We train the second layer weights $\oa^t$ while keeping the first layer weights fixed at $\obu^t = \obu^{T_1}$. We set $\xi^a (t) = 1$ and $\xi^w (t) = 0$ for $t \in [T_1,T_2]$.
\end{description}
\end{itemize}

\begin{remark}
As in Section~\ref{ssec:vanillaproof}, the learning rate schedules can be made Lipschitz at $T_1$ with a change of variables, falling under the assumptions of strong SGD learnability.
\end{remark}

\paragraph{Perturbing the activation:} We consider an activation function $\sigma$ that verifies $\rA0'$, i.e., that is sufficiently smooth in a neighborhood of $0$. However, unlike the proof for the vanilla staircase, we add the following technical caveat. At time $T_1$, we randomly perturb the activation $\sigma$ to get an activation $\sigma_{pert}$. We use activation $\sigma$ in the training of Phase 1 when training \eqref{eq:DF-PDE} during time $[0,T_1]$ but we use the perturbed activation function $\sigma_{pert}$ when training \eqref{eq:DF-PDE} during time $[T_1,T_2]$ in Phase 2. By perturbing the activation, we mean the following: let $0 < \tau_{pert} < 1$ be a parameter that controls the amount of perturbation. Draw $\rho_{i} \sim \mathrm{Unif}([-\tau_{pert},\tau_{pert}])$ for each $i \in \{0,\ldots,2^{8P}\}$. The perturbed activation is defined as $\sigma_{pert}(x) = \sigma(x) + \sum_{r=0}^{2^{8P}} \frac{\rho_r}{r!} x^r$.

The dynamics of \eqref{eq:DF-PDE} in time $[0,T_1]$ with activation $\sigma$ stitched together with the dynamics in time $[T_1,T_2]$ with activation $\sigma_{pert}$ corresponds to an algorithm that falls under the definition of strong $O(d)$-SGD-learnability, when extended to allow such a perturbation (in particular, the equivalent characterization and necessary condition in Theorems \ref{thm:equivalence} and \ref{thm:MSP_necessary} would still hold). See Section~\ref{ssec:discussionperturb} for more discussion.
\\

We restate the sufficient condition, proving that for any MSP set structure $\cS$, generic functions $h_*$ with that set structure $\cS$ are strongly $O(d)$-SGD-learnable:

\begin{theorem}\label{thm:genericmspsuffrestated}
Consider $\cS \subseteq 2^{[P]}$ a MSP set structure, and $0 < \tau_{pert} < 1$ a perturbation parameter. Assume that the activation function $\sigma$ satisfies $\rA$0' and has nonzero derivatives $\sigma^{(r)} (0) \neq 0$ for $r =0 , \ldots, P$. Then, almost surely for $h_*$ with respect to to $\mu_{\cS}$ and almost surely for perturbation $\brho \sim \Unif ( [-\tau_{pert},\tau_{pert}]^{2^{8P}} ) $, the following hold: for any $\eps > 0$, there exist $T_1,T_2 > 0$ such that training with the above hyperparameters and activation perturbation will learn $h_*$ to accuracy $\eps$.

This implies that almost surely over $\mu_{\cS}$, $h_*$ is strongly $O(d)$-SGD-learnable (under the expanded definition of $O(d)$-SGD-learnability where the SGD algorithm is allowed to perturb the activation function once).
\end{theorem}

\subsection{Discussion on the perturbation of the activation}\label{ssec:discussionperturb}

The perturbation is convenient to show that a polynomial is not identically zero for arbitrary MSP set structure. Note that given a set structure $\cS$, these polynomials are fully explicit (given by recurrence relations) and one can verify by hand that they have a non zero coefficient. It is an interesting direction to show this result directly without relying on perturbing the activation function. In the setting of discrete-time regime and polynomial activations (cf. Theorem~\ref{thm:discrete-msp}), such a perturbation is not needed: the weights $\obu^{T_1}$ are exact polynomials of $\oa^0$ and one can use algebraic tricks involving linear independence of powers of polynomials (see Proposition \ref{prop:newmanslater}).

Note that we can extend the definition of strong SGD-learnability in $O(d)$-scaling to allow such a perturbation. In that case, the dimension-free dynamics \eqref{eq:DF-PDE} corresponds to gluing two dynamics with activations $\sigma$ between $[0,T_1]$ and $\sigma_{pert}$ between $[T_1,T_2]$. The equivalent characterization (Theorem \ref{thm:equivalence}) and necessary condition (Theorem \ref{thm:MSP_necessary}) still hold using this extended definition.

\subsection{Outline of the proof} 

Similarly to the proof for the vanilla staircase in Section~\ref{ssec:vanillaproof}, the proof of Theorem~\ref{thm:genericmspsuffrestated} follows by analyzing the solution $\obu^t(\oa^0)$ to the evolution equations \eqref{eq:evolution_effective_MF} obtained from initialization $ (\oa^0, \bzero, 0)$. Again, for clarity, we will suppress some notations: we denote $\bu$ instead of $\obu$, and $a$ instead of $\oa^0$. We also forget about $\os^t = 0$ and simply consider $\orho_t \in \cP(\R^{P+1})$ the distribution of $(a,\bu)$. This last simplification can be done since we initialize the first-layer weights to 0, so in particular $\os^0 = 0$, and by the evolution equation of \eqref{eq:evolution_effective_MF} we have $\os^t = 0$ throughout training. Furthermore, we will denote $K$ a generic constant that only depends on $P$ and the constants in the assumptions. The value of $K$ can change from line to line.

For MSP functions beyond the vanilla staircase, the approach used to prove \cref{thm:vanillastaircasesuffrestated} no longer works, and a finer-grained analysis is needed.\footnote{Indeed, for MSP functions that are not vanilla staircases, $\bM = (\E_a [ a^{\beta(S)+\beta(S')}])_{S,S' \subseteq [P]}$ (introduced in Section \ref{app:outline_staircase}) can have some sets $S \neq S'$ such that $\beta(S) = \beta(S')$, and $\bM$ is not a positive matrix anymore.} The argument is more involved because we need to track higher-order corrections to $\bu^t$. We present here the finer-grained analysis.

The proof analyzes Phase 1 and Phase 2 of training separately.

\paragraph{Phase 1 (nonlinear dynamics)}

We break our analysis of the nonlinear training in Phase 1 into several parts. The goal is to understand the evolution under the dimension-free PDE of each neuron's weights $(a,\bu^t(a))$. Because we initialize the first layer to $0$, it suffices to study the dynamics of $\bu^t$, ignoring the dynamics of $\os^t$ since it stays at $\os^t = \os^0 = 0$ throughout. The dynamics of $\bu^t$ are given by
\begin{align}\label{eq:butdynamics}
\frac{d}{dt} \bu^t = a\E_{\bz}[g_t(\bz)\sigma'(\<\bu^t,\bz\>)\bz],
\end{align}
where $g_t(\bz) = h_*(\bz) - \fNN(\bz;\bar\rho_t)$ is the residual at time $t$.
\\

\textit{Reducing to analyzing with polynomial approximation}. Our first step is to analyze a polynomial approximation of $\bu^t$ instead of analyzing $\bu^t$ directly. Let $L > 0$ be an integer governing the degree of approximation. We will choose $L$ to be a large enough constant depending on $P$. We first prove in Section~\ref{ssec:capprox-one} that for small times $t$ we can approximate the dynamics of $\bu^t(a)$ by an approximate dynamics $\tilde{\bu}^t$ defined as
\begin{align*}
\tilde{\bu}^t(a) = \bQ^t [a, a^2, a^3, \ldots, a^L]^{\top},
\end{align*}
where $[a, a^2, a^3, \ldots,a^L]$ denotes the vector with the powers of $a$, and $\bQ^t \in \R^{P \times L}$ is a time-dependent matrix with $\bQ^0 = \bzero$ and which is updated according to a certain non-linear dynamics defined as follows (this corresponds essentially to truncating the dynamics of $\bu^t$ by only keeping the order-$L$ approximation). Let $g_t(\bz) = h_*(\bz) - \fNN(\bz;\bar\rho_t)$ denote the residual at time $t$. For $l = 1$,
\begin{align}\label{eq:Qdynamics1}
\frac{d}{dt} Q_{i1}^t = \E_{\bz}[z_i g_t(\bz) m_1]\, ,
\end{align}
and for $2 \leq l \leq L$,
\begin{align}\label{eq:Qdynamics2}
\frac{d}{dt} Q_{il}^t =  \E_{\bz}\left[z_i g_t(\bz) \sum_{1 \leq r \leq L-1} \frac{m_{r+1}}{r!} \sum_{i_1,\ldots,i_r \in [P]} \sum_{\substack{l_1,\ldots,l_r \in [L] \\ \sum_{r'=1}^r l_{r'} = l-1}} \prod_{r'=1}^r z_{i_{r'}} Q_{i_{r'}l_{r'}}^t\right] \, .
\end{align}
We prove in Claim~\ref{claim:buapproxtildebu} that we have $\|\bu^t(a) - \tilde{\bu}^t(a)\| \leq O(t^{L})$ for small enough times $t$, so it suffices to study $\tilde{\bu}^t$ instead of $\bu^t$. Of course, the dynamics of $\tilde{\bu}^t$ still present a challenge to analyze.
\\

\textit{Reducing to analyzing the simplified dynamics}. One significant challenge is that the residual $g_t$ is time-dependent, as it depends on $\fNN(\cdot;\bar\rho_t)$. This interaction term complicates the picture significantly. However, if we train for small time $t$, then $|\fNN(\bz;\bar\rho_t)| \leq O(t)$, and we can expect the contribution of this term to be negligible. To make this intuition precise, in Section~\ref{ssec:ignorefnn} we introduce a time-dependent matrix $\hat\bQ^t \in \R^{P \times L}$ which is initialized at $\hat\bQ^0 = \bzero$ and which has the same evolution equations \eqref{eq:Qdynamics1} and \eqref{eq:Qdynamics2} as $\bQ^t$, except with $g_t$ replaced by $h_*$. We obtain a ``simplified dynamics'' by letting $\hat\bu^t(a) = \hat\bQ^t [a,a^2,a^3,\ldots,a^L]^{\top}$. This is easier to analyze since it neglects the interaction term.

However, unlike the comparison of $\bu^t$ to its polynomial approximation $\tilde\bu^t$, where we could prove that $\|\bu^t - \tilde\bu^t\| \leq O(t^L)$, it is not the case that the simplified dynamics $\hat\bu$ give such a good accuracy approximation to $\bu$ in $L_2$ norm. Indeed, we may unfortunately have $\|\bu^t - \hat\bu^t\| \geq \Omega(t^2)$, which is a bound that would be far too loose for our analysis of higher-order terms in the dynamics. To overcome this issue, we prove that $|Q_{il}^t - \hatQ_{il}^t| \leq O(t^{l+1})$ for each $i \in [P], l \in [L]$. We then use the fact that $\tilde{\bu}^t$ and $\hat{\bu}^t$ are both polynomials in $a$ with coefficients $\bQ$ and $\hat{\bQ}$, respectively, to reduce to analyzing the $\hat{\bu}^t$ dynamics (see Section~\ref{ssec:reductiontosimplified} for details).
\\

\textit{Analyzing the simplified dynamics with a recurrence relation}. We analyze the $\hat{\bu}^t$ dynamics by deriving recurrence relations for the coefficients $\hatQ_{il}^t$. In particular, we may express each coefficient $\hatQ_{il}^t$ as a polynomial in $a$, $t$, and the nonzero Fourier coefficients $\{\alpha_S\}_{s \in \cS}$ of $h_*$ (see Section~\ref{ssec:continuousrecurrence}). This allows us to prove that almost surely over the choice of $h_*$ each coordinate $\hat{\bu}_i^t$ has distinct dynamics: namely, $\hat{\bu}_i^t - \hat{\bu}_{i'} \not\equiv 0$ for all $i \neq i' \in [P]$. This is where we must use the fact that the MSP function $h_*$ is ``generic'', i.e., the coefficients $\{\alpha_S\}_{S \in \cS}$ are chosen randomly. (In fact, we prove and use the stronger result that for any $\bz \neq \bz' \in \{+1,-1\}^P$, we have
$\<\hat{\bu}^t,\bz - \bz'\> \not\equiv 0$,
and this difference has nonzero low-degree terms.)

\paragraph{Phase 2 (linear dynamics)} The linear dynamics are analyzed by showing a lower-bound on $\lambda_{\mathrm{min}}(\bK^{T_1})$, as was the case for the proof of the vanilla staircase in Section~\ref{ssec:vanillaproof}. We show in Section \ref{ssec:reductiontosimplified} (and similarly to the discrete case) that it is sufficient to show that a polynomial depending on the simplified dynamics is non-zero. In Sections \ref{ssec:mspconclusion} and \ref{sec:mainmspclaim}, we show how this can be achieved using the perturbation on the activation function: one of the coefficient of the polynomial can be rewritten as the determinant of a Vandermonde matrix with entries $\{ \< \hat{\bu}^t , \bz \> \}_{ \bz \in \{+1,-1\}^P} $. Using that $\<\hat{\bu}^t,\bz - \bz'\> \not\equiv 0$ for $\bz \neq \bz'$, this determinant is non zero and we conclude the proof.

\subsection{Approximating the $u_i^t$ with polynomials}\label{ssec:capprox-one}

As outlined above, we study the dynamics of the dimension-free PDE. Let us first analyze Phase 1, when we train for time $T_1$ using activation function $\sigma$, and keep the second layer fixed. In particular, we analyze the dynamics of $\bu^t(a)$ given by \cref{eq:butdynamics} and the initialization $\bu^t = \bzero$. In the proof below, we sometimes omit the dependence on $a$ and time $t$, e.g., writing $\bu$ instead of $\bu^t (a)$, when the dependence on $t$ and $a$ is clear.

The first step of the proof is to approximate $\bu^t(a)$ with a polynomial in $a$. Let $L > 0$ be an integer which corresponds to the degree of approximation. We prove in this section that we have the approximation $\bu^t \approx \tilde{\bu}^t$, where we define $\tilde\bu^t$ as:
\begin{align*}
\tilde\bu^t = \bQ^t [a,a^2,a^3,\ldots,a^L]^{\top}.
\end{align*}
Here, recall that $\bQ^t \in \R^{P \times L}$ is given by initializing $\bQ^0 = \bzero$ and training with \cref{eq:Qdynamics1} and \cref{eq:Qdynamics2}.

We first prove for each $l \in [L], i \in [L]$, that each coefficient $Q_{il}^t$ of $a^l$ scales as $O(t^l)$. 

\begin{claim}\label{claim:qilbound}
There is a constant $C$ depending on $K,L,P$ such that for any $i \in [P]$, $l \in [L]$, and $0 \leq t \leq T_1$, $|Q_{il}^t| \leq Ct^l$.
\end{claim}
\begin{proof}
We prove this by induction on $l$. For the base case of $l = 1$, we know that
\begin{align*}
\Big\vert \pd{Q_{i1}^t}{t} \Big\vert  = |\E_{\bz}[z_i g_t(\bz) m_1]| \leq K,
\end{align*}
since $\|g_t\|_{\infty} \leq K$ throughout the dynamics, and $|m_1| \leq K$. So $|Q_{i1}^t| \leq K t \leq C_1 t$ for a constant $C_1$. For the inductive step, let $2 \leq l \leq L$ and suppose $|Q_{il'}^t| \leq C_{l'} t^{l'}$ for all $1 \leq l' < l$. Then
\begin{align*}
\Big\vert\pd{Q_{il}^t}{t}\Big\vert &\leq K\|g_t\|_{\infty}\E_{\bz}\Big[\sum_{1 \leq r \leq L-1}  \sum_{i_1,\ldots,i_r \in [P]} \sum_{\substack{l_1,\ldots,l_r \in [L] \\ \sum_{r'=1}^r l_{r'} = l-1}} \prod_{r'=1}^r |Q_{i_{r'}l_{r'}}^t|\Big] \\
&\leq K\|g_t\|_{\infty}\E_{\bz}\Big[\sum_{1 \leq r \leq L-1}  \sum_{i_1,\ldots,i_r \in [P]} \sum_{\substack{l_1,\ldots,l_r \in [L] \\ \sum_{r'=1}^r l_{r'} = l-1}} t^{l-1} \prod_{r'=1}^r |C_{l_{r'}}|\Big] \\
&\leq K L (2P)^L
\max(|C_1|,\ldots,|C_{l-1}|)^L  t^{l-1} \leq C_l t^{l-1}.
\end{align*}
So $|Q_{il}^t| \leq C_l t^l$, defining $C_l$ appropriately.
\end{proof}

Let us prove that $\bu^t$ and $\tilde{\bu}^t$ have norm $O(t)$.
\begin{claim}\label{claim:butildebuOt}
There is a constant $C$ depending only on $K,P$ and a constant $C'$ depending only on $K,L,P$ such that for any $0 \leq t \leq T_1$,
$\|\bu^t\| \leq Ct, \|\tilde{\bu}^t\| \leq C't$.
\end{claim}
\begin{proof}
Note $\bu^0 = \bzero$ and $\|\pd{\bu^t}{t}\| \leq |a|\|g_t(\bz)\|_{\infty}\|\sigma'\|_{\infty} \|\bz\| \leq (2K)K\sqrt{P} \leq C$. So $\|\bu^t\| \leq Ct$. Similarly, $\|\tilde{\bu}^t\| \leq \sum_{i \in [P]}, \sum_{l \in [L]} |Q_{il}^t| \leq PLCt \leq C't$ by Claim~\ref{claim:qilbound}.
\end{proof}

Let us prove that $\bu^t \approx \tilde{\bu}^t$ throughout the dynamics.
\begin{claim}\label{claim:buapproxtildebu}
There are constants $c,C > 0$ depending on $K,L,P$ such that if $T_1 \leq c$ then for any $0 \leq t \leq T_1$,
$\|\bu^t - \tilde{\bu}^t\| \leq C t^{L}$.
\end{claim}
\begin{proof}
The proof will use Gronwall's inequality. First, by triangle inequality
\begin{align*}
\pd{}{t} \|\bu^t - \tilde{\bu}^t\| \leq \Big\|\pd{\tilde{\bu}^t}{t} - \pd{\bu^t}{t}\Big\|.
\end{align*}
Notice that
\begin{align*}
&~\Big\|\pd{\tilde{\bu}^t}{t} - \pd{\bu^t}{t}\Big\| \\
=&~ \Big\|\E[\bz g_t(\bz) m_1] + \sum_{l=1}^L a^l \E_{\bz}\Big[\bz g_t(\bz) \sum_{1 \leq r \leq L-1} \frac{m_{r+1}}{r!} \sum_{i_1,\ldots,i_r \in [P]} \sum_{\substack{l_1,\ldots,l_r \in [L] \\ \sum_{r'=1}^r l_{r'} = l-1}} \prod_{r'=1}^r z_{i_{r'}} Q_{i_{r'}l_{r'}}^t\Big] - \pd{\bu^t}{t}\Big\| \\
\leq&~ \Big\|\E[\bz g_t(\bz) m_1] + \E_{\bz}\Big[\bz g_t(\bz) \sum_{1 \leq r \leq L-1} \frac{m_{r+1}}{r!} h_{L,r}^t(\bz)\Big] - \pd{\bu^t}{t}\Big\| \, ,
\end{align*}
where for any $1 \leq r \leq L-1$,
\begin{align*}
h_{L,r}^t(\bz) =  \sum_{l=1}^L a^l \sum_{i_1,\ldots,i_r \in [P]} \sum_{\substack{l_1,\ldots,l_r \in [L] \\ \sum_{r'=1}^r l_{r'} = l-1}} \prod_{r'=1}^r z_{i_{r'}} Q_{i_{r'}l_{r'}}^t,
\end{align*}
which can be thought of as a degree-$(L-1)$ approximation to $\<\tilde{\bu}^t,\bz\>^r$, in the sense that
\begin{align*}
&~|\<\tilde{\bu}^t,\bz\>^r - h_{L,r}^t(\bz)| \\
=&~ \Big\vert \<\tilde{\bu}^t,\bz\>^r - \sum_{i_1,\ldots,i_r \in [P]} \sum_{\substack{l_1,\ldots,l_r \in [L] \\ 0 \leq \sum_{r'=1}^r l_{r'} \leq L-1}} \prod_{r'=1}^r z_{i_{r'}} Q_{i_{r'}l_{r'}}^t a^{l_{r'}}\Big\vert \\
=&~ \Big\vert\sum_{i_1,\ldots,i_r \in [P]} \Big\{\sum_{l_1,\ldots,l_r \in [L]} \prod_{r'=1}^r z_{i_{r'}} Q_{i_{r'}l_{r'}}^t a^{l_{r'}} -  \sum_{\substack{l_1,\ldots,l_r \in [L] \\ 0 \leq \sum_{r'=1}^r l_{r'} \leq L-1}} \prod_{r'=1}^r z_{i_{r'}} Q_{i_{r'}l_{r'}}^t a^{l_{r'}}\Big\}\Big\vert \\
=&~ \Big\vert\sum_{i_1,\ldots,i_r \in [P]}  \sum_{\substack{l_1,\ldots,l_r \in [L] \\ L \leq \sum_{r'=1}^r l_{r'} \leq rL}} \prod_{r'=1}^r z_{i_{r'}} Q_{i_{r'}l_{r'}}^t a^{l_{r'}}\Big\vert \\
\leq&~ \sum_{i_1,\ldots,i_r \in [P]}  \sum_{\substack{l_1,\ldots,l_r \in [L] \\ L \leq \sum_{r'=1}^r l_{r'} \leq rL}} \prod_{r'=1}^r C t^{l_{r'}} |a|^{l_{r'}} \\
\leq&~ P^r 2^{rL} C^r |a|^L t^L \\
\leq&~ P^L 2^{L^2} C^L t^L \\
\leq&~ C t^L \, ,
\end{align*}
for a constant $C$ depending on $K,L,P$, where used Claim~\ref{claim:qilbound} to bound $Q_{i_{r'}l_{r'}}^t$ and that $|a| \leq 1$ and $t < 1$ in the final bound.

We conclude that 
\begin{align*}
\Big\|\pd{\tilde{\bu}^t}{t} - \pd{\bu^t}{t}\Big\| &\leq \Big\|\E_{\bz}\Big[\bz g_t(\bz) \sum_{1 \leq r \leq L-1} \frac{m_{r+1}}{r!} (h_{L,r}^t(\bz) - \<\bu^t,\bz\>^r)\Big]\Big\| \\
&\leq 2K^2PL \max_{1 \leq r \leq L-1,\bz \in \{+1,-1\}^P} |h_{L,r}^t(\bz) - \<\bu^t,\bz\>^r| \\
&\leq 2K^2PL \max_{1 \leq r \leq L-1,\bz \in \{+1,-1\}^P} |h_{L,r}^t(\bz) - \<\tilde{\bu}^t,\bz\>^r| + |\<\tilde{\bu}^t,\bz\>^r - \<\bu^t,\bz\>^r| \\
&\leq 2K^2PL( Ct^L + rP^r \|\tilde{\bu}^t - \bu^t\|) \\
&\leq Ct^L + C\|\tilde{\bu}^t - \bu^t\|\, ,
\end{align*}
where for the second-to-last line we have used $\|\tilde{\bu}^t\|, \|\bu^t\| \leq Ct \leq 1/(r\sqrt{P})$ if we take small enough time $T_1 \leq c$ for a constant $c > 0$ depending on $L,P,K$. The claim follows by Gronwall's inequality, since $\tilde{\bu}^0 = \bu^0 = \bzero$ and we train for time $T_1 \leq c < 1$.
\end{proof}

\subsection{Simplified dynamics without interaction term}\label{ssec:ignorefnn}

We have introduced the dynamics $\tilde{\bu}^t$ and proved that they give a $O(t^L)$-approximation of the true dynamics $\bu^t$. We now reduce further, to analyzing the dynamics of $\hat{\bu}^t$, where we have dropped the $\fNN$ term, replacing $g_t(\bz)$ with $h_*(\bz)$ in the definition of the dynamics \eqref{eq:Qdynamics1} and \eqref{eq:Qdynamics2}:
\begin{align*}
    \hat{\bu}(t) = \hatbQ^t[a,a^2,a^3,\ldots,a^L]^{\top},
\end{align*}
where for $l = 1$,
\begin{align*}
\pd{\hatQ_{i1}^t}{t} = \E_{\bz}[z_i h_*(\bz) m_1],
\end{align*}
and for $2 \leq l \leq L$,
\begin{align*}
\pd{\hatQ_{il}^t}{t} =  \E_{\bz}\Big[z_i h_*(\bz) \sum_{1 \leq r \leq L-1} \frac{m_{r+1}}{r!} \sum_{i_1,\ldots,i_r \in [P]} \sum_{\substack{l_1,\ldots,l_r \in [L] \\ \sum_{r'=1}^r l_{r'} = l-1}} \prod_{r'=1}^r z_{i_{r'}} \hatQ_{i_{r'}l_{r'}}^t\Big].
\end{align*}

To show that the new dynamics is close to the old dynamics, we first show that $\|\fNN(\bz;\rho_t)\|_{\infty} = O(t)$, is small when $t$ is small:
\begin{claim}\label{claim:fnnOt}
There is a constant $C$ depending on $K,L,P$ such that for all $0 \leq t \leq T_1$, $\|\fNN(\bz;\rho_t)\|_{\infty} \leq Ct$.
\end{claim}
\begin{proof}
For any $\bz$, $|\fNN(\bz;\rho_t) - \fNN(\bz;\rho_0)| \leq \E_{a}[|a \sigma(\<\bu^t(a),\bz\>) - a\sigma(\<\bu^0(a),\bz\>)|] \leq \E_{a}[K| \<\bu^t(a)-\bu^0(a),\bz\>|] \leq K\|\bu^t(a) - \bzero\|_1 \leq KPCt \leq Ct$ by Claim~\ref{claim:butildebuOt} and $K$-Lipschitzness. And $\fNN(\bz;\rho_0) = \E_{a}[a \sigma(0)] = 0$, since $\E_a[a] = 0$ and $\bu^0(a) = \bzero$.
\end{proof}

We also prove the analogue of Claim~\ref{claim:qilbound} for $\hatQ$:
\begin{claim}\label{claim:hatqilhatbubound}
There is a constant $C$ depending on $K,L,P$ such that for all $i \in [P]$, $l \in [L]$, and $0 \leq t \leq T_1$, $|\hatQ_{il}^t| \leq Ct^l$. Also, $\|\hat{\bu}^t\| \leq Ct$.
\end{claim}
\begin{proof}
The bound on $|\hatQ_{il}|$ is the same as Claim~\ref{claim:qilbound}, but using the bound $\|h_*\|_{\infty} \leq K$ instead of the bound $\|g_t\|_{\infty} \leq 2K$. The bound on $\|\hat{\bu}^t\|$ is the same as Claim~\ref{claim:butildebuOt}, using the bound on $|\hatQ_{il}^t|$.
\end{proof}

We show that $|Q_{il}^t - \hatQ_{il}^t| \leq O(t^{l+1})$ for each $l \in [L]$:
\begin{lemma}\label{lem:qqhatclose}
There is a constant $C$ depending on $K,L,P$ such that for any $i \in [P]$, $l \in [L]$, $|Q_{il}^t - \hatQ_{il}^t| \leq Ct^{l+1}$.
\end{lemma}
\begin{proof}
We prove this by induction on $l$. For $l = 1$,
\begin{align*}
\Big\vert \pd{\hatQ_{i1}^t}{t} - \pd{Q_{i1}^t}{t}\Big\vert \leq \|h_*(\bz) - g_t(\bz)\|_{\infty} |m_1| \leq K \|\fNN(\bz;\rho_t)\|_{\infty} \leq CKt \leq C_1 t,
\end{align*}
by Claim~\ref{claim:fnnOt}, for some large enough constant $C_1$. Therefore $|\hatQ_{i1}^t - Q_{i1}^t| \leq C_1 t^2$.
For the inductive step, let $2 \leq l \leq L$, and assume that $|\hatQ^t_{il'} - Q_{il'}^t| \leq C_1 t^{l'+1}$ for all $1 \leq l' \leq l-1$. Then
\begin{align*}
&~\Big\vert \pd{\hatQ_{il}^t}{t} - \pd{Q_{il}^t}{t}\Big\vert \\
=&~ \Big\vert \E_{\bz}\Big[z_i \sum_{1 \leq r \leq L-1} \frac{m_{r+1}}{r!} \sum_{i_1,\ldots,i_r \in [P]} \sum_{\substack{l_1,\ldots,l_r \in [L] \\ \sum_{r'=1}^r l_{r'} = l-1}} \Big\{ h_*(\bz)\prod_{r'=1}^r z_{i_{r'}} \hatQ^t_{i_{r'}l_{r'}} - g_t(\bz)\prod_{r'=1}^r z_{i_{r'}} Q^t_{i_{r'}l_{r'}}\Big\}\Big]\Big\vert  \\
\leq&~ KL P^L \sum_{\substack{l_1,\ldots,l_r \in [L] \\ \sum_{r'=1}^r l_{r'} = l-1}} \max_{\bz} \left|h_*(\bz)\prod_{r'=1}^r \hatQ^t_{i_{r'}l_{r'}} - g_t(\bz)\prod_{r'=1}^r Q^t_{i_{r'}l_{r'}}\right| \\
\leq&~ C \sum_{\substack{l_1,\ldots,l_r \in [L] \\ \sum_{r'=1}^r l_{r'} = l-1}} \left(\max_{\bz} |h_*(\bz) - g_t(\bz)|\left|\prod_{r'=1}^r \hatQ^t_{i_{r'}l_{r'}}\right| + |g_t(\bz)| \left|\prod_{r'=1}^r \hatQ^t_{i_{r'}l_{r'}} - \prod_{r'=1}^r Q^t_{i_{r'}l_{r'}}\right|\right) \\
\leq&~ \sum_{\substack{l_1,\ldots,l_r \in [L] \\ \sum_{r'=1}^r l_{r'} = l-1}} \left(Ct C^r t^{l-1} + (2K) \left|\prod_{r'=1}^r \hatQ^t_{i_{r'}l_{r'}} - \prod_{r'=1}^r Q^t_{i_{r'}l_{r'}}\right|\right) \\
\leq&~ Ct^l + C\sum_{\substack{l_1,\ldots,l_r \in [L] \\ \sum_{r'=1}^r l_{r'} = l-1}} \sum_{r''=1}^{r}\left|\prod_{r'=1}^{r''} \hatQ^t_{i_{r'}l_{r'}} \prod_{r'=r''+1}^r Q^t_{i_{r'}l_{r'}} - \prod_{r'=1}^{r''-1} \hatQ^t_{i_{r'}l_{r'}} \prod_{r'=r''}^r Q^t_{i_{r'}l_{r'}}\right| \\
\leq&~ Ct^l + C\sum_{\substack{l_1,\ldots,l_r \in [L] \\ \sum_{r'=1}^r l_{r'} = l-1}} \sum_{r''=1}^{r}\left| \hatQ^t_{i_{r''}l_{r''}} - Q^t_{i_{r''}l_{r''}}\right| C^{r-1} t^{l-1-l_{r''}} \\
\leq&~ Ct^l + C\sum_{\substack{l_1,\ldots,l_r \in [L] \\ \sum_{r'=1}^r l_{r'} = l-1}} \sum_{r''=1}^{r}\left( C_{l_{r''}}t^{l_{r''}+1}\right) C^{r-1} t^{l-1-l_{r''}} \\
\leq&~ C_lt^l\, ,
\end{align*}
where the second-to-last-line was by the inductive hypothesis. Since $Q_{il}^0 = \hatQ_{il}^0 = 0$, we conclude $|\hatQ_{il}^t - Q_{il}^t| \leq C_lt^{l+1}$.
\end{proof}

The above lemma will be used in Section~\ref{ssec:reductiontosimplified} to show that it suffices to analyze the dynamics of $\hat\bQ^t$ instead of the dynamics of $\bQ^t$, and in turn instead of the dynamics of $\bu^t$.

\subsection{Recurrence relation of the coefficients in the simplified dynamics}\label{ssec:continuousrecurrence}

We prove that each entry of the matrix $\hatQ^t$ is a polynomial in $t$ and the Fourier coefficients of $h_*$, and we give a recurrence relation for the coefficients. Define $\balpha = \{\alpha_S\}_{S \subseteq [P]} \in \R^{2^P}$ where $\alpha_S = \E_{\bz}[\chi_S(\bz) h_*(\bz)]$, and $\bm = ( m_0 , \ldots , m_{L}) \in \R^{L+1}$ where $m_i = \sigma^{(i)}(0) $ for all $i \in \{0,\ldots,L\}$..

\begin{lemma}\label{lem:hatburecurrence}
For each $i \in [P]$, $l \in [L]$, we have $\hatQ^t_{il} = t^l p_{il}(\balpha,\bm)$, where $p$ is a polynomial in the Fourier coefficients $\balpha$ of $h_*$ and in the first $L$ derivatives $\bm$ of $\sigma$. Furthermore, $\{p_{il}\}_{i \in [P], l \in [L]}$ satisfies the recurrence relations $p_{i1} = \alpha_{\{i\}} m_1$ and
\begin{align*}
&~ p_{il}(\balpha,\bm) \\
=&~ \frac{1}{l}\sum_{S \subseteq [P]} \sum_{1 \leq r \leq L-1} \sum_{i_1,\ldots,i_r \in [P]} \sum_{\substack{l_1,\ldots,l_r \in [L] \\ \sum_{r'=1}^r l_{r'} = l-1}} 1(\{i\} \oplus S \oplus \{i_1\} \dots \oplus \{i_r\} = \emptyset)\frac{m_{r+1}}{r!} \alpha_{S}  \prod_{r'=1}^{r} p_{i_{r'}l_{r'}}(\balpha,\bm).
\end{align*}
\end{lemma}
\begin{proof}
The proof is by induction on $l$. In the base case, for any $i \in [P]$,
\begin{align*}
\pd{\hatQ_{i1}^t}{t} = \E_{\bz}[z_i h_*(\bz) m_1] = \alpha_{\{i\}} m_1,
\end{align*}
so $\hatQ_{i1}^t = t \alpha_{\{i\}} m_1$. For the inductive step, suppose that the lemma is true for all $i \in [P]$ and $l' \in \{1,\ldots,l-1\}$. Then
\begin{align*}
&~\pd{\hatQ_{il}^t}{t} \\
=&~ \E_{\bz}\Big[z_i \Big\{\sum_{S \subseteq [P]} \alpha_{S} \chi_{S}(\bz)\Big\} \sum_{1 \leq r \leq L-1} \frac{m_{r+1}}{r!} \sum_{i_1,\ldots,i_r \in [P]} \sum_{\substack{l_1,\ldots,l_r \in [L] \\ \sum_{r'=1}^r l_{r'} = l-1}} \prod_{r'=1}^r z_{i_{r'}} t^{l_{r'}} p_{i_{r'}l_{r'}}(\balpha,\bm) \Big] \\
=&~ t^{l-1} \sum_{S \subseteq [P]} \sum_{1 \leq r \leq L-1} \sum_{i_1,\ldots,i_r \in [P]} \sum_{\substack{l_1,\ldots,l_r \in [L] \\ \sum_{r'=1}^r l_{r'} = l-1}} \Big\{\E_{\bz}\Big[z_i \chi_{S}(\bz) \prod_{r'=1}^r z_{i_{r'}} \Big]\Big\} \frac{m_{r+1}}{r!}  \alpha_{S} \prod_{r'=1}^{r} p_{i_{r'}l_{r'}}(\balpha,\bm) \\
=&~ t^{l-1} \sum_{S \subseteq [P]} \sum_{1 \leq r \leq L-1} \sum_{i_1,\ldots,i_r \in [P]} \sum_{\substack{l_1,\ldots,l_r \in [L] \\ \sum_{r'=1}^r l_{r'} = l-1}} 1(\{i\} \oplus S \oplus \{i_1\} \dots \oplus \{i_r\} = \emptyset)\frac{m_{r+1}}{r!} \alpha_{S}  \prod_{r'=1}^{r} p_{i_{r'}l_{r'}}(\balpha,\bm)\, .
\end{align*}
The recurrence relation follows by integrating with respect to $t$.
\end{proof}

We will subsequently prove that it suffices to study $\hatbu$, for which the recurrence relation in Lemma \ref{lem:hatburecurrence} becomes useful.

\subsection{Reduction to analyzing the simplified dynamics}\label{ssec:reductiontosimplified}

Let us study the training in Phase 2, where we train the second layer from time $T_1$ to time $T_2$, while keeping the first layer fixed. Furthermore, we train with the perturbed activation function $\sigma_{pert}$. In order to prove that the training of the second layer converges, it is sufficient to prove that the kernel obtained as the linearization of the second layer weights, after the training in Phase 1 has condition number bounded by a constant $C$ depending only on $K,L,P$. Define the kernel $\bK^{T_1} : \{+1,-1\}^P \times \{+1,-1\}^P \to \R$ for times $t \geq T_1$ as
\begin{align*}
\bK^{T_1}(\bz,\bz') = \E_{(a,\bu^{T_1}) \sim \bar\rho_{T_1}}[\sigma_{pert}(\<\bu^{T_1},\bz\>)\sigma_{pert}(\<\bu^{T_1},\bz'\>)].
\end{align*}

In order to bound the learning in Phase 2, it is sufficient to bound the minimum eigenvalue of $\bK^{T_1}$. To this aim, define the kernel $\tilde{\bK}^{T_1} : \{+1,-1\}^P \times \{+1,-1\}^P \to \R$ corresponding to the $\tilde{\bu}^t$ dynamics as:
\begin{align*}
\tilde{\bK}^{T_1}(\bz,\bz') &= \E_{a \sim \mu_a}[\hat\sigma_{pert}(\<\tilde\bu^{T_1}(a),\bz\>)\hat\sigma_{pert}(\<\tilde\bu^{T_1}(a),\bz'\>)],
\end{align*}
where
$$\hat\sigma_{pert}(s) = \sum_{r=0}^{L} \frac{m_r + \rho_r}{r!} s^r$$ is the degree-$L$ approximation of the perturbed activation function $\sigma_{pert}$. Recall that the perturbation $\brho$ is chosen so that $\rho_r \sim \mathrm{Unif}[-\tau_{pert},\tau_{pert}]$ for all $r \in \{0,\ldots,2^{8P}\}$ and $0$ otherwise.\\

We bound the minimum eigenvalue of $\bK^{T_1}$ by the minimum eigenvalue of $\tilde{\bK}^{T_1}$ by showing that the kernel $\tilde{\bK}^{T_1}$ is $O(t^L)$ close in spectral norm to the kernel $\bK^{T_1}$.
\begin{claim}\label{claim:Kbound}
There are constants $c,C > 0$ depending on $K,L,P$ such that, for all $T_1 \leq c$,
\begin{align*}
\lambda_{\mathrm{min}}(\bK^{T_1}) \geq \lambda_{\mathrm{min}}(\tilde{\bK}^{T_1}) - C(T_1)^L
\end{align*}
\end{claim}
\begin{proof}
This follows by proving that $\bK^{T_1}$ and $\tilde{\bK}^{T_1}$ are close in spectral norm: i.e., $\|\bK^{T_1} - \tilde{\bK}^{T_1}\| \leq C(T_1)^L$. For any $\bz,\bz'$,
\begin{align*}
|&\bK^{T_1}(\bz,\bz') - \tilde{\bK}^{T_1}(\bz,\bz')| \\
&\leq \E_{a \sim \mu_a}[|\sigma_{pert}(\<\bu^{T_1}(a),\bz\>)\sigma_{pert}(\<\bu^{T_1}(a),\bz'\>) - \hat\sigma_{pert}(\<\tilde\bu^{T_1}(a),\bz\>)\hat\sigma_{pert}(\<\tilde\bu^{T_1}(a),\bz'\>)|] \\
&\leq \E_{a \sim \mu_a}[|\sigma_{pert}(\<\bu^{T_1}(a),\bz\>)\sigma_{pert}(\<\bu^{T_1}(a),\bz'\>) - \sigma_{pert}(\<\tilde\bu^{T_1}(a),\bz\>)\sigma_{pert}(\<\tilde\bu^{T_1}(a),\bz'\>)| \\
&\quad\quad + |\sigma_{pert}(\<\tilde{\bu}^{T_1}(a),\bz\>)\sigma_{pert}(\<\tilde{\bu}^{T_1}(a),\bz'\>) - \hat\sigma_{pert}(\<\tilde\bu^{T_1}(a),\bz\>)\hat\sigma_{pert}(\<\tilde\bu^{T_1}(a),\bz'\>)|]  \\
&\leq \E_{a \sim \mu_a}[2K^2\sqrt{P}\|\bu^{T_1}(a)-\tilde\bu^{T_1}(a)\| + C\|\tilde\bu^{T_1}(a)\|^{L+1}] \\
&\leq C(T_1)^L,
\end{align*}
where we use that $\|\sigma\|_{\infty}, \|\sigma'\|_{\infty} \leq K$,  $\|\bz\|,\|\bz'\| \leq \sqrt{P}$, and also the Taylor series error bound and the fact that $\|\tilde{\bu}^t\| \leq Ct$ (by Claim~\ref{claim:butildebuOt}), and $\|\bu - \tilde{\bu}\| \leq Ct^L$ by Claim~\ref{claim:buapproxtildebu} for $0 \leq t \leq T_1$.
So $\|\bK^{T_1} - \tilde{\bK}^{T_1}\| \leq \|\bK^{T_1} - \tilde{\bK}^{T_1}\|_F \leq 2^P C(T_1)^L \leq C(T_1)^L$.
\end{proof}

So if we can prove that $\lambda_{\mathrm{min}}(\tilde{\bK}^{T_1}) \geq \Omega((T_1)^l)$ for any $l < L$, then for sufficiently small $T_1$ this implies that $\lambda_{min}(\bK^{T_1}) > c(T_1)^l$ for some constant $c > 0$ depending on $K,L,P$. This would prove that the condition number of $\bK^{T_1}$ is bounded by a constant independent of $d$.

We now show a strategy to prove that $\lambda_{\mathrm{min}}(\tilde{\bK}^{T_1}) \geq \Omega((T_1)^l)$, by analyzing the $\hat\bu^t$ dynamics instead of the $\tilde\bu^t$ dynamics. We must use a much more delicate argument than the bound used to compare $\bK^{T_1}$ and $\tilde{\bK}^{T_1}$. The reason is that we used $\|\bu^{T_1} - \tilde{\bu}^{T_1}\| \leq O((T_1)^L)$, but it is not necessarily true that $\tilde{\bu}^{T_1}$ and $\hat{\bu}^{T_1}$ are $O((T_1)^L)$-close in $L_2$ norm. In fact, we typically have $\|\tilde\bu^{T_1} - \hat\bu^{T_1}\| \geq \Omega((T_1)^2)$. So we instead use the fact that $\tilde{\bu}^{T_1}$ and $\hat{\bu}^{T_1}$ are polynomials in $a$, and their coefficients are close as polynomials in $a$ (previously proved in Lemma~\ref{lem:qqhatclose}).

Let us first prove a lower-bound on $\lambda_{\mathrm{min}}(\tilde{\bK}^{T_1})$ in terms of the determinant of a certain ``feature matrix'' $\tilde{\bM} : \{+1,-1\}^P \times [2^P] \to \R$ indexed by $\bz \in \{+1,-1\}^P$ and $j \in [2^P]$ as
\begin{align*}
\tilde{M}(\bz,j) &= \hat{\sigma}_{pert}(\<\tilde{\bu}^{T_1}(a_j),\bz\>),
\end{align*}
where $\ba = [a_1,\ldots,a_{2^P}] \in \R^{2^P}$ is a vector of indeterminate variables. We bound $\lambda_{\mathrm{min}}(\tilde{K}^{T_1})$ in terms of the expected magnitude of the determinant of $\tilde{M}$, for random $\ba \sim \mu_a^{\otimes 2^P} = \mathrm{Unif}([-1,1]^{\otimes 2^P})$.
\begin{claim}\label{claim:tildeKbound}
There is a constant $c > 0$ depending on $K,L,P$ such that, for all $T_1 \leq c$,
\begin{align*}
\lambda_{\mathrm{min}}(\tilde{\bK}^{T_1}) \geq c\E_{\ba \sim \mathrm{Unif}([-1,1]^{\otimes 2^P})}[\det(\tilde{\bM})^2].
\end{align*}
\end{claim}
\begin{proof}
Since
\begin{align*}[\tilde{\bM}\tilde{\bM}^{\top}](\bz,\bz') = \sum_{j \in [2^P]} \hat\sigma_{pert}(\<\tilde{\bu}^{T_1}(a_j),\bz\>)\hat\sigma_{pert}(\<\tilde{\bu}^{T_1}(a_j)\bz'\>),\end{align*}
we can write $\tilde{\bK}^{T_1}$ in terms of this matrix product
\begin{align*}
    \tilde{\bK}^{T_1} = \frac{1}{2^P} \E_{\ba \sim \mathrm{Unif}[-1,1]^{\otimes 2^P}}[\tilde{\bM}\tilde{\bM}^{\top}].
\end{align*}
So, since $\tilde{\bM} \tilde{\bM}^{\top}$ is p.s.d.,
\begin{align}\label{eq:lambdaminkt1complambdaminmmt}
\lambda_{\mathrm{min}}(\tilde{\bK}^{T_1}) \geq \frac{1}{2^P} \E_{\ba \sim \mathrm{Unif}^{\otimes 2^P}}[\lambda_{\mathrm{min}}(\tilde{\bM}\tilde{\bM}^{\top})].
\end{align}
For any $a_j \in [-1,1]$ and assuming $T_1 \leq c$ is small enough, we have $|\<\tilde{\bu}^{T_1}(a),\bz\>| \leq c$ for some small enough constant $c > 0$ so that $|\hat\sigma_{pert}(\<\tilde{\bu}^{T_1}(a_j),\bz\>)| \leq C$ for some large enough constant $C$ depending on $K,L,P$. This means $\lambda_{\mathrm{max}}(\tilde{\bM}\tilde{\bM}^{\top}) \leq 2^{2P} C^2 \leq C$ almost surely. So $\lambda_{\mathrm{min}}(\tilde{\bM}\tilde{\bM}^{\top}) \geq (\lambda_{\mathrm{max}}(\tilde{\bM}\tilde{\bM}^{\top}))^{-2^P+1}\prod_{i=1}^{2^P} \lambda_i(\tilde{\bM} {\bM}^{\top}) \geq \det(\tilde{\bM})^2 / C^{2^P - 1} \geq c \det(\tilde{\bM})^2$. This proves the claim when combined with the lower bound \eqref{eq:lambdaminkt1complambdaminmmt}.
\end{proof}

It remains to lower-bound the magnitude of the determinant of $\tilde{\bM}$, for $\ba \sim \mathrm{Unif}([-1,1]^{\otimes 2^P})$. First, we note that the determinant is a polynomial in $\ba$.
\begin{claim}\label{claim:tildebMpolyzeta}
For each $\bgamma \in \{0,\ldots,L^2\}^{2^P}$, there is a coefficient $\tilde{h}_{\bgamma}$ depending only on $t$, $h_*$, $\bm$, and $\brho$ such that 
\begin{align*}
\det(\tilde{\bM}) = \sum_{\bgamma \in \{0,\ldots,L^2\}^{2^P}} \tilde{h}_{\bgamma} \ba^{\bgamma}.
\end{align*}
In other words, the determinant is a polynomial in $\ba$ of individual degree at most $L^2$.
\end{claim}
\begin{proof}
For each $i \in [P]$, and $j \in [2^P]$, recall that $\tilde{\bu}_{i}^t(a_j) = \sum_{l \in [L]} (a_j)^l Q_{il}^t$. Here $Q_{il}^t$ depends only on $i,l,t,\balpha,\bm$ and does not depend on $\ba$. So for each $\bz \in \{+1,-1\}^P$, $\tilde{\bM}(\bz,j) = \hat\sigma_{pert}(\<\tilde{\bu}^t(a_j),\bz\>) = \sum_{r=0}^L \frac{m_r + \rho_r}{r!} \<\tilde{\bu}^t(a_j),\bz\>^r = \sum_{r=0}^L \frac{m_r + \rho_r}{r!} \left(\sum_{i \in [P]} z_i \sum_{l \in [L]} (a_j)^l Q_{il}^t\right)^r$. So since each entry of $\tilde{\bM}$ is a polynomial in $\ba$, the determinant is also a polynomial in $\ba$.
\end{proof}

We can prove that in expectation over $\ba \sim \mathrm{Unif}([-1,1]^{\otimes 2^P})$ this determinant is nonzero if it has nonzero coefficients of low degree:
\begin{claim}\label{claim:anticoncentration-in-expectation}
There is a constant $c > 0$ depending on $L$ and $P$ such that for all $0 \leq t \leq T_1$,
\begin{align*}
\E[\det(\tilde{\bM})^2] \geq c \sum_{\bgamma \in \{0,\ldots,L^2\}^{2^P}} |\tilde{h}_{\bgamma}|^2
\end{align*}
\end{claim}
\begin{proof}
The proof is by writing $\det(\tilde{\bM})$ in the Legendre basis, lower-bounding its coefficients in this basis, and using the orthogonality of the Legendre polynomials. This is Lemma~\ref{lem:polynomial-anticoncentration}.
\end{proof}

This leaves the question of how to prove that $\det(\tilde{\bM})$ is a nonzero polynomial with some nonzero term of degree $t^l$ where $l < 2L$. Here we show that this problem can be reduced to analyzing the $\hat{\bu}$ dynamics, which are simpler to analyze since they do not have the dependence on $\fNN$ and admit the recurrence relations of Lemma~\ref{lem:hatburecurrence}. Similarly to the definition of $\tilde\bM$, we can define $\hat\bM : \{+1,-1\}^P \times [2^P] \to \R$ by
\begin{align*}
\hat\bM(\bz,j) &= \hat{\sigma}(\<\tilde{\bu}^{T_1}(a_j),\bz\>).
\end{align*}
Similarly to $\tilde\bM$, we can prove that each entry of $\hat\bM$ is a polynomial in $\ba$.
\begin{claim}
For each $\bgamma \in \{0,\ldots,L^2\}^{2^P}$, there is a coefficient $\hat{h}_{\bgamma}$ depending only on $t$, $h_*$, $\bm$ and $\brho$ such that 
\begin{align*}
\det(\hat\bM) = \sum_{\bgamma \in \{0,\ldots,L^2\}^{2^P}} \hat{h}_{\bgamma} \ba^{\bgamma}.
\end{align*}
In other words, the determinant is a polynomial in $\ba$ of individual degree at most $L^2$.
\end{claim}
\begin{proof}
Same as the proof of Claim~\ref{claim:tildebMpolyzeta}.
\end{proof}

This is useful, since we can show that the coefficients of $\det(\hat\bM)$ are close to those of $\det(\tilde{\bM})$.

\begin{claim}\label{claim:tildehapproxhath}
There is a constant $C > 0$ depending on $K, L, P$, such that for any $\bgamma \in \{0,\ldots,L^2\}^{2^P}$,
\begin{align*}
|\tilde{h}_{\bgamma} - \hat{h}_{\bgamma}| \leq C(T_1)^{\|\bgamma\|_1+1}.
\end{align*}
\end{claim}
\begin{proof}
We write $\det(\tilde{\bM})$ as a sum over permutations $\tau$,
\begin{align*}
\det(\tilde{\bM}) &= \sum_{\tau \in S_{2^P}} \sgn(\tau) \prod_{j=1}^{2^P} \sum_{r=0}^L \frac{m_r + \rho_r}{r!} \left(\sum_{i \in [P]} \tau(j)_i \sum_{l \in [L]} (a_j)^l Q_{il}^{T_1}\right)^r.
\end{align*}
Therefore,
\begin{align*}
\tilde{h}_{\bgamma} 
=&~ \sum_{\tau \in S_{2^P}} \sgn(\tau) \sum_{r_1,\ldots,r_{2^P} \in \{0,\ldots,L\}}  \\
&~\phantom{AAAAAAA} \prod_{j=1}^{2^P} \frac{m_{r_j} + \rho_{r_j}}{r_j!} \times \Big[\delta_{r_j,0} + (1 - \delta_{r_j,0})\Big\{\sum_{i_1,\ldots,i_{r_j} \in [P]} \sum_{\substack{l_1,\ldots,l_{r_j} \in [L] \\ \sum_{r'=1}^{r_j} l_{r'} = \gamma_j}} \prod_{r'=1}^{r_j} \tau(j)_{i_{r'}} Q_{i_{r'}l_{r'}}^{T_1}\Big\} \Big]\, ,
\end{align*}
and the same expression holds for $\hat{h}_{\bgamma}$, with $\hatQ$ replacing $Q$. Since $Q_{il}^t,\hatQ_{il}^t \leq Ct^l$ by Claim~\ref{claim:qilbound} and \ref{claim:hatqilhatbubound} and $Q_{il}^t - \hatQ_{il}^t \leq Ct^{l+1}$ by Lemma~\ref{lem:qqhatclose}, we conclude by a triangle inequality and telescoping that
$|\tilde{h}_{\bgamma} - \hat{h}_{\bgamma}| \leq C(T_1)^{l+1}$ for a constant $C$ depending only on $K,L,P$.
\end{proof}

Furthermore, in fact $\det(\hatbM)$ has the special structure that each coefficient $\hat{h}_{\bgamma}$ is of size proportional to $(T_1)^{\|\bgamma\|_1}$ if it is nonzero:
\begin{claim}\label{claim:hathpoly}
For any $\bgamma \in \{0,\ldots,L^2\}^{2^P}$, there is a polynomial $q_{\bgamma}(\balpha,\bm,\brho)$ such that $$\hat{h}_{\bgamma} = (T_1)^{\|\bgamma\|_1} q_{\bgamma}(\balpha,\bm,\brho).$$
\end{claim}
\begin{proof}
By direct calculation,
\begin{align*}
\hat{h}_{\bgamma} &= \sum_{\tau \in S_{2^P}} \sgn(\tau) \sum_{r_1,\ldots,r_{2^P} \in \{0,\ldots,L\}} \\
&\quad\quad\quad\quad \prod_{j=1}^{2^P} \frac{m_{r_j} + \rho_{r_j}}{r_j!} \Big[\delta_{r_j,0} + (1 - \delta_{r_j,0})\Big\{\sum_{i_1,\ldots,i_{r_j} \in [P]} \sum_{\substack{l_1,\ldots,l_{r_j} \in [L] \\ \sum_{r'=1}^{r_j} l_{r'} = \gamma_j}} \prod_{r'=1}^{r_j} \tau(j)_{i_{r'}} \hatQ_{i_{r'}l_{r'}}^{T_1}\Big\}\Big].
\end{align*}
Since $\hatQ_{i_{r'}l_{r'}} = (T_1)^{l_{r'}} p_{i_{r'}l_{r'}}(\balpha,\bm)$ by Lemma~\ref{lem:hatburecurrence}, we have
\begin{align*}
\hat{h}_{\bgamma} &= \sum_{\tau \in S_{2^P}} \sgn(\tau) \sum_{r_1,\ldots,r_{2^P} \in \{0,\ldots,L\}} \\
&\quad\quad\quad\prod_{j=1}^{2^P} \frac{m_{r_j} + \rho_{r_j}}{r_j!} (T_1)^{\gamma_j} \Big[\delta_{r_j,0} + (1 - \delta_{r_j,0})\Big\{\sum_{i_1,\ldots,i_{r_j} \in [P]} \sum_{\substack{l_1,\ldots,l_{r_j} \in [L] \\ \sum_{r'=1}^{r_j} l_{r'} = \gamma_j}} \prod_{r'=1}^{r_j} \tau(j)_{i_{r'}} p_{i_{r'}l_{r'}}(\balpha,\bm)\Big\} \Big]\, .
\end{align*}
We deduce that
\begin{align*}
\hat{h}_{\bgamma} = (T_1)^{\|\bgamma\|_1} q_{\bgamma}(\balpha,\bm,\brho),
\end{align*}
where $q_{\bgamma}(\balpha,\bm,\brho)$ is the polynomial defined by
\begin{align*}
&~ q_{\bgamma}(\balpha,\bm,\brho) \\
=&~ \sum_{\tau \in S_{2^P}} \sgn(\tau) \sum_{r_1,\ldots,r_{2^P} \in \{0,\ldots,L\}} \\
&\quad\quad\quad\quad\prod_{j=1}^{2^P} \frac{m_{r_j} + \rho_{r_j}}{r_j!} \Big[\delta_{r_j,0} + (1 - \delta_{r_j,0})\Big\{\sum_{i_1,\ldots,i_{r_j} \in [P]} \sum_{\substack{l_1,\ldots,l_{r_j} \in [L] \\ \sum_{r'=1}^{r_j} l_{r'} = \gamma_j}} \prod_{r'=1}^{r_j} \tau(j)_{i_{r'}} p_{i_{r'}l_{r'}}(\balpha,\bm)\Big\} \Big] \, ,
\end{align*}
which concludes the proof of the claim.
\end{proof}

Combining the above claims we obtain a bound on the determinant of $\tilde{\bM}$ in terms of the $\hat\bu$ dynamics.
\begin{claim}\label{claim:boundontildedetintermsofq}
Suppose that for some $\bgamma \in \{0,\ldots,L^2\}^{2^P}$, we have $q_{\bgamma}(\balpha,\bm,\brho) \neq 0$. Then there is a small enough constant $c > 0$ depending on $K,L,P,\balpha,\bm,\bgamma$, such that for all $T_1 \leq c$,
\begin{align*}
\E[\det(\tilde{\bM})^2] &\geq c(T_1)^{2\|\bgamma\|_1}.
\end{align*}
\end{claim}
\begin{proof}
By combining Claims~\ref{claim:anticoncentration-in-expectation}, \ref{claim:tildehapproxhath}, and \ref{claim:hathpoly}, we know that there is a large enough constant $C > 0$ and small enough constant $c > 0$ such that
\begin{align*}
\E[\det(\tilde{\bM})^2] \geq c \sum_{\bgamma \in \{0,\ldots,L^2\}^{2^P}} |\min(0,(T_1)^{\|\bgamma\|_1}q_{\bgamma}(\balpha,\bm,\brho) - C(T_1)^{\|\bgamma\|_1 + 1})|^2.
\end{align*}
Choosing $c > 0$ smaller than $|q_{\bgamma}(\balpha,\bm,\brho)| / (2C)$ concludes the claim.
\end{proof}

We conclude by combining all of the above claims to get the result of this subsection:
\begin{lemma}\label{lem:finalhatsuffbound}
Suppose that for some $\bgamma \in \{0,\ldots,L^2\}^{2P}$ such that $\|\bgamma\|_1 < L/2$ we have $q_{\bgamma}(\balpha,\bm,\brho) \neq 0$. Then there is a small enough constant $c > 0$ depending on $K,L,P,\balpha,\bm,\bgamma,\brho$ such that for all $T_1 \leq c$ we have
\begin{align*}
\lambda_{\mathrm{min}}(\bK^{T_1}) \geq c(T_1)^{2\|\bgamma\|_1}.
\end{align*}
\end{lemma}
\begin{proof}
This is immediate by combining Claims~\ref{claim:Kbound}, \ref{claim:tildeKbound} and \ref{claim:boundontildedetintermsofq}.
\end{proof}

\subsection{Proving learnability of generic MSP functions, Theorem~\ref{thm:genericmspsuffrestated}}\label{ssec:mspconclusion}

Here we give the final technical step to proving that generic MSP functions are learnable. The proof idea is to use Lemma~\ref{lem:finalhatsuffbound} to lower-bound the minimum eigenvalue of the kernel matrix $\bK^{T_1}$. By Lemma~\ref{lem:finalhatsuffbound}, it suffices to prove that for any minimal MSP structure $\cS \subseteq 2^{[P]}$, if we plug in $\alpha_S = 0$ for all $S \not\in \cS$ the determinant $\det(\hat\bM)$ almost surely is a non-zero polynomial in $t$ with nonzero low-order terms. In other words, the main technical lemma that remains to be proved is the following.

\begin{lemma}\label{lem:mainmspclaim}
Let $\cS \subseteq 2^{[P]}$ be any MSP set structure on $P$ variables. Then there are constants $l_{\cS}$ and $L_{\cS}$ depending only $\cS$ such that if we take the truncation to the dynamics to be $L \geq L_{\cS}$ then $\det(\hat\bM) \mid_{(\alpha_S)_{S \not\in \cS} = \bzero}$ is a polynomial in $t,\ba,\{\balpha\}_{S \in \cS},\bm,\brho$ that has a nonzero term with degree $l_{\cS}$ in $t$.
\end{lemma}

Before we show this lemma, let us see how it implies the main theorem.
\begin{proof}[Proof of Theorem~\ref{thm:genericmspsuffrestated}]
Let $L_{\cS}$ and $l_{\cS}$ be as in Lemma~\ref{lem:mainmspclaim}.
Choose the approximation parameter $L = \max(L_{\cS},2l_{\cS}+1)$ for defining the dynamics $\hat\bu$. We know that $\det(\hatbM) \mid_{(\alpha_S)_{S \not\in S} = \bzero}$ is a polynomial in $t$, $\ba$, $\{\balpha\}_{S \in \cS}$, $\bm$, and $\rho$ that has a nonzero term with degree $l_{\cS}$ in $t$. Therefore, almost surely over plugging in the activation perturbation $\brho = [\rho_0,\ldots,\rho_{2^{8P}}] \sim \mathrm{Unif}[-\tau_{pert},\tau_{pert}]^{\otimes 2^{8P}}$, the generic Fourier coefficients on the MSP set structure $(\alpha_S)_{S \in \cS} \sim \mathrm{Unif}[-1,1]^{\otimes |\cS|}$, and the zero Fourier coefficients outside the MSP set structure $(\alpha_S)_{S \not\in \cS} = 0$, we must have that $\det(\hatbM)$ is a polynomial in $t$ with a nonzero term of degree $ l_{\cS}$.

Since $\cS \subseteq 2^{[P]}$, $L$ and $l_{\cS}$ are upper-bounded by a constant $C$ that depends only on $P$. So by Lemma~\ref{lem:finalhatsuffbound}, we conclude that almost surely over $\brho$ and $\balpha$ there is a constant $c > 0$ depending only on $K,P,\brho,\balpha$ such that we have $\lambda_{\mathrm{min}}(\bK^{T_1}) \geq c(T_1)^{2l_{\cS}}$ as long as $T_1 \leq c$.  In particular, choosing $T_1 = c$, then $\lambda_{\mathrm{min}}(\bK^{T_1}) \geq c^2$.

For $t \geq T_1$, let $\bg_t = (g_t(\bz))_{\bz \in \{+1,-1\}^P}$ denote the residual vector where $g_t(\bz) = h_*(\bz) - \fNN^{pert}(\bz;\bar\rho_{t})$. Here $\fNN^{pert}$ is $\fNN$ but with the activation $\sigma$ replaced by the perturbed activation $\sigma_{pert}$ that is used in Phase 2. Recall that during Phase 2 the dynamics are linear since we are training the second layer, and are governed by kernel $\bK^{T_1}$. We have following bound on the norm of the residuals for $t \geq T_1$: $$\|\bg_t\|_2^2 \leq e^{-\lambda_{\mathrm{min}}(\bK^{T_1})(t-T_1)}\|\bg_{T_1}\|^2.$$ Choose $T_1 = c$, and $T_2 = T_1 + \log(\|\bg_{T_1}\|^2 / \eps) / c^2$ to achieve error $\eps > 0$. Since $\|\bg_{T_1}\|^2 \leq 2^P (\|h_*\|^2 + \|\fNN^{pert}(\cdot;\bar\rho_{T_1})\|^2) \leq K$, we have that $T_1$ and $T_2$ are constants depending on $K,P,\balpha,\brho$. This proves strong $O(d)$-SGD learnability (with the variation that the activation function is perturbed at time $T_1$) almost surely over the Fourier coefficients $\balpha$ and the perturbation $\brho$.
\end{proof}

\subsection{Proof of Lemma~\ref{lem:mainmspclaim}}
\label{sec:mainmspclaim}

It only remains to show Lemma~\ref{lem:mainmspclaim}. To show this lemma, we will use the fact from Claim~\ref{claim:hathpoly} that $\det(\hat\bM)$ is a polynomial in all relevant parameters: $t,\ba,\balpha,\bm,\brho$.
\begin{claim}
There is a large enough integer $D$ depending on $L,P$ such that $\det(\hatbM)$ is a polynomial of degree at most $D$ in $\ba,\bm,\balpha,\brho$, and $t$.

\end{claim}
\begin{proof}
This is by writing $\det(\hat\bM) = \sum_{\bgamma} t^{\|\bgamma\|_1} \ba^{\bgamma} q_{\bgamma}(\balpha,\bm,\brho)$ where each $q_{\bgamma}$ is a polynomial, as proved in Claim~\ref{claim:hathpoly}.
\end{proof}

To study this polynomial, we first reduce to studying ``minimal'' MSP set structures, defined as follows.
\begin{definition}
We say that $\cS = \{S_1,\ldots,S_P\}$ is a \textit{minimal} MSP set structure if the sets can be ordered such that for each $i \in [P]$ we have $S_i \subset [i]$ and $i \in S_i$.
\end{definition}

The following claim shows that it is sufficient to restrict our attention to minimal MSP set structures.
\begin{claim}
Suppose that for every $P$ there are constants $l_P,L_{P,0}$ depending only on $P$ such that for any $L > L_{P,0}$, and every minimal MSP set structure $\cS \subseteq 2^{[P]}$, the polynomial
$\det(\hatbM) \mid_{(\alpha_S)_{S \not\in \cS} = \bzero}$ has a nonzero term with degree at most $l_P$ in $t$.

Then, for any $L > L_{P,0}$ and MSP set structure $\cS \subseteq 2^{[P]}$, the polynomial $\det(\hatbM) \mid_{(\alpha_S)_{S \not\in \cS'} = \bzero}$ has a nonzero term with degree at most $l_P$ in $t$.
\end{claim}
\begin{proof}
For any MSP set structure $\cS' \subseteq 2^{[P]}$, up to a permutation of the variables there is a minimal MSP set structure $\cS \subseteq 2^{[P]}$ such that $\cS \subseteq \cS'$. Since $\det(\hatbM)\mid_{(\alpha_S)_{S \not\in \cS} = \bzero}$ has a nonzero term with degree at most $l_P$, so does $\det(\hatbM) \mid_{(\alpha_S)_{S \not\in \cS'} = \bzero}$, because the former polynomial can be constructed from the latter by additionally setting $(\alpha_S)_{S \in \cS' \sm \cS} = 0$, which could only zero out monomials.
\end{proof}

Because of the above claim, for the remainder of this section, we fix a minimal MSP set structure $\cS = \{S_1,\ldots,S_P\}$. Let us analyze the behavior of the dynamics of $\hat{\bu}$ on a function $h_*(\bz) = \sum_{S \subseteq [P]} \alpha_S \chi_S(\bz)$ with this structure, i.e., with $\alpha_S = 0$ for all $S \not\in \cS$. Let us explicitly compute the leading order terms of the weights $\hat{\bu}_i^t$ using the recurrence relations for the simplified dynamics. Recall that $\hatbu_i^t(a) = \sum_{l=1}^L a^l t^l p_{il}(\balpha,\bm)$.

\begin{claim}\label{claim:pioirecurrence}
Suppose that $L > 2^P$. For each $i \in [P]$, define
\begin{align*}
o_i = 1 + \sum_{i' \in S_i \sm \{i\}} o_{i'}.
\end{align*}
We have
$p_{il}(\balpha,\bm) \mid_{(\alpha_S)_{S \not\in \cS} = \bzero} \equiv 0$ for all $l < o_i$, and for $l = o_i$ we have
$$p_{io_i}(\balpha,\bm) \mid_{(\alpha_S)_{S \not\in \cS} = \bzero} =  \frac{\alpha_{S_i} m_{|S_i|}}{o_i}\prod_{i' \in S_i \sm \{i\}} p_{i'o_{i'}},$$ with the convention that a product over an empty set is $1$ and a sum over an empty set is $0$.
\end{claim}
\begin{proof}
We prove this by induction on $l$ using the recurrence relations for $p_{il}$ derived in Lemma~\ref{lem:hatburecurrence}. For simplicity, we write $p_{il} = p_{il}((\alpha_S)_{S \in \cS},\bm) = p_{il}(\balpha,\bm) \mid_{(\alpha_S)_{S \not\in S} = \bzero}$. First consider the base case of $l = 1$. For any $i$ such that $o_i = 1$, we have $S_i = \{i\}$. Therefore, from the base case of the recurrence relations, we have $p_{io_i} = p_{i1} = t \alpha_{\{i\}} m_1$. On the other hand, if $o_i > 1$, then $S_i \neq \{i\}$. By the minimality of the MSP structure we have $\{i\} \not\in \cS$ so $\alpha_{\{i\}} = 0$. Therefore $p_{i1} = t \alpha_{\{i\}} m_1 = 0$.

For the inductive step, suppose $l \geq 2$ and that the result is true for $l' \in \{0,\ldots,l-1\}$. Now consider any $S \in \cS$, any $1 \leq r \leq L$ and any $(i_1,\ldots,i_r) \in [P]^r$ such that $\{i\} \oplus S \oplus \{i_1\} \dots \oplus \{i_r\} = \emptyset$. Consider also any $l_1,\ldots,l_r \in [L-1]^r$ such that $\sum_{r'} l_{r'} = l - 1$. Each of these corresponds to a possible contribution to $p_{il}$ in the recurrence relation of Lemma~\ref{lem:hatburecurrence}. Suppose that $l \leq o_i$.

\textit{Case 1}: Suppose there is $i' \in \{i_1,\ldots,i_r\}$ such that $o_{i'} \geq o_i$. Without loss of generality take $i' = i_1$. But since $l_1,\ldots,l_r \leq l-1 < o_i \leq o_{i'}$, we have $p_{i_1l_1} = p_{i'o_{i'}} = 0$ by the inductive hypothesis, so the terms in case 1 do not contribute.

\textit{Case 2}: Suppose for all $r' \in [r]$ we have $o_{i_{r'}} < o_i$. Then $i \in S$ since otherwise $i \in \{i\} \oplus S \oplus \{i_1\} \oplus \dots \oplus \{i_r\}$ and of course $o_{i} \geq o_{i}$. If $S = S_{i'}$ for some $i' > i$, then we have $i' \in S$. And, as a consequence $i' \in \{i_1,\ldots,i_r\}$, because otherwise $i' \in \{i\} \oplus S \oplus \{i_1\} \oplus \dots \oplus \{i_r\}$ However, $o_{i'} > o_i$ since $i \in S_{i'}$, so this is a contradiction. We conclude that $S = S_i$, and so $S_i \sm \{i\} = \{i_1,\ldots,i_r\}$. Since $\sum_{i' \in S_i \sm \{i\}} = o_i - 1$ and $\sum_{r'=1}^r i_{r'} = l-1 \leq o_i$, we conclude that either \textit{Case a}: there is some $r'$ such that $l_{r'} < o_{i_{r'}}$, or \textit{Case b}: $l_{r'} = o_{i_{r'}}$ for all $r' \in [r]$. In Case a, we have $p_{i_{r'}l_{r'}} = 0$ by the inductive hypothesis, so the term does not contribute to $p_{il}$. Case b occurs if and only if $l = o_i$ and $i_1,\ldots,i_r$ are a permutation of $S_i \sm \{i\}$. There are exactly $(|S_i| - 1)!$ such terms, so the recurrence relation for $p_{il}$ holds.
\end{proof}

For any $\bz \in \{+1,-1\}^P$, define the multivariable polynomial 
\begin{align*}
q_{\bz}(a,T_1,\balpha,\bm) = \<\hat\bu^{T_1}(a),\bz\> = \sum_{l=1}^L \sum_{i=1}^P (a T_1)^{l} p_{il}(\balpha,\bm).
\end{align*}

\begin{claim}\label{claim:qzminusqzprime}
There is a constant $L_0$ depending on $P$ such that for large enough truncation $L > L_0$, for any $\bz \neq \bz' \in \{+1,-1\}^P$, $\pd{}{a} (q_{\bz} - q_{\bz'})$ has a nonzero term of degree at most $2^{P-1}$ in $T_1$.
\end{claim}
\begin{proof}
Let us take a constant $L_0 = 2^P$. Then the low-order solutions to the recursion from Claim~\ref{claim:pioirecurrence} are valid. There must be an index $i \in [P]$ such that $z_i \neq z'_i$. Choose $i \in \{i' : z_{i'} \neq z'_{i'}\}$ such that $o_i$ is minimized, breaking ties in favor larger $i$. Consider the terms of $\pd{}{a} (q_{\bz} - q_{\bz'})$ which are of degree $o_i$ in $T_1$. The degree $o_i$ part is equal to
 \begin{align*}
 [T_1^{o_{i}}]\pd{}{a} (q_{\bz} - q_{\bz'}) =  \sum_{i'=1}^P (z_{i'} - z'_{i'}) o_i a^{o_i - 1} p_{i'o_i}(\balpha,\bm) = \sum_{\substack{i'=1 \\ z_{i'} \neq z'_{i'}}}^P (z_{i'} - z'_{i'}) o_i a^{o_i - 1} p_{i'o_i}(\balpha,\bm).
 \end{align*}
Notice that if $z_{i'} \neq z'_{i'}$, then have $o_{i'} \geq o_i$ by the choice of $i$. And if $o_{i'} > o_i$ then $p_{i'o_i} \equiv 0$ by Claim~\ref{claim:pioirecurrence}. So
\begin{align*}
[T_1^{o_{i}}]\pd{}{a} (q_{\bz} - q_{\bz'}) = \sum_{\substack{i'=1 \\ o_{i'} = o_i}}^P (z_{i'} - z'_{i'}) o_i a^{o_i - 1} p_{i'o_{i'}}(\balpha,\bm).
\end{align*}
By the recurrence relations for $p_{i'o_{i'}}$ in Claim~\ref{claim:pioirecurrence}, one can see that $p_{io_i}$ is a monomial with degree 1 in $\alpha_{S_i}$. On the other hand, for all $i' < i$, the polynomial $p_{i'o_{i'}}$ does not depend on $\alpha_{S_i}$. Therefore $[T_1^{o_{i}}]\pd{}{a} (q_{\bz} - q_{\bz'})$ is a nonzero polynomial. So $\pd{}{a} (q_{\bz} - q_{\bz'})$ has a nonzero degree $o_i$ term in $T_1$. One can prove using the recurrence relation of Claim~\ref{claim:pioirecurrence} inductively on $i$ that $o_i \leq 2^{i-1}$.
\end{proof}

Now consider the following matrix $\bN \in \R^{2^P \times 2^P}$ indexed by $\bz \in \{+1,-1\}^P$ and $j \in [2^P]$, and depending on some indeterminate scaling factor $\nu \in \R$,
\begin{align*}
N_{\bz,j} = \pd{^{j-1}}{a^{j-1}} \exp(\nu q_{\bz})
\end{align*}

We prove that $\det(\bN)$ has a low-order non-zero term in the analytic expansion of $T_1$ at $0$. This is an auxiliary result that will allow us to prove the corresponding result for $\det(\hat\bM)$.

\begin{claim}\label{claim:detnnonzeroterm}
There is a constant $L_0$ depending on $\cS \subseteq 2^{[P]}$ such that for large enough $L > L_0$, there exists $l \leq 2^{3P}$ where $$\frac{\partial^l}{(\partial T_1)^l} \det(\bN) \mid_{T_1 = 0}$$ equals a nonzero polynomial in $\nu,a,\balpha,\bm$. 
\end{claim}
\begin{proof}
By the chain rule we may write $\bN_{\bz,j} = \exp(\nu q_{\bz}) R_{\bz,j}$, for a function  $R_{\bz,j}(a,T_1,\balpha,\bm)$ defined inductively on $j$ as $R_{\bz,1}(a,T_1,\balpha,\bm) \equiv 1$, and $$R_{\bz,j+1} = R_{\bz,j} \pd{}{a} \nu q_{\bz} + \pd{}{a}R_{\bz,j}.$$ So $\det(\bN) = \left(\prod_{\bz \in \{+1,-1\}^P} \exp(\nu q_{\bz})\right) \det(\bR)$, where $\bR$ is the matrix with entries $R_{\bz,j}$.
Since each $R_{\bz,j}$ is a polynomial of degree $j-1$ in $\nu$, $\det(\bR)$ is a polynomial of degree at most $\sum_{j=1}^{2^P} j-1 = (2^P-1)(2^P)/2 = \binom{2^P}{2}$ in $\nu$. Let us consider the part of $\det(\bR)$ that has degree $\binom{2^P}{2}$ in $\nu$. This must come from the degree $j-1$ part of each $R_{\bz,j}$, which can inductively be shown to be $[\nu^{j-1}]R_{\bz,j} = (\pd{}{a} q_{\bz})^{j-1}$. So $[\nu^{\binom{2^P}{2}}]\det(\bR) = \det(\hat\bR)$, where $\hat\bR$ is the matrix with entries
\begin{align*}
\hat R_{\bz,j} = \Big(\pd{}{a} q_{\bz}\Big)^{j-1}.
\end{align*}
This matrix is Vandermonde, so its determinant is (up to a factor of $+1$ or $-1$):
\begin{align*}
\det(\hat{\bR}) = \prod_{\bz \neq \bz' \in \{+1,-1\}^P} \Big(\pd{}{a} (q_{\bz} - q_{\bz'})\Big).
\end{align*}
From Claim~\ref{claim:qzminusqzprime}, we know that for each distinct $\bz,\bz'$, we have that $(\pd{}{a} (q_{\bz} - q_{\bz'}))$ has a nonzero term of degree at most $2^{P-1}$ in $T_1$. Therefore $\det(\hat{\bR})$ has a nonzero term of degree at most $\binom{2^{P}}{2} 2^{P-1} \leq 2^{3P}$ in $T_1$. In particular, we have proved that $\det(\bR)$ is a polynomial in $\nu,a,T_1,\balpha,\bm$ that has a nonzero term of degree at most $2^{3P}$ in $T_1$. Let $0 \leq l \leq 2^{3P}$ be the smallest $l$ such that $[T_1^l] \det(\bR) \not\equiv 0$. Then we have
\begin{align*}
\frac{\partial^l}{(\partial T_1)^l} \det(\bN) \mid_{T_1 = 0} &= \frac{\partial^l}{(\partial T_1)^l}  \det(\bR)
\prod_{\bz} \exp(\nu q_{\bz}) \mid_{T_1 = 0} \\
&= \left(\frac{\partial^l}{(\partial T_1)^l}  \det(\bR)\right)
\prod_{\bz} \exp(\nu q_{\bz}) \mid_{T_1 = 0} \\
&= \left(\frac{\partial^l}{(\partial T_1)^l}  \det(\bR)\right) \mid_{T_1 = 0},
\end{align*}
since $q_{\bz} \mid_{T_1 = 0} \equiv 1$, since $T_1$ divides the polynomial $q_{\bz}$ by its definition.

\end{proof}

Now consider the following matrix $\hat{\bN} \in \R^{2^P \times 2^P}$. We will eventually compare the determinant of $\hat{\bN}$ to that of $\bN$. Each entry of $\hat{\bN}$ is a polynomial in $a,\balpha,\bm,\brho$
\begin{align*}
\hat{N}_{\bz,j} = \pd{^{j-1}}{a^{j-1}} \hat\sigma_{pert}(q_{\bz}).
\end{align*}

Let us prove that $\det(\hat\bN)$ has a low-order nonzero term in $T_1$ by comparing it to $\det(\bN)$.
\begin{claim}
For any $\cS$ there is large enough truncation parameter $L_0$, such that for $L > L_0$ there exists $l \leq 2^{3P}$ with
$\frac{\partial^l}{(\partial T_1)^l} \det(\hat{\bN}) \mid_{T_1 = 0} \not\equiv 0$.
\end{claim}
\begin{proof}
Suppose that we were to make the substitution $\rho_r = -m_r + \nu^r$ for each $r \in \{0,\ldots,2^{8P}\}$. Then we would get $\hat{N}_{\bz,j} = \sum_{r=0}^{2^{8P}} \frac{\nu^r}{r!} (q_{\bz}(a,T_1,\balpha,\bm))^r + \sum_{r=2^{8P}+1}^{L} \frac{m_r}{r!} (q_{\bz}(a,T_1,\balpha,\bm))^r$. Then since $T_1$ divides $q_{\bz}$ and $\sum_{r=0}^{2^{8P}} \frac{\nu^r}{r!} s^r$ is the first few order expansion of $\exp(\nu s)$, for any $l \leq 2^{3P}$, we have
\begin{align*}
\frac{\partial^l}{(\partial T_1)^l} \det(\hat{\bN}) \mid_{T_1 = 0} &= \frac{\partial^l}{(\partial T_1)^l} \sum_{\tau \in S_{2^P}} \sgn(\tau) \prod_{j=1}^{2^P} \hat{\bN}_{\tau(j),j} \mid_{T_1 = 0} \\
&= \frac{\partial^l}{(\partial T_1)^l} \prod_{j=1}^{2^P}  \left(\sum_{r=0}^{2^{8P}} \frac{\nu^r}{r!} (q_{\tau(j)}(a,T_1,\balpha,\bm))\right) \mid_{T_1 = 0} \\
&= \frac{\partial^l}{(\partial T_1)^l} \prod_{j=1}^{2^P}  \left(\exp(\nu q_{\tau(j)}(a,T_1,\balpha,\bm))\right) \mid_{T_1 = 0} \\
&= \frac{\partial^l}{(\partial T_1)^l} \det(\bN) \mid_{T_1 = 0}.
\end{align*}
Recall that by Claim~\ref{claim:detnnonzeroterm}, there is a $l \leq 2^{3P}$ such that $\frac{\partial^l}{(\partial T_1)^l} \det(\bN) \mid_{T_1 = 0}$ is a nonzero polynomial. Since we have derived the above by substituting $\rho_r = -m_r + \nu^r$, we must have that without substituting we have $\frac{\partial^l}{(\partial T_1)^l} \det(\hat{\bN}) \mid_{T_1 = 0}$ is a nonzero polynomial in $a,\balpha,\bm,\brho$.
\end{proof}

Furthermore, $\det(\hat\bN)$ is related to $\det(\hat\bM)$.
\begin{claim}
$\det(\hat\bN) = \frac{\partial}{\partial a_2} \frac{\partial^2}{(\partial a_3)^2} \dots \frac{\partial^{2^P-1}}{(\partial a_{2^P})^{2^P-1}} \det(\hat\bM) \mid_{a_1 = a_2 = \dots a_{2^P} = a}$.
\end{claim}
\begin{proof}
By linearity of the derivative,
\begin{align*}
\frac{\partial}{\partial a_2} \frac{\partial^2}{(\partial a_3)^2} \dots \frac{\partial^{2^P-1}}{(\partial a_{2^P})^{2^P-1}} \det(\hat\bM) &= \sum_{\tau \in S_{2^P}} \mathrm{sgn}(\tau) \frac{\partial}{\partial a_2} \frac{\partial^2}{(\partial a_3)^2} \dots \frac{\partial^{2^P-1}}{(\partial a_{2^P})^{2^P-1}} \prod_{j=1}^{2^P} \hat\sigma_{pert}(q_{\tau(j)}(a_j)) \\
&= \sum_{\tau \in S_{2^P}} \mathrm{sgn}(\tau) \prod_{j=1}^{2^P} \frac{\partial ^{j-1}}{(\partial a_j)^{j-1}} \hat\sigma_{pert}(q_{\tau(j)}(a_j)) \\
&= \det(\hat\bN).
\end{align*}
\end{proof}

Combining the above two claims allows us to conclude that there is a nonzero term in $\det(\hat\bM)$ that has low degree in $T_1$. This concludes the proof of the lemma, which implies the theorem.
\begin{proof}[Proof of Lemma~\ref{lem:mainmspclaim}]
By the above two claims, there is $l \leq 2^{3P}$ such that $$\frac{\partial^l}{(\partial T_1)^l} \frac{\partial}{\partial a_2} \frac{\partial^2}{(\partial a_3)^2} \dots \frac{\partial^{2^P-1}}{(\partial a_{2^P})^{2^P-1}} \det(\hat\bM) \mid_{a_1 = a_2 = \dots a_{2^P} = a, T_1 = 0} \not\equiv 0.$$
This implies that $\det(\hat\bM)$ has a nonzero term of degree $l \leq 2^{3P}$ in $T_1$.
\end{proof}

\clearpage

\section{Explicit sample-complexity bounds in all parameters}\label{sec:sample-complexity-explicit}

In this paper, we focused on the dependence of the sample complexity on the ambient dimension $d$. In particular, our main result shows that MSP is a necessary and nearly sufficient condition for a sparse function $h_*$ to be learnable in $n = C(\eps,\delta,h_*) d$ samples in the mean-field scaling (i.e., to achieve test error $\eps$ with probability $1 - \delta$). While this was not our goal, we note for the interested reader that our proof techniques provide explicit dependencies of the sample size in all parameters $\eps,\delta,h_*$. In this appendix, we gather these fully explicit sample-complexity bounds and leave for future work the task of improving them.

\subsection{Vanilla staircase functions}

As a first bound, let us naively use the propagation-of-chaos bound comparison between the dimension-free dynamics and batch-SGD \eqref{eq:bSGD} presented in Theorem \ref{thm:bSGD-to-DF}.

\begin{proposition}\label{prop:vanilla_explicit_propChaos}
Consider learning a vanilla staircase:
\[
h_* (\bz) = \alpha_{\{1\}} z_1 + \alpha_{\{1,2\}} z_1z_2 + \alpha_{\{1,2,3\}} z_1z_2z_3 + \ldots + \alpha_{\{1,\ldots,P\}} z_1 z_2 \cdots z_P\, ,
\]
such that $\sum_{j \in [P]} | \alpha_{\{1, \ldots , j\}}| \leq 1$ and denote $\alpha_* := \min_{j \in [P]} | \alpha_{\{1, \ldots , j\}}| $. Then there exist a constant $K >0$ that only depends on the activation $\sigma$ and a numerical constant $C>0$, such that the dynamics described in Section \ref{ssec:vanillaproof} reaches $\eps > 0 $ test error with probability at least $1 - 2^{-2^{CP}}$, with 
\[
n = C e^{ \log(1/\eps)^C ( K/\alpha_*)^{2^{CP}}} d\, , \qquad N = C e^{ \log(1/\eps)^C ( K/\alpha_*)^{2^{CP}}} \, . 
\]
\end{proposition}

\begin{proof}[Proof of Proposition \ref{prop:vanilla_explicit_propChaos}]
The limiting mean-field dynamics verify at $T_2$: 
\[
\| f_* - \hat f_{\NN} (\cdot; \rho_{T_2} ) \|_{L^2}^2 \leq \eps/4\, .
\]
With $\bw^0 = \bzero$ initialization, the mean-field and dimension free dynamics are the same and we can use the comparison bound between batch-SGD and the mean-field dynamics in Proposition \ref{prop:general-PDE-SGD-bound}, which yields that there exists a numerical constant $C>0$ such that if 
\begin{equation}\label{eq:Nn_req}
N = C K e^{K T_2^3} / \eps \, , \qquad \eta/b =  e^{- KT_2^3} \eps/ (CK) \, ,
\end{equation}
then $\| \hat f_{\NN} (\cdot ; \bTheta^{k_2} ) - \hat f_{\NN} (\cdot ; \rho_{T_2} ) \|_{L^2}^2 \leq \eps/4$ with probability at least $1 - 1/N$. Note that we choose $T_2 = \log ( K / \eps)/ \lambda_{\min} (\bK^{T_1})$, and $T_1 = \lambda_{\min} (\bM) /2$. By Lemma \ref{lem:staircase_lambda_min}, we have $\lambda_{\min} (\bM)  \geq 2^{-2^{-CP}}$. Furthermore, in Section \ref{app:outline_staircase}, we showed that
\[
\lambda_{\min} (\bK^{T_1}) \geq \{ \min_{S} D_S^2 \} \lambda_{\min} (\bM) /2 \, ,
\]
where 
\[
 \min_{S} D_S \geq (1/K) \cdot \prod_{s \in [P]} \nu_k (T_1) \geq (\alpha_* / K)^{2^{CP}} \, .
\]
Injecting this lower bound on $\lambda_{\min} (\bK^{T_1})$ in Eq.~\eqref{eq:Nn_req} yields the bounds in the proposition.
\end{proof}

This first bound uses a worst case bound that depends exponentially on the training time, which scales as $e^{1/\lambda_{\min} (\bK^{T_1})} \leq e^{e^{e^{P}}}$ because of Phase 2 of linear training. A more careful analysis of Phase 2 yields an error that scales as $1/\lambda_{\min} (\bK^{T_1}) \leq e^{e^P}$ (see Section \ref{ssec:linear-sgd-discrete}). This results in the following improved bound:

\begin{proposition}\label{prop:vanilla_explicit_SGD}
Follow the same setup as in Proposition \ref{prop:vanilla_explicit_propChaos}. Then there exist a constant $K >0$ that only depends on the activation $\sigma$ and a numerical constant $C>0$, such that $h_*$ is strongly $O(d)$-learnable with the following dependency on $n$ and $N$:
\[
n = \Big( \frac{K}{\alpha_*} \Big)^{2^{CP}}  \frac{\log(1/\eps)}{\eps^5} \log (1/\delta) d\, , \qquad N = \Big( \frac{K}{\alpha_*} \Big)^{2^{CP}}  \frac{\log(1/\delta)^2}{\eps^3} \, . 
\]
\end{proposition}

\begin{proof}[Proof of Proposition \ref{prop:vanilla_explicit_SGD}]
This follows from applying Lemma \ref{lem:reduce-to-cond-discrete} presented in Section \ref{ssec:reducetolambdamin-discrete}, using Lemma \ref{lem:K_to_M} and the lower bound on the kernel matrix provided in the proof of Proposition \ref{prop:vanilla_explicit_SGD}.
\end{proof}

\subsubsection{Technical lemma}

\begin{lemma}\label{lem:staircase_lambda_min}
Consider $\bM = ( \E_{a} [ a^{i + j - 2}] )_{i,j \in [2^P]} \in \R^{2^P \times 2^P}$ with $a \sim \Unif ([+1,-1])$, i.e., $M_{ij} = \frac{1}{i+j - 1}\delta_{i+j \equiv 0[2]}$. Then there exists $C>0$ independent of $P$ such that
\[
\lambda_{\min} (\bM) \geq 2^{ - 2^{CP}} \, .
\]
\end{lemma}

\begin{proof}[Proof of Lemma \ref{lem:staircase_lambda_min}]
We follow a similar argument as in the proof of Lemma \ref{lem:polynomial-anticoncentration}. First, note that
\[
\lambda_{\min} (\bM)  = \inf_{\| \bu \|_2 = 1} \bu^\sT \bM \bu = \inf_{\| \bu \|_2 = 1} \E_{a} [ h(a ; \bu )^2 ]\, , 
\]
where $h(a ; \bu ) = \sum_{j = 0}^{2^P - 1} u_j a^j$. Consider $P_l$ the degree-$l$ Legendre polynomial on $[-1,1]$ and denote
\[
P_l (z) = \sum_{j = 0}^l p_{l,j} z^j \, .
\]
The polynomial $h(a ; \bu )$ can be decomposed in this basis as
\[
h(a ; \bu ) = \sum_{l = 0}^{2^D- 1} g_l P_l (z)\, .
\]
In particular, we have $u_j = \sum_{l\geq j} g_l p_{l,j}$. Therefore, 
\[
\begin{aligned}
1=\| \bu \|_2^2 =&~ \sum_{j =0}^{D-1} \Big( \sum_{l \geq j} g_l p_{l,j} \Big)^2 \\
\leq&~ 2^{2^P} \cdot \Big\{ \max_{j,l = 0 , \ldots , 2^P - 1} p_{l,j}^2 \Big\}\cdot  \sum_{j,l = 0}^{2^P-1} g_l^2 \leq  2^{2^{CP}} \cdot \E_{a} [ h(a ; \bu )^2 ] \, ,
\end{aligned}
\]
where we used that $|p_{l,l-2k}| = 2^{-l}{{l}\choose{k}} {{2l - 2k}\choose{l}} \leq 2^{Cl}$ for some $C>0$ and $|p_{l,l-2k - 1}| = 0$ otherwise. This concludes the proof.
\end{proof}

\subsection{Merged-staircase functions in the smoothed complexity model}

Similarly, the tighter analysis of second-layer training in Lemmas~\ref{lem:K_to_M} and \ref{lem:reduce-to-cond-discrete} can be used to provide a complexity bound for learning MSP functions under a smoothed model of complexity. This corresponds to making the dependencies of Theorem~\ref{thm:discrete-msp} on parameters other than $d$ tighter.

\begin{proposition}\label{prop:msp_explicit}
Let $h_* : \{+1,-1\}^P \to \R$ be any function normalized so that $\max_{\bz} |h_*(\bz)| \leq 1/P$ and $\cS = \{S : \hat{h}_*(S) \neq 0\}$ is an MSP set structure. Then for any $0 < \mu < 1$, there is a function $\tilde{h} : \{+1,-1\}^P \to \R$ such that
\begin{align*}
\tilde{h}(\bz) = h_*(\bz) + \sum_{S \in \cS} c_S \chi_S(\bz)
\end{align*}
for some $c_S$ such that $|c_S| \leq \mu$, and such that $f_*(\bx) = \tilde{h}(\bz)$ can be learned by SGD to $\epsilon$ error in $n = d \cdot \poly(2^{2^{2^{O(P)}}}, (1/\mu)^{2^{2^{O(P)}}}, 1/\epsilon)$ samples with a neural network of width $N = \poly(2^{2^{2^{O(P)}}}, (1/\mu)^{2^{2^{O(P)}}}, 1/\epsilon)$ neurons.
\end{proposition}
\begin{proof}
The proof is the same as the proof of Theorem~\ref{thm:discrete-msp}. The main difference is that we lower-bound $\lambda_{\mathrm{min}}(\bK^{k_1})$ explicitly and apply Lemmas~\ref{lem:K_to_M} and \ref{lem:reduce-to-cond-discrete} as in Proposition~\ref{prop:vanilla_explicit_SGD}. Let $\balpha = (\alpha_S)_{S \in 2^{[P]}}$, where $\alpha_S = \E_{\bz}[\tilde{h}(\bz)\chi_S(\bz)]$. Also consider activation function $\sigma(u) = \sum_{i=0}^L \frac{m_i}{i!} u^i$ for $L = 2^{8P}$. By Lemmas~\ref{lem:k-min-from-det-m} and \ref{lem:poly-anticoncentration-and-bar-hat-comparison} we know that we can bound the minimum eigenvalue of the expected kernel $\lambda_{\min}(\bK^{k_1})$ in terms of the coefficients of $\det(\bM(\bzeta,\balpha,\bm))$, viewed as a polynomial in $\bzeta$. Here, $\bM(\bzeta,\bxi,\brho)$ is the $2^P \times 2^P$ matrix given by the recurrence relations in Lemma~\ref{lem:pki-baru-discrete} and the definition in \eqref{eq:bM-def-MSP}. Notice that by construction, for any $\bz \in \{+1,-1\}^P$ and $j \in [2^P]$ we have that
\begin{align*}
(L!)^{L+1} \cdot M_{\bz,j}(\bzeta,\bxi,\brho)
\end{align*}
is a polynomial in $\bzeta,\bxi,\brho$ with integral coefficients and has degree at most $(L!)^{O(L+1)} = 2^{2^{O(P)}}$. Therefore the polynomial $N_{\bz,j}(\bzeta,\bphi,\brho)$ constructed in \eqref{eq:N-construction-discrete-msp} is such that $$(L!)^{(L+1)2^P} \det(\bN(\bzeta,\bphi,\brho))$$ has integral coefficients in $\bzeta,\bphi,\brho$ and degree at most $(L!)^{O((L+1)2^P)} = 2^{2^{O(P)}}$. Let $\kappa = \min(\mu, 1/L^2)$, and choose $\bm \sim \Unif[-\kappa, \kappa]^{\otimes L}$ and $c_S \sim \Unif[-\kappa, \kappa]$ for each $S \in \cS$. By the polynomial anti-concentration of Lemma~\ref{lem:not-shifted-polynomial-anticoncentration} and a Markov bound, with probability at least $9/10$ we must have that $\det(\bM(\bzeta,\balpha,\bm))^2$ has some monomial of $\bzeta$ with coefficient at least $\kappa^{2^{2^{O(P)}}}$. This implies by \eqref{eq:Kmin-bound-smoothed-msp} that
$$\lambda_{\mathrm{min}}(\bK^{k_1}) \geq (1/2^{2^{O(P)}}) \eta^{2^{2^{O(P)}}} \kappa^{2^{2^{O(P)}}}.$$
Taking learning rate $\eta = 2^{-\Theta(P)}$ which is small enough, we obtain the result that 
$$\lambda_{\mathrm{min}}(\bK^{k_1}) \geq \min(\mu^{2^{2^{O(P)}}}, 1/2^{2^{2^{O(P)}}}).$$
The result when we combine with Lemmas~\ref{lem:K_to_M} and \ref{lem:reduce-to-cond-discrete}.
\end{proof}

\subsection{Merged-staircase functions when including all degree-1 monomials}\label{ssec:msp-one-included}

We conjecture that the optimal dependence on $P$ for learning vanilla staircase functions with unregularized SGD by two-layer neural networks should be on the order of $\exp(O(P))$, but the results of Propositions~\ref{prop:vanilla_explicit_SGD} and \ref{prop:msp_explicit} have $P$ dependence on the order of at least $\exp(\exp(\Omega(P))$ and $\exp(\exp\exp(\Omega(P)))$, respectively. Therefore, we focus here on improving our understanding of the $P$ dependence. We prove in Proposition~\ref{prop:perturb-linear-deterministic-cond} that SGD can succeed with $\exp(O(P))$ sample complexity dependence, but our result has two qualifications:
\begin{itemize}
    \item \textbf{Stronger non-degeneracy assumption}. We assume that all of the {\em degree-1} terms of $h_*$ -- i.e., $\hat{h}_*(\{1\}),\ldots,\hat{h}_*(\{P\})$ -- are nonzero and are sufficiently distinct. This is in contrast to Propositions~\ref{prop:vanilla_explicit_SGD} and \ref{prop:msp_explicit}, which did not make this assumption.
    \item \textbf{Ad hoc activation function}. We use an ad hoc activation function. This is in contrast with Propositions~\ref{prop:vanilla_explicit_SGD} and \ref{prop:msp_explicit}, which apply to ``most'' activation functions.
\end{itemize}
Because of the stronger non-degeneracy assumption, the neural network can learn even when we only train the first layer for $k_1 = 1$ step and then train the second layer for a sufficiently large number of steps, $k_2$.

\subsubsection{The bump and gradient bump functions}\label{sssec:bump-gradbump} We define the $\sigma_{\mathrm{bump}}$ and $\sigma_{\mathrm{gradbump}}$ functions that are used to construct our ad hoc activation function.
For any $\alpha < \beta$ and $\gamma < (\beta - \alpha) / 2$, define the ``bump'' function:
\begin{align*}
\sigma_{\mathrm{bump}}(u;\alpha,\beta,\gamma) = \begin{cases}1, &  u \in [\alpha+\gamma,\beta-\gamma] \\
0, & u \not\in [\alpha,\beta] \\
(6(u-\alpha)^2/\gamma^2 - 15(u-\alpha)/\gamma+10)(u-\alpha)^3 /\gamma^3, & u \in [\alpha, \alpha+\gamma] \\
(6(\beta-u)^2 / \gamma^2 - 15(\beta-u)/\gamma+10)(\beta-u)^3 /\gamma^3, & u \in [\beta-\gamma,\beta].
\end{cases}
\end{align*}
This function is twice-differentiable since $\sigma_{\mathrm{bump}}(\alpha) = \sigma_{\mathrm{bump}}'(\alpha) = \sigma''_b(\alpha) = \sigma_{\mathrm{bump}}(\beta) = \sigma_{\mathrm{bump}}'(\beta) = \sigma_{\mathrm{bump}}''(\beta) = 0$ and it also satisfies $$\|\sigma_{\mathrm{bump}}\|_{\infty} \leq 1, \quad \|\sigma'_b\|_{\infty} \leq \frac{2}{\gamma}, \quad \mbox{ and } \quad \|\sigma''_b\|_{\infty} \leq \frac{6}{\gamma^2}.$$ If we take $\gamma$ small relative to $\beta - \alpha$, then this function is effectively an indicator on the set $[\alpha,\beta]$. See Figure~\ref{fig:bump_function} for an example.
\begin{figure}
    \centering
    \includegraphics[scale=0.5]{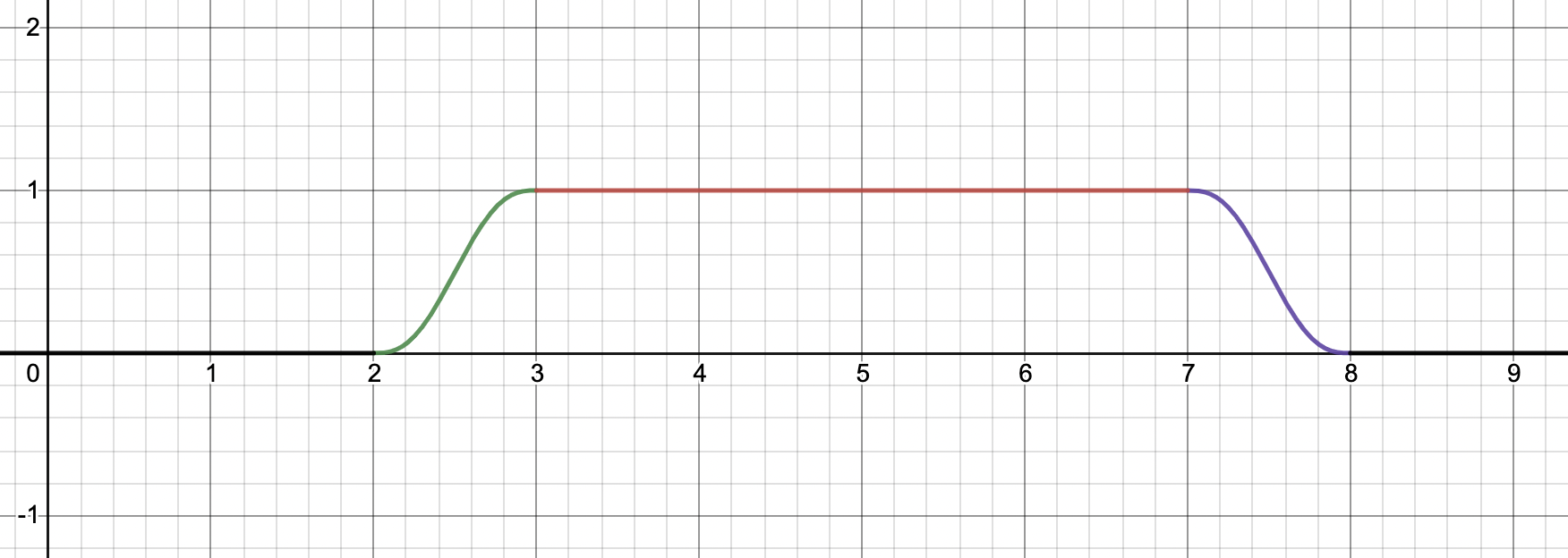}
    \caption{The bump function $\sigma_{\mathrm{bump}}(u;\alpha,\beta,\gamma)$ with $\alpha = 2,\beta = 8,\gamma = 1$.}
    \label{fig:bump_function}
\end{figure}

Also define the ``gradient bump'' function:
\begin{align*}
\sigma_{\mathrm{gradbump}}(u;\kappa) = \begin{cases}
u, & u \in [-1,1] \\
-9u^{5}-68u^{4}-198u^{3}-276u^{2}-184u-48, & u \in [-2,-1] \\
-9u^{5}+68u^{4}-198u^{3}+276u^{2}-184u+48, & u \in [1, 2] \\
0, & u \not\in [-2,2].
\end{cases}
\end{align*}
This function is twice-differentiable since the pieces agree up to second derivative. Furthermore, it satisfies $$\sigma_{\mathrm{gradbump}}(0) = 0, \quad \sigma_{\mathrm{gradbump}}'(0) = 1,$$ and it and its first and second derivatives are bounded by $O(1)$. We will use this function to ensure that our activation function has nonzero gradient at zero. See Figure~\ref{fig:gradient_bump_function}.
\begin{figure}
    \centering
    \includegraphics[scale=0.5]{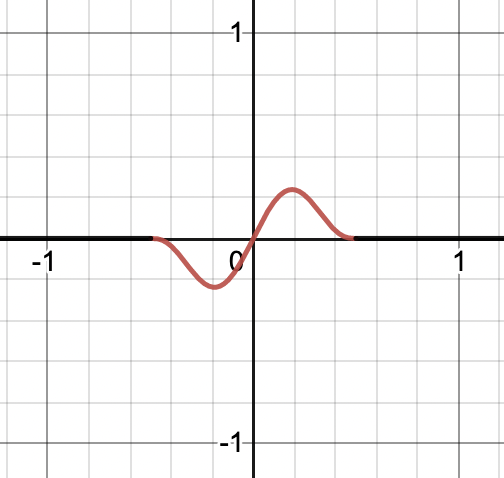}
    \caption{The gradient bump function $\sigma_{\mathrm{gradbump}}(u)$.}
    \label{fig:gradient_bump_function}
\end{figure}

\subsubsection{Statement of Proposition~\ref{prop:perturb-linear-deterministic-cond} and Corollary~\ref{cor:one-msp}}

\begin{proposition}\label{prop:perturb-linear-deterministic-cond}
Let $h_* : \{+1,-1\}^P \to \R$, normalized for convenience so that $\max_{\bz} |h_*(\bz)| \leq 1/P$. Define $\bc = [\hat{h}(\{1\}),\ldots, \hat{h}(\{P\})]$ be the vector of Fourier coefficients of degree 1, and suppose that $\bc$ satisfies the following conditions for some $0 < \mu < 1$.
\begin{itemize}
\item Bounded away from zero: for all $\bz \in \{+1,-1\}^P$,
\begin{align}\label{eq:cz-lower-bound}
|\<\bc, \bz\>| \geq \mu\,. 
\end{align}
\item Diverse: for any $\bz \neq \bz' \in \{+1,-1\}^P$,
\begin{align}\label{eq:cz-mult-different}
\frac{\<\bc,\bz\>}{\<\bc,\bz'\>} \not\in (1-\mu, 1+\mu)\,.
\end{align}
\end{itemize}

Then $f_*(\bx) = h_*(\bz)$ is SGD-learnable to any error $0 < \eps < 1$ with probability at least $1 - \delta$ in $d \cdot \poly(2^P \log(1/\delta) / (\mu \eps))$ samples on a network with $N = \poly(2^P \log(1/\delta) / (\mu \eps))$ neurons.
Furthermore, this SGD learnability is with initialization $\rho_0 = \Unif[-1,1] \otimes \delta_{\bzero}$ and activation function
\begin{align*}
\sigma(t) = \sigma_{\mathrm{gradbump}}(u) + \gamma^2 \sigma_{\mathrm{bump}}(u;\alpha,\beta,\gamma),
\end{align*}
where $\gamma = \mu / 2^{8P}$, $\alpha = \mu / 2^{4P}$, $\beta = \alpha + \mu / 2^{4P}$.
\end{proposition}

The following result shows that the condition of Proposition~\ref{prop:perturb-linear-deterministic-cond} is true under a smoothed complexity model, where we perturb any function $h_*$ slightly on its degree-1 Fourier coefficients.
\begin{corollary}\label{cor:one-msp}
For any $h : \{+1,-1\}^P \to \R$ with the normalization $\max_{\bz} |h_*(\bz)| \leq 1/(2P)$ and any $0 < \mu' < 1$, there is a $\tilde{h} : \{+1,-1\}^P \to \R$ satisfying the conditions of Proposition~\ref{prop:perturb-linear-deterministic-cond} with $\mu = \mu' / 2^{O(P)}$ and $\tilde{h}(\bz) = h(\bz) + \<\bdelta, \bz\>$ and $\|\bdelta\|_{\infty} \leq \mu'$.
\end{corollary}

\subsubsection{Proof of Proposition~\ref{prop:perturb-linear-deterministic-cond} and Corollary~\ref{cor:one-msp}}

\begin{proof}[Proof of Proposition~\ref{prop:perturb-linear-deterministic-cond}]
Initialize with $\rho_0 = \Unif[-1,1] \otimes \delta_{\bzero}$. Let the regularization parameter on the second layer be $\lambda = \eps \cdot \poly(\mu / 2^{P})$ for some large enough polynomial. Train the first layer with one step of \eqref{eq:bSGD}, i.e., take $k_1 = 1$ with learning rate $\eta = 1$. Train the second layer for $k_2 \geq \poly(2^P \log(1/\delta) / (\eps \mu))$ steps for a large enough polynomial. Let the number of neurons be $N \geq \poly(2^P \log(1/\delta) / (\eps \mu))$ and let the batch size be any $b \geq \poly(2^P \log(1/\delta) / (\eps \mu)) \log(N) d$ for large enough polynomials.

\paragraph{Computing 1 step of dynamics} Let us analyze the \eqref{eq:d-DF-PDE} dynamics for $k_1 = 1$ step. Since $\barbu^0(a) = \bzero$, and $\sigma(0) = 0$, $\sigma'(0) = 1$, there is a simple formula for $\barbu^1(a)$:
\begin{align*}
\barbu^1(a) &= \barbu^0(a)-a \E_{\bz}[(\fNN(\bz;\bar\rho_0) - h_*(\bz)) \sigma'(\<\barbu^0(a),\bz\>)\bz] \\
&= a \E_{\bz}[ h_*(\bz) \sigma'(0)\bz] \\
&= a \E_{\bz}[ h_*(\bz)\bz] \\
&= a \bc,
\end{align*}
where $\bc = [\hat{h}_*(\{1\}),\ldots, \hat{h}_*(\{P\})]$. By the choice of activation, and since $|a| \leq 1$ and $|\<\bc,\bz\>| \leq 1$, we have
\begin{align}\label{eq:plug-into-ad-hoc-activation}
\sigma(\<\barbu^1(a), \bz\>) = a \<\bc, \bz\> + \gamma^2 \sigma_{\mathrm{bump}}(a \<\bc, \bz\>; \alpha,\beta,\gamma),
\end{align}

\paragraph{Definition of events $E_{\bz,j}$ and $F_j$} Write $\bu_j^k = \bu(a_j^0)$ to denote the first-layer weight corresponding to neuron $j$. For any $\bz \in \{+1,-1\}^P$ and $j \in [N]$, define the event
$$E_{\bz,j} = \{a_j^0 \<\bc, \bz\> \in [\alpha+\gamma,\beta-\gamma]\}.$$

Notice that if event $E_{\bz,j}$ holds, then by \eqref{eq:cz-mult-different}, for any $\bz' \neq \bz$ we have $$a_j^0 \<\bc, \bz'\> \not\in [\beta(1-\mu), \alpha(1+\mu)] \supset [\alpha,\beta].$$ So by \eqref{eq:plug-into-ad-hoc-activation} and the definition of $\sigma_{\mathrm{bump}}$, under event $E_{\bz,j}$, we have
\begin{align}
\sigma(\<\barbu_j^1,\bz\>) = a_j^0 \<\bc, \bz\> + \gamma^2, \quad  \mbox{ and } \quad  \sigma(\<\barbu_j^1,\bz'\>) = a_j^0 \<\bc, \bz'\> \mbox{ for all } \bz \neq \bz'.
\end{align}

Also, for each $j \in [N]$ define the event $F_{j}$, which is
\begin{align*}
F_j = \{a_j^0 \geq 1/2\}
\end{align*}

Since $\min |\<\bc,\bz\>| > 2\alpha$ for all $\bc$, under event $F_j$ we have $a_{j}^0 \<\bc,\bz\> \not\in [\alpha,\beta]$. So by \eqref{eq:plug-into-ad-hoc-activation}, under event $F_j$
\begin{align}
\sigma(\<\barbu_j^1,\bz\>) = a_j^0 \<\bc, \bz\> \mbox{ for all $\bz$}.
\end{align}

\paragraph{Lower-bounding the event probabilities}

Recall that $a_j^0 \sim \mathrm{Unif}[-1,1]$. By \eqref{eq:cz-lower-bound}, we know that $\min_{\bz} |\<\bc, \bz\>| \geq \mu > \beta - \gamma$, so for any $\bz, j$
\begin{align}\label{eq:low-bound-prob-in-interval}
\P[E_{\bz,j}] = \frac{\beta - \alpha - 2 \gamma}{2 |\<\bc,\bz\>|} \geq \mu / 2^{4P + 1}.
\end{align}
And clearly since $a_j^0 \sim \mathrm{Unif}[-1,1]$,
\begin{align}
\P[F_j] = 1/4.
\end{align}

Let $E_{good}$ be the event that
\begin{align*}
E_{good} = \{\frac{\min_{\bz} |\{j : E_{\bz,j} \mbox{ holds}\}|}{N} \geq \mu / 2^{4P+2}\} \cap \{\frac{|\{j : F_j \mbox{ holds}\}|}{N} \geq 1/5\}.
\end{align*}

By a Hoeffding bound and a union bound
$$
\P[E_{good}] \geq 1 - \delta, 
$$
as long as $$N \geq 2^{\Omega(P)} \log(1/\delta).$$

\paragraph{Lower-bounding the empirical kernel eigenvalues}
Note that $\sigma$ satisfies $\|\sigma\|_{\infty},\|\sigma'\|_{\infty},\|\sigma''\|_{\infty} \leq O(1)$. Therefore, by Lemma~\ref{lem:reduce-to-cond-discrete}, the final loss is bounded in terms of 
$\lambda_{\min}(\bM \bM^{\top})$, where $\bM$ is the feature matrix matrix at iteration $k_1 = 1$ with entries
$$\bM_{\bz,j} = \frac{1}{\sqrt{N}} \sum_{j=1}^N \sigma(\<\barbu_j^{k_1},\bz\>).$$

Under event $E_{good}$, let us prove that the minimum eigenvalue of the empirical kernel $\bM \bM^{\top}$ is lower-bounded by $\lambda_{\mathrm{\min}}(\bM\bM^{\top}) \geq \mu^5 2^{-100P}$. Consider any test function $\phi : \{+1,-1\}^P \to \R$. Suppose by contradiction that \begin{align}\label{eq:ad-hoc-contradiction}\E_{\bz,\bz'}[\phi(\bz) \{\bM \bM^{\top}\}_{\bz,\bz'} \phi(\bz')] \leq \mu^52^{-100P}\E_{\bz}[\phi(\bz)^2].\end{align} Under event $E_{good}$, we have
\begin{align*}
\E_{\bz,\bz'}[\phi(\bz) \{\bM \bM^{\top}\}_{\bz,\bz'} \phi(\bz')] &\geq \frac{1}{N} \sum_{j : F_j \mbox{holds}} \E_{\bz,\bz'}[\phi(\bz) \sigma(\<\barbu_j^1, \bz\>) \sigma(\<\barbu_j^1, \bz'\>) \phi(\bz')] \\
&= \frac{1}{N} \sum_{j : F_j \mbox{holds}} \E_{\bz,\bz'}[\phi(\bz) (a_j^0)^2 \<\bc,\bz\>\<\bc,\bz'\> \phi(\bz')] \\
&\geq \frac{1}{20} \E_{\bz}[\phi(\bz) \<\bc,\bz\>]^2.
\end{align*}
In order to avoid contradiction with \eqref{eq:ad-hoc-contradiction} it follows that
\begin{align}\label{eq:ad-hoc-linear-upper-bound}
\E_{\bz}[\phi(\bz) \<\bc,\bz\>]^2 \leq 20 \mu^5 \cdot 2^{-100P}\E_{\bz}[\phi(\bz)^2]
\end{align}
Similarly, for any $\bz'' \in \{+1,-1\}^P$
\begin{align*}
\E_{\bz,\bz'}&[\phi(\bz) \{\bM \bM^{\top}\}_{\bz,\bz'} \phi(\bz')] \geq \frac{1}{N} \sum_{j : E_{\bz'',j} \mbox{holds}} \E_{\bz,\bz'}[\phi(\bz) \sigma(\<\barbu_j^1, \bz\>) \sigma(\<\barbu_j^1, \bz'\>) \phi(\bz')] \\
&\geq \frac{1}{N} \sum_{j : E_{\bz'',j} \mbox{holds}} \E_{\bz,\bz'}[\phi(\bz) (a_j^0 \<\bc,\bz\> + \gamma^2 1(\bz = \bz'')) (a_j^0 \<\bc,\bz'\> + \gamma^2 1(\bz' = \bz'')) \phi(\bz')] \\
&\geq \frac{1}{N} \sum_{j : E_{\bz'',j} \mbox{holds}} \gamma^4 \phi(\bz'') \phi(\bz'') - 2\gamma^2|\E_{\bz}[\phi(\bz)\<\bc,\bz\>\phi(\bz'')]| - \E_{\bz}[\phi(\bz) \<\bc,\bz\>]^2 \\
&\geq \mu 2^{-4P-2} (\gamma^4 \phi(\bz'') \phi(\bz'') - 2\gamma^2|\E_{\bz}[\phi(\bz)\<\bc,\bz\>\phi(\bz'')]| - \E_{\bz}[\phi(\bz) \<\bc,\bz\>]^2) \\
&\geq \mu 2^{-4P-2} (\gamma^4 \phi(\bz'')^2 - 2\gamma^2|\phi(\bz'')|\sqrt{20 \cdot \mu^52^{-100P} \E_{\bz}[\phi(\bz)^2]} - 20 \cdot  \mu^52^{-100P} \E_{\bz}[\phi(\bz)^2]),
\end{align*}
where in the last line we use \eqref{eq:ad-hoc-linear-upper-bound}.
So in order to avoid contradiction with \eqref{eq:ad-hoc-contradiction} we conclude that
\begin{align*}
\phi(\bz'')^2 \leq \mu^5\frac{2^{-40P}}{\mu \gamma^4}\E_{\bz}[\phi(\bz)^2].
\end{align*}
However, since $\mu^5 2^{-40P} / (\mu \gamma^4) = 2^{-8P} < 1$, we get a contradiction by taking $\bz'' = \arg\max_{\bz} \phi(\bz)^2$. Therefore, we conclude that $\lambda_{\mathrm{min}}(\{\bM \bM^{\top}\}_{\bz,\bz'}) \geq \mu^5 2^{-100P}$. Plugging this into the guarantees for the linear regression of the last layer (Lemma~\ref{lem:reduce-to-cond-discrete}) concludes the proof.

\end{proof}

\begin{proof}[Proof of Corollary~\ref{cor:one-msp}]
We prove that such a $\bdelta$ exists by the probabilistic method. Let $\bc = \E_{\bz}[h(\bz)\bz]$ and let $\tilde{\bc} = \E_{\bz}[\tilde{h}(\bz)\bz] = \bc + \bdelta$. If we take random $\bdelta \sim \Unif([0,\mu'])^{\otimes P}$ then for any distinct pair $\bz,\bz'$ such that without loss of generality $z_1 = 1 \neq -1 = z'_1$, we have \begin{align*}\P_{\bdelta}[|\<\tilde{\bc}, \bz - \bz'\>| \leq \mu' 2^{-3P}] &= \P_{\bdelta}[|\<\bc,\bz-\bz'\> + \<\bdelta,\bz-\bz'\>| \leq \mu' 2^{-3P}] \\
&= \P_{\bdelta}[\delta (z_1 - z_1') \in \<\bc,\bz'-\bz\> + \<\bdelta,\bz'-\bz\> + [-\mu' 2^{-3P},\mu' 2^{-3P}]] \\
&= \P_{\bdelta}[2\delta_1 \in \<\bc,\bz'-\bz\> + \<\bdelta_{-1},\bz'_{-1}-\bz_{-1}\> + [-\mu' 2^{-3P},\mu' 2^{-3P}]] \\
&\leq 2^{-3P+1}.\end{align*}
Similarly, for any $\bz$,
\begin{align*}
\P_{\bdelta}[|\<\tilde{\bc},\bz\>| \leq \mu' 2^{-3P}] \leq 2^{-3P+1}.
\end{align*}

Therefore, taking a union bound over the $\binom{2^P}{2} < 2^{2P-1}$ distinct pairs $\bz,\bz'$, we have $$\P_{\bdelta}[|\<\tilde{\bc},\bz\>| \geq \mu' 2^{-3P} \mbox{ and } |\<\tilde{\bc},\bz - \bz'\>| \geq \mu' 2^{-3P} \mbox{ for all } \bz \neq \bz'] > 1 - (2^{2P-1} + 2^P)(2^{-3P+1}) \geq 0,$$ so by the probabilistic method a deterministic choice of $\bdelta$ satisfying $|\<\tilde{\bc}, \bz\>| \geq \mu' 2^{-3P}$ and $|\<\tilde{\bc}, \bz\> - \<\tilde{\bc}, \bz\>| \geq \mu' 2^{-3P}$ for all $\bz \neq \bz'$ exists. These conditions are sufficient to satisfy the conditions of Proposition~\ref{prop:perturb-linear-deterministic-cond} with $\mu = \mu' / 2^{O(P)}$.
\end{proof}

\clearpage

\section{Lower bounds on learning with linear methods}

We recall the general definition for linear methods from Section \ref{sec:linear-methods} in the main text. Given a Hilbert space $(\cH , \< \cdot, \cdot \>_{\cH} )$, a feature map $\psi : \{+1 , -1 \}^d \to \cH$, an empirical loss function $L:\R^{2n} \to  \R \cup {\infty}$ and a regularization parameter $\lambda >0$, a linear method construct from data points $(y_i , \bx_i )_{i \in [n]}$ a prediction model $\hf( \bx) := \< \hba , \psi (\bx ) \>_{\cH}$ where $\hba \in \cH$ is obtained by minimizing a regularized empirical risk functional 
\begin{equation}\label{eq:linear_method_bis}
\hba = \argmin_{\ba \in \cH} \Big[ L \big((y_i, \< \ba , \psi (\bx_i) \>)_{i\in [n]}\big) + \lambda \| \ba \|_{\cH}^2 \Big]\, .
\end{equation}
Recall that we denote $q = \dim  (\cH)$.

\begin{example}
Popular examples of linear methods include
\begin{itemize}
    \item[(a)] \textit{Random Feature models:} take $(\bw_i )_{i \in [N]} \sim_{iid} \tau$, $\bw_i \in \cV$, and an activation $\phi : \cX \times \cV \to \R$, then the Hilbert space and the feature map are defined by $\cH = \spn \{ \phi (\cdot ; \bw_i ) : i \in [N] \}$ and $\psi ( \bx) = ( \phi( \bx ; \bw_1) , \ldots , \phi ( \bx ; \bw_N ) )$. For generic examples, we have $q = \dim (\cH) =N$ almost surely.

    \item[(b)] \textit{Kernel methods:} take $\cH$ a reproducing kernel Hilbert space (RKHS) with reproducing kernel $K : \cX \times \cX \to \R$. There exists a Hilbert space $(\cF , \< \cdot, \cdot \>_{\cF})$ (the feature space) and a feature map $\psi : \cX \to \cF$ such that $K ( \bx_1 , \bx_2 ) = \< \psi(\bx_1) , \psi (\bx_2) \>_{\cF}$ and $\cH = \{ \< \ba , \psi (\cdot ) \>_{\cF}: \ba \in \cF \}$. We have typically $q = \dim (\cH) = \infty$.
\end{itemize}
Ridge regression corresponds to taking the functional: $L\big((y_i, \hf_i )_{i\in[n]}\big) = \frac{1}{n} \sum_{i\in [n]} \big(y_i - \hf_i \big)^2 $.
\end{example}

We will be interested in providing lower bounds on the number of samples necessary to learn some classes of functions for any linear methods. We first present the following general dimension-based (see discussion bellow) approximation lower bound that is a slight variation of \cite{hsu2021approximation,hsudimension,kamath_dim}; it improves on \cite{hsu2021approximation,hsudimension} for target functions that are not (almost) orthogonal, and it uses the operator norm of the gram matrix rather than its min-eigenvalue as in \cite{kamath_dim}.

\begin{proposition}[Dimension lower bound]\label{prop:dimension_lower_bound} 
Let $\cR$ be a Hilbert space with inner product denoted by $\< \cdot, \cdot \>_\cR$. Fix $\cF = \{f_1 , \ldots , f_M \} \subset \cR$ a set of target functions with $\| f_i \|_{\cR}^2 = \< f_i , f_i \>_{\cR} = 1 $ for all $i \in [M]$. Let $\cT$ be a (potentially random) finite-dimensional subspace of $\cR$, with $r = \E_{\cT} [ \dim (\cT) ] < \infty$.

Define the average\footnote{This is a  lower-bound on the worst-case approximation error considered in \cite{kamath_dim}.} approximation error of the target functions $\cF$ by the subspace $\cT$ 
\[
\eps := \frac{1}{M} \sum_{i\in [M]} \E_{\cT} \Big[ \inf_{g \in \cT} \| g - f_i \|_{\cR}^2 \Big] \, ,
\]
and  $\bG = ( \< f_i , f_j \>_{\cR} )_{ij \in [M]}$ the Gram matrix associated to the $f_i$'s. Then
\begin{equation}\label{eq:dim_r_app}
    r \geq \frac{M}{\| \bG \|_{\op}} (1 - \eps) \, .
\end{equation}
\end{proposition}

Note that the results in \cite{hsu2021approximation,hsudimension} are simply obtained by using 
\[
\| \bG \|_{\op} \leq \| \id \|_{\op} + \| \bG - \id \|_{\op} \leq 1 + \| \bG - \id \|_F = 1 + \sqrt{ \sum_{i\neq j} \< f_i , f_j \>^2_{\cR} }\, .
\]

\begin{proof}[Proof of Proposition \ref{prop:dimension_lower_bound}]
In the proofs in \cite{hsu2021approximation,hsudimension}, we simply replace the Boas-Bellman inequality by (for any $g\in \cR$)
\begin{equation}\label{eq:altern_Boas}
\begin{aligned}
\sum_{i \in [M]} \< g , f_i \>_{\cR}^2 =&~ \Big\< g , \sum_{i \in [M]}  \< g, f_i \>_{\cR} f_i  \Big\>_{\cR} \\
\leq &~ \| g \|_{\cR} \Big( \sum_{ij \in [M]} \< g, f_i \>_{\cR} \< g, f_j \>_{\cR} \< f_i, f_j \>_{\cR} \Big)^{1/2} \\
= &~ \| g \|_{\cR} \big( \bb^\sT \bG \bb \big)^{1/2} \leq \| g \|_{\cR} \| \bG \|_{\op}^{1/2} \| \bb \|_2 \, ,
\end{aligned}
\end{equation}
where we denoted $\bb = (\< g, f_1 \>_{\cR} , \ldots , \< g, f_M \>_{\cR} )$. Noticing that $\| \bb \|_2^2$ is equal to the left-hand side of Eq.~\eqref{eq:altern_Boas}, we get
\[
\sum_{i \in [M]} \< g , f_i \>_{\cR}^2 \leq \| \bG \|_{\op} \| g \|_{\cR}^2 \, ,
\]
which together with $\|g\|_{\cR} \leq \frac{1}{M} \sum_{i=1}^M ||f_i||_{\cR} \leq 1$ yields the improved bound \eqref{eq:dim_r_app}.
\end{proof}

Let us explain how to derive lower-bounds on the performance of linear methods using Proposition \ref{prop:dimension_lower_bound}. Consider $\cR = L^2 (\cX)$ and $\cT$ the space of functions $f = \< \ba , \psi ( \cdot ) \>_{\cF}$ with $\ba \in \spn \{ \psi (\bx_i) : i \in [n] \}$. We can consider $\cT$ random or fixed conditional on $\psi$ (e.g., random feature map) and the $\bx_i$'s. We always have $r \leq \min (p,n)$. Consider learning a set of $M$ functions $\cF = \{ f_1 , \ldots , f_M \}$ with the linear estimator obtained by \eqref{eq:linear_method_bis}. From the above discussion, we must have that the estimator $\hf \in \cT$ and the generalization error is lower bounded by the approximation error $\| f_i - \hf \|_{L^2}^2 \geq \inf_{g \in \cT} \| f_i - g \|_{L^2}^2$. Therefore $\eps$ lower bound the average generalization error over learning $\cF$. Therefore, Proposition \ref{prop:dimension_lower_bound} implies the following: if the average generalization error over $\cF$ is less than $\eps$, then we must have
\begin{equation}\label{eq:dimBound1}
\min(n,q) \geq \frac{M}{\| \bG \|_{\op}} (1 - \eps) \, .
\end{equation}
This bound is a dimension lower bound in the sense that it does not assume anything about the statistical model (e.g., the $\bx_i$ can be arbitrary and do not have to be independent), only that the estimator lies in a $\min (n,q)$-dimensional subspace $\cT$: this subspace can be a good approximation of $M$ orthogonal functions only if $\min (n,q) \geq \Omega(M)$.

To get Proposition \ref{prop:dimens-lower-bound} in the main text, we make the following two modifications of the bound in Proposition \ref{prop:dimension_lower_bound}.
In Eq.~\eqref{eq:dimBound1}, we upper bound $\| \bG \|_{\op} \leq \| \bG \|_{1,\infty} = \max_{i \in [M]} \sum_{j \in [M]} | \< f_i , f_j \>_{\cR} | $. Second, some linear subspaces $\Omega \subseteq \cR$ are harder to fit for linear methods (see for example \cite{ghorbani2021linearized,mei2021generalization}). For instance, vanilla staircase functions of large degree contain monomials of large degree that have a large dimension lower-bound, but the overall staircase functions do not have a large dimension lower-bound per se. 
We next present a corollary that applies to any decomposition $\cR=\Omega \oplus \Omega^{\perp}$, and distinguishes the error incurred on each of the two orthogonal subspaces. Denote $\proj_{\Omega}$ and $\proj_{\Omega^{\perp}} = \id - \proj_{\Omega}$ the orthogonal projections onto $\Omega$ and $\Omega^\perp$ respectively.

\begin{corollary}\label{cor:dimension_lower_bound} 
Let $\cR$ be a Hilbert space with inner product denoted by $\< \cdot, \cdot \>_\cR$ and $\Omega$ a linear subspace of $\cR$. Fix $\cF = \{f_1 , \ldots , f_M \} \subset \cR$ a set of target functions with $\| \proj_{\Omega} f \|_{\cR}^2 = 1$ for all $f \in \cF$. Define $\E_{\cF}$ the expectation over $f \sim \Unif (\cF)$. Let $\cT$ be a (potentially random) finite-dimensional subspace of $\cR$, with $r = \E_{\cT} [ \dim (\cT \cap \Omega) ] < \infty$.

Define the average approximation error on $\Omega$ of the target functions $\cF$ by the subspace $\cT$ 
\[
\eps  :=\frac{1}{M} \sum_{i\in [M]} \E_{\cT} \Big[ \inf_{g \in \cT} \| \proj_{\Omega} (g - f_i) \|_{\cR}^2 \Big] \, ,
\]
Then
\begin{equation}
    r \geq \frac{1 - \eps}{\max_{i \in [M]}\frac{1}{M} \sum_{j \in [M]}  | \< f_i , \proj_{\Omega} f_j \>_{\cR} |  }\, .
\end{equation}
\end{corollary}

This is a direct consequence of Proposition \ref{prop:dimension_lower_bound} whith $\cR$ and $\cT$ replaced by $\Omega$ and $\cT \cap \Omega$, and the target functions by $\cF ' = \{ \proj_{\Omega} f_1 , \ldots , \proj_{\Omega} f_M \}$. Proposition \ref{prop:dimens-lower-bound} in the main text is simply Corollary \ref{cor:dimension_lower_bound} rewritten in the context of linear methods.

Consider a set  of target functions $\cF$ such that $\| \proj_{\Omega} f \|_{L^2}^2 = 1 - \kappa $ and $\| (\id - \proj_{\Omega} ) f \|_{L^2}^2 = \kappa $ for any $f \in \cF$. If the averaged generalization error is less than $1 - u$, we can take $\eps = (1- u)/(1 - \kappa)$ and get
\begin{equation}\label{eq:dimBound}
\min(n,q) \geq \frac{u - \kappa}{\max_{i \in [M]}\frac{1}{M} \sum_{j \in [M]}  | \< f_i , \proj_{\Omega} f_j \>_{\cR} |  } \, .
\end{equation}

Let us apply this bound to the examples described in the main text. We take $\cX = \{+1 , -1 \}^d$. First consider $\Omega$ the span of all degree $k$ monomials and a target function $f_*$ such that $\| \proj_{\Omega} f_* \|_{L^2}^2 = \| \proj_{\Omega^\perp} f_* \|_{L^2}^2 = \frac{1}{2}$, and $\proj_{\Omega} f_*$ is supported on $m$ monomials $\{S_1 , \ldots , S_m \}$, with $S_i \subseteq [d]$, $|S_i| = k$:
\[
\proj_{\Omega} f_*  ( \bx) =  \sum_{ i \in [m]} \alpha_{S_i} \chi_{S_i} ( \bx) \, .
\]
We consider the class of functions 
\[
\cF_* = \Big\{ f_* ( \tau ( \bx) )  : \tau \in \Pi (d) \Big\} \, .
\]
$\cF_*$ is the smallest class of functions containing $f_*$ that is invariant under a permutation of the input coordinates. The generalization error $\E_{\cF_*} \big[ \| f - \hf \|_{L^2}^2 \big] $ corresponds to the test error with uniform prior distribution over all permutation of the input space. Note that any method that is equivariant with respect to permutations (e.g., kernel methods with inner-product kernel) will have the same generalization error $\E_{\cT} \big[ \| f\circ \tau - \hf \|_{L^2}^2 \big] $ for any $\tau \in \Pi (d)$.

Applying Eq.~\eqref{eq:dimBound}, we obtain the following lower bound:

\begin{proposition}\label{prop:F1}
For any linear method, in order to get an average generalization error over $\cF_*$ that is smaller than $1/2 \cdot (1 - \eta)$, we must have
\[
\min(n,q) \geq \frac{\eta}{m}{{d}\choose{k}} \, .
\]
\end{proposition}

\begin{proof}[Proof of Proposition \ref{prop:F1}]
Fix $\tau_1 \in \Pi (d)$. We have
\[
\begin{aligned}
\E_{\tau_2} \big[ | \<  f_* \circ \tau_1 , \proj_{\Omega} f_* \circ \tau_2 \>_{L^2 } | \big] = &~ \sum_{ij \in [m]} \alpha_{S_i} \alpha_{S_j} \E_{\tau_2} \big[ \< \chi_{S_i} \circ \tau_1 , \chi_{S_j} \circ \tau_2 \>_{L^2} \big] \\
=&~  \sum_{ij \in [m]} \alpha_{S_i} \alpha_{S_j} \cdot \frac{k! (d - k)!}{d!}\, \\
\leq&~ m \frac{k! (d - k)!}{d!}\sum_{i \in [m]} \alpha_{S_i}^2  = \frac{m}{2}\frac{k! (d - k)!}{d!} \, .
\end{aligned}
\]
We can then apply Eq.~\eqref{eq:dimBound} with $u := 1/2 \cdot (1 + \eta)$ and $\kappa = 1/2$.
\end{proof}

Proposition \ref{prop:F1} shows that for $k$ fixed, $n = \Omega_d (d^k / m)$ samples are necessary to learn $\cF_*$. 

As a second example, consider the vanilla staircase function of degree $P$: 
\[
f_P ( \bx) = \frac{1}{\sqrt{P}} \sum_{i = 1}^P x_1 \cdots x_i \, ,
\]
and the function class of all staircase functions of degree $P$:
\[
\cF_P = \Big\{ f_P ( \tau ( \bx) )  : \tau \in \Pi (d) \Big\} \, .
\]

\begin{proposition}\label{prop:FP}
Let $P  \leq d/ 2$ and $\eta \in (0,1)$. For any linear method, in order to get an average generalization error over $\cF_P$ that is smaller than $1 - \eta$, we must have
\[
\min(n,q) \geq \frac{\eta}{2}{{d}\choose{\lfloor \frac{\eta P}{2} \rfloor}} \, .
\]
\end{proposition}

\begin{proof}[Proof of Proposition \ref{prop:FP}]
Denote now $S_i = \{1, 2 , \ldots , i \}$ and $\proj_{\ell}$ the projection on every monomials of degree at least $\ell$. Notice that $\| \proj_{\ell} f_P \|_{L^2}^2 = 1 - \frac{\ell -1}{P} $ and $\| (\id - \proj_{\ell} ) f_P \|_{L^2}^2 = \frac{\ell -1}{P} $.

Fix $\tau_1 \in \Pi (d)$. We have
\[
\begin{aligned}
\E_{\tau_2} \big[ | \<  f_P \circ \tau_1 , \proj_{\ell} f_P \circ \tau_2 \>_{L^2 } | \big] =&~ \frac{1}{P} \sum_{i = \ell}^P \E_{\tau_2} \big[ \< \chi_{S_i} \circ \tau_1 , \chi_{S_i} \circ \tau_2 \>_{L^2} \big] \\
=&~ \frac{1}{P } \sum_{i = \ell}^P \frac{i! (d - i)!}{d!} \leq  \frac{1}{{{d}\choose{\ell}}}\, .
\end{aligned}
\]
The proposition follows by taking $\ell = \lfloor \frac{\eta P}{2 } \rfloor$ and applying Eq.~\eqref{eq:dimBound} with $\proj_{\ell}$, $u = \eta$ and $\kappa = \frac{\ell -1 }{P}$.
\end{proof}

In our case, we are interested in $P = \omega_d(1)$. Letting $\eta$ decay at moderate rate, such as $\eta =1 / \sqrt{P}$ in Proposition \ref{prop:FP}, we get the following superpolynomial lower bound on the number of samples $n \geq d^{\omega_d ( 1 )}$.

\clearpage

\section{Technical results}

In this appendix, we gather a few technical results needed to prove the main results in this paper.

\subsection{Bound between batch-SGD and discrete mean-field dynamics} 
\label{app:discrete-prop-of-chaos}

While the results in \cite{mei2019mean} are written for continuous-time dynamics, we note that their proof can be easily adapted to the discrete-time regime, as described in Appendix \ref{app:strong-discrete}. More precisely, following the notations in \cite{mei2019mean}, we compare the solution $(\bTheta^k)_{k \in \naturals}$ of batch-SGD:
\begin{equation}\label{eq:disc-dyn1}
    \btheta_j^{k+1} = \btheta_j^k + \frac{1}{b} \sum_{i \in [b]} \{ y_{ki} - \hat f_{\NN} ( \bx_{ki} ; \bTheta^k ) \} \cdot \bH_k  \nabla_{\btheta} \sigma_* ( \bx_{ki} ; \btheta )  - \bH_k \bLambda \btheta_j^k\, ,
\end{equation}
to the solution $(\rho_k)_{k \in \naturals}$ of the discrete mean-field dynamics:
\begin{equation}\label{eq:disc-dyn2}
    \btheta^{k+1} = \btheta^k + E_{(y,\bx)} \big[ \{ y - \hat f_{\NN} ( \bx ; \rho_k ) \}  \bH_k  \nabla_{\btheta} \sigma_* ( \bx ; \btheta ) \big]  - \bH_k \bLambda \btheta^k\, .
\end{equation}

We consider the same assumptions as \cite[Theorem 1]{mei2019mean}, with the difference that $\rA 1$ is replaced by $\rA1' : \| \bLambda \|_{\op} , \| \bH_k \|_{\op} \leq K$.

\begin{proposition}[Discrete propagation-of-chaos]\label{prop:discrete_prop} Assume that conditions $\rA1',\rA2$-$\rA4$ in \cite{mei2019mean} hold and let $k_0 \in \naturals$. There exists a constant $K$ depending only on the constants in the assumptions (in particular independent of $d,k_0$) such that:

\begin{itemize}
    \item[(A)] \emph{Fixed second-layer:} \[
    \sup_{k = 0, \ldots , k_0 } \big\Vert \hat f_{\NN} (\cdot; \bTheta^k ) - \hat f_{\NN} (\cdot; \rho_{k} ) \big\Vert_{L^2} \leq K e^{K k_0} \left\{ \frac{\sqrt{\log N} + z}{\sqrt{N}} +  \frac{\sqrt{d+\log N}+z}{\sqrt{b}}  \right\} \, ,
    \]
    with probability at least $1 - e^{-z^2}$.
    
        \item[(B)] \emph{Training both layers:} \[
    \sup_{k = 0, \ldots , k_0 } \big\Vert \hat f_{\NN} (\cdot; \bTheta^k ) - \hat f_{\NN} (\cdot; \rho_{k} ) \big\Vert_{L^2} \leq K e^{e^{K k_0}} \left\{ \frac{\sqrt{\log N} + z}{\sqrt{N}} +  \frac{\sqrt{d+\log N}+z}{\sqrt{b}}  \right\} \, ,
    \]
    with probability at least $1 - e^{-z^2}$.
\end{itemize}

\end{proposition}

\begin{proof}[Proof of Proposition \ref{prop:discrete_prop}]
The proof of this proposition follows by adapting the proof of \cite[Theorem 1]{mei2019mean} to the discrete setting described above (see also Appendix \ref{app:proof-MF-approx}). In particular, part (A) (fixed second layer coefficients) follows from the Appendix B in \cite{mei2019mean}: the comparison between discrete and continuous gradient is not needed anymore, and the only difference is in the first part of Proposition 16 in \cite{mei2019mean}, which can simply be rewritten by noting that
\[
\begin{aligned}
&~ \| \obtheta_i^{k+1} - \ubtheta_i^{k+1} \|_2^2 \\
\leq&~ | \<  \obtheta_i^{k+1} - \ubtheta_i^{k+1},  \obtheta_i^{k} - \ubtheta_i^{k}\>| + | \< \obtheta_i^{k+1} - \ubtheta_i^{k+1}, (\obtheta_i^{k+1} - \obtheta_i^{k}) -  (\ubtheta_i^{k+1} -\ubtheta_i^{k}) \> | \\
\leq&~ \| \obtheta_i^{k+1} - \ubtheta_i^{k+1} \|_2 \|  \obtheta_i^{k} - \ubtheta_i^{k}\|_2 + | \< \obtheta_i^{k+1} - \ubtheta_i^{k+1}, (\obtheta_i^{k+1} - \obtheta_i^{k}) -  (\ubtheta_i^{k+1} -\ubtheta_i^{k}) \> |\, ,
\end{aligned}
\]
and the rest of the proof follows similarly.

For part (B), the main difference comes from bounding $a^k$: we have
\[
|a^{k+1}| \leq | a^k | +  \E_{\bx} \big[ | \hg_k ( \bx) | | \sigma ( \bx; \bw^k ) | \big] \leq | a^k | + K + K \| a^k \|_{\infty}\, ,
\]
where we denoted $\| a^k \|_{\infty} = \sup_{a \in \text{supp}(\rho_k)} | a |$ and $\hg_k ( \bx)  = f_* (\bx) - \hf_{\NN} (\bz ; \rho_k)$ and used that by assumption $\| f_*\|_\infty, \| \sigma\|_{\infty} \leq K$ and $\| \hf_{\NN} \|_{\infty} \leq \int |a | \| \sigma \|_{\infty} \rho_k (\de \btheta) \leq K $. We can then use the discrete Gr\"onwall inequality to get $\| a^t \|_{\infty} \leq Ke^{Kk_0}$. This explains the worse dependency (double exponential) in $k_0$ in the bound, than for continuous time, where one can use properties of continuous gradient flows to get a bound on $a^t$ linear in time. With this modification, the rest of the proof follow by adapting Appendix C in \cite{mei2019mean}, where we can assume that the activation function $\sigma_* (\bx ; \btheta)$ is bounded by $Ke^{Kk_0}$. 
\end{proof}

\subsection{Discrete-time analysis of SGD on second layer}\label{ssec:linear-sgd-discrete}

In this section, we provide a tighter analysis for the training of the second-layer weights by SGD using the bias-variance decomposition, instead of a propagation-of-chaos argument. This technical tool is used in Section \ref{sec:sample-complexity-explicit} to provide tighter sample-complexity bounds. Indeed, the second-layer weights are trained during a time $T = \log (1 / \eps) / \lambda_{\min}$, where $\lambda_{\min}$ is the minimum eigenvalue of the kernel matrix associated to the neural network after training the first layer weights (see for example Section \ref{app:outline_staircase}). If we naively use the propagation-of-chaos comparison between the dimension-free dynamics and batch-SGD \eqref{eq:bSGD} in Theorem~\ref{thm:bSGD-to-DF}, we need
\[
e^{K_1T^7} \left\{ \sqrt{\frac{P + \log(d)}{d}} + \sqrt{\frac{\log N}{N}} +  \sqrt{\frac{d+\log N}{b}}  \sqrt{ \eta} \right\} \leq \eps \, ,
\]
and therefore we need the sample size to scale as $n = d e^{\log (1 / \eps)^7 / \lambda_{\min}^7 }$. In this section, we consider instead a direct analysis of \eqref{eq:bSGD} on the second layer weights and we obtain the tighter requirement  $n = d/(\eps \lambda_{\min})^C$ for some constant $C>0$.

\paragraph{Training setup:} We assume that the weights of the first layer are fixed at some $\bW^0 \in \R^{N \times d}$. We train the weights $\ba$ of the second layer and obtain a discrete-time dynamics $\ba^k$ on these weights. Namely, given initial weights $\bW^0 \in \R^{N \times d}$ and $\ba^0 \in \R^N$, we train $\ba^k$ with the batch-SGD dynamics \eqref{eq:bSGD}, with the step size $\eta > 0$ and regularization parameter $\lambda^a = \lambda > 0$ on $\ba$.
\begin{align*}
\bw_j^{k+1} &= \bw_j^0, \\
a_j^{k+1} &= a_j^{k} + \eta  \Big[ \frac{1}{b} \sum_{i\in [b]} \{ y_{ki} - \hat f_{\NN} ( \bx_{ki} ; \bTheta^k ) \}  \sigma  ( \< \bw_j^k , \bx_{ki} \> ) \bx_{ki} \Big] - \eta \lambda a_j^k \,, \nonumber
\end{align*}
where at each time step $k$ we take fresh i.i.d. data samples $\{(y_{ki},\bx_{ki})\}_{i \in [b]}$.
In the analysis, it is convenient to define the feature map which depends on $\bW^0$ as
$$\phi(\bx) = \frac{1}{\sqrt{N}} [\sigma(\<\bw_1^0,\bx\>) \dots \sigma(\<\bw_N^0,\bx\>)] \in \R^N.$$
With this notation, the neural network function while only the second layer is being trained can be written as:
\begin{align*}
\fNN(\bx;\ba,\bW^0) = \frac{1}{\sqrt{N}}\<\ba, \phi(\bx)\>,
\end{align*}
This is simply a linear function in $\ba$ with a fixed feature map $\phi$ depending on $\bW^0$. Therefore, the training of $\ba$ can be studied with classical ideas for analyzing linear methods. Namely, define the regularized loss\footnote{The $1/N$ scaling in the regularization is needed because each neuron is scaled as $1/N$ in the expression for $\fNN$.}:
$$\ell(\ba) = \frac{1}{2}\EE_{\bx}\Big[(h_* (\bz) - \frac{1}{\sqrt{N}}\langle\ba, \phi(\bx)\rangle)^2 \Big] + \frac{\lambda}{2 N} \|\ba\|^2.$$

The main result of this subsection is that if there exists a low-norm ``certificate'' $\ba^{cert}$ that achieves small loss, then batch-SGD will achieve a loss that is approximately upper-bounded by the loss at $\ba^{cert}$, after a short number of iterations.
This is a key ingredient in our more quantitative bounds on the sample complexity in Section~\ref{sec:sample-complexity-explicit}.

\begin{lemma}[Suffices to prove existence of certificate]\label{lem:cert-suffices-sgd}
Let $B_1,B_2 > 0$ be such that $|y| \leq \sqrt{B_1}$ and $\|\phi(\bx)\|^2 \leq B_2$ almost surely. For any step size $0 < \eta \leq 1/(B_2 + \lambda)$, any $\delta > 0$, and any $\ba^{cert} \in \R^N$, and any time step $k$, with probability at least $1 - \delta$
\begin{align*}
\ell(\ba^k) &\leq \ell(\ba^{cert}) + 2(B_2 + \lambda) \Big((1-\eta\lambda)^{2k} \Big(\frac{\|\ba^0\|^2}{N} + \frac{2\ell(\ba^{cert})}{\lambda}\Big) \\
&\quad\quad\quad\quad\qquad\qquad\qquad\qquad + \left(\frac{9B_1B_2}{\lambda^2 b N} +  \frac{18(B_2 + \lambda)^2 \ell(\ba^{cert})}{\lambda^3 b}\right)\log(1/\delta)\Big).
\end{align*}
\end{lemma}

\begin{remark}[Relation of Lemma~\ref{lem:cert-suffices-sgd} to prior work]
The proof of this lemma relies on the well-known bias-variance decomposition idea in for analyzing least-squares regression with SGD (e.g., \cite{bach2013non,jain2017markov}). However, we wish to prove a statement about the final iterate of SGD, and most works analyze averaged iterates rather than the final iterate. Standard bounds for last-iterate SGD (e.g., \cite{shamir2013stochastic,jain2019making}) do not apply because they assume Lipschitzness of the loss function, which cannot be assumed because we use the squared loss. Furthermore, the bound of \cite{jain2018parallelizing} for last-iterate batch-SGD holds in expectation, rather than with high probability. Nevertheless, our proof is a straightforward modification of \cite{jain2018parallelizing}, making stronger assumptions and obtaining a suboptimal rate in order to get a simpler proof.
\end{remark}

We first prove the following lemma, where $\ba^*$ is the minimizer of $\ell$, guaranteed to be unique by strong convexity:

\begin{lemma}\label{lem:distgapsgd}
Under the same conditions as Lemma~\ref{lem:cert-suffices-sgd}, for any time step $k > 0$, with probability at least $1 - \delta$ the following bound holds:
$$\|\ba^k- \ba^*\| \leq (1 - \eta \lambda)^k \|\ba^0 - \ba^*\| + \frac{3(\sqrt{B_1B_2} + (B_2 + \lambda)\|\ba^*\|)\sqrt{\log(1/\delta)}}{\lambda \sqrt{b}}$$
\end{lemma}
\begin{proof}

Define $\bH = \frac{1}{N}\E[\phi(\bx) \otimes \phi(\bx)] + \frac{\lambda}{N} \id$ and $\bv = \frac{1}{\sqrt{N}}\E[\phi(\bx) y]$. The loss at $\ba$ can be rewritten as:
$$\ell(\ba) = \frac{1}{2}\E[y^2] - \<\bv, \ba\>  + \frac{1}{2} \<\ba \otimes \ba, \bH\>.$$

Define the gap to optimality $\bw^k = \ba^k - \ba^*$. We track the evolution of $\bw^k$. Let $\bP^k = \id - \eta (\frac{1}{b} \sum_{i=1}^b \phi(\bx_{ki}) \otimes \phi(\bx_{ki}) + \lambda \id)$. Note that $0 \preceq \bP^k \preceq (1 - \eta \lambda \id)$, since $\|\phi(\bx_{ki})\|^2 \leq B_2$, and $\eta \leq 1/(B_2 + \lambda)$. We have
\begin{align*}
\bw^{k+1} &= \bP^k \ba^k + \frac{\eta}{b} \sum_{i=1}^b y_{ki} \phi(\bx_{ki}) - \ba^* \\
&= \bP^k \bw^k - \frac{\eta}{b} \sum_{i=1}^b y_{ki} \phi(\bx_{ki}) + \eta(\frac{1}{b} \sum_{i=1}^b \phi(\bx_{ki}) \otimes \phi(\bx_{ki}) + \lambda \id) \ba^* \\
&= \bP^k \bw^k + \eta \bzeta^k,
\end{align*}
where for any $k \geq 0$, we have the noise vector
\begin{align*}
\bzeta^k = \frac{1}{b} \bxi^{ki} \mbox{ and } \bxi^{ki} = y_{ki} \phi(\bx_{ki}) - (\phi(\bx_{ki}) \otimes \phi(\bx_{ki}) + \lambda \id) \ba^*
\end{align*}
Recursively expanding this, we have the well-known ``bias-variance'' decomposition $$\bw^k = \Big(\prod_{l=0}^{k-1} \bP^l\Big) \bw^0 + \eta \sum_{j=0}^{k-1} \Big(\prod_{l=1}^j \bP^{l} \Big) \bzeta^j.$$
The first term tracks how close $\ba^t$ would be to $\ba^0$ if there were no noise, and the second term controls how much error the noise in the batch-SGD contributes if we had started at the optimal solution $\ba^*$.

By the spectral norm bound on $\bP^l$ and the triangle inequality, we have
\begin{align}\|\bw^k\| \leq \underbrace{(1 - \eta\lambda)^k \|\bw^0\|}_{\text{(Term 1)}} + \underbrace{\eta \sum_{j=0}^{k-1} (1- \eta \lambda)^{k-1-j} \|\bzeta^j\|}_{\text{(Term 2)}}.\label{ineq:wtbound}\end{align}

The first term is already essentially in the form that we want. Let us bound the second term:
\begin{claim}\label{claim:bxibound}
For any time step $j$, and $i \in [b]$, $\|\bxi^{ji}\| \leq B_3 := \sqrt{B_1B_2} + (B_2 + \lambda)\|\ba^*\|$ almost surely.
\end{claim}
\begin{proof}
By triangle inequality and the almost-sure bounds on $y_{ji}$ and $\phi(\bx_{ji})$.
\end{proof}
\begin{claim}
For any time-step $j$, $\EE[\|\bzeta^j\|] \leq B_3 / \sqrt{b}$.
\end{claim}
\begin{proof}
By the first-order optimality conditions on $\ba^*$, we have $\bH \ba^* = \bv$, so $\E[\bxi^{ji}] = \bv - \bH \ba^* = \bzero$. Furthermore, $\bxi^{ji}$ and $\bxi^{ji'}$ are independent for all $i \neq i' \in [b]$.
By Cauchy-Schwarz, and the bound from Claim~\ref{claim:bxibound}, 
$$\E[\|\bzeta^j\|] \leq \sqrt{\E[\|\bzeta^j\|^2]} = \sqrt{\E\Big[\frac{1}{b^2} \sum_{i,i'=1}^b \bxi^{ji} \cdot \bxi^{ji'}\Big]} = \sqrt{\frac{1}{b}\E[ \|\bxi^{j1}\|^2]} \leq B_3 / \sqrt{b}.$$
\end{proof}
\begin{claim}
The second term in \cref{ineq:wtbound} is bounded with probability $1 - \delta$:
$$\P\Big[\text{(Term 2)} \geq \frac{3B_3\sqrt{\log(1/\delta)}}{\lambda\sqrt{b}}\Big] \leq \delta$$
\end{claim}
\begin{proof}
Group the samples $\{(\by_{ji},\bx_{ji})\}_{j \in \{0,\ldots,k-1\},i \in [b]}$ into groups where $i$ is the same, letting $$Z_i = ((\by_{0,i},\bx_{0,i}),(\by_{1,i},\bx_{1,i}),\ldots,(\by_{k-1,i},\bx_{k-1,i}))$$ denote the collection of $i$th samples at all time-steps, for all $i \in [b]$. We can write the term that we want to bound as a function of these samples as $$\text{(Term 2)} = h(Z_1,\ldots,Z_b) = \sum_{j=0}^{k-1} (1-\eta \lambda)^{k-1-j} \|\bzeta^j\|.$$
By the previous claim, the expectation is bounded by
$$\E[h(Z_1,\ldots,Z_b)] \leq \frac{B_3}{\lambda \sqrt{b}}.$$ We bound deviation from the expectation using McDiarmid's inequality. Note that if we replace $Z_1$ with a new independent draw $\tilde{Z}_1 = ((\tilde{\by}_{0,i},\tilde{\bx}_{0,i}),(\tilde{\by}_{1,i},\tilde{\bx}_{1,i}),\ldots,(\tilde{\by}_{k-1,i},\tilde{\bx}_{k-1,i}))$, then by triangle inequality almost surely
\begin{align*}
|h(Z_1,Z_2,\ldots,Z_b) - h(\tilde{Z}_1,Z_2,\ldots,Z_b)| \leq \eta \sum_{j=0}^{k-1} (1 - \eta \lambda)^{k-1} \frac{1}{b} \|\bxi^{j,1} - \tilde{\bxi}^{j,1}\| \leq \frac{2 B_3}{b\lambda}.
\end{align*}
By symmetry, the same is true of replacing any $Z_i$ with an independent draw $\tilde{Z}_i$. Furthermore, $Z_1,\ldots,Z_b$ are independent, so by McDiarmid's inequality
\begin{align*}
\P\Big[h(Z_1,\ldots,Z_b) \geq \E[h(Z_1,\ldots,Z_b)] + \frac{2B_3 \sqrt{\log(1/\delta)}}{\lambda \sqrt{b}}\Big] \leq \delta.
\end{align*}
Combining the above bounds yields the claim.
\end{proof}
The above claim combined with \cref{ineq:wtbound} proves the lemma.
\end{proof}

\begin{lemma}\label{lem:lossgapsgd}
Under the same conditions as Lemma~\ref{lem:distgapsgd}, for any $\delta > 0$, with probability $1 - \delta$ we have
$$\ell(\ba^k) \leq \ell(\ba^*) + 
\frac{(B_2 + \lambda)}{2N} \left((1-\eta\lambda)^{k} \|\ba^0 - \ba^*\| + \frac{3(\sqrt{B_1B_2} + (B_2 + \lambda)\|\ba^*\|)\sqrt{\log(1/\delta)}}{\lambda \sqrt{b}}\right)^2$$
\end{lemma}
\begin{proof}
The excess loss at $\ba$ equals
\begin{align*}\ell(\ba) - \ell(\ba^*) &= \frac{1}{2} \<\ba \otimes \ba, \bH\> - \frac{1}{2} \<\ba^* \otimes \ba^*, \bH\>-\<\bv,\ba - \ba^*\> \\
&= \frac{1}{2} \<(\ba - \ba^*) \otimes (\ba - \ba^*), \bH\> + \<(\ba - \ba^*) \otimes \ba^*, \bH\> - \<\bv,\ba - \ba^*\>  \\
&= \frac{1}{2} \<(\ba - \ba^*) \otimes (\ba - \ba^*), \bH\> + \<\bH \ba^* - \bv,\ba - \ba^*\> \\
&= \frac{1}{2} \<(\ba - \ba^*) \otimes (\ba - \ba^*), \bH\> ,
\end{align*}
where the last step is by the first-order optimality condition $\bH \ba^* = \bv$. We conclude by using the bound on $\|\ba^k - \ba^*\|$ from Lemma~\ref{lem:distgapsgd}, and the fact that $\|\bH\| \leq (B_2 + \lambda) / N$.
\end{proof}

We may now prove the main result of this subsection: i.e., that it suffices to prove that there is a certificate $\ba^{cert}$ achieving low loss.
\begin{proof}[Proof of Lemma~\ref{lem:cert-suffices-sgd}]
Because of the quadratic regularization term and the optimality of $\ba^*$,
\begin{align*}
\frac{\lambda}{2N} \|\ba^*\|^2 \leq \ell(\ba^*) \leq \ell(\ba^{cert}),
\end{align*}
so $$\|\ba^*\|^2 \leq \frac{2N}{\lambda}\ell(\ba^{cert}).$$
Plugging this into the bound from Lemma~\ref{lem:lossgapsgd}, with probability $1 - \delta$,
\begin{align*}
\ell(\ba^k) &\leq \ell(\ba^*) + \frac{(B_2 + \lambda)}{N} \left(2(1-\eta\lambda)^{2k} (\|\ba^0\|^2 + \|\ba^*\|^2) + \frac{18(B_1B_2 + (B_2 + \lambda)^2\|\ba^*\|^2)\log(1/\delta)}{\lambda^2 b}\right) \\
&\leq \ell(\ba^{cert}) + 2(B_2 + \lambda) \Big((1-\eta\lambda)^{2k} (\frac{\|\ba^0\|^2}{N} + \frac{2\ell(\ba^{cert})}{\lambda}) \\
&\quad\quad\quad\quad\qquad\qquad\qquad\qquad + \left(\frac{9B_1B_2}{\lambda^2 b N} +  \frac{18(B_2 + \lambda)^2 \ell(\ba^{cert})}{\lambda^3 b}\right)\log(1/\delta)\Big).
\end{align*}
\end{proof}

\subsection{Bounding the loss reduces to lower bounding the eigenvalues of the kernel matrix}\label{ssec:reducetolambdamin-discrete}

We give a lemma that bounds the final loss of batch-SGD on the second layer weights, after an initial phase of training the first layer weights. In Phase 1, we train $\bW$ for $k_1$ iterations while keeping $\ba$ fixed. In Phase 2, we train $\ba$ for $k_2-k_1$ iterations, while keeping $\bW$ fixed.
We reduce the problem of bounding the final loss $R(\ba^{k_2},\bW^{k_1})$ to the problem of lower bounding the kernel matrix associated with the neural network constructed by the dimension-free dynamics during phase 1. We will use this result to provide explicit sample-complexity bounds in Section \ref{sec:sample-complexity-explicit}.

\paragraph{Training setup} Let $\bTheta^k = (\btheta_1^k,\ldots,\btheta_N^k)$ denote the parameters on step $k$ of batch-SGD training. For $j \in [N]$ we independently draw $$\btheta^0_j = (a_j^0,\bw_j^0)  \sim \mathrm{Unif}([-1,1]) \times \delta_{\bzero}.$$
We train $\bTheta$ with \eqref{eq:bSGD} with regularization $\lambda^w = 0$ and $\lambda^a = \lambda > 0$ which is a parameter to be set. Let $\eta_1,\eta_2 > 0$ be the step sizes during phase 1 and phase 2 respectively.
\begin{itemize}
    \item In Phase 1, for time step $k \in \{0,\ldots,k_1 - 1\}$, we update according to
    \begin{align*}
    \bw_j^{k+1} = &~ \bw_j^{k} + \eta_1 a_j^k \Big[ \frac{1}{b} \sum_{i\in [b]} \{ y_{ki} - \hat f_{\NN} ( \bx_{ki} ; \bTheta^k ) \}  \sigma ' ( \< \bw_j^k , \bx_{ki} \> ) \bx_{ki} \Big] \,  \\
    a_j^{k+1} = &~ a_j^k = a_j^0 .\nonumber
    \end{align*}
    \item In Phase 2, for time step $k \in \{k_1,\ldots,k_2 - 1\}$, we update according to
    \begin{align*}
    \bw_j^{k+1} &= \bw_j^{k_1} \\
    a_j^{k+1} &= a_j^{k} + \eta_2  \Big[ \frac{1}{b} \sum_{i\in [b]} \{ y_{ki} - \hat f_{\NN} ( \bx_{ki} ; \bTheta^k ) \}  \sigma  ( \< \bw_j^k , \bx_{ki} \> ) \bx_{ki} \Big] - \eta \lambda a_j^k \,. \nonumber
    \end{align*}
\end{itemize}

\paragraph{Reduction to analyzing the limiting mean-field dynamics} For $k \in \{0,\ldots,k_1\}$, consider either $\rho_k$  or $\rho_{k_1 \eta_1}$ the limiting dynamics in the discrete-time or in the continuous time setting, and let $\bar{\bTheta}^k = (\barbtheta_1^k,\ldots,\barbtheta_N^k)$ or $\bar{\bTheta}^{t} = (\barbtheta_1^t,\ldots,\barbtheta_N^t)$ be the parameters if we had trained with the limiting discrete or continuous-time mean-field dynamics, initialized at $\barbTheta^0 = \bTheta^0$. If we denote $\bar\bw_j^k = (\bar\bu_j^k , \bar \bv_j^k)$ and $\bar\bw_j^t = (\bar\bu_j^t , \bar \bv_j^t)$, recall that the discrete mean-field dynamics is given by (note that with $\bar\bw^0 = \bzero$ initialization, the mean-field PDE and dimension free dynamics are the same)
\begin{align*}
\bar\bu_j^{k+1} &= \bar\bu_j^k + \eta_1 \bar a_j^0 \EE_{\bx}[\{\fNN(\bz;\bar\rho_k) - h_*(\bz)\}\sigma'(\<\bar\bu_j^k,\bz\>)\bx] \,   \\
\bar\bv_j^{k+1} &= \bar\bv_j^k = \bzero \, , \nonumber 
\end{align*}
while for the continuous time mean-field dynamics
\begin{align*}
\frac{\de}{\de t}\bar\bu_j^{t} &= \bar a_j^0 \EE_{\bx}[\{\fNN(\bz;\bar\rho_t) - h_*(\bz)\}\sigma'(\<\bar\bu_j^t,\bz\>)\bx] \,   \\
\bar\bv_j^{t} &= \bzero \, . \nonumber 
\end{align*}
For convenience, denote $T_1 = k_1$ or $T_1 = \eta k_1$, and $\bar\bu^{T_1}_j$ the weights at the end of phase 1. Denote $\err = \max_{j \in [N]} \| \bw_j^{T_1} - \bar \bw_j^{T_1} \|_2$. The propagation-of-chaos argument in Propositions \ref{prop:general-PDE-SGD-bound} and \ref{prop:discrete_prop} yields the following bound with probability $1 - 1/N$ (when training only the first layer-weights): 
\begin{equation}\label{eq:propChaos}
 \err \leq K e^{K T_1} \left\{  \sqrt{\frac{\log N}{N}} +  \sqrt{\frac{d+\log N}{b}} \right\}\, .
\end{equation}

We define $\bK^{T_1} \in \R^{2^P \times 2^P}$ the kernel at the end of phase 1, with entries
\[
K^{T_1} (\bz , \bz ') = \int \sigma ( \< \bar\bu , \bz\>) \sigma ( \< \bar\bu , \bz '\>) \bar\rho_{T_1} (\de \bar \bu) \, .
\]
We further define the following matrix $\bM$ associated with the features computed by the mean-field dynamics from initialization $\bTheta_0$. $\bM \in \R^{2^P \times N}$ is indexed by $\bz \in \{+1,-1\}^P$ and $j \in [N]$, and has entries
$$M_{\bz,j} = \frac{1}{\sqrt{N}} \sigma(\<\barbu_j^{T_1},\bz\>).$$
We call this the mean-field feature matrix (MF feature matrix, for short). Because $\ba^0_j \in \R^N$ are iid with $\bar \rho_0$, the $\bar\bu^{T_1}_j$ are iid distributed with respect to $\bar\rho_{T_1}$ and $\bM \bM$ is a random approximation of $\bK^{T_1}$. Indeed, we have the following bound:
\begin{lemma}\label{lem:K_to_M}
There exists $C>0$ such that if $N \geq 2^{CP}  \log (1/\delta)^2 / \lambda_{\min} ( \bK^{T_1} )^2$, then with probability at least $1 - \delta$, we have
\[
\lambda_{\min } ( \bM \bM^\sT ) \geq \lambda_{\min} (\bK^{T_1} ) / 2 \, .
\]
\end{lemma}

\begin{proof}[Proof of Lemma \ref{lem:K_to_M}]
Note that $\bM = [ \bm_1 , \ldots , \bm_N ]$ has iid columns with covariance $N \E [ \bm_j \bm_j^\sT ]  = \bK^{T_1}$. Furthermore $\| \bm_j \|_2^2 \leq K 2^P$. The result follows by applying \cite[Theorem 4.44]{vershynin2010introduction}.
\end{proof}

We prove a bound on the risk $\bTheta^{k_2}$ found using batch-SGD in terms of the minimum singular value of $\bM$.

\begin{lemma}\label{lem:reduce-to-cond-discrete}[Sufficient to prove $\bM$ is well-conditioned]
There is a universal constant $C > 0$ such that for any $0 < \delta < 1, 0 < \eps < 1$, any $k_1$, if we pick hyperparameters
\begin{align*}\lambda &= c_1\eps \quad \mbox{ for }\quad  c_1 = \frac{\min(1,\lambda_{\mathrm{min}}(\bM\bM^T)^2)}{2^{2P+2}K^6} \\
\eta &\leq 1 / (2K^2) \\
N &\geq c_2 / \eps^3, \quad \mbox{ for } c_2 = (CK^2)^{3k_1} \log(1/\delta)^2 / c_1^3 \\
b &\geq c_3 \log(N) d / \eps^4, \quad \mbox{ for } c_3 = (4K^2+c_1)^2(CK^2)^{2k_1} \log(1 / \delta) / c_1^3 \\
k_2 &\geq k_1 + c_4 \log(1/\eps) / (\eta \eps), \quad \mbox{ for } c_4 = \log(100(4K^2 + c_1)) / c_1,
\end{align*}
then, with probability at least $1 - \delta$, the final loss is bounded by
$$R(\ba^{k_2},\bW^{k_1}) \leq \eps.$$
\end{lemma}

Let us first prove that if $\bM \bM^T$ is well-conditioned then a low-error certificate exists. As in Section~\ref{ssec:linear-sgd-discrete}, let $$\ell(\ba) = R(\ba,\bW^{k_1}) + \frac{\lambda}{2N}\|\ba\|^2.$$
\begin{lemma}\label{lem:cert-exists}
There is a universal constant $C > 0$ such that for any $\delta > 0$ and mini-batch size $b \geq \max(d, C\log(1/\delta))$, with probability at least $1 - \delta$, there is $\ba^{cert} \in \R^N$ such that
$$\ell(\ba^{cert}) \leq (\err + \lambda) 2^{2P} K^6 / \lambda_{\mathrm{min}}(\bM\bM^T)^2,$$
where $\err = \err(k_1,b,N,\delta)$ is the error bound for convergence to the limiting discrete or continuous-time mean-field dynamics.
\end{lemma}
\begin{proof}
Let $\bbeta \in \R^{2^P}$ be the vector of values of $h_*$, indexed by $\bz \in \{+1,-1\}^P$:
$$\bbeta_{\bz} = h_*(\bz).$$
We construct a certificate for the least-squares training in Phase 2. Let $$\ba^{cert} = \sqrt{N}\bM^T (\bM \bM^T)^{-1} \bbeta.$$ We bound the norm of $\ba^{cert}$ in terms of the minimum singular value of $\bM$:
\begin{align}
\|\ba^{cert}\| &\leq \sqrt{N} \|\bM\| \|\bbeta\|  / \lambda_{\mathrm{min}}(\bM \bM^T) \nonumber \\
&\leq 2^PK^2 \sqrt{N}  / \lambda_{\mathrm{min}}(\bM \bM^T), \label{eq:certnormbound}
\end{align}
for a constant $K$ since $\|\bM\| \leq \sqrt{2^P}\|\sigma\|_{\infty} \max_{j} |a_j^{k_1}| \leq \sqrt{2^P}K$, and $\|\beta\| \leq \sqrt{2^P} \|h_*\|_{\infty} \leq K$. We bound the risk given by using $\ba^{cert}$ as a certificate. By the triangle inequality,
\begin{align}
R(\ba^{cert},\bW^{k_1}) = \frac{1}{2}\EE_{\bx}[(f_*(\bx) - \frac{1}{N} \sum_{j=1}^{N} a^{cert}_j \sigma(\<\bw_j^{k_1},\bx\>))^2] \leq \Upsilon_1 + \Upsilon_2,\label{eq:splitintoupsilons}
\end{align}
where
\begin{align*}
\Upsilon_1 &= \EE_{\bx}[(f_*(\bx) - \frac{1}{N} \sum_{j=1}^{N} a^{cert}_j \sigma(\<\barbw_j^{k_1},\bx\>))^2] \\
\Upsilon_2 &= \EE_{\bx}[(\frac{1}{N} \sum_{j=1}^{N} a^{cert}_j (\sigma(\<\bw_j^{k_1},\bx\>) - \sigma(\<\barbw_j^{k_1},\bx\>)))^2].
\end{align*}
We first bound $\Upsilon_1$, using the fact that that $f_*(\bx) = h_*(
\bz)$ and $\frac{1}{\sqrt{N}}\sigma(\<\barbw^{k_1}_j,\bx\>) = \frac{1}{\sqrt{N}}\sigma(\<\barbu^{k_1}_j,\bz\>) = M_{\bz,j}$:
\begin{align}
\Upsilon_1 &= \EE_{\bx=(\bz,\br)} [(h_*(\bz) - \frac{1}{\sqrt{N}}\sum_{j=1}^{N} a^{cert}_j M_{\bz,j})^2] = \frac{1}{2^P}\|\bbeta - \frac{1}{\sqrt{N}}\bM \ba^{cert}\|^2 = 0.\label{eq:upsilon1bound}
\end{align}
For $\Upsilon_2$, use (a) Jensen's inequality, (b) $K$-Lipschitzness of $\sigma$, and (c) linearity of expectation:
\begin{align*}
\Upsilon_2 &\overset{(a)}{\leq} \frac{1}{N} \sum_{j=1}^{N} \EE_{\bx}[(a^{cert}_j( \sigma(\<\bw_j^{k_1},\bx\>) - \sigma(\<\barbw_j^{k_1},\bx\>)))^2] \\
&\overset{(b)}{\leq} \frac{1}{N} \sum_{j=1}^{N} K^2|a^{cert}_j|^2 \EE_{\bx}[\<\bw_j^{k_1} - \barbw_j^{k_1},\bx\>^2] \\
&\overset{(c)}{=} \frac{1}{N} \sum_{j=1}^{N} K^2|a^{cert}_j|^2 \|\bw_j^{k_1} - \barbw_j^{k_1}\|_2^2
\end{align*}
Finally, by the propagation of chaos bound in Eq.~\eqref{eq:propChaos}, at time $T_1$ with probability at least $1 - \delta$,
$$\max_{j\in [N]} \|\bw_j^{k_1} - \barbw_j^{k_1}\|^2 \leq \err := \err(k_1,b,N,\delta) = e^{K T_1}(\sqrt{\log(N / \delta)/N} + \sqrt{d \log(N / \delta)/b}),$$
where $C > 0$ is some universal constant. So
\begin{align}
\Upsilon_2 &\leq K^2 \|\ba^{cert}\|^2 (\err) / N.\label{eq:upsilon2bound}
\end{align}

Combining \cref{eq:certnormbound,eq:splitintoupsilons,eq:upsilon1bound,eq:upsilon2bound}, with probability at least $1 - \delta$,
\begin{align*}\ell(\ba^{cert}) &= R(\ba^{cert},\bW^{k_1}) + \frac{1}{2N}\lambda \|\ba^{cert}\|^2 \\
&\leq K^2 \|\ba^{cert}\|^2 (\err) /N + \frac{1}{2N} \lambda \|\ba^{cert}\|^2 \\
&\leq (\err + \lambda) 2^{2P} K^6 / \lambda_{\mathrm{min}}(\bM\bM^T)^2
\end{align*}
\end{proof}

Now we may prove the main result of this subsection.

\begin{proof}[Proof of Lemma~\ref{lem:reduce-to-cond-discrete}]
By the choice of hyperparameters $N$ and $b$, we have the following error bound between the batch-SGD dynamics on the first layer and the infinite-width, population limit,
\begin{align}\err(T_1,N,b,\delta/2) &= 
e^{KT_1}(\sqrt{\log(N / \delta)/N} + \sqrt{d \log(N/\delta)/b}) \\
&\leq \eps / (2^{2P+1} K^6 \lambda_{\mathrm{min}}(\bM\bM^T)^2)
\end{align}

By Lemma~\ref{lem:cert-exists}, with probability $1 - \delta/2$ over the choice of the initialization $\bTheta^0$ and the samples $\{(\bx_{ki},y_i)\}_{k \in \{0,\ldots,k_1-1\}, i\in [b]}$ there exists a certificate $\ba^{cert}$ with loss:
\begin{align*}
\ell(\ba^{cert}) &\leq (\err + \lambda) 2^{2P} K^6 / \lambda_{\mathrm{min}}(\bM\bM^T)^2 \leq \eps / 2.
\end{align*}

Let us conclude by applying Lemma~\ref{lem:cert-suffices-sgd}, which shows that with probability $1 - \delta/2$ the final error satisfies the following bound in terms of $\ell(\ba^{cert})$:
\begin{align*}
\ell(\ba^{k_2}) &\leq \ell(\ba^{cert}) + 2(B_2 + \lambda) \Big((1-\eta\lambda)^{2(k_2-k_1)} (\frac{\|\ba^0\|^2}{N} + \frac{2\ell(\ba^{cert})}{\lambda}) \\
&\quad\quad\quad\quad\qquad\qquad\qquad\qquad + \left(\frac{9B_1B_2}{\lambda^2 b N} +  \frac{18(B_2 + \lambda)^2 \ell(\ba^{cert})}{\lambda^3 b}\right)\log(2/\delta)\Big).
\end{align*}

Here we take $B_1 = 4K^2 \geq |y_{ki}|^2$ almost surely, and $B_2 = K^2$, since 
\[
\max_{\bx} \|\phi(\bx)\|^2 = \max_{\bx} \frac{1}{N} \sum_{i=1}^N |\sigma(\<\bw_i^{k_1},\bx\>)|^2 \leq K^2 
\]
almost surely. Also, $\|\ba^0\|^2 \leq N$. Finally, $(1 - \eta \lambda)^{2(k_2-k_1)} \leq \eps/(100(4K^2 + c_1 \eps))$. So with probability $1 - \delta/2$,
\begin{align*}
\ell(\ba^{k_2}) &\leq \ell(\ba^{cert}) + 2(4K^2 + c_1 \eps)((1-\eta \lambda)^{2(k_2-k_1)} (1 + 2 / c_1) + (\frac{36K^4}{c_1^2 \eps^2 b N} + \frac{18(4K^2 + c_1 \eps)^2 \eps}{c_1^3 \eps^3 b}) \log(2/\delta)) \\
&\leq \ell(\ba^{cert}) + \eps / 2 \leq \eps.
\end{align*}
    
We conclude by noting $R(\ba^{k_2},\bW^{k_1}) \leq \ell(\ba^{k_2})$.
\end{proof}

\subsection{Anti-concentration of polynomials}\label{sec:polynomial-anticoncentration}

We prove the technical lemma that polynomials anti-concentrate when evaluated at random inputs. Concretely, we lower-bound the variance of the polynomial evaluated at a random input based on the sum of the magnitudes of its coefficients. Our bound is crude, but suffices for our purposes. We remark that anti-concentration bounds for polynomials in terms of their variance (and other moments) are a well-studied subject. For instance, the seminal paper \cite{carbery2001distributional} bounds the probability that a polynomial of random variables lies in an interval in terms of the variance (or other moments) of the polynomial. In contrast, we bound the variance based on the sum of magnitudes of the polynomial's coefficients.

\begin{lemma}[Polynomial anticoncentration]\label{lem:polynomial-anticoncentration}
For any integers $D,m>0$, there exists a constant $c > 0$ such that the following hold. For any polynomial $h : \R^m \to \R$ of the form:  $$h(\bz) = \sum_{\balpha \in \{0,\ldots,D\}^m} h_{\balpha} \prod_{l \in [m]} z_l^{\alpha_l}\, ,$$ we have
$$\E_{\bu \sim \Unif([-1,1]^{\otimes m})}[h(\bu)^2] \geq c \left(\sum_{\balpha} |h_{\balpha}|\right)^2\, .$$
\end{lemma}
\begin{proof}
Define $M = \sum_{\balpha} |h_{\balpha}|$ and let $\bu \sim \Unif([-1,1]^{\otimes m})$.
Let us lower bound $\EE_{\bu}[h(\bu)^2]$. For this purpose, we decompose $h(\bz)$ in the multivariate Legendre polynomial basis, which is defined as follows. Let $P_l : \R \to \R$ denote the degree-$l$ Legendre polynomial in one-dimension. In particular, they satisfy the orthogonality relations: for any $k,l \in \Z_{\geq 0}$,
$$\EE_{u \sim \Unif([-1,1])}[P_k(u)P_l(u)] = \delta_{kl}.$$
For any $\balpha \in \Z_{\geq 0}^m$, the multivariate Legendre polynomial is then given by
$$P_{\balpha}(\bz) = \prod_{l \in [m]} P_{\alpha_l}(z_l).$$
These polynomials inherit the orthogonality relations over the multivariate uniform distribution, i.e., for any $\balpha,\bbeta \in \Z_{\geq 0}^m$,
$$\EE_{\bu}[P_{\balpha}(\bu)P_{\bbeta}(\bu)] = \delta_{\balpha\bbeta}.$$

The polynomials $\{P_{\balpha}\}_{\balpha \in \{0,\ldots,D\}^m}$ therefore form an orthonormal basis over the multivariate polynomials whose degree in each variable is bounded by $D$. Writing $h(\bz)$ in this basis, we get
$$h(\bz) = \sum_{\balpha \in \{0,\ldots,D\}^m} g_{\balpha} P_{\balpha}(\bz) \, ,$$
for some coefficients $g_{\balpha} \in \R$. For each multivariate Legendre polynomial we also write its expansion
$$P_{\balpha}(z) = \sum_{\bbeta \in \{0,\ldots,D\}^m} p_{\balpha,\bbeta} \prod_{l \in [m]} z_l^{\beta_l}\, ,$$
for some coefficients $p_{\balpha,\bbeta} \in \R$. Therefore, for any $\bbeta$, we have $$h_{\bbeta} = \sum_{\balpha \in \{0,\ldots,D\}^m} g_{\balpha} p_{\balpha,\bbeta}\, ,$$ and so
\begin{equation}\label{eq:M-upp-bound}
\begin{aligned}
M = \sum_{\bbeta} |h_{\bbeta}| \leq&~ \sum_{\bbeta,\balpha \in \{0,\ldots,D\}^m} |g_{\balpha} p_{\balpha,\bbeta}| \\
\leq&~ (D+1)^m \max_{\balpha',\bbeta' \in \{0,\ldots,D\}^m} |p_{\balpha',\bbeta'}| \cdot \sum_{\balpha} |g_{\balpha}| \leq C \sum_{\balpha} |g_{\balpha}|,
\end{aligned}
\end{equation}
for some constant $0 < C < \infty$ depending on $m,D$.

Therefore,
\begin{align*}\EE_{\bu}[h(\bu)^2] &= \sum_{\balpha,\balpha' \in \{0,\ldots,D\}^m} \EE_{\bu}[g_{\balpha} g_{\balpha'} P_{\balpha}(\bu)P_{\balpha'}(\bu)] &\mbox{(linearity of expectation)}\\
&= \sum_{\balpha \in \{0,\ldots,D\}^m} (g_{\balpha})^2 &\mbox{(orthogonality relations)} \\
&\geq \frac{1}{(D+1)^m}\left(\sum_{\balpha \in \{0,\ldots,D\}^m} |g_{\balpha}|\right)^2 &\mbox{(Cauchy-Schwarz)} \\
&\geq \frac{1}{(D+1)^m} (M/C)^2 &\mbox{(by \cref{eq:M-upp-bound})}
\end{align*}

\end{proof}

We will also use the following corollary:
\begin{lemma}[Polynomial anticoncentration for shifted input distribution]\label{lem:not-shifted-polynomial-anticoncentration}
For any integers $m,D >0$ and constant $C >0$, there exists a constant $c > 0$ such that the following hold. For any $0 < \rho < 1$, any $\bw \in \R^m$ with $\|\bw\|_{\infty} \leq C$, and any polynomial $h : \R^m \to \R$ of the form
$$h(\bz) = \sum_{\balpha \in \{0,\ldots,D\}^m} h_{\balpha} \prod_{l \in [m]} z_l^{\alpha_l}\, ,$$ we have, writing $\|\balpha\|_1 = \sum_{i \in [m]} \alpha_i$,
$$\E_{\bu \sim \Unif([-\rho,\rho]^{\otimes m})}[h(\bw + \bu)^2] \geq c\left(\sum_{\balpha} |h_{\balpha}| \rho^{\|\balpha\|_1}\right)^2.$$
\end{lemma}
\begin{proof}
Fix $0 < \rho < 1$ and $\bw \in \R^m$ such that $\|\bw\|_{\infty} \leq C$.
Write $$g(\bz) = h(\bw + \rho \bz) = \sum_{\balpha \in \{0,\ldots,D\}^m} g_{\balpha} \prod_{l \in [m]} (z_l)^{\alpha_l}.$$
Then
$$h(\bz) = g((\bz - \bw)/\rho) = \sum_{\balpha} \sum_{\bbeta} g_{\balpha} q_{\balpha,\bbeta} \prod_{l \in [m]} (z_l)^{\beta_l},$$
where we denoted
$$q_{\balpha}(\bz) = \prod_{l \in [m]} ((z_l - w_l)/\rho)^{\alpha_l} = \sum_{\bbeta} q_{\balpha,\bbeta} \prod_{l \in [m]} (z_l)^{\beta_l}.$$

We have the following easy bound $$|q_{\balpha,\bbeta}| \leq \left((C+1)/\rho\right)^{\|\balpha\|_1}\, ,$$
so that we have the upper bound \begin{align*}\sum_{\balpha} |h_{\balpha}| \rho^{\|\balpha\|_1} =&~ \sum_{\bbeta} |\sum_{\balpha} g_{\balpha} q_{\balpha,\bbeta}| \rho^{\|\balpha\|_1}\\
\leq&~ \sum_{\bbeta} \sum_{\balpha} |g_{\balpha}| |q_{\balpha,\bbeta}| \rho^{\|\balpha\|_1} \leq (C+1)^{m D}(D+1)^m \sum_{\balpha} |g_{\balpha}|\, . \end{align*}
  We deduce that 
  $$\sum_{\balpha} |g_{\balpha}| \geq \sum_{\balpha} |h_{\balpha}| \rho^{\|\balpha\|_1} (C+1)^{-m D} (D+1)^{-m}.$$
The lemma follows by noting that $ h(\bw + \bu)$, where 
$\bu \sim \Unif([-\rho,\rho])$, is equal in distribution to $ g(\bv)$, where $\bv \sim \Unif([-1,1])$, and applying Lemma~\ref{lem:polynomial-anticoncentration}.
\end{proof}

\end{document}